%% file: main.tex
\documentclass[letterpaper,10pt]{article}

\usepackage{probml}

\ShortHeadings{Neural SDEs on Compact Spaces: Theory, Methods, and Application to Suicide Risk}

\usepackage{wrapfig}
\usepackage{amsmath}
\usepackage{amssymb}
\usepackage{mathtools}
\usepackage{amsthm}
\usepackage{bm}
\usepackage{comment}
\usepackage{colortbl}
\usepackage{cancel}
\usepackage{circledsteps}
\usepackage{placeins}
\usepackage{soul}
\usepackage{multirow}
\usepackage[table]{xcolor}
\usepackage{booktabs}
\usepackage{tabularx}
\usepackage{array}
\usepackage{enumitem}

\colorlet{colorWSPStrong}{red!75}
\colorlet{colorVanillaStrong}{blue!75}
\colorlet{colorWSP}{red!25}
\colorlet{colorVanilla}{blue!25}
\colorlet{colorVanillaClipStrong}{blue!45!purple}
\colorlet{colorVanillaClip}{blue!45!purple!25}
\colorlet{colorWSPNoClipStrong}{orange}
\colorlet{colorWSPNoClip}{orange!25}

\creflabelformat{equation}{#2\textup{#1}#3}

\input{math_commands.tex} 

\usepackage{minitoc}

\begin{document}

\doparttoc 
\faketableofcontents 
\part{} 

\title{
  Neural Stochastic Differential Equations on Compact State Spaces: Theory, Methods, and Application to Suicide Risk Modeling
}

\author[1,$*$]{Malinda Lu}
\author[1,$*$]{Yue-Jane Liu}
\author[2]{Matthew K.~Nock}
\author[1,$\dagger$]{Yaniv Yacoby}

\affil[1]{Wellesley College}
\affil[2]{Harvard University}
\affil[$*$]{Equal contribution}
\affil[$\dagger$]{Correspondence to: \url{yy109@wellesley.edu}}

\maketitle

\begin{abstract}
Ecological Momentary Assessment (EMA) studies enable the collection of high-frequency self-reports of suicidal thoughts and behaviors (STBs) via smartphones.
Latent stochastic differential equations (SDEs) are a promising model class for EMA data, as it is irregularly sampled, noisy, and partially observed.
But SDE-based models suffer from two key limitations.
(a) These models often violate domain constraints, undermining scientific validity and clinical trust of the model.
(b) Training is numerically unstable without ad hoc fixes (e.g.~oversimplified dynamics) that are ill-suited for high-stakes applications.
Here, we develop a novel class of expressive SDEs whose solutions are provably confined to a prescribed compact polyhedral state space, matching the domains of EMA data. 
In this work, (1) we show why chain-rule based constructions of SDEs on compact domains fail, theoretically and empirically; (2) we derive constraints on drift and diffusion for general and stationary SDEs so their solutions remain in the desired state space; and (3), we introduce a parameterization that maps arbitrary (neural or expert-given) dynamics into constraint-satisfying SDEs. 
On several real EMA datasets, including a large suicide-risk study, our parameterization improves forecasts and optimization dynamics over standard latent neural SDE baselines. 
These contributions pave the way for principled, trustworthy continuous-time models of suicide risk and other clinical time series and extend applications of SDE-based methods (e.g.~diffusion models) to domains with hard state constraints.
\end{abstract}

\section{Introduction and Related Work} \label{sec:intro}

Suicide is a leading cause of death worldwide~\citep{who2025suicide}, yet our ability to understand and predict suicidal thoughts and behaviors (STBs) remains limited~\citep{nock2026understanding}. 
A key difficulty is that transitions from suicidal ideation to action can occur rapidly~\citep{millner2017describing}, leaving narrow windows for observation. 
Ecological Momentary Assessment (EMA) studies~\citep{shiffman2008ecological} address this using smartphones to collect high-frequency self-reports of affects, emotions, and STBs in daily life. 
However, EMA time series are irregularly sampled, noisy, and partially observed, with dynamics influenced by unobserved external factors. 
These characteristics violate the assumptions of many conventional approaches, which rely on regular sampling, complete observations or ad hoc treatment of missingness.

Latent stochastic differential equations (SDEs) (e.g.~\citet{archambeau2007variational,li2020scalable}) are an expressive probabilistic framework that naturally accommodates irregular sampling, explicitly models latent psychological dynamics, and separates process noise from observation noise. 
SDEs take the form, $d z_t = h(t, z_t) \cdot dt + g(t, z_t) \cdot d B_t$, wherein the change in state, $d z_t$, is modeled as a sum of deterministic and stochastic components, the ``drift,'' $h$, and ``diffusion,'' $g$, respectively.
By parameterizing the dynamics (drift and diffusion) with neural networks (NNs), ``neural SDEs'' can capture volatile, nonlinear trajectories, making them one of the few viable models for EMA data.
However, these models suffer from two important limitations.

\paragraph{Limitation 1: Many applications of SDE-based models require SDE solutions to satisfy domain-specific constraints.}
Without additional structure, SDEs with expressive (e.g.~NN-based) dynamics can encode inappropriate inductive biases for a given domain, producing invalid forecasts that call into question the validity of the entire model.
In EMA data, most survey questions use a 0–10 Likert scale, so SDE solutions must also lie in a compact state space.
But instead, these SDEs may yield nonsensical forecasts, like a ``12 out of 10'' intensity of suicidal ideation.
To illustrate the importance of addressing such model misspecification, imagine a weather model that forecasts $70\%$ chance of 70\textdegree F, $20\%$ chance of 80\textdegree F, and $10\%$ chance of 5,000\textdegree F.
Once the model assigns significant probability mass to an impossible event, how correct are the $70\%$ and $20\%$ forecasts? 
And how should the impossible $10\%$ be reallocated?
For suicide prevention, analogous misspecifications are not merely inconvenient; they are dangerous.
A model that allocates probability mass to impossible psychological states may (a) learn incorrect dynamics, leading to false scientific conclusions, and/or (b) mislead downstream clinical decision-making, either missing opportunities to prevent suicide attempts or triggering unnecessary alarms that erode clinicians' trust.

\paragraph{Limitation 2: SDE-based models are unstable during training.}
Numerical instabilities arise both when solving the SDE \emph{and} when backpropagating through the solver (e.g.~\cite{zhang2024trajectory}).
When SDEs are incorporated into deep probabilistic models, these numerical instabilities exacerbate existing training difficulties.
In practice, the adoption of SDE-based models often hinges on simplified dynamics (e.g.~\citet{ansari2023neural,oh2024stable}) and training tricks (like KL-annealing, e.g.~\citet{li2020scalable}).
For suicide prevention, however, these workarounds are problematic: (a) oversimplified dynamics may fail to capture the complex mechanisms that drive rapid shifts in STBs, and (b) dependence on training tricks undermines reproducibility, hinders clinical deployment, and reduces trust and uptake among domain experts.
Finally, although recent work explores promising simulation-free inference for latent SDEs (e.g.~\cite{zhang2024trajectory,bartosh2025sde}), approximate inference methods introduce their own, hard-to-characterize inductive biases (e.g.~\cite{yacoby2020failure,foong2020expressiveness,yacoby2022mitigating,coker22a}); thus, the empirical behavior and trade-offs implicit in simulation-free methods are not yet well understood.

\paragraph{Insight: Both limitations are addressed by enforcing that SDE solutions stay within the data domain.}
Doing so not only addresses model mismatch---it also improves the stability of model fit.
For example, recent work on diffusion models for image data observed that parameterizing SDEs on compact state spaces can improve their performance~\citep{saharia2022photorealistic,lou2023reflected,fishman2023diffusion,christopher2024constrained}. 
We hypothesize that, in early training, unconstrained SDE trajectories often exit the compact domain of the data and require many gradient steps to return, increasing chances of landing in poor local optima.
Later in training, even small perturbations to the dynamics can push trajectories outside the data region.
By constraining dynamics to respect the natural compact data domain, we induce a bias toward admissible models, improving both training dynamics and generalization.
In this work, we focus on compact polyhedra, which represent a large class of commonly-used spaces, including rectangular spaces and simplexes, useful for a variety of temporal natural phenomena (e.g.~\citet{cresson2016stochastic}) as well as models for image data (e.g.~diffusion models \citep{lou2023reflected}).

\paragraph{Shortcomings of existing SDE-based models on compact state spaces.}
Of these recent works, SDEs with reflected dynamics are promising because they apply to \emph{any} SDE-based model.
Reflected SDEs (RSDEs) on a compact space $K$ augment the original SDE with a third term~\citep{pilipenko2014introduction}:
$d z_t = h(t, z_t) \cdot dt + g(t, z_t) \cdot d B_t + r(z_t) \cdot dC_t$, where $r(z_t)$ is a reflecting vector field that pushes $z_t$ back into the interior of $K$, $\mathrm{int}(K)$, when $z_t$ is on the boundary of the space, $\partial K$, and $C_t$ is a non-decreasing process that increases only when $z_t \in \partial K$ (i.e.~$C_t$ satisfies $\int_0^t \mathbb{I}(z_s \notin \partial K) \cdot dC_s = 0$).
Thus, RSDEs behave like SDEs on the interior of the space, but are reflected inwards at the boundary. 
Despite their rich theory, RSDEs have two shortcomings.
First, the instantaneous equal-and-opposite push towards the interior, represented by $r$, may not faithfully describe many phenomena in physics, biology, engineering, and medicine (e.g.~\citet{d2013bounded,rohanizadegan2020discrete}), including the dynamics of STBs.
Second, RSDEs currently lack efficient, high-order solvers (e.g.~\citet{ding2008numerical,fishman2023metropolis}).
These challenges become barriers for applications in suicide research, where model interpretation is important.

Complementing work on RSDEs, recent work has drawn on stochastic viability theory~\citep{aubin1991} to parameterize SDEs on compact state spaces (without reflections) via careful construction of the dynamics (e.g.~\cite{cai1996generation,cresson2012validating,d2013bounded,cresson2016stochastic,cresson2018note,rohanizadegan2020discrete}).
While promising, the dynamics proposed in these works are tailored to specific phenomena and do not readily generalize.
Here, we generalize these ideas to obtain arbitrarily flexible SDE dynamics on any compact polyhedral state space.
We summarize the position of this work relative to these prior approaches in \cref{apx:related-work}.

\paragraph{Contributions.} 
In this paper, we propose a novel class of expressive SDEs on compact polyhedral spaces using insights from stochastic viability theory. 
\textbf{(1)} We explain why chain-rule based approaches struggle theoretically and empirically (\cref{sec:challenges}).
\textbf{(2)} We prove constraints on the drift/diffusion that ensure general and stationary SDEs have an inductive bias for compact state spaces (\cref{sec:theory}). 
\textbf{(3)} We propose a parameterization that provably satisfies these constraints, allowing us to transform any dynamics---whether NN-based or expert-specified---into SDEs whose solutions remain in a prescribed compact polyhedral state space (\cref{sec:method}), enabling us to capture a different class of natural phenomena than RSDEs.
Finally, \textbf{(4)} on several real EMA datasets, we empirically demonstrate that enforcing viability filters out psychologically inadmissible models, thereby steering learning toward dynamics closer to the data-generating process, resulting in improved forecasts and optimization dynamics relative to baselines (\cref{sec:results}).
We include a demo of our method at: \url{https://github.com/mogu-lab/wsp-demo}.

\paragraph{Broader Impact.}
This work impacts two communities: (i) ML researchers developing SDE-based models whose latent states live on constrained domains, and (ii) clinical psychology theorists for whom it is crucial that the dynamics respect clinical knowledge. 
For ML researchers, our method provides a general mechanism for imposing hard state constraints on expressive SDE dynamics, and can be integrated into a broad range of SDE-based frameworks (e.g.~diffusion models~\citep{song2021scorebased} and infinitely deep models~\citep{xu2022infinitely}), while remaining compatible with standard inference methods for latent SDEs~\citep{archambeau2007variational,kidger2021neural,issa2023non,zhang2025efficient}, including simulation-free approaches~\citep{bartosh2025sde}. 
For clinical psychology theorists, our viability results have implications for ongoing efforts to formalize suicide and other mental health challenges as dynamical systems (e.g.~\cite{millner2020advancing,wang2023mathematical,robinaugh2024advancing}).
Ultimately, this work paves way for hybrid models~\citep{schweidtmann2024review} of suicide, jointly guided by domain expertise \emph{and} by data, to advance our understanding of suicide and improve our ability to identify individuals at imminent risk in time for intervention.

\paragraph{Notation.}
Consider the following It\^o SDE: 
\begin{align}
    d z_t = h(t, z_t) \cdot dt + (\mathrm{diag} \circ g)(t, z_t) \cdot d B_t.
    \label{eq:unconstrained}
\end{align}
Here, $t \geq 0$ is time, $z_t \in K$ is the SDE's solution, which lies in a compact space, $K \subset \mathbb{R}^{D_z}$.
Next, $h: \mathbb{R}_{\geq 0} \times K \rightarrow \mathbb{R}^{D_z}$ and $g: \mathbb{R}_{\geq 0} \times K \rightarrow \mathbb{R}_{\geq 0}^{D_z}$ are the drift and diffusion, respectively.
We overload $\mathrm{diag}(\cdot)$ to embed (or extract) a vector into (or from) the diagonal of a matrix, and define $I_D$ as a $D$-dimensional identity matrix.
Finally, we use $e_d$ for the $d$th standard basis vector, $\nabla$ for the Jacobian, $\langle \cdot, \cdot \rangle$ for inner products, and $z_t^d$, $h^d$, and $g^d$ for the $d$th dimension of $z_t$, $h$ and $g$, respectively.

\section{Challenges with Chain-Rule Based Methods}
\label{sec:challenges}

Our goal is to find an expressive parameterization of $h$ and $g$ so that the SDE in \cref{eq:unconstrained} is \emph{viable}:
\begin{definition}[\citet{milian1995stochastic}] \label{def:viable}
A stochastic process $z_t$ is \emph{viable} in $K$ if, for every initial
condition $z_0 \in K$, $\mathbb{P}\left( z_t \in K, \forall t \in [0,\infty) \right) = 1$.
\end{definition}

\paragraph{Transforming SDE solutions on $\mathbb{R}^{D_z}$ to solutions on $K$.}
The simplest way to ensure $z_t$ lies on a compact space $K$ is to derive a closed-form SDE for $f(z_t)$, where $f: \mathbb{R}^{D_z} \to K$.
This is achieved with It\^o's lemma for It\^o SDEs and with the standard chain-rule for Stratonovich SDEs.
While simple, however, this approach does not work theoretically or empirically for three reasons.

\paragraph{Challenge 1: Theory.}
There does not exist a diffeomorphism, $f$, from a non-compact set, $\mathbb{R}^{D_z}$, to a compact set, $K$.
In practice, we often ignore this and map $\mathbb{R}^{D_z}$ to the \emph{interior} of $K$, as in classification models that use sigmoid/softmax to map Euclidean outputs to a unit cube/simplex. However, this can cause pathologies: for linearly separable classes, the logistic regression maximum likelihood estimation drives parameters to infinity~\citep{santner1986note}, undermining interpretability. 
Analogous issues can arise for SDEs parameterizing time-varying Bernoulli probabilities.

\paragraph{Challenge 2: Numerical Stability.}
If we are willing to overlook Challenge 1, we find ourselves with numerically unstable dynamics.
To see this, consider a 1D SDE $y_t \in \mathbb{R}$ with drift $\tilde{h}$ and diffusion $\tilde{g}$, and suppose we construct $z_t \in (0,1)$ via the sigmoid transform $z_t = f(y_t)$.
We begin with It\^o's lemma and plug in $y_t = f^{-1}(z_t) = \mathrm{sigmoid}^{-1}(z_t)$:
\begin{align} 
    d z_t &= \left[ \tilde{h}(t, y_t) \cdot \frac{\partial f(y_t)}{\partial y_t} + \frac{1}{2} \cdot \tilde{g}(t, y_t)^2 \cdot \frac{\partial^2 f(y_t)}{\partial y_t^2} \right] \cdot dt + \left[ \tilde{g}(t, y_t) \cdot \frac{\partial f(y_t)}{\partial y_t} \right] \cdot dB_t, \nonumber \\ 
    &= (z_t - z_t^2) \cdot \left[ \tilde{h}(t, f^{-1}(z_t)) + \tilde{g}(t, f^{-1}(z_t))^2 \cdot \left( \tfrac{1}{2} - z_t \right) \right] \cdot dt + (z_t - z_t^2) \cdot \tilde{g}(t, f^{-1}(z_t)) \cdot dB_t.
    \label{eq:ito}
\end{align}
Here, $f^{-1}$ is unbounded, and can induce unbounded, numerically unstable dynamics.
Moreover, existence and uniqueness proofs for SDEs typically require linearly bounded dynamics (e.g.~\citet[Theorem 5.2.1]{oksendal2013stochastic}).

\paragraph{Challenge 3: Inductive Bias.}
If we are willing to overlook Challenge 1, we can overcome Challenge 2 as follows.
We observe that arbitrarily expressive $\tilde{h}$ and $\tilde{g}$ can ``undo'' $f^{-1}$ by internally composing it with $f$.
Thus, we can define $h(t, z_t) = \tilde{h}(t, f^{-1}(z_t))$ and $g(t, z_t) = \tilde{g}(t, f^{-1}(z_t))$ and parameterize $h$ and $g$ directly, e.g.~via NNs:
\begin{align} 
    d z_t &= (z_t - z_t^2) \cdot \left[ h(t, z_t) + g(t, z_t)^2 \cdot \left( \tfrac{1}{2} - z_t \right) \right] \cdot dt + (z_t - z_t^2) \cdot g(t, z_t) \cdot dB_t. 
    \label{eq:ito-absorbed}
\end{align}
This way, so long as $h$ and $g$ are bounded, this SDE has bounded dynamics, overcoming Challenge 2.
But, this surfaces yet another challenge: the inductive bias of this SDE is appropriate for few phenomena.
Because the sigmoid maps $\mathbb{R}$ to $(0,1)$, large excursions of $y_t$ toward $\pm\infty$ correspond to vanishingly small changes in $z_t$ near 0 or 1. 
In the induced SDE, this appears as the multiplicative factor, $(z_t - z_t^2)$, on both the drift and diffusion, which tends to $0$ at the boundary, causing trajectories to slow down and effectively ``stick'' there. 
We verify this behavior empirically in \cref{sec:results}.
This behavior is undesirable when modeling suicidal ideation, which can oscillate rapidly during short durations (e.g.~\citet{coppersmith2023mapping}).

\section{Constraints on Dynamics for SDEs on Compact Polyhedra} \label{sec:theory}

Motivated by the challenges from \cref{sec:challenges}, we prove constraints on the drift/diffusion to ensure that both stationary and general SDEs have an inductive bias for compact polyhedral state spaces.
To begin, we define polyhedral subspaces of Euclidean space:
\begin{definition} \label{def:polyhedron}
Let $u, v \in \mathbb{R}^{D_z}$ and $\mathcal{H}(u, v) = \{ z \in \mathbb{R}^{D_z} : \langle z - u, v \rangle \geq 0 \}$ denote a closed half-space. $K \subset \mathbb{R}^{D_z}$ is a polyhedron if it is a finite intersection of closed half-spaces: $K = \bigcap_{s \in \{1, \cdots, S\}} \mathcal{H}(u_s, v_s)$.
\end{definition} 

\paragraph{General SDEs.}
In \cref{thm:milian}, \citet{milian1995stochastic} shows that, under standard linear-growth and Lipschitz assumptions, an It\^o SDE is viable in a polyhedron, $K$, if and only if, on each half-space boundary, the drift and diffusion have no outward-pointing components: thus, stochastic motion cannot push it across the boundary, and any normal motion at the boundary is deterministic and non-outward.
Compared to the sigmoid-based chain-rule construction for the unit cube (\cref{sec:challenges}), these are less restrictive requirements, since they do not force the drift to be 0 at the boundary.
While \cref{thm:milian} also holds for non-compact polyhedra, we focus on compact polyhedra from here on.
In \cref{apx:proof-stratonovich}, we extend this result to Stratonovich SDEs on compact polyhedra.

\begin{theorem}[Adapted from \citet{milian1995stochastic}] \label{thm:milian}
    \textbf{Suppose} that the drift and diffusion, $h(t, z_t)$ and $g(t, z_t)$, of an It\^o SDE, defined for $t \geq 0$ and $z_t \in \mathbb{R}^{D_z}$, satisfy three conditions: 
    \begin{enumerate}[nosep, label=(\roman*), font=\bfseries]
        \item For each $T > 0$, there exists $C_T > 0$ such that for all $z_t \in K$ and $t \in [0, T]$, $\lVert h(t, z_t) \rVert^2 + \lVert g(t, z_t) \rVert^2 \leq C_T \cdot ( 1 + \lVert z_t \rVert^2)$. 
        \item For all $T > 0$, $z_t, z_t' \in K$, and $t \in [0, T]$, $\lVert h(t, z_t) - h(t, z_t') \rVert + \lVert g(t, z_t) - g(t, z_t') \rVert \leq C_T \cdot \lVert z_t - z_t' \rVert$. 
        \item For each $z_t \in K$, $h(t, z_t)$ and $g(t, z_t)$ are continuous.
    \end{enumerate}

    \textbf{Then} $z_t$ is viable in $K$ \textbf{if and only if}: for all $s \in \{1, \dots, S\}$ and $z_t \in K$ such that $\langle z_t - u_s, v_s \rangle = 0$,
    \begin{enumerate}[nosep, label=(\alph*), font=\bfseries]
    \item $\langle h(t, z_t), v_s \rangle \geq 0$, and
    \item $\langle g(t, z_t) \odot e_d, v_s \rangle = 0$ for $t \geq 0$ and $d \in \{1, \dots, D_z\}$.
    \end{enumerate}
\end{theorem}

\paragraph{Stationary SDEs.}
In many applications, it is important to model stationary dynamics (e.g.~\cite{tank2015bayesian}), particularly over short time horizons (e.g.~\cite{tonekaboni2021unsupervised,weatherhead22learning}). 
Similarly, in EMA studies of STBs, patients are typically tracked for relatively short periods---often from one to several weeks~\citep{ammerman2022using}---and their survey responses are often highly noisy due to unobserved external influences. 
Assuming stationarity for the duration of the study may reduce model complexity and provide a more appropriate inductive bias.
We therefore extend \cref{thm:milian} to stationary SDEs on r-polyhedra, which are, informally speaking, compact polyhedra with non-zero ``volume'' (avoiding polyhedra for which every point is on $\partial K$):
\begin{definition} \label{def:r-polyhedron}
A set $K \subset \mathbb{R}^{D_z}$ is an r-polyhedron if it is a compact polyhedron that is also regular: for every $z \in \partial K$ and every $\epsilon > 0$, there exists $z' \in \mathcal{B}(z, \epsilon)$ in a ball of radius $\epsilon$ centered at $z$ that lies in the interior of $K$.
\end{definition} 

We do this by selecting a diffusion, $g$, that satisfies (i)-(iii) and (b) from \cref{thm:milian}, deriving a closed-form equation for a drift $h$ as a function of $g$ that ensures stationarity, and proving that $h$ also satisfies all conditions from \cref{thm:milian} (see proof in \cref{apx:proof-stationary}).

\begin{theorem} \label{thm:stationary}
Let $K$ be an r-polyhedron, and let $h(z_t)$ and $g(z_t)$ be the drift and diffusion, respectively, of an autonomous It\^o SDE, defined for $t \geq 0$ and $z_t \in \mathbb{R}^{D_z}$.
\textbf{Suppose} that for all $T > 0$, $z_t, z_t' \in K$, and $t \in [0, T]$, there exists $C_T > 0$ such that:
\begin{enumerate}[nosep, label=(\roman*), font=\bfseries]
\item $\lVert g(z_t) \rVert^2 \leq C_T \cdot ( 1 + \lVert z_t \rVert^2)$.
\item $\lVert g(z_t) - g(z_t') \rVert \leq C_T\cdot \lVert z_t - z_t' \rVert$ and
$\lVert \mathrm{diag}( \nabla_{z_t} g(z_t) ) - \mathrm{diag}( \nabla_{z_t'} g(z_t') ) \rVert \leq C_T\cdot \lVert z_t - z_t' \rVert$.
\item The unnormalized time-marginal, $\tilde{p}(z_t)$, has support on $\mathrm{int}(K)$, and $g(z_t) \odot \nabla_{z_t} \log \tilde{p}(z_t)$ admits a continuous extension to $K$ such that
$\bigl\| g(z_t) \odot \nabla_{z_t} \log \tilde{p}(z_t) - g(z_t') \odot \nabla_{z_t'} \log \tilde{p}(z_t') \bigr\| \leq C_T \cdot \|z_t - z_t'\|$.
\end{enumerate}
\textbf{Suppose further that:} for all $z_t \in \mathrm{int}(K)$, 
\begin{enumerate}[nosep, label=(\roman*), font=\bfseries, start=4]
\item $g^d(z_t) > 0$ for all $d \in \{ 1, \dots, D_z \}$, and
\item for each $s\in\{1,\dots,S\}$, there exists $L_s<\infty$ such that $\big\lVert n_s \odot g(z) \big\rVert \le L_s \cdot d_s(z)$, where $n_s = v_s / \lVert v_s \rVert$ and $d_s(z) = \langle z - u_s, n_s \rangle$.
\end{enumerate}
\textbf{Then} for any $z_0 \sim p(z_t)$, the solution $z_t$ is viable in $K$, with time-marginal $p(z_t)$, \textbf{provided that}:
\begin{enumerate}[nosep, label=(\alph*), font=\bfseries]
\item $h(z_t) = \frac{1}{2} \mathrm{diag} \left( \nabla_{z_t} [g(z_t)^2] \right) + \frac{1}{2}  g(z_t)^2 \odot \nabla_{z_t} \log \tilde{p}(z_t)$, and
\item For all $s \in \{1, \dots, S\}$ and $z_t \in K$ such that $\langle z_t - u_s, v_s \rangle = 0$, $\langle g(z_t) \odot e_d, v_s \rangle = 0$ for all $d \in \{1, \dots, D_z\}$.
\end{enumerate}
\end{theorem}
The parameterization of $h(z_t)$ in \cref{thm:stationary} is easily implemented by obtaining the score function, $\nabla_{z_t} \log \tilde{p}(z_t)$, via auto-differentiation of $\log \tilde{p}(z_t)$, parameterized by an unconstrained NN.
Also, we note that assumptions (i)-(v) are easily satisfiable---see discussion in \cref{apx:assumptions}. 

\section{Parameterization of Expressive SDEs on r-Polyhedra} \label{sec:method}

We propose a parameterization that maps arbitrary (neural or expert-given) dynamics into constraint-satisfying dynamics from \cref{thm:milian,thm:stationary}. 

\paragraph{Weighted Sums Parameterization (WSP).}
We observe that we can satisfy both constraints on the drift and diffusion using the same mechanism: $\mathrm{WSP}(f, c, t, z) = w(z) \cdot f(t, z) + (1 - w(z)) \cdot c(z)$, using a different choice of $c(z)$ for each.
Here, $f(\cdot)$ is the original, unconstrained dynamics, given by domain experts or by some flexible function class, like NNs;
$c(\cdot)$ is a simple function that satisfies the constraints;
finally, $w(z) \in [0, 1]$ weighs the sum of $f(\cdot)$ and $c(\cdot)$, approaching 0 at $\partial K$ to favor $c(\cdot)$, and approaching 1 at the interior of $K$ to favor $f(\cdot)$.
Of many possible choices, we define:
\begin{align}
    w(z) = \tanh\left( \beta \cdot \prod_s \left[ \frac{e^{-d_s(z)}}{\sum_{s'} e^{-d_{s'}(z)}} \cdot \tanh\left( \alpha \cdot d_s(z) \right) \right] \right), \qquad d_s(z) = \frac{\langle z - u_s, v_s \rangle}{\lVert v_s \rVert}. \label{eq:w}
\end{align}
Here, $d_s(z) \geq 0$ is the shortest distance from $z$ to the boundary of $\mathcal{H}(u_s, v_s)$, and $\alpha, \beta > 0$ determine the transition rate between $f(z)$ and $c(z)$. 
They can be learned jointly with the model parameters.
The intuition behind \cref{eq:w} is that $w(z)$ should approach $0$ as $z$ approaches the \emph{closest} of the $S$ boundaries. 
As such, we take a product of distances from $z$ to each boundary, weighted by a softmin; this, in a sense, ``selects the closest distance.''
By taking a product of these weighted distances, we obtain a function that is $0$ at all boundaries and positive elsewhere, using $\tanh$ to ensure $w(z) \in [0, 1]$.
\cref{fig:intuition} (top-left) visualizes $w(z)$ for different polyhedra.

\begin{figure*}[t!]
    \centering

    \includegraphics[width=0.485\textwidth]{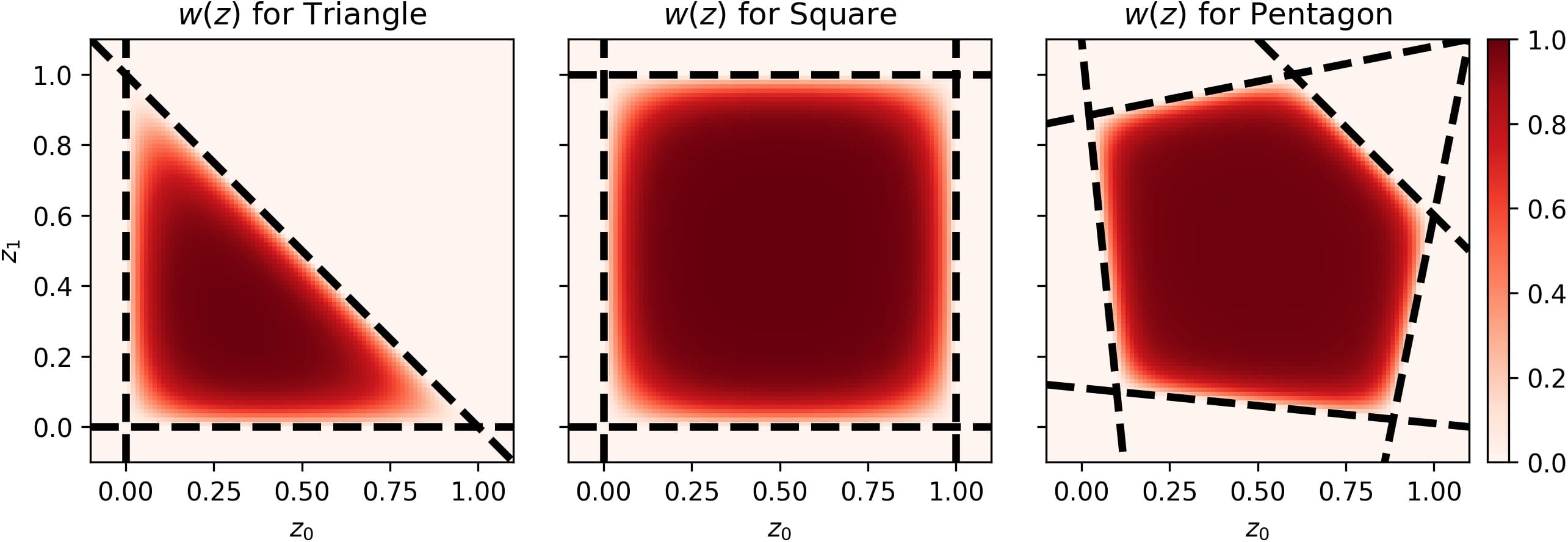}
    ~    
    \includegraphics[width=0.485\textwidth]{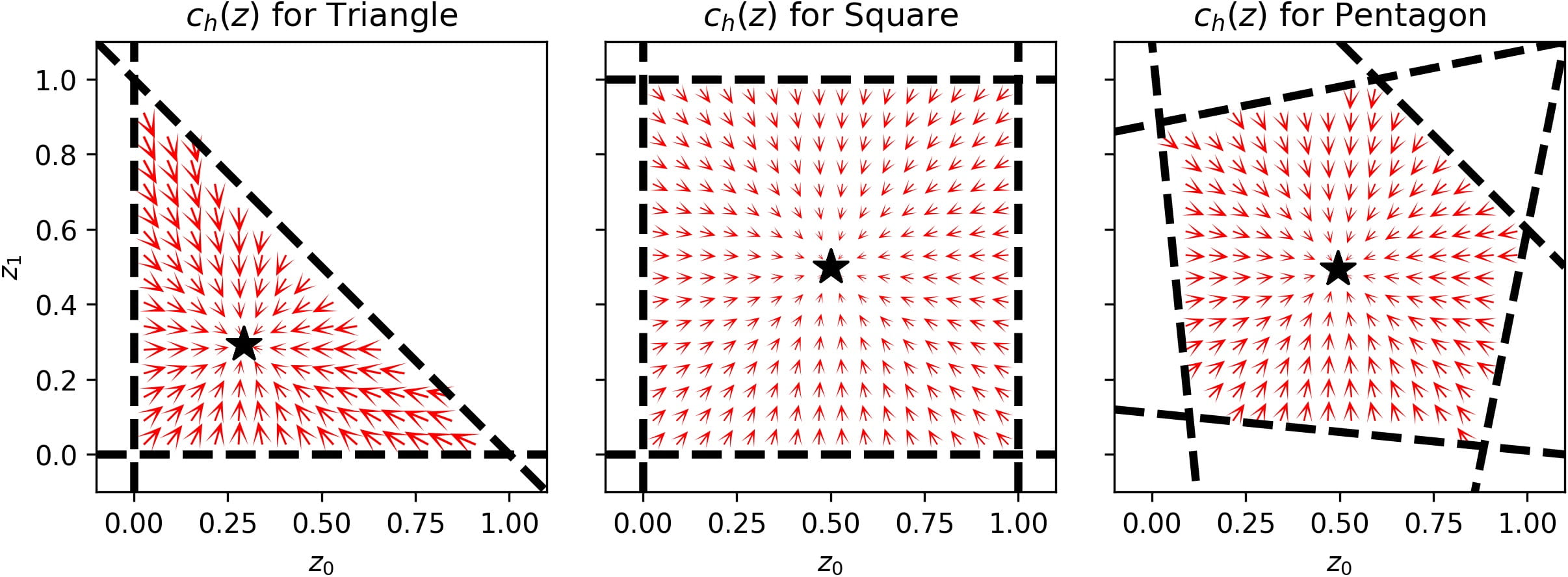}

    \includegraphics[width=0.485\textwidth]{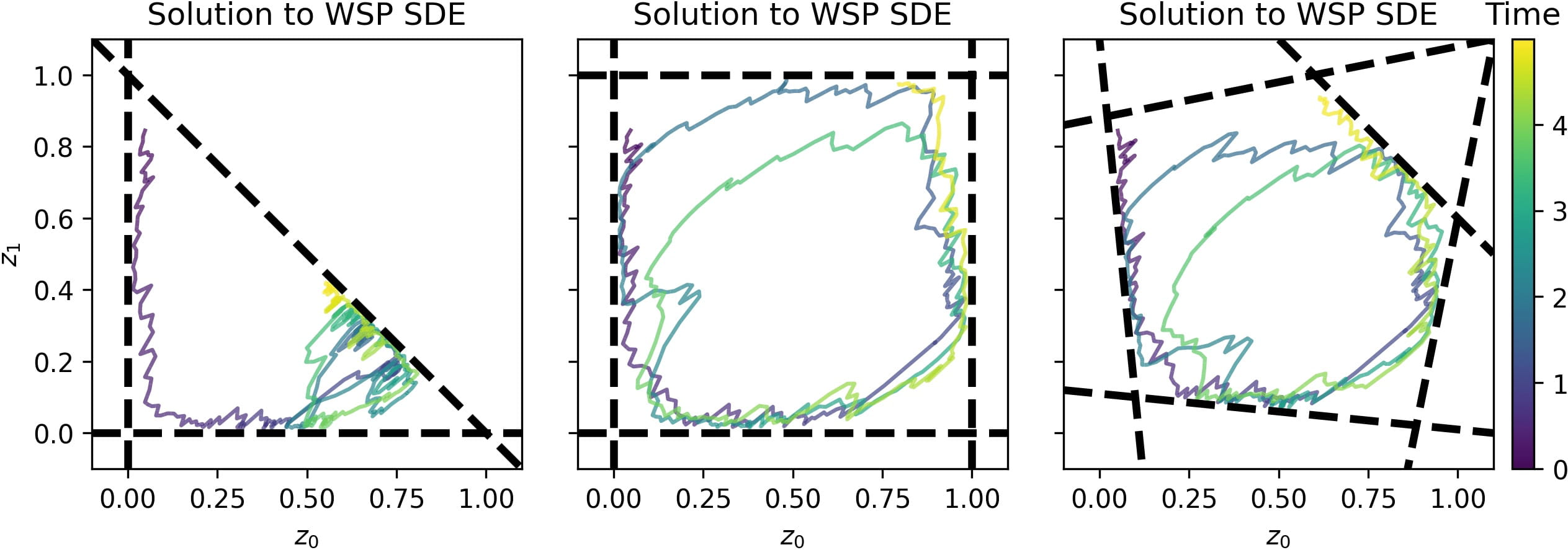}    
    ~
    \includegraphics[width=0.485\textwidth]{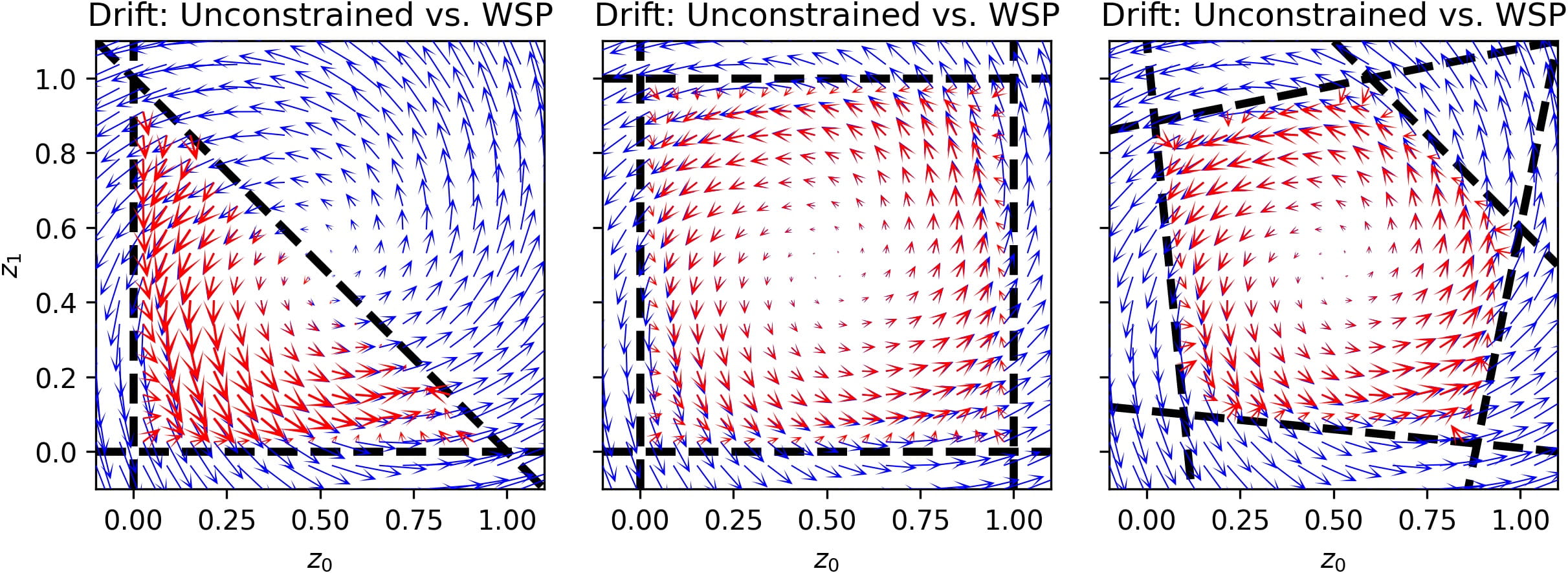}    
    
    \caption{\textbf{Intuition Underlying WSP.} Top left: $w(z)$ from \cref{eq:w}, approaching 0 at the boundaries and 1 in the interior. Top right: $c_h(z)$ from \cref{eq:wsp}, pointing towards the Chebyshev center $\bigstar$. Bottom left: solutions to WSP SDE, successfully remaining in $K$. Bottom right: an \textcolor{colorVanillaStrong}{``expert-given'' drift $\tilde{h}$} vs.~\textcolor{colorWSPStrong}{WSP drift $h$} (\cref{eq:wsp}) matching in the interior of $K$, but differing near the boundary. Details in \cref{apx:expert-dynamics-setup}.}
    \label{fig:intuition}
\end{figure*}

\paragraph{WSP-based SDEs.}
Given any unconstrained dynamics, $\tilde{h}: \mathbb{R}_{\geq 0} \times K \to \mathbb{R}^{D_z}$ and $\tilde{g}: \mathbb{R}_{\geq 0} \times K \to \mathbb{R}_{\geq 0}^{D_z}$, WSP returns new dynamics, $h$ and $g$, that satisfy (a)-(b) from \cref{thm:milian}:
\begin{align}
\begin{split}
    h(t, z_t) &= \mathrm{WSP}(\tilde{h}, c_h, t, z_t), \qquad c_h(z_t) = \gamma \cdot \frac{z^* - z_t}{\lVert z^* - z_t \rVert + \epsilon}, \\
    g(t, z_t) &= \mathrm{WSP}(\tilde{g}, c_g, t, z_t), \qquad c_g(z_t) = 0,    
\end{split}  
\label{eq:wsp}
\end{align}
where $z^*$ is the Chebyshev center of $K$, easily computed via linear programming once per polyhedron~\citep{boyd2004convex}.
Because when $z_t \in \partial K$, $\langle z^* - z_t, v_s\rangle = \langle z^* - u_s, v_s\rangle - \langle z_t - u_s, v_s\rangle = \langle z^* - u_s, v_s\rangle > 0$, $c_h(z_t)$ points to the interior of $K$ with magnitude controlled by $\gamma, \epsilon > 0$, learned jointly with the other model parameters.
We choose the Chebyshev center due to its generality and its large margin to the boundary; in future work, we are excited to investigate alternative choices for $c_h(z_t)$ that incorporate domain expertise. 
\cref{fig:intuition} visualizes $c_h(z)$ and $h$ for different polyhedra, showing WSP SDEs remain viable in $K$.
To instantiate WSP for neural SDEs, we can simply set $\tilde{h}, \tilde{g}$ to NNs with linear and softplus activations, respectively.
Next, we prove that, under the same mild conditions on $\tilde{h}$ and $\tilde{g}$ (e.g.~Lipschitz continuity, etc.), WSP dynamics satisfy the conditions of \cref{thm:milian,thm:stationary} (proof in \cref{apx:proof-wsp}).
One limitation of our proposed $w(z)$ is that $\alpha$ and $\beta$ need to grow quickly with the number of boundaries $S$. 
Below, there is a simple fix for the unit cube.

\begin{theorem} \label{thm:wsp}
    Let $K$ be an r-polyhedron.
    \textbf{Suppose} that $\tilde{h}$ and $\tilde{g}$, defined above, satisfy
    conditions \textbf{(i)-(iii)} of \cref{thm:milian}.
    \textbf{Then} the drift and diffusion $h$ and $g$, defined in
    \cref{eq:wsp}, satisfy \textbf{(i)-(iii)} and \textbf{(a)-(b)} from \cref{thm:milian}.
    \textbf{Suppose further} that $\tilde{g}$ additionally satisfies conditions
    \textbf{(i)}, \textbf{(ii)}, \textbf{(iv)} of \cref{thm:stationary};
    \textbf{then} $g$, defined in \cref{eq:wsp}, satisfies \textbf{(i)-(ii)}, \textbf{(iv)-(v)} and \textbf{(b)} from \cref{thm:stationary}.
\end{theorem}
\cref{thm:wsp} does not cover condition (iii) from \cref{thm:stationary}, which concerns the joint behavior of $g$ \emph{and} the score function, $\nabla_{z_t} \log \tilde{p}(z_t)$.
One simple way to satisfy (iii) is by ensuring $\nabla_{z_t} \log \tilde{p}(z_t)$ is Lipschitz, enabling WSP to easily satisfy both \cref{thm:milian,thm:stationary}.

\paragraph{WSP on the Unit Cube and Simplex.}
For a unit cube, $K=[0,1]^{D_z}$, the constraints decouple across dimensions, so we apply WSP element-wise, reducing the number of elements in the product of \cref{eq:w} from $2D_z$ to $2$.
For a simplex, we first define $K \subset \mathbb{R}^{D_z - 1}$ as the projection of a simplex onto $D_z - 1$ dimensions (e.g.~the triangle in the top-left of \cref{fig:intuition} is the projection of a 3D simplex).
We do this by setting $u,v \in \mathbb{R}^{D_z - 1}$ from \cref{def:polyhedron} as follows.
For $1 \leq s < D_z$, $u_s$ is the zero vector and $v_s = e_s$. 
For $s = D_z$, $u_s$ and $v_s$ are vectors filled with $(D_z - 1)^{-1}$ and $-(D_z - 1)^{-1/2}$, respectively.
Using WSP, we then define an SDE whose state, $z_t$, evolves in this projected space, $K$.
The resulting process has $D_z - 1$ dimensions, each nonnegative with a sum $\leq 1$.
Finally, we recover the full $D_z$-dimensional simplex by adding the last component, $z_t^{D_z} = 1 - \sum_{s=1}^{D_z - 1} z_t^s$, to $z_t$ so that all components are nonnegative and sum to 1.

\begin{figure*}[t!]
    \centering

    \includegraphics[width=0.31\textwidth]{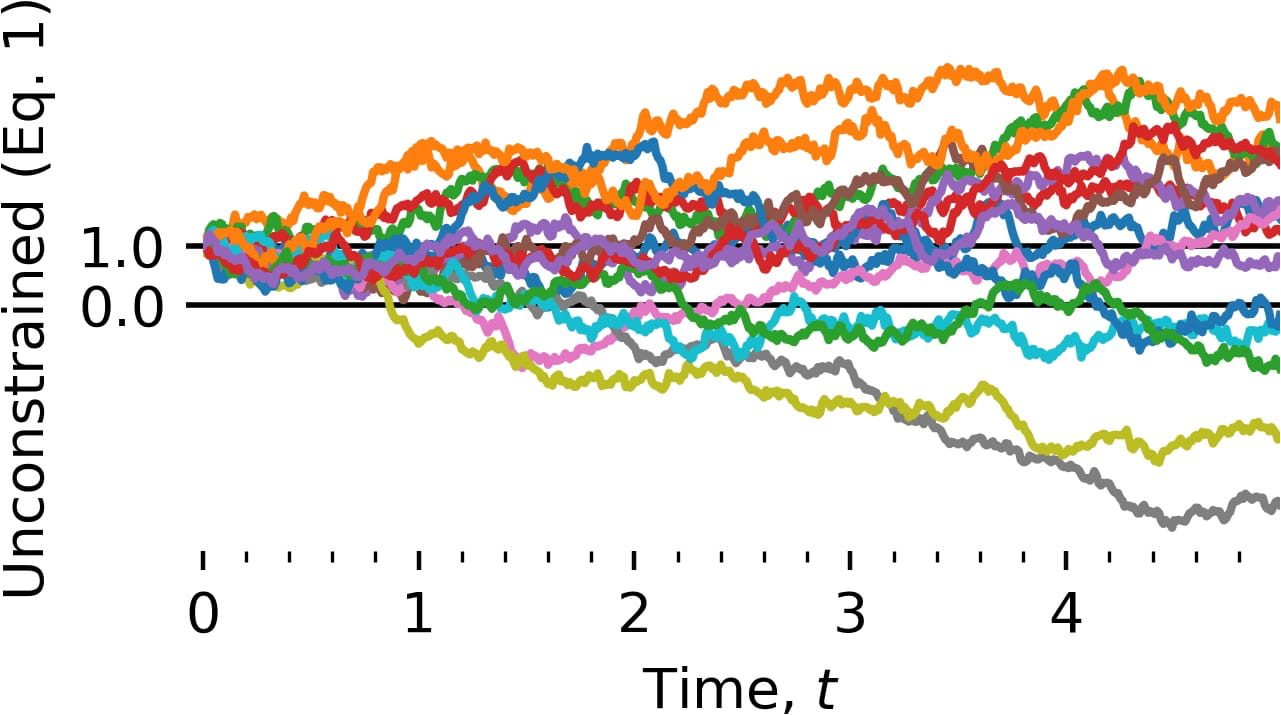}
    ~
    \includegraphics[width=0.32\textwidth]{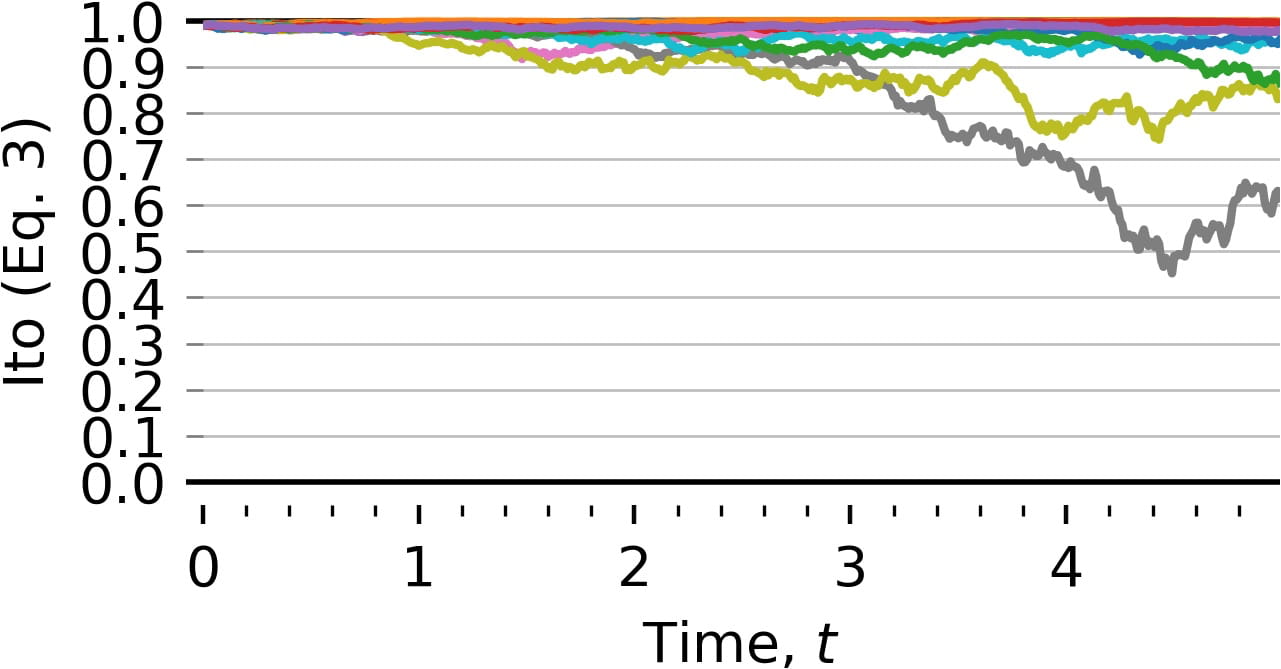} 
    ~
    \includegraphics[width=0.32\textwidth]{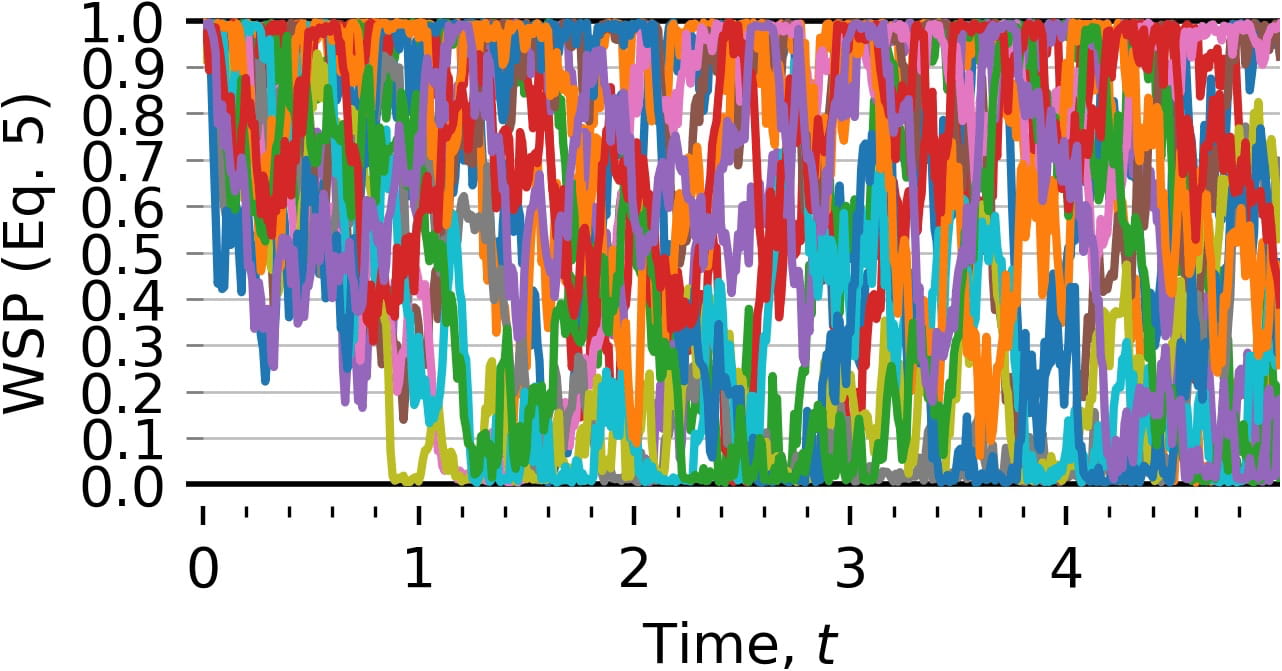}

    \vspace{2mm}
    \hrule
    \hrule
    \vspace{2mm}
    
    \includegraphics[width=1.0\textwidth]{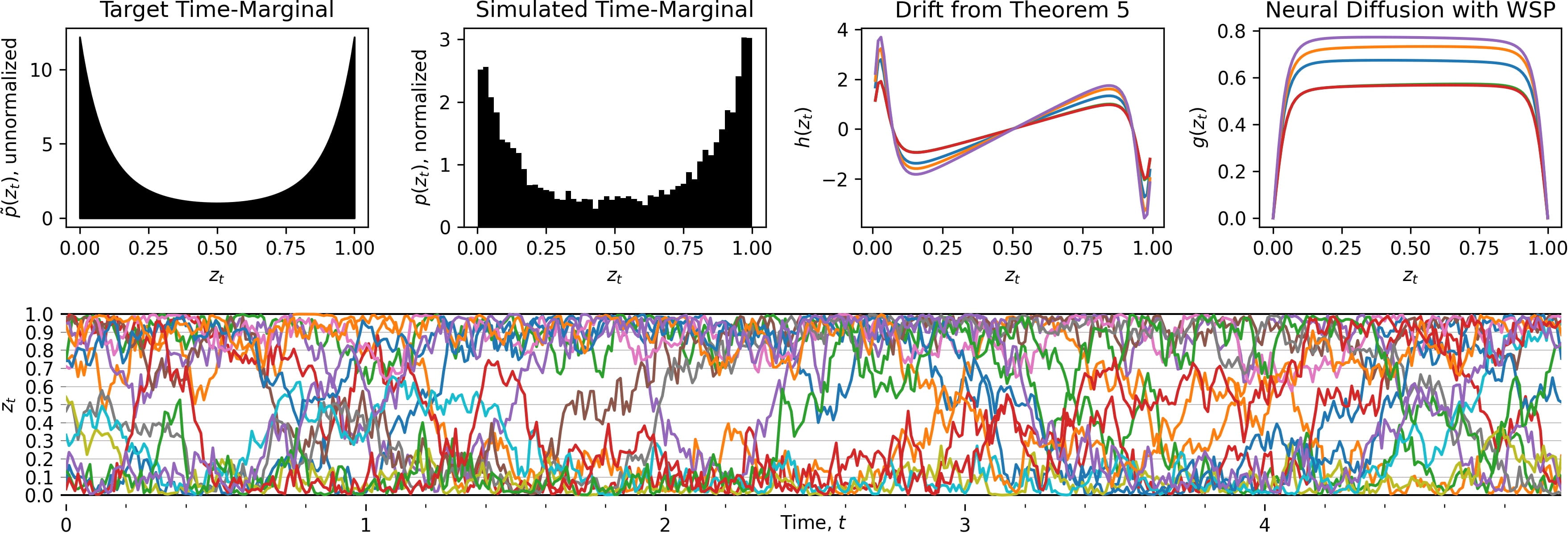}   
    
    \caption{\textbf{Top: Inductive bias of WSP vs.~baselines.} Left: unconstrained SDE (\cref{eq:unconstrained}) with NN quickly leaves $K = [0, 1]$. Middle: SDE transformed via sigmoid (\cref{eq:ito,eq:ito-absorbed}) sticks to the boundary. Right: SDE with WSP (\cref{eq:wsp}) successfully remains in $K$.
    \textbf{Bottom: Stationary WSP matches target time-marginal while remaining viable.} Additional results in \cref{apx:results-stationary}.}
    \label{fig:ito-inductive-bias}
\end{figure*}

\section{Experiments and Results} \label{sec:results}

We present experiments that stress-test WSP against baselines, comparing inductive biases, faithfulness to expert knowledge, and effects on optimization dynamics.

\subsection{Synthetic Data Experiments}

\paragraph{WSP maps expert-given dynamics to viable dynamics on arbitrary polyhedra.}
We \emph{stress-test} WSP by transforming the given dynamics, in which $z_t$ \emph{spirals out} of a target region, to remain viable in three arbitrary polyhedra: triangle, square, pentagon. 
The bottom-right of \cref{fig:intuition} compares the given and the WSP-transformed dynamics in blue and red, respectively, showing that they match on the interior of the space while differing near the boundary.
The bottom-left of \cref{fig:intuition} then shows that sample trajectories from the WSP-transformed dynamics remain within the polyhedra while still spiraling according to the original dynamics. 
Details in \cref{apx:expert-dynamics-setup}.

\paragraph{WSP produces volatile trajectories, qualitatively similar to EMA data, while baselines do not.}
Due to Limitation 2 (\cref{sec:intro}), initialization plays a crucial role in the success of expressive SDE-based models.
As such, to empirically compare the inductive bias of WSP (\cref{eq:wsp}) against baselines (\cref{eq:unconstrained,eq:ito,eq:ito-absorbed}), we solve SDEs given by NNs $h$ and $g$ with randomly sampled weights.
We define the viable region, $K = [0, 1]$, and specifically choose $z_0 = 0.99$ near the boundary to \emph{stress-test} the chain-rule based SDEs in \cref{eq:ito,eq:ito-absorbed} to show that, close to the boundary, they struggle to return to the interior of $K$ (setup in \cref{apx:synthetic-setup}).
While simple, \emph{these experiments show WSP boasts a stark improvement over baselines in its ability to capture volatile data}.
\cref{fig:ito-inductive-bias} (top) shows WSP ensures SDE samples are viable in $K$ (faithful to expert knowledge), while freely moving across $K$ (appropriate inductive bias for volatile data).
Notably, WSP trajectories qualitatively resemble the dynamics of suicidal ideation in EMA data.
In contrast, unconstrained SDEs quickly leave $K$, and SDEs based on It\^o's lemma stick to the boundary.
In \cref{fig:ito-inductive-bias} (with additional results in \cref{apx:results-kl-expansion,apx:results-stationary}), we show these behaviors persist under the Karhunen-Lo\`eve approximation of Brownian motion (from \cref{apx:latent-sde}), and in the stationary construction from \cref{thm:stationary}.

\subsection{Real EMA Data Experiments}

\paragraph{Latent SDEs.}
We model patient $n$'s $d$th EMA response with an ordinal likelihood conditioned on their latent psychological state, $z_t^n \in [0,1]^{D_z}$, evolving via an autonomous Stratonovich SDE:
\begin{align*}
z_{t_0}^n \sim \mathcal{N}_{[0,1]}(\mu_0,\sigma_0^2 \cdot I_{D_z}), \phantom{a} 
dz_t^n = h(z_t^n; \theta) \cdot dt + \mathrm{diag}(g(z_t^n;\theta)) \circ dB_t, \phantom{a} x_{t_m}^{n,d} \mid z_{t_m}^{n,d} \sim \mathrm{Cat}\left(\lambda(z_{t_m}^{n,d};\sigma_\epsilon)\right),
\end{align*}
where $\mathcal{N}_{[0,1]}$ is a truncated normal and $\lambda(\cdot;\sigma_\epsilon)$ is an ordinal link function with fixed cutpoints in $[0,1]$.
Since traditional variational methods for latent SDEs remain strong competitors to simulation-free methods, we choose the traditional method, but approximate Brownian motion with a truncated Karhunen-Lo\`eve expansion, as done by~\citet{ghosh2022differentiable}, allowing us to replace the SDE solver with a stable, high-order, adaptive ODE solver.
Details in \cref{apx:latent-sde}.

\paragraph{Viable Solvers.}
Although WSP is viable in continuous time, standard numerical solvers can step outside $K$ due to discretization.
Developing solvers with viability guarantees is an important future direction. 
In our real data experiments, we therefore include a pragmatic safeguard: we clip the state into $K$ before evaluating drift/diffusion to improve numerical stability.

\paragraph{Baselines.}
We evaluate a set of ablations on 4 real EMA datasets (\cref{apx:real-data}) to isolate the effect of (i) imposing viability through WSP and (ii) applying the clipping safeguard used for numerical stability during training (since viable solvers are an open problem).
Specifically, we compare latent SDE models with the dynamics below. 
Let $(\tilde h,\tilde g)$ be NNs with linear and softplus output activations, respectively, and let $\mathrm{clip}(z)=\min(\max(z,0),1)$ denote an element-wise clipping onto the target unit cube. 
\textcolor{colorVanillaStrong}{\textbf{Vanilla}}:  $(h,g)=(\tilde h,\tilde g)$. 
\textcolor{colorVanillaClipStrong}{\textbf{Vanilla+Clip}}: $(h,g)=(\tilde h\circ \mathrm{clip}, \tilde g\circ \mathrm{clip})$.
\textcolor{colorWSPNoClipStrong}{\textbf{WSP}}: applies \cref{eq:wsp} to $(\tilde h,\tilde g)$, yielding viable dynamics $(h_{\mathrm{WSP}}, g_{\mathrm{WSP}})$. 
\textcolor{colorWSPStrong}{\textbf{WSP+Clip}}: $(h,g)=(h_{\mathrm{WSP}}\circ \mathrm{clip}, g_{\mathrm{WSP}}\circ \mathrm{clip})$.

\paragraph{Evaluation.}
We split each patient's time horizon at the median time step: observations after this point form the ``forecast set'' to stress-test forecasting, while 20\% of time steps before it are held out at random as the ``interpolation set'' (evaluating fit within the training window) to assess whether optimization got stuck in poor local optima.
We evaluate the log posterior predictive on these two heldout sets, as well as constraint satisfaction metrics on the learned dynamics (\cref{apx:real-setup}).

\textbf{Isolating inductive bias under widening forecasts.}
Because STBs are influenced by external factors, they are intrinsically uncertain; thus, latent SDE posteriors quickly revert to the prior with growing forecast horizons. 
Naively comparing forecast distributions would therefore mostly reflect whether a model assigns mass outside the measurement domain, rather than whether its \emph{dynamics} better capture the underlying process. 
We hypothesize that enforcing viability \emph{filters out psychologically inadmissible dynamics, steering learning toward models closer to the ground truth}. 
We therefore fix the variational variances to force more confident forecasts, making differences in inductive bias (rather than uncertainty inflation) observable.
Details in \cref{apx:real-setup}.

\begin{table}[t]
  \caption{\textbf{WSP-based latent neural SDEs satisfy constraints; they improve training dynamics and forecasts on \emph{all} EMA datasets.} We report log posterior predictive on the interpolation and forecasting sets for models trained on a large-scale EMA suicide-risk study (``SMART'') and several EMA mental health datasets (GLOBEM DS2-4)---details in \cref{apx:real-data}. We also report three constraint-satisfaction metrics, summarized below. $\uparrow$/$\downarrow$ indicate whether higher/lower values are better. Details in \cref{apx:real-setup}.}
  \label{table:performance}

  \begin{minipage}[t]{0.77\textwidth}
    \centering
    \scriptsize
    \renewcommand{\arraystretch}{1.2} 
    \begin{tabular}{l|l||c|c|c|c|c}
     & & \multicolumn{2}{c|}{\textbf{Log Predictive}} & \multicolumn{3}{c}{\textbf{Constraint Satisfaction}} \\ \cline{3-7}

    \multirow{-2}{*}{\textbf{Data}} & \multirow{-2}{*}{\textbf{Method}} & Interpolation $\uparrow$  & Forecast $\uparrow$ & DRVP $\uparrow$ & DIVP $\uparrow$ & DIDV $\downarrow$ \\ \hline \hline
    \multirow{4}{*}{SMART} & \cellcolor{colorWSPStrong} \textcolor{white}{WSP+Clip} & \cellcolor{colorWSP}\textbf{-34.17} & \cellcolor{colorWSP}\textbf{-187.64} & \cellcolor{colorWSP}\textbf{1.00} & \cellcolor{colorWSP}\textbf{1.00} & \cellcolor{colorWSP}\textbf{0.00} \\ \cline{2-7}

      & \cellcolor{colorWSPNoClipStrong} \textcolor{white}{WSP} & \cellcolor{colorWSPNoClip}\textbf{-34.00}& \cellcolor{colorWSPNoClip}\textbf{-189.23}& \cellcolor{colorWSPNoClip}\textbf{1.00}& \cellcolor{colorWSPNoClip}\textbf{1.00}& \cellcolor{colorWSPNoClip}\textbf{0.00}
    \\ \cline{2-7}
     & \cellcolor{colorVanillaClipStrong} \textcolor{white}{Vanilla+Clip} & \cellcolor{colorVanillaClip}-44.17 & \cellcolor{colorVanillaClip}-218.20 & \cellcolor{colorVanillaClip}\textbf{1.00}& \cellcolor{colorVanillaClip}0.00& \cellcolor{colorVanillaClip}3.62
     
     \\ \cline{2-7}
    & \cellcolor{colorVanillaStrong} \textcolor{white}{Vanilla} & \cellcolor{colorVanilla}-46.50 & \cellcolor{colorVanilla}-226.50 & \cellcolor{colorVanilla}\textbf{1.00} & \cellcolor{colorVanilla}0.00 & \cellcolor{colorVanilla}3.98 \\ \hline
    
    \multirow{4}{*}{DS2} & \cellcolor{colorWSPStrong} \textcolor{white}{WSP+Clip} & \cellcolor{colorWSP}\textbf{-40.34} & \cellcolor{colorWSP}\textbf{-103.73} & \cellcolor{colorWSP}\textbf{1.00} & \cellcolor{colorWSP}\textbf{1.00} & \cellcolor{colorWSP}\textbf{0.00} \\ \cline{2-7}     
    & \cellcolor{colorWSPNoClipStrong} \textcolor{white}{WSP} & \cellcolor{colorWSPNoClip}\textbf{-46.97}& \cellcolor{colorWSPNoClip}-124.19& \cellcolor{colorWSPNoClip}\textbf{1.00}& \cellcolor{colorWSPNoClip}\textbf{1.00}& \cellcolor{colorWSPNoClip}\textbf{0.00}

    \\ \cline{2-7}
    & \cellcolor{colorVanillaClipStrong} \textcolor{white}{Vanilla+Clip} & \cellcolor{colorVanillaClip}-108.19& \cellcolor{colorVanillaClip}-182.68& \cellcolor{colorVanillaClip}0.89& \cellcolor{colorVanillaClip}0.00 & \cellcolor{colorVanillaClip}6.20 
     \\ \cline{2-7}
    & \cellcolor{colorVanillaStrong} \textcolor{white}{Vanilla} & \cellcolor{colorVanilla}-139.63 & \cellcolor{colorVanilla}-423.54 & \cellcolor{colorVanilla}0.55 & \cellcolor{colorVanilla}0.00 & \cellcolor{colorVanilla}9.23 \\ \hline
    
    \multirow{4}{*}{DS3} & \cellcolor{colorWSPStrong} \textcolor{white}{WSP+Clip} & \cellcolor{colorWSP}\textbf{-34.37} & \cellcolor{colorWSP}\textbf{-95.07} & \cellcolor{colorWSP}\textbf{1.00} & \cellcolor{colorWSP}\textbf{1.00} & \cellcolor{colorWSP}\textbf{0.00} \\ \cline{2-7}
    & \cellcolor{colorWSPNoClipStrong} \textcolor{white}{WSP} & \cellcolor{colorWSPNoClip}\textbf{-34.07}& \cellcolor{colorWSPNoClip}\textbf{-96.12}& \cellcolor{colorWSPNoClip}\textbf{1.00}& \cellcolor{colorWSPNoClip}\textbf{1.00}& \cellcolor{colorWSPNoClip}\textbf{0.00}

    \\ \cline{2-7}
    & \cellcolor{colorVanillaClipStrong} \textcolor{white}{Vanilla+Clip} & \cellcolor{colorVanillaClip}-110.18& \cellcolor{colorVanillaClip}-268.36& \cellcolor{colorVanillaClip}0.96& \cellcolor{colorVanillaClip}0.00& \cellcolor{colorVanillaClip}6.78 
     \\ \cline{2-7}
    & \cellcolor{colorVanillaStrong} \textcolor{white}{Vanilla} & \cellcolor{colorVanilla}-129.48 & \cellcolor{colorVanilla}-558.02 & \cellcolor{colorVanilla}0.59 & \cellcolor{colorVanilla}0.00 & \cellcolor{colorVanilla}4.15 \\ \hline   
    
    \multirow{4}{*}{DS4} & \cellcolor{colorWSPStrong} \textcolor{white}{WSP+Clip} & \cellcolor{colorWSP}\textbf{-41.67} & \cellcolor{colorWSP}\textbf{-93.73} & \cellcolor{colorWSP}\textbf{1.00} & \cellcolor{colorWSP}\textbf{1.00} & \cellcolor{colorWSP}\textbf{0.00} \\ \cline{2-7}
    & \cellcolor{colorWSPNoClipStrong} \textcolor{white}{WSP} & \cellcolor{colorWSPNoClip}\textbf{-41.75}& \cellcolor{colorWSPNoClip}\textbf{-93.88}& \cellcolor{colorWSPNoClip}\textbf{1.00}& \cellcolor{colorWSPNoClip}\textbf{1.00}& \cellcolor{colorWSPNoClip}\textbf{0.00}

    \\ \cline{2-7}
    & \cellcolor{colorVanillaClipStrong} \textcolor{white}{Vanilla+Clip} & \cellcolor{colorVanillaClip}-114.86 & \cellcolor{colorVanillaClip}-120.47& \cellcolor{colorVanillaClip}0.78& \cellcolor{colorVanillaClip}0.02& \cellcolor{colorVanillaClip}6.86
     \\ \cline{2-7}
     & \cellcolor{colorVanillaStrong} \textcolor{white}{Vanilla} & \cellcolor{colorVanilla}-147.43 & \cellcolor{colorVanilla}-246.90 & \cellcolor{colorVanilla}0.58 & \cellcolor{colorVanilla}0.00 & \cellcolor{colorVanilla}8.49 \\ 
    \end{tabular}
  \end{minipage}
  ~
  \begin{minipage}[t]{0.21\textwidth}
    \footnotesize
    \vspace{-43.5pt}
    \textbf{Metrics: DRVP and DIVP} measure the fraction of the boundary that satisfies \cref{thm:milian}'s constraints on the drift and diffusion, respectively. \textbf{DIDV} measures the average difference between the diffusion value at the boundary and 0 (where 0 distance indicates a viable diffusion).
  \end{minipage}
\end{table}

\paragraph{WSP-based latent neural SDEs satisfy constraints, leading to better training dynamics and forecasts on real EMA data over Vanilla neural SDEs.}
\cref{table:performance} shows that:
(1) WSP-based dynamics avoid assigning probability mass to impossible outcomes: the constraint-satisfaction metrics are perfect for WSP, by construction, and are far from perfect for the baselines.
(2) These constraints guide optimization toward better optima: WSP-based models achieve higher interpolation log predictive performance. 
Since WSP constrains the dynamics, Vanilla could theoretically match it; thus, the gap suggests Vanilla training often converges to poorer optima.
(3) Altogether, the WSP-based model exhibits a higher log-predictive for forecasting, indicating its inductive bias is more appropriate for mental health EMA data.
\cref{fig:wsp-qualitative-inline} confirms this qualitatively: in comparison to WSP, the Vanilla dynamics escape the valid $[0, 1]$ space, fit the observed data poorly, and make worse forecasts.
Finally, we note that although WSP substantially improves forecasts, it still cannot reliably forecast a patient's true state for the \emph{entirety} of the second half of the study just from the first half---EMA data is too stochastic for \emph{any} model to be this accurate.
\emph{Instead, these results suggest that embedding additional clinical knowledge in latent SDEs could yield similarly large gains.}

\paragraph{Ablation: Clipping yields modest gains for Vanilla and little to no gain for WSP.}
As shown in \cref{table:performance}, Vanilla+Clip improves over Vanilla on interpolation and forecasting, yet remains substantially worse than WSP, while WSP+Clip is comparable to WSP.
This confirms that WSP's gains stem from viability-constrained dynamics, not clipping.

\begin{figure*}[t!]
    \centering

    \includegraphics[width=0.48\textwidth]{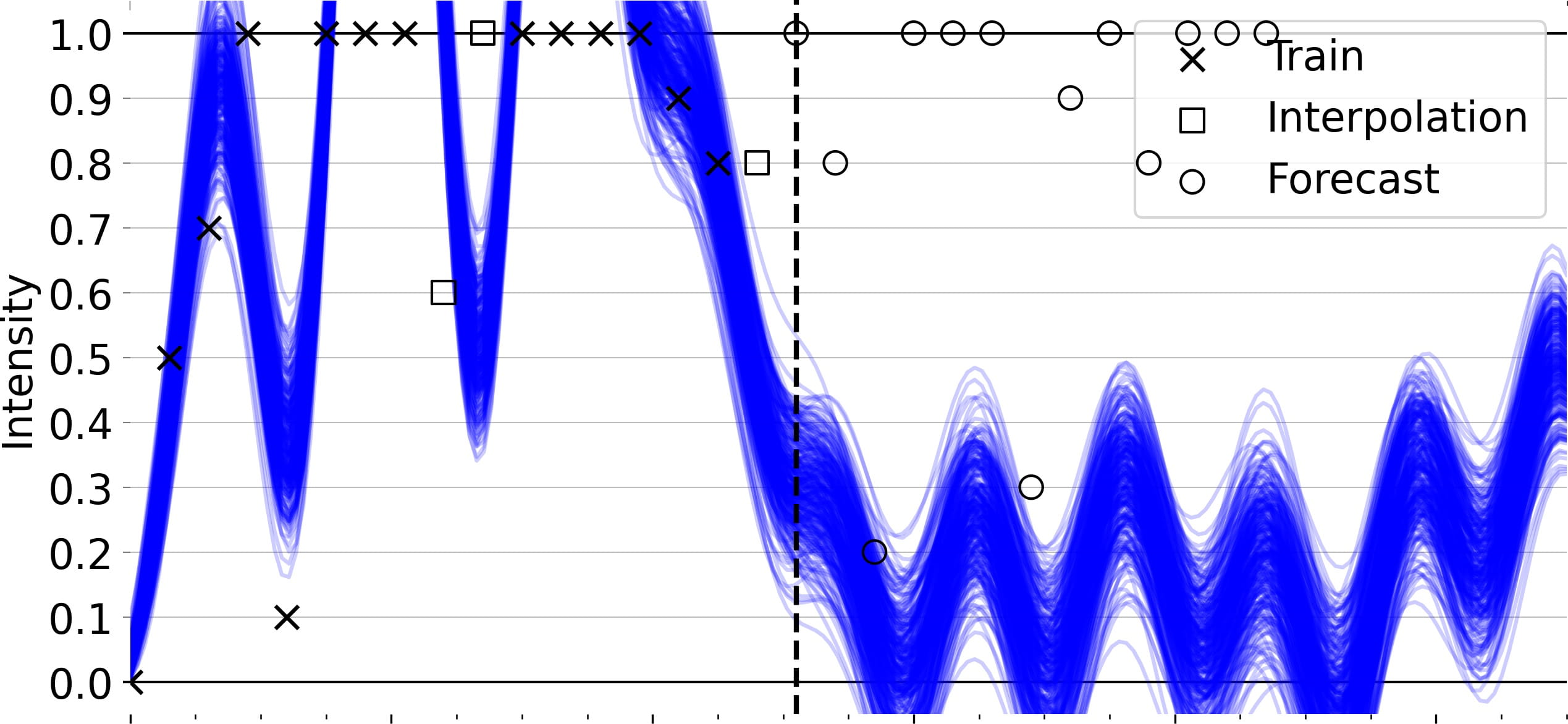}
    ~
    \includegraphics[width=0.48\textwidth]{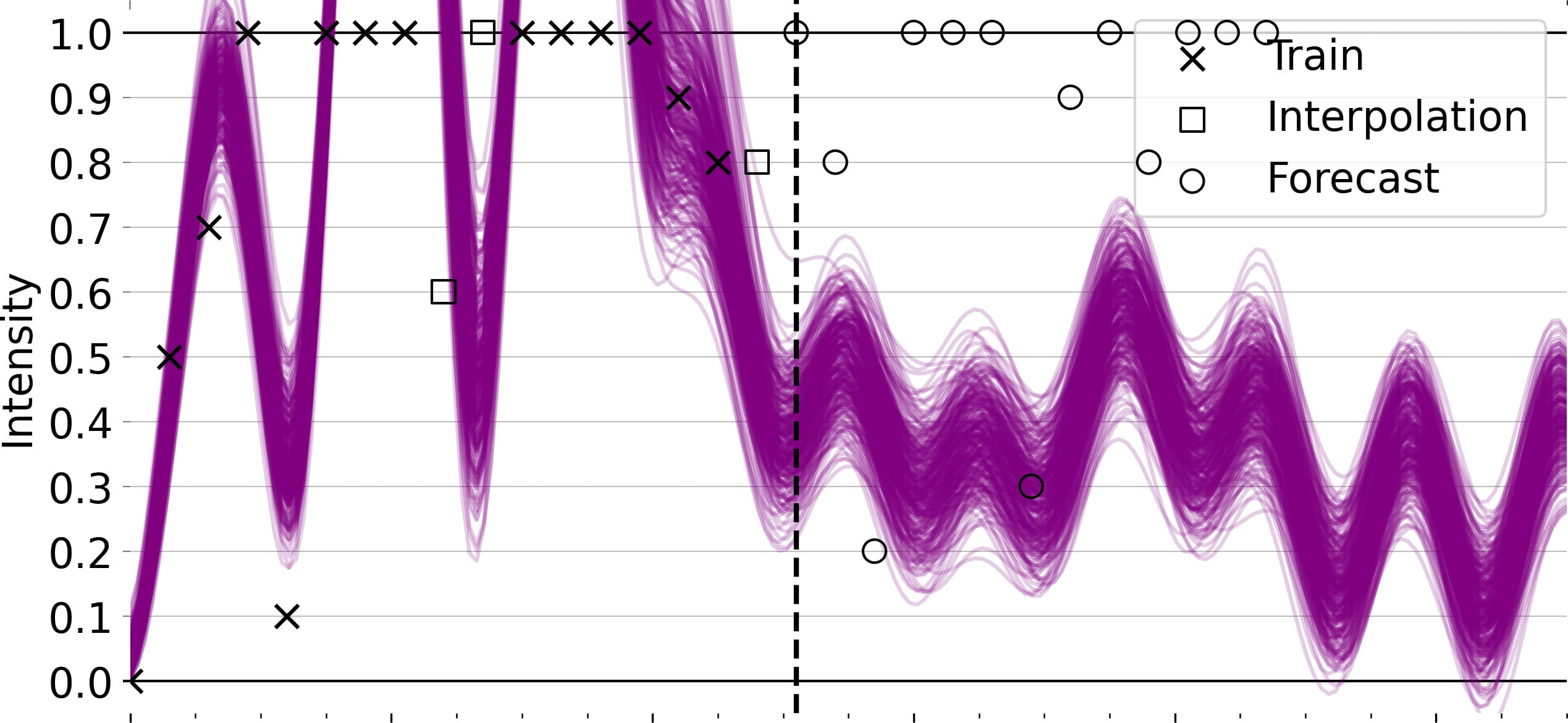}

    \includegraphics[width=0.48\textwidth]{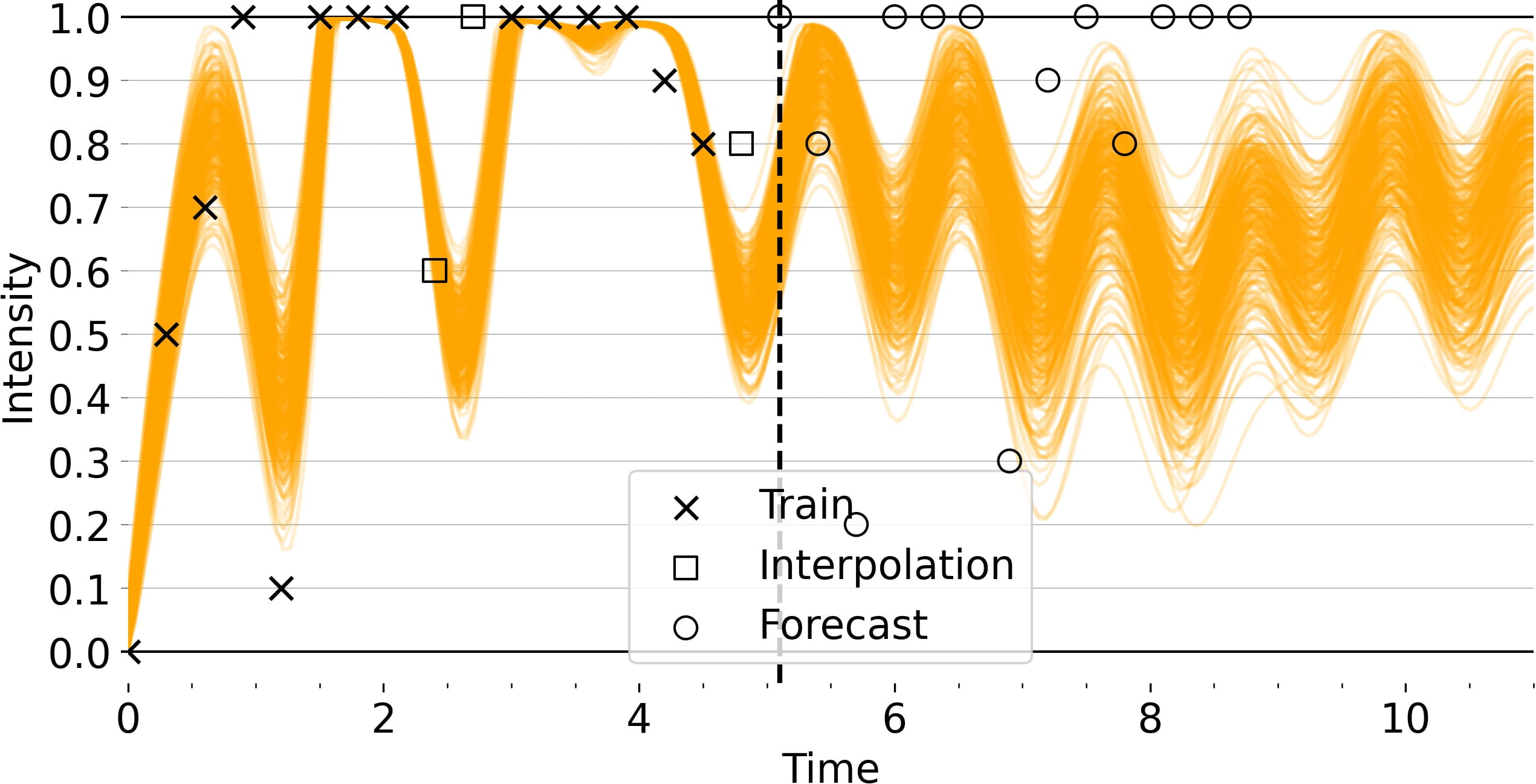}
    ~
    \includegraphics[width=0.48\textwidth]{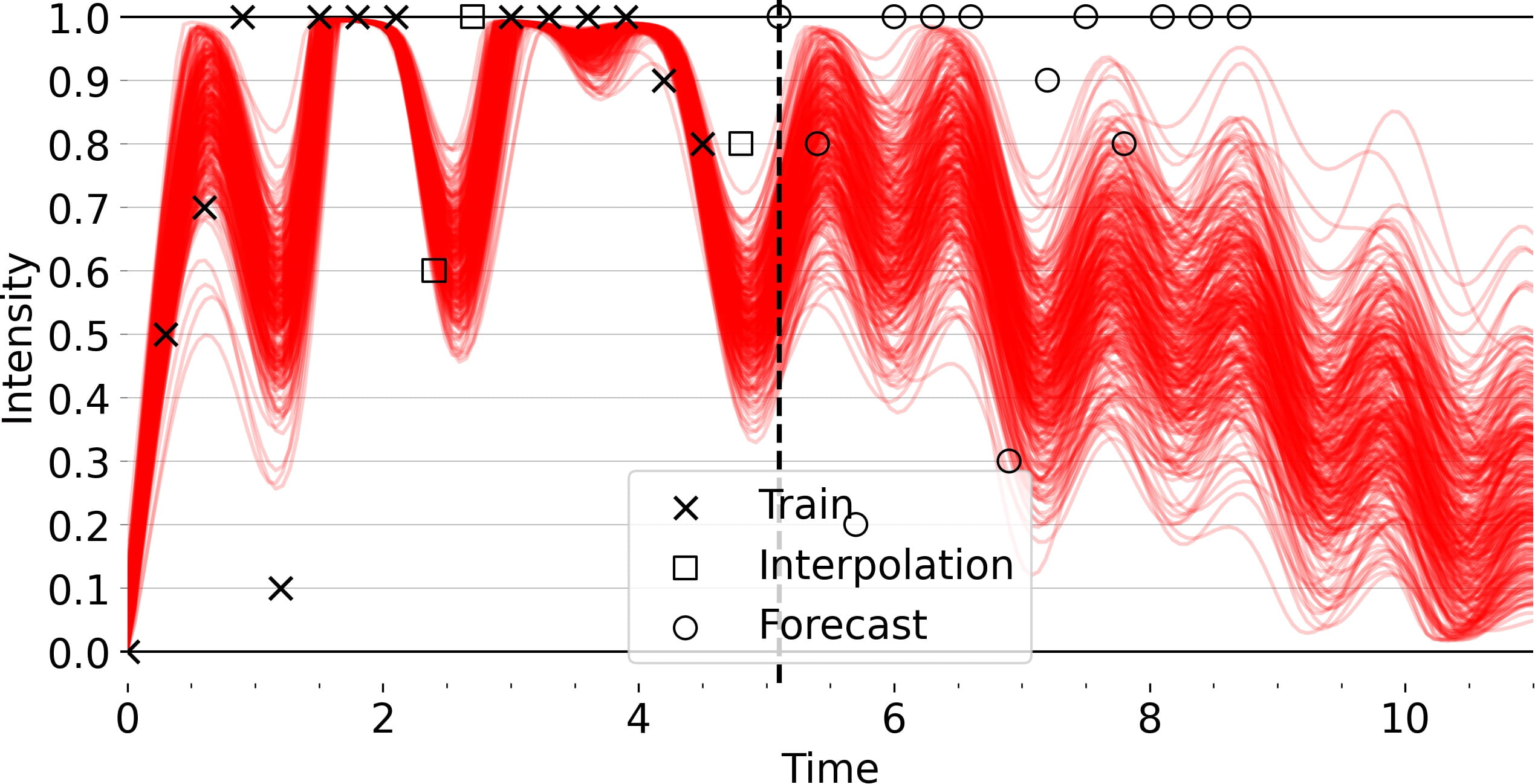}
    
    \caption{\textbf{WSP-based latent neural SDE exhibits better inductive bias than Vanilla neural SDE baseline on the ``SMART'' data.} Each plot visualizes posterior samples for a patient, corresponding to EMA item ``urge to die,'' along with data 0-10 Likert-scale EMA items. Top-left: \textcolor{colorVanillaStrong}{\textbf{Vanilla}}. Top-right: \textcolor{colorVanillaClipStrong}{\textbf{Vanilla+Clip}}. Bottom-left: \textcolor{colorWSPNoClipStrong}{\textbf{WSP}}. Bottom-right: \textcolor{colorWSPStrong}{\textbf{WSP+Clip}}. Additional results in \cref{apx:viz-qualitative}.}
    \label{fig:wsp-qualitative-inline}
\end{figure*}

\section{Discussion and Future Work} \label{sec:discussion}

In this work, we introduce a general method for transforming arbitrary---neural or expert-specified---dynamics into SDEs on arbitrary compact polyhedra. 
Empirically, our method substantially improves optimization and forecasts on EMA data for suicide risk and mental health. 
These results are a first step toward aligning expressive probabilistic models with clinical knowledge, but trustworthy models for scientific insight and intervention require further work, which we outline next.

\paragraph{Hybrid Modeling to Empirically Test Psychological Theories of Suicide.}
WSP naturally supports hybrid models that blend mechanistic, expert-specified structure with data-informed neural components, providing a vehicle for empirically testing psychological theories of suicide~\citep{millner2020advancing}. 
Using WSP, domain experts can encode theoretical constraints into the drift and diffusion---such as directional influences among symptoms---while NNs learn the functional form of these influences. 
In future work, we plan to (a) instantiate specific theories as families of constrained SDEs, and (b) empirically assess these theories by comparing their forecasting performance.

\paragraph{Encoding Relationships Between EMA Items via Linear Inequality Constraints.}
Although our main application focuses on constraints for 0–10 Likert scales, the same framework accommodates a richer class of linear inequality constraints to encode relationships among symptoms. 
For example, if we believe suicidal ideation is driven by an intolerable internal state, we can impose that the ``amount'' of suicidal ideation is always lower than the ``amount'' of internal distress.
In future work, we plan to systematically test such clinically motivated constraints.

\paragraph{Integration with Other Popular SDE-based Models.}
Our results may benefit other types of SDE-based models, like diffusion models~\citep{song2021scorebased} and infinitely deep Bayesian NNs~\citep{xu2022infinitely}, where the data often live in bounded or simplex-like domains (e.g.~pixel values or probabilities), but the underlying SDEs are typically defined on the entire Euclidean space. 
In future work, we hope to evaluate whether our constraint-satisfying dynamics translate into improved optimization and performance for other SDE-based models.

\FloatBarrier

\section*{Acknowledgments}
Research reported in this publication was supported by the National Institute of Mental Health (NIMH) of the National Institutes of Health (NIH) under Award Numbers R01MH139486 and U01MH116928. The content is solely the responsibility of the authors and does not necessarily represent the official views of  the National Institutes of Health. The authors are also grateful for funding from the Tommy Fuss Collaborative STAR Fund, the Chet and Will Griswold Suicide Prevention Fund, and a gift from an anonymous donor. Finally, the authors are grateful to Wellesley College for supporting YL in the summer of 2025.

\bibliography{references}

\newpage
\appendix

\addcontentsline{toc}{section}{Appendix} 
\part{Appendix} 
\parttoc 

\section{Summary of Related Work} \label{apx:related-work}

{
\footnotesize

\begin{tabularx}{\linewidth}{
>{\raggedright\arraybackslash}p{0.30\linewidth}||
>{\centering\arraybackslash}X
>{\centering\arraybackslash}X
>{\centering\arraybackslash}X
}
\toprule
\textbf{Approach} &
\textbf{Smoothness of Drift/Diffusion} &
\textbf{Any r-Polyhedron} &
\textbf{Expressive} \\
\midrule
\midrule
\textbf{RSDE} (e.g.~\cite{pilipenko2014introduction}) & Non-smooth dynamics activated at boundary &
\checkmark &
Expressive in interior; fixed at boundary \\
\midrule
\textbf{Chain Rule} (\cref{sec:challenges}) &
Smooth on interior; degenerate at boundary &
To the best of our knowledge, limited to spaces explored in prior work &
\checkmark \\
\midrule
\textbf{Handcrafted Viability Constraints} (e.g.~\cite{cai1996generation,cresson2012validating,d2013bounded,cresson2016stochastic,cresson2018note,rohanizadegan2020discrete}) &
\checkmark &
Limited to spaces explored in prior work &
Limited to dynamics explored in prior work \\
\midrule
\textbf{This Work (WSP)} &
\checkmark &
\checkmark &
\checkmark \\
\bottomrule
\end{tabularx}
}

\FloatBarrier

\section{Extending \cref{thm:milian} to Stratonovich SDEs on Compact Polyhedra}
\label{apx:proof-stratonovich}

\begin{corollary} \label{thm:stratonovich} 
    \textbf{Suppose} that the drift and diffusion, $h(t, z_t)$ and $g(t, z_t)$, of a Stratonovich SDE, defined for $t \geq 0$ and $z_t \in \mathbb{R}^{D_z}$, satisfy conditions (i)-(iii) from \cref{thm:milian}.
    Suppose further that for all $T > 0$, $z_t, z_t' \in K$, and $t \in [0, T]$, $\lVert \mathrm{diag} (\nabla_{z_t} g(t, z_t)) - \mathrm{diag}(\nabla_{z_t'} g(t, z_t')) \rVert \leq C_T\cdot \lVert z_t - z_t' \rVert$.
    \textbf{Then} $z_t$ is viable in compact polyhedron $K$ if and only if (a)-(b) from \cref{thm:milian} hold.
\end{corollary}

\subsection{Proof of the forward direction}

\begin{proof}
Given the Stratonovich interpretation of the SDE in \cref{eq:unconstrained},
\begin{align}
    d z_t = h(t, z_t) \cdot dt + (\mathrm{diag} \circ g)(t, z_t) \circ d B_t,
\end{align}
we can write the equivalent It\^o SDE as follows:
\begin{align}
    d z_t &= \hat{h}(t, z_t) \cdot dt + (\mathrm{diag} \circ g)(t, z_t) \cdot d B_t, \label{eq:strat-from-ito}
\end{align}
where,
\begin{align}
    \hat{h}(t, z_t) &= h(t, z_t) + \frac{1}{2} \cdot \mathrm{diag}( \nabla_{z_t} g(t, z_t)) \odot g(t, z_t).    
\end{align}
Since in \cref{eq:strat-from-ito}, the diffusion is unchanged, we only need to show that when $h$ satisfies (i)-(iii) and (a), so does $\hat{h}$.

\paragraph{Proof that $\hat{h}$ satisfies (i) from \cref{thm:milian}.} 
We prove that for each $T > 0$, there exists $C_T > 0$ such that for all $z_t \in K$ and $t \in [0, T]$, $\lVert \hat{h}(z_t) \rVert^2 \leq C_T \cdot ( 1 + \lVert z_t \rVert^2)$. 
We do this as follows:
\begin{align}
    \lVert \hat{h}(z_t) \rVert &= \left\lVert h(t, z_t) + \frac{1}{2} \cdot \mathrm{diag}( \nabla_{z_t} g(t, z_t)) \odot g(t, z_t) \right\rVert \\
    &\leq \lVert h(t, z_t) \rVert + \frac{1}{2} \cdot \lVert \mathrm{diag}( \nabla_{z_t} g(t, z_t)) \odot g(t, z_t) \lVert \\
    &\leq \lVert h(t, z_t) \rVert + \frac{1}{2} \cdot \lVert \mathrm{diag}( \nabla_{z_t} g(t, z_t)) \rVert \cdot \lVert g(t, z_t) \lVert \\
    &\leq \underbrace{\sqrt{C_T' \cdot ( 1 + \lVert z_t \rVert^2)}}_{\text{bounded via condition (i)}} + \underbrace{\frac{1}{2} \cdot \lVert \mathrm{diag}( \nabla_{z_t} g(t, z_t)) \rVert}_{\text{bounded by const. via condition (ii)}} \cdot \underbrace{\sqrt{C_T' \cdot ( 1 + \lVert z_t \rVert^2)}}_{\text{bounded via condition (i)}}
\end{align}
The above line can be written in the form $(1 + B) \cdot \sqrt{C_T' \cdot ( 1 + \lVert z_t \rVert^2)}$, which, when squared, gives us an inequality of the form, $\lVert \hat{h}(z_t) \rVert^2 \leq C_T \cdot ( 1 + \lVert z_t \rVert^2)$.

\paragraph{Proof that $\hat{h}$ satisfies (ii) from \cref{thm:milian}.} 
Here, we prove that for all $T > 0$, $z_t, z_t' \in K$, and $t \in [0, T]$, $\lVert \hat{h}(z_t) - \hat{h}(z_t') \rVert \leq C_T \cdot \lVert z_t - z_t' \rVert$:
\begin{align}
    \lVert \hat{h}(z_t) - \hat{h}(z_t') \rVert
    &= \left\lVert h(t, z_t) + \frac{1}{2} \cdot \mathrm{diag}( \nabla_{z_t} g(t, z_t)) \odot g(t, z_t) - h(t, z_t') - \frac{1}{2} \cdot \mathrm{diag}( \nabla_{z_t'} g(t, z_t')) \odot g(t, z_t') \right\rVert \\
    &= \left\lVert h(t, z_t) - h(t, z_t') + \frac{1}{2} \cdot \mathrm{diag}( \nabla_{z_t} g(t, z_t)) \odot g(t, z_t)  - \frac{1}{2} \cdot \mathrm{diag}( \nabla_{z_t'} g(t, z_t')) \odot g(t, z_t') \right\rVert \\
    &\leq \left\lVert h(t, z_t) -h(t, z_t')\right\rVert + \left\lVert\frac{1}{2} \cdot \mathrm{diag}( \nabla_{z_t} g(t, z_t)) \odot g(t, z_t)  - \frac{1}{2} \cdot \mathrm{diag}( \nabla_{z_t'} g(t, z_t')) \odot g(t, z_t') \right\rVert
\end{align}
Using the trick by~\citet{mo2171798}, we have
\begin{align}
    \lVert \hat{h}(z_t) - \hat{h}(z_t') \rVert
    &\leq \underbrace{\left\lVert h(t, z_t) - h(t, z_t')\right\rVert}_{\leq C_T' \cdot \lVert z_t - z_t' \rVert} \\ \nonumber
    &\quad + \frac{1}{2} \cdot \underbrace{\left\lVert \mathrm{diag}( \nabla_{z_t} g(t, z_t)) - \mathrm{diag}( \nabla_{z_t'} g(t, z_t'))\right\rVert}_{\leq C_T' \cdot \lVert z_t - z_t' \rVert} \cdot \left\lVert g(t, z_t)\right\rVert \\ \nonumber
    &\quad + \frac{1}{2} \cdot \left\lVert \mathrm{diag}( \nabla_{z_t'} g(t, z_t'))\right\rVert \cdot \underbrace{\left\lVert g(t, z_t) - g(t, z_t') \right\rVert}_{\leq C_T' \cdot \lVert z_t - z_t' \rVert} 
\end{align}
Finally, since both $g(t, z_t)$ and $\mathrm{diag}( \nabla_{z_t'} g(t, z_t'))$ are Lipschitz on a bounded domain, they can be bounded by a constant.

\paragraph{Proof that $\hat{h}$ satisfies (iii) from \cref{thm:milian}.} 
Since $\hat{h}$ is comprised of addition and scaling operations on continuous functions, it is also continuous.

\paragraph{Proof that $\hat{h}$ satisfies (a) from \cref{thm:milian}.}
When $\langle z_t - u_s, v_s \rangle = 0$, we show (a) holds for $\hat{h}$ as follows:
\begin{align}
    \langle \hat{h}(t, z_t), v_s \rangle &= \underbrace{\langle h(t, z_t), v_s \rangle}_{\geq 0} + \frac{1}{2} \cdot \sum\limits_{d=1}^{D_z} \underbrace{\frac{\partial g^d(t, z_t)}{\partial z_t^d}}_{\text{bounded}} \cdot \underbrace{g^d(t, z_t) \cdot v_s^d}_{= 0} \geq 0.
\end{align}
The first term is non-negative.
The second term is $0$ since $\langle g(t, z_t) \odot e_d, v_s \rangle = 0$ when condition (b) holds for $g$, and when $0$ is multiplied by the partial (bounded thanks to condition (ii) for $g$), we get $0$. 
Thus, $\langle \hat{h}(t, z_t), v_s \rangle \geq 0$.

\end{proof}

\subsection{Proof of the reverse direction}

\begin{proof}
Suppose $z_t$ is viable in $K$ as a solution to the Stratonovich SDE.
Since the Stratonovich SDE and its equivalent It\^o form in \cref{eq:strat-from-ito} share the same sample paths, $z_t$ is equally viable as a solution to the It\^o SDE with drift $\hat{h}$ and diffusion $g$.
From the forward-direction proof, $\hat{h}$ satisfies conditions (i)--(iii) of \cref{thm:milian}; and $g$ satisfies conditions (i)--(iii) by assumption.
Applying the \emph{only-if} direction of \cref{thm:milian} to the viable It\^o SDE with dynamics $(\hat{h}, g)$ yields, for all $s \in \{1, \dots, S\}$ and $z_t \in K$ such that $\langle z_t - u_s, v_s \rangle = 0$:
\begin{align}
    \langle \hat{h}(t, z_t), v_s \rangle \geq 0, 
    \qquad 
    \langle g(t, z_t) \odot e_d, v_s \rangle = 0 
    \quad \text{for all } d \in \{1, \dots, D_z\}.
\end{align}
The second of these is exactly condition (b) from \cref{thm:milian} for the Stratonovich SDE.
It remains to recover condition (a) for $h$.
At any boundary point where $\langle z_t - u_s, v_s \rangle = 0$, condition (b) gives
$g^d(t, z_t) \cdot v_s^d = \langle g(t, z_t) \odot e_d, v_s \rangle = 0$
for every $d \in \{1, \dots, D_z\}$. Therefore,
\begin{align}
    \langle \hat{h}(t, z_t), v_s \rangle
    &= \left\langle h(t, z_t) 
        + \frac{1}{2} \cdot \mathrm{diag} \left(\nabla_{z_t} g(t, z_t)\right) \odot g(t, z_t),  v_s \right\rangle \\
    &= \langle h(t, z_t), v_s \rangle 
        + \frac{1}{2} \cdot \sum_{d=1}^{D_z} 
        \underbrace{\frac{\partial g^d(t, z_t)}{\partial z_t^d}}_{\text{finite by (ii)}} 
        \cdot 
        \underbrace{g^d(t, z_t) \cdot v_s^d}_{=   0 \ \text{by (b)}} \\
    &= \langle h(t, z_t), v_s \rangle.
\end{align}
Since $\langle \hat{h}(t, z_t), v_s \rangle \geq 0$, we conclude $\langle h(t, z_t), v_s \rangle \geq 0$,
which is condition (a) from \cref{thm:milian} for $h$.

\end{proof}

\section{Proof of \cref{thm:stationary}}
\label{apx:proof-stationary}

To prove \cref{thm:stationary}, we will first derive a sufficient form for the drift, $h$, as a
function of the diffusion, $g$, so that it induces stationary dynamics with target (unnormalized)
time-marginal $\tilde{p}(z_t)$ (\cref{apx:sufficient-drift}). 
Then, we show that, although the drift and diffusion vanish at the boundary (\cref{apx:sufficient-drift-vanish}), $\partial K$ is unreachable from the interior, so the boundary does not form an absorbing state (\cref{apx:not-absorbing-state}).
Finally, we will show that the drift satisfies all conditions from \cref{thm:milian}, implying that it is viable in $K$ (\cref{apx:stationary-drift-satisfies-condition}).

\subsection{Derivation of a drift sufficient for stationary dynamics}
\label{apx:sufficient-drift}

We find $h$ by drawing inspiration from the derivation of \citet{cai1996generation}, which sets
the Fokker--Planck--Kolmogorov (FPK) equation to $0$ to obtain stationarity and then solves for
the dynamics. In contrast, instead of solving for the diffusion (which typically requires computing
intractable integrals), we solve for the drift. This derivation parallels classical Langevin dynamics
and stochastic gradient MCMC (e.g.~\cite{ma2015complete,ormandy2019linking}).

We begin with a general autonomous It\^o SDE of the form,
\begin{align}
    dz_t = h(z_t) \cdot dt + g(z_t) \cdot dB_t,
\end{align}
where $g(z_t)\in\R^{D_z\times D_z}$ is a full matrix. Let
$G(z_t)=g(z_t)\Sigma g(z_t)^\intercal$, where $\Sigma$ is the covariance of the Brownian motion.
Stationarity means $\frac{\partial}{\partial t}p(t,z_t)=0$.
Plugging in the FPK equation gives:
\begin{align}
    0
    &=
    -\sum_{d=1}^{D_z} \frac{\partial}{\partial z_t^d}\Big(h^d(z_t) p(z_t)\Big)
    + \frac{1}{2}\sum_{d=1}^{D_z}\sum_{d'=1}^{D_z}\frac{\partial^2}{\partial z_t^d\partial z_t^{d'}}
    \Big(G^{d,d'}(z_t) p(z_t)\Big) \\
    &=
    \sum_{d=1}^{D_z}\frac{\partial}{\partial z_t^d}\left(
    -h^d(z_t)p(z_t)
    + \frac{1}{2}\sum_{d'=1}^{D_z}\frac{\partial}{\partial z_t^{d'}}\Big(G^{d,d'}(z_t)p(z_t)\Big)
    \right).
\end{align}

For this theorem, we only consider identity-covariance Brownian motion and diagonal diffusion:
$\Sigma=I$ and $g(z_t)\in\R_{\ge 0}^{D_z}$ interpreted as $\mathrm{diag}(g(z_t))$. Then $G$ is
diagonal with entries $g^d(z_t)^2$, and the stationary FPK equation simplifies to
\begin{align}
    0
    &=
    \sum_{d=1}^{D_z}\frac{\partial}{\partial z_t^d}\left(
    -h^d(z_t)p(z_t) +  \frac{1}{2}\frac{\partial}{\partial z_t^d}\Big(g^d(z_t)^2 p(z_t)\Big)
    \right).
\end{align}
One sufficient (strong) way to satisfy this is to enforce the element-wise identities:
for every $d\in\{1,\dots,D_z\}$,
\begin{align}
    h^d(z_t)p(z_t) =  \frac{1}{2} \frac{\partial}{\partial z_t^d}\Big(g^d(z_t)^2 p(z_t)\Big).
\end{align}
Solving for $h^d$ and using $p(z_t)=\tilde{p}(z_t)/A$ yields:
\begin{align}
    h^d(z_t)
    &=
    \frac{1}{2p(z_t)}\frac{\partial}{\partial z_t^d}\Big(g^d(z_t)^2 p(z_t)\Big) \\
    &=
    \frac{1}{2\tilde{p}(z_t)}\frac{\partial}{\partial z_t^d}\Big(g^d(z_t)^2 \tilde{p}(z_t)\Big) \\
    &=
     \frac{1}{2} \frac{\partial}{\partial z_t^d}\big[g^d(z_t)^2\big]
    + \frac{1}{2} g^d(z_t)^2 \frac{\partial}{\partial z_t^d}\log\tilde{p}(z_t).
\end{align}
Equivalently,
\begin{align}
    h(z_t)
    =
     \frac{1}{2} \mathrm{diag} \left(\nabla_{z_t}[g(z_t)^2]\right)
    + \frac{1}{2} g(z_t)^2\odot \underbrace{\nabla_{z_t}\log\tilde{p}(z_t)}_{\text{score function}},
    \label{eq:sufficient-drift}
\end{align}
giving us condition (a) from \cref{thm:stationary}.

\subsection{Proof that the sufficient drift can vanish at the boundary} \label{apx:sufficient-drift-vanish}

We now show that the stationary construction from \cref{eq:sufficient-drift} can be \emph{degenerate} at boundary points if (i) the diffusion vanishes on $\partial K$ as a convenient way to enforce viability and (ii) the score remains bounded on the boundary. 
In that case, $dz_t = 0$, indicating a potential absorbing state that would prevent the intended stationary density on $\mathrm{int}(K)$ from being realized.
In \cref{apx:not-absorbing-state}, we then prove that under the assumptions from \cref{thm:stationary}, the boundary is unreachable from the interior of $K$, fixing the problem.

\begin{lemma} \label{lem:absorbing-boundary}
Let $z_t^\partial\in\partial K$ be any boundary point. Consider the stationary drift construction from \cref{eq:sufficient-drift}, based on a corresponding diffusion, $g(z_t)$, that satisfies all conditions from \cref{thm:milian}, in part, using the following, convenient properties:
\begin{enumerate}[label=(A\arabic*)]
    \item $g^d(z_t) \geq 0$ for all
    $z_t \in \mathbb{R}^{D_z}$ and all $d\in\{1,\dots,D_z\}$;
    \item $g(z_t^\partial)=0$;
    \item $\|\nabla_{z_t}\log\tilde{p}(z_t)\|_2 < \infty$.
\end{enumerate}
Then $h(z_t^\partial)=0$. In particular, at $z_t^\partial$ both drift and diffusion vanish; hence if the
SDE reaches $z_t^\partial$, it remains on the boundary.
\end{lemma}

\begin{proof}
By (A1), $g^d(z_t)^2 \geq 0$ and by (A2), $g^d(z_t^\partial) = 0$,
$z_t^\partial$ is the global minimizer of $g^d(z_t)^2$; therefore,
\begin{align}
    \frac{\partial}{\partial z_t^d}\big[g^d(z_t)^2\big]\Big|_{z_t=z_t^\partial}=0
    \qquad \forall d\in\{1,\dots,D_z\},
\end{align}
and
\begin{align}
    \mathrm{diag} \left(\nabla_{z_t}[g(z_t)^2]\right)\Big|_{z_t=z_t^\partial}=0.
\end{align}
Next, by (A2) and (A3),
\begin{align}
    g(z_t^\partial)^2 \odot \nabla_{z_t}\log\tilde{p}(z_t^\partial)=0.
\end{align}
Substituting these into \cref{eq:sufficient-drift} gives $h(z_t^\partial)=0$. 
Finally, since $g(z_t^\partial)=0$ as well, the SDE has $dz_t=0$ at $z_t^\partial$, so $z_t^\partial$ is absorbing.
\end{proof}

\subsection{Proof that the boundary is unreachable from the interior} \label{apx:not-absorbing-state}

We now prove that, under the assumptions from \cref{thm:stationary}, the boundary is unreachable from the interior, ensuring the induced target marginal remains valid.
Our proof draws on ``McKean's Argument'' (e.g.~Section 2.9, Problem 7 from \citet{mckean2024stochastic}, Proposition 1 from \citet{bru1991wishart}, or Proposition 4.3 by \citet{mayerhofer2011strong}).
Specifically, we first derive a 1D SDE characterizing the distance from $z_t$ to boundary $s$, and apply It\^o's lemma to the logarithm of the distance to the boundary. 
By leveraging the linear bounds on the drift and diffusion, we show that the deterministic components of this log-distance remain finite. 
We then apply Doob's $L^2$ martingale convergence theorem to the remaining stochastic integral to show that it also converges to a finite limit almost surely. 
Since reaching the boundary in finite time would require the log-distance to diverge to $-\infty$, this finite limit yields a contradiction, thus proving the boundary is never reached.

\begin{lemma} \label{lem:boundary-unreachable}
Assume the conditions of \cref{thm:stationary}, and let $z_0\in\mathrm{int}(K)$.
Then
\begin{align}
\mathbb{P}\left(\exists t \in [0, \infty):\ z_t\in\partial K \right)=0.
\end{align}
Equivalently, the boundary is unreachable from the interior.
\end{lemma}

\begin{proof}
Define the per-boundary hitting times and the first exit time from $K$:
\begin{align}
    \tau_s = \inf\{t \ge 0:\ d_s(z_t) = 0\}, \qquad
    \tau = \inf\{t \ge 0:\ z_t \in \partial K\} = \min_s \tau_s.
\end{align}
Since $z_0 \in \mathrm{int}(K)$ and $z_t$ is continuous (the SDE has Lipschitz, 
linearly bounded dynamics by (i)--(ii), so a unique strong solution exists), 
we have $z_t \in \mathrm{int}(K)$ for all $t \in [0, \tau)$.
On the event $\{\tau < \infty\}$, continuity of $z_t$ and each $d_s(z_t)$ ensures 
there exists $s^* \in \{1,\dots,S\}$ such that $d_{s^*}(z_\tau) = 0$ and 
$\tau_{s^*} = \tau$; fix such an $s^*$.
We then define the distance-to-boundary process for boundary $s^*$:
\begin{align}
    y_t = d_{s^*}(z_t) = \langle z_t - u_{s^*}, n_{s^*}\rangle, \qquad 
    n_{s^*} = \frac{v_{s^*}}{\|v_{s^*}\|}.
\end{align}
Since $z_0 \in \mathrm{int}(K)$, we have $y_0 > 0$, and $y_\tau = 0$ on 
$\{\tau < \infty\}$.
All bounds below that are stated for $z_t \in K$ apply on $[0, \tau)$, 
since $z_t \in \mathrm{int}(K) \subset K$ throughout this interval.

\paragraph{A 1D SDE for $y_t$ on $[0,\tau)$.}
Because $d_{s^*}(\cdot)$ is affine, $\nabla_{z_t} d_{s^*}(z)=n_{s^*}$ and $\nabla^2_{z_t} d_{s^*}(z)=0$.
Applying It\^o's lemma yields, for $t<\tau$,
\begin{align}
    dy_t
    &= \langle n_{s^*}, h(z_t) \rangle  dt 
    + (n_{s^*} \odot g(z_t))^\intercal \cdot dB_t.
    \label{eq:y_sde_boundary_appx}
\end{align}
We define the scalar drift and diffusion magnitudes,
\begin{align}
    \mu_{s^*}(z_t) = \langle n_{s^*}, h(z_t)\rangle, \qquad
    \sigma_{s^*}(z_t) = \|n_{s^*} \odot g(z_t)\|_2.
\end{align}
By assumption (iv) in \cref{thm:stationary}, $g(z)>0$ for all $z\in\mathrm{int}(K)$,
hence $\sigma_{s^*}(z_t)>0$ for all $t<\tau$.
Next, we define the 1D process,
\begin{align}
    \widetilde B_t = \int_0^t 
    \frac{(n_{s^*}\odot g(z_u))^\intercal}{\sigma_{s^*}(z_u)} \cdot dB_u,
    \qquad t<\tau,
\end{align}
where $dB_u$ denotes the $D_z$-dimensional Brownian increment.
By L\'evy's characterization of Brownian motion,
$\widetilde B_t$ is a standard 1D Brownian motion on $[0,\tau)$, because 
$\widetilde B_0 = 0$ and because its quadratic variation up to $t$ equals $t$:
\begin{align}
    \int_0^t \left\lVert \frac{(n_{s^*}\odot g(z_u))}{\sigma_{s^*}(z_u)} 
    \right\rVert_2^2 \cdot du 
    = \int_0^t 1 \cdot du 
    = t, \qquad t < \tau.
\end{align}
Therefore, for $t<\tau$, we can write \cref{eq:y_sde_boundary_appx} as:
\begin{align}
    dy_t = \mu_{s^*}(z_t) \cdot dt + \sigma_{s^*}(z_t) \cdot d\widetilde B_t.
    \label{eq:y_sde_scalar_appx}
\end{align}

\paragraph{Bounding $\sigma_{s^*}(z)$ and $\mu_{s^*}(z)$ linearly in $d_{s^*}(z)$.}
By assumption (v), there exists $L_{s^*}<\infty$ such that for all $z\in K$,
\begin{align}
    \sigma_{s^*}(z)=\|n_{s^*}\odot g(z)\|_2 \le L_{s^*} d_{s^*}(z).
    \label{eq:sigma_linear_appx}
\end{align}
Next, we bound $\mu_{s^*}(z_t)=\langle n_{s^*},h(z_t)\rangle$ linearly in 
$d_{s^*}(z_t)$.
Using condition (a) from \cref{thm:stationary},
\begin{align}
\mu_{s^*}(z_t)
&=\left\langle n_{s^*}, \frac{1}{2} \mathrm{diag} \left(\nabla_{z_t}[g(z_t)^2]
\right)\right\rangle
+\left\langle n_{s^*}, \frac{1}{2} g(z_t)^2\odot \nabla_{z_t}\log\tilde p(z_t)
\right\rangle .
\end{align}
For diagonal diffusion, the $d$th component of $\mathrm{diag}(\nabla_{z_t}[g(z_t)^2])$ is
$\frac{\partial}{\partial z_t^d}(g^d(z_t)^2)=2g^d(z_t) 
\frac{\partial}{\partial z_t^d}g^d(z_t)$; hence,
\begin{align}
\left\langle n_{s^*}, \frac{1}{2} \mathrm{diag} 
\left(\nabla_{z_t}[g(z_t)^2]\right)\right\rangle 
&=\sum_{d=1}^{D_z} n_{s^*}^d  g^d(z_t) \frac{\partial}{\partial z_t^d}g^d(z_t) \\
&=\left\langle n_{s^*}\odot g(z_t), \mathrm{diag}(\nabla_{z_t} g(z_t))\right\rangle .
\end{align}
Similarly,
\begin{align}
\left\langle n_{s^*}, \frac{1}{2} g(z_t)^2\odot \nabla_{z_t}\log\tilde p(z_t)
\right\rangle
&= \frac{1}{2}\sum_{d=1}^{D_z} n_{s^*}^d  g^d(z_t)^2 
(\nabla_{z_t}\log\tilde p(z_t))^d \\
&= \frac{1}{2}\left\langle n_{s^*}\odot g(z_t),  
g(z_t)\odot \nabla_{z_t}\log\tilde p(z_t)\right\rangle .
\end{align}
This gives us:
\begin{align}
\mu_{s^*}(z_t)
&= \left\langle n_{s^*}\odot g(z_t), \mathrm{diag}(\nabla_{z_t} g(z_t))\right\rangle
+ \frac{1}{2}\left\langle n_{s^*}\odot g(z_t),  
g(z_t)\odot \nabla_{z_t}\log\tilde p(z_t)\right\rangle
\end{align}
Therefore, by the Cauchy-Schwarz inequality,
\begin{align}
\lVert \mu_{s^*}(z_t) \rVert_2
&\le \|n_{s^*}\odot g(z_t)\|_2 \cdot \|\mathrm{diag}(\nabla_{z_t} g(z_t))\|_2 
+ \frac{1}{2} \|n_{s^*} \odot g(z_t)\|_2 
\cdot \|g(z_t)\odot \nabla_{z_t}\log\tilde p(z_t)\|_2 
\label{eq:mu_bound_step1_appx}
\end{align}
Now, we define the finite constants (since $K$ is compact and the relevant 
quantities are continuous by (i)--(iii)):
\begin{align}
M_{\nabla g}=\sup_{z_t\in K}\|\mathrm{diag}(\nabla_{z_t} g(z_t))\|_2,\quad
M_{gs} = \sup_{z \in K} \|g(z) \odot \nabla \log \tilde{p}(z)\|_2 .
\end{align}
Using assumption (v), $\|n_{s^*} \odot g(z_t)\|_2 \le L_{s^*} d_{s^*}(z_t)$, 
and plugging into \cref{eq:mu_bound_step1_appx} yields, for all $z_t\in K$,
\begin{align}
\lVert \mu_{s^*}(z_t) \rVert_2
&\le L_{s^*} d_{s^*}(z_t) M_{\nabla g}
+ \frac{1}{2} L_{s^*} d_{s^*}(z_t) M_{gs} 
=  \underbrace{L_{s^*} \cdot \left( M_{\nabla g} + \frac{1}{2} M_{gs} \right)}_{A_{s^*}} \cdot d_{s^*}(z_t) .
\label{eq:mu_linear_appx}
\end{align}
Thus, both $\sigma_{s^*}(z)$ and $\mu_{s^*}(z)$ are linearly bounded in 
$d_{s^*}(z)$.

\paragraph{Un-reachability of $y_t = 0$ in finite time.}
For $t<\tau$, we have $y_t=d_{s^*}(z_t)>0$, so $\log y_t$ is well-defined.
Applying It\^o's lemma to $f(y_t)=\log y_t$ in \cref{eq:y_sde_scalar_appx} gives,
\begin{align}
d\log y_t
&=\left(\frac{\mu_{s^*}(z_t)}{y_t}
-\frac12\left(\frac{\sigma_{s^*}(z_t)}{y_t}\right)^{2}\right) dt
+\frac{\sigma_{s^*}(z_t)}{y_t} d\widetilde B_t.
\label{eq:dlogy_appx}
\end{align}
By the linear bounds in \cref{eq:sigma_linear_appx,eq:mu_linear_appx} and 
since $y_t = d_{s^*}(z_t) > 0$ for $t < \tau$, the ratios are uniformly bounded:
\begin{align}
\left|\frac{\mu_{s^*}(z_t)}{y_t}\right|\le A_{s^*},
\qquad
0\le \frac{\sigma_{s^*}(z_t)}{y_t}\le L_{s^*}.
\label{eq:ratios_bounded_appx}
\end{align}
Integrating \cref{eq:dlogy_appx} from $0$ to $t<\tau$ yields:
\begin{align}
\log y_t
= \underbrace{\log y_0}_{\text{finite}}
+ \int_0^t \underbrace{\left(\frac{\mu_{s^*}(z_u)}{y_u}
- \frac{1}{2}\left(\frac{\sigma_{s^*}(z_u)}{y_u}\right)^{2}
\right)}_{\text{bounded by } A_{s^*}+\frac{1}{2}L_{s^*}^2} du
+ \underbrace{\int_0^t \frac{\sigma_{s^*}(z_u)}{y_u}  d\widetilde B_u}_{M_t}.
\label{eq:logy_rep_appx}
\end{align}
The term inside the deterministic integral is bounded by $A_{s^*} + \frac{1}{2}L_{s^*}^2$, so the deterministic integral is finite for any finite $t < \tau$.
The stochastic term, $M_t$, is a continuous local martingale on $[0,\tau)$ (as is any It\^o integral with a locally
square-integrable integrand).
We next show that $\lim_{t\uparrow\tau} M_t$ exists and is finite a.s.~in the event $\{\tau < \infty\}$.

The difficulty is that $M_t$ is only defined on $[0, \tau)$, since the integrand $\frac{\sigma_{s^*}(z_t)}{y_t}$ involves $1/y_t$, which may blow up as
$t \uparrow \tau$.
In particular, although $M_t$ is a local martingale, we cannot directly apply martingale convergence to it.
Instead, for each fixed $T < \infty$, we introduce the stopped process $M^T_t$, to which we can apply the martingale convergence theorem and connect it back to our original $M_t$. 
We define the stopped process,
\begin{align}
M^T_t
= \int_0^{t\wedge\tau\wedge T} \frac{\sigma_{s^*}(z_u)}{y_u}  d\widetilde B_u,
\end{align}
which agrees with $M_t$ for all $t < \tau \wedge T$, but is frozen at its value at time $\tau \wedge T$ thereafter.
The key advantage is that stopping at the deterministic time $T$ bounds the integration horizon; combined with the bound from \cref{eq:ratios_bounded_appx}, the integrand of $M^T_t$ is bounded, so $M^T_t$ is a well-defined It\^o integral with continuous sample paths.
Because $M^T_0 = 0$ and the quadratic variation of $M^T_t$, $\langle M^T_t \rangle_{\infty}$, satisfies,
\begin{align}
\mathbb{E} \left[ \langle M^T_t \rangle_{\infty} \right] 
= \mathbb{E}\left[\lim_{t\to\infty} \int_0^{t\wedge\tau\wedge T} \left(\frac{\sigma_{s^*}(z_u)}{y_u}\right)^{2}du\right] 
= \mathbb{E}\left[\int_0^{\tau\wedge T} \left(\frac{\sigma_{s^*}(z_u)}{y_u}\right)^{2}du\right] 
\le L_{s^*}^2\cdot T 
< \infty,
\end{align}
both conditions of \citet[Proposition~1.23]{revuz2013continuous} are satisfied, so $M^T_t$ is a (true) square-integrable martingale.
We can therefore apply Doob's $L^2$ martingale convergence theorem, which says, $M^T_\infty := \lim_{t\to\infty}M^T_t$ exists and is finite a.s.
Relating $M^T_t$ back to $M_t$, in the event $\{\tau \le T\}$,
\begin{align}
\lim_{t\uparrow\tau}M_t
= \lim_{t\uparrow\tau}M^T_t
= M^T_\tau
= M^T_\infty,
\end{align}
which is finite a.s.
And since $T$ is arbitrary, $\lim_{t\uparrow\tau}M_t$ is finite a.s.~in the event $\{\tau<\infty\}$.

Finally, we derive a contradiction to show the impossibility of the event $\{\tau < \infty\}$.
In this event, since all three terms in \cref{eq:logy_rep_appx} are a.s.~finite,
$\lim_{t\uparrow\tau}\log y_t$ is finite, and therefore,
\begin{align}
\lim_{t\uparrow\tau}y_t = \exp\left(\lim_{t\uparrow\tau}\log y_t\right) > 0.
\end{align}
However, by continuity of $z_t$ and $d_{s^*}$, and since
$\tau_{s^*} = \tau$ means $z_\tau$ lies exactly on boundary $s^*$,
\begin{align}
\lim_{t\uparrow\tau}y_t = y_\tau = d_{s^*}(z_\tau) = 0.
\end{align}
These two conclusions are contradictory, so the event $\{\tau < \infty\}$ must have probability zero.
Therefore $\mathbb{P}(\tau<\infty)=0$, and so $\mathbb{P}\left(\exists t \in [0,\infty):z_t\in\partial K\right)=0$.

\end{proof}

\subsection{Proof that the sufficient drift satisfies conditions from \cref{thm:milian}} \label{apx:stationary-drift-satisfies-condition}

Since \cref{apx:not-absorbing-state} shows the solution to the SDE will never reach the boundary, we already know it is viable.
Still, for completeness, we show here that the drift for the stationary dynamics satisfies conditions from \cref{thm:milian}, allowing us to apply \cref{thm:milian} as well.

\paragraph{Proof that $h$ satisfies (i) from \cref{thm:milian}.}
Here, we prove that for each $T > 0$, there exists $C_T > 0$ such that for all $z_t \in K$ and $t \in [0, T]$, $\lVert h(z_t) \rVert^2 \leq C_T \cdot ( 1 + \lVert z_t \rVert^2)$. 
We do this as follows:
\begin{align}
    \lVert h(z_t) \rVert 
    &= \frac{1}{2} \cdot \left\lVert \mathrm{diag}\left( \nabla_{z_t} [g(z_t)^2] \right) + g(z_t)^2 \odot \nabla_{z_t} \log \tilde{p}(z_t) \right\rVert \\
    &\leq \frac{1}{2} \cdot \left\lVert \mathrm{diag}\left( \nabla_{z_t} [g(z_t)^2] \right) \right\rVert + \frac{1}{2} \left\lVert g(z_t)^2 \odot \nabla_{z_t} \log \tilde{p}(z_t) \right\rVert \\
    &\leq \frac{1}{2} \cdot \underbrace{\left\lVert \mathrm{diag}\left( \nabla_{z_t} [g(z_t)^2] \right) \right\rVert}_{\Circled{1}} + \frac{1}{2} \underbrace{\left\lVert g(z_t) \right\rVert}_{M_g} \cdot \underbrace{\left\lVert g(z_t) \odot \nabla_{z_t} \log \tilde{p}(z_t) \right\rVert}_{M_{gs}}
\end{align}
where $M_g = \sup_{z_t\in K}\|g(z_t)\|_2$ by continuity of $g(z_t)$ on a compact $K$.
Next, we bound the gradient of $g$ using condition (ii):
{\allowdisplaybreaks
\begin{align} 
    \Circled{1} 
    &= \left\lVert \mathrm{diag}\left( \nabla_{z_t} [g(z_t)^2] \right) \right\rVert \\
    &= \left\lVert \lim_{\epsilon \to 0} \begin{bmatrix}
        \frac{g^1(z_t + \epsilon \cdot e_1)^2 - g^1(z_t)^2}{\epsilon} \\
        \vdots \\
        \frac{g^{D_z}(z_t + \epsilon \cdot e_{D_z})^2 - g^{D_z}(z_t)^2}{\epsilon} \\
    \end{bmatrix} \right\rVert \\
    &= \lim_{\epsilon \to 0} \left\lVert \begin{bmatrix}
        \frac{g^1(z_t + \epsilon \cdot e_1)^2 - g^1(z_t)^2}{\epsilon} \\
        \vdots \\
        \frac{g^{D_z}(z_t + \epsilon \cdot e_{D_z})^2 - g^{D_z}(z_t)^2}{\epsilon} \\
    \end{bmatrix} \right\rVert \quad (\text{by continuity of the norm}) \\
    &\leq \sum\limits_{d=1}^{D_z} \lim_{\epsilon \to 0} \left\lVert \frac{g^d(z_t + \epsilon \cdot e_d)^2 - g^d(z_t)^2}{\epsilon} \right\rVert \quad (\text{since the $\ell_2$-norm is upper bounded by the $\ell_1$-norm}) \\
    &= \sum\limits_{d=1}^{D_z} \lim_{\epsilon \to 0} \frac{1}{\epsilon} \cdot \left\lVert g^d(z_t + \epsilon \cdot e_d)^2 - g^d(z_t)^2 \right\rVert \\
    &\leq \sum\limits_{d=1}^{D_z} \lim_{\epsilon \to 0} \frac{1}{\epsilon} \cdot \left\lVert g(z_t + \epsilon \cdot e_d)^2 - g(z_t)^2 \right\rVert \\
    &= \sum\limits_{d=1}^{D_z} \lim_{\epsilon \to 0} \frac{1}{\epsilon} \cdot \left\lVert \left( g(z_t + \epsilon \cdot e_d) - g(z_t) \right) \odot g(z_t + \epsilon \cdot e_d) - g(z_t) \odot \left( g(z_t) - g(z_t + \epsilon \cdot e_d) \right) \right\rVert \label{eq:trick1-mo2171798} \\
    &\leq \sum\limits_{d=1}^{D_z} \lim_{\epsilon \to 0} \frac{1}{\epsilon} \cdot \left\lVert \left( g(z_t + \epsilon \cdot e_d) - g(z_t) \right) \odot g(z_t + \epsilon \cdot e_d) \right\rVert + \frac{1}{\epsilon} \cdot \left\lVert g(z_t) \odot \left( g(z_t) - g(z_t + \epsilon \cdot e_d) \right) \right\rVert \\
    &\leq \sum\limits_{d=1}^{D_z} \lim_{\epsilon \to 0} \frac{1}{\epsilon} \cdot \left\lVert g(z_t + \epsilon \cdot e_d) - g(z_t) \right\rVert \cdot \left\lVert g(z_t + \epsilon \cdot e_d) \right\rVert + \frac{1}{\epsilon} \cdot \left\lVert g(z_t) \right\rVert \cdot \left\lVert g(z_t) - g(z_t + \epsilon \cdot e_d) \right\rVert \\
    &\leq \sum\limits_{d=1}^{D_z} \lim_{\epsilon \to 0} \frac{M_g}{\epsilon} \cdot \left\lVert g(z_t + \epsilon \cdot e_d) - g(z_t) \right\rVert + \frac{M_g}{\epsilon} \cdot \left\lVert g(z_t) - g(z_t + \epsilon \cdot e_d) \right\rVert \\
    &\leq \sum\limits_{d=1}^{D_z} \lim_{\epsilon \to 0} \frac{2 \cdot M_g}{\epsilon} \cdot C_T \cdot \left\lVert \epsilon \cdot e_d \right\rVert \\
    &= 2 \cdot D_z \cdot M_g \cdot C_T \cdot \left\lVert e_d \right\rVert \\
    &= 2 \cdot D_z \cdot M_g \cdot C_T
\end{align}
}
where, in \cref{eq:trick1-mo2171798}, we use the trick from \citet{mo2171798}.
Putting all of this together, we have that $\lVert h(z_t) \rVert$ is bounded by some constant, for which we can always find a new constant $C_T$ to further bound it: $\lVert h(z_t) \rVert^2 \leq C_T \cdot ( 1 + \lVert z_t \rVert^2)$.

\paragraph{Proof that $h$ satisfies (ii) from \cref{thm:milian}.}
Here, we prove that for all $T > 0$, $z_t, z_t' \in K$, and $t \in [0, T]$, $\lVert h(z_t) - h(z_t') \rVert \leq C_T \cdot \lVert z_t - z_t' \rVert$.

We begin as follows:
\begin{align}
\begin{split}
    \lVert h(z_t) - h(z_t') \rVert 
    &\leq \frac{1}{2} \cdot \underbrace{\left\lVert \mathrm{diag}\left( \nabla_{z_t} [g(z_t)^2] \right) - \mathrm{diag}\left( \nabla_{z_t'} [g(z_t')^2] \right) \right\rVert}_{\Circled{2}} \\ 
    &\qquad + \frac{1}{2} \cdot \underbrace{\left\lVert g(z_t)^2 \odot \nabla_{z_t} \log \tilde{p}(z_t) - g(z_t')^2 \odot \nabla_{z_t'} \log \tilde{p}(z_t') \right\rVert}_{\Circled{3}}.
\end{split}
\end{align}
Using the trick by~\citet{mo2171798} again, we bound $\Circled{2}$ as follows:
\begin{align}
    \Circled{2} &= \left\lVert \mathrm{diag}\left( \nabla_{z_t'} [g(z_t')^2] \right) - \mathrm{diag}\left( \nabla_{z_t} [g(z_t)^2] \right) \right\rVert \\
    &= 2 \cdot \left\lVert g(z_t') \odot \mathrm{diag}\left( \nabla_{z_t'} g(z_t') \right) - g(z_t) \odot \mathrm{diag}\left( \nabla_{z_t} g(z_t) \right) \right\rVert \\
    &= 2 \cdot \left\lVert \left( g(z_t') - g(z_t) \right) \odot \mathrm{diag}\left( \nabla_{z_t'} g(z_t') \right) - g(z_t) \odot \left( \mathrm{diag}\left( \nabla_{z_t} g(z_t) \right) - \mathrm{diag}\left( \nabla_{z_t'} g(z_t') \right) \right) \right\rVert \\ 
    &\leq 2 \cdot \left\lVert \left( g(z_t') - g(z_t) \right) \odot \mathrm{diag}\left( \nabla_{z_t'} g(z_t') \right) \right\rVert + 2 \cdot \left\lVert g(z_t) \odot \left( \mathrm{diag}\left( \nabla_{z_t} g(z_t) \right) - \mathrm{diag}\left( \nabla_{z_t'} g(z_t') \right) \right) \right\rVert \\
    &\leq 2 \cdot \left\lVert g(z_t') - g(z_t) \right\rVert \cdot \left\lVert \mathrm{diag}\left( \nabla_{z_t'} g(z_t') \right) \right\rVert + 2 \cdot \left\lVert g(z_t) \right\rVert \cdot \left\lVert \mathrm{diag}\left( \nabla_{z_t} g(z_t) \right) - \mathrm{diag}\left( \nabla_{z_t'} g(z_t') \right) \right\rVert \\
    &\leq 2 \cdot C_T \cdot \lVert z_t - z_t' \rVert \cdot \underbrace{\left\lVert \mathrm{diag}\left( \nabla_{z_t'} g(z_t') \right) \right\rVert}_{\text{bounded by const.~via (ii)}} +2 \cdot M_g \cdot \underbrace{\left\lVert \mathrm{diag}\left( \nabla_{z_t} g(z_t) \right) - \mathrm{diag}\left( \nabla_{z_t'} g(z_t') \right) \right\rVert}_{\leq C_T \cdot \lVert z_t - z_t' \rVert \text{ via (ii)}}
\end{align}
We similarly bound $\Circled{3}$ as follows:
\begin{align}
    \Circled{3}
    &= \left\lVert g(z_t)^2 \odot \nabla_{z_t} \log \tilde{p}(z_t) 
        - g(z_t')^2 \odot \nabla_{z_t'} \log \tilde{p}(z_t') \right\rVert \\
    &= \left\lVert g(z_t) \odot \bigl(g(z_t) \odot \nabla_{z_t} \log \tilde{p}(z_t)\bigr) 
        - g(z_t') \odot \bigl(g(z_t') \odot \nabla_{z_t'} \log \tilde{p}(z_t')\bigr) \right\rVert.
\end{align}
Using the trick by~\citet{mo2171798},
\begin{align}
    \Circled{3}
    &\leq \underbrace{\left\lVert g(z_t) - g(z_t') \right\rVert}_{\leq  C_T \cdot \lVert z_t - z_t'\rVert\;\text{via (ii)}}
        \cdot \underbrace{\left\lVert g(z_t) \odot \nabla_{z_t} \log \tilde{p}(z_t) \right\rVert}_{\leq M_{gs}} \\
    &\quad + \underbrace{\left\lVert g(z_t') \right\rVert}_{\leq M_g} \cdot \underbrace{\left\lVert g(z_t) \odot \nabla_{z_t} \log \tilde{p}(z_t) 
            - g(z_t') \odot \nabla_{z_t'} \log \tilde{p}(z_t') \right\rVert}
            _{\leq  C_T \cdot \lVert z_t - z_t'\rVert\;\text{via (iii)}} \\
    &\leq \left( M_{gs} + M_g \right) \cdot C_T \cdot \lVert z_t - z_t' \rVert,
\end{align}

\paragraph{Proof that $h$ satisfies (iii) from \cref{thm:milian}.}
Here, we prove that for each $z_t \in K$, $h(z_t)$, defined for $t \geq 0$, is continuous. 

Since continuous functions are closed under all operations used to define $h(z_t)$, and since $h(z_t)$ is defined in terms of other continuous functions, it is also continuous.

\paragraph{Proof that $h$ satisfies (a) from \cref{thm:milian}.}
Here, we prove that, for all $s \in \{ 1, \dots, S \}$ and $z_t \in K$ such that $\langle z_t - u_s, v_s \rangle = 0$, we have $\langle h(z_t), v_s \rangle \geq 0$. 
\begin{align}
\langle h(z_t), v_s \rangle
    &= \langle g(z_t) \odot \mathrm{diag}\left( \nabla_{z_t} g(z_t) \right), v_s \rangle + \frac{1}{2} \cdot \langle g(z_t)^2 \odot \nabla_{z_t} \log \tilde{p}(z_t), v_s \rangle \\
    &= v_s^\intercal \cdot \left( g(z_t) \odot \mathrm{diag}\left( \nabla_{z_t} g(z_t) \right) \right) + \frac{1}{2} \cdot v_s^\intercal \cdot \left( g(z_t)^2 \odot \nabla_{z_t} \log \tilde{p}(z_t) \right) \\
    &= v_s^\intercal \cdot \left( \sum\limits_{d=1}^{D_z} (g(z_t) \odot e_d) \cdot \nabla_{z_t^d} g^d(z_t) \right) + \frac{1}{2} \cdot v_s^\intercal \cdot \left( g(z_t)^2 \odot \nabla_{z_t} \log \tilde{p}(z_t) \right) \\
    &= \left( \sum\limits_{d=1}^{D_z} v_s^\intercal \cdot (g(z_t) \odot e_d) \cdot \nabla_{z_t^d} g^d(z_t) \right) + \frac{1}{2} \cdot v_s^\intercal \cdot \left( g(z_t)^2 \odot \nabla_{z_t} \log \tilde{p}(z_t) \right) \\
    &= \frac{1}{2} \cdot v_s^\intercal \cdot \left( g(z_t)^2 \odot \nabla_{z_t} \log \tilde{p}(z_t) \right) \quad (\text{since $\langle g(z_t) \odot e_d, v_s \rangle = 0$}) \\
    &= \frac{1}{2} \cdot v_s^\intercal \cdot \left( \sum\limits_{d=1}^{D_z} (g(z_t) \odot e_d) \cdot g^d(z_t) \cdot \nabla_{z_t^d} \log \tilde{p}(z_t) \right) \\
    &= \frac{1}{2} \cdot \left( \sum\limits_{d=1}^{D_z} v_s^\intercal \cdot (g(z_t) \odot e_d) \cdot g^d(z_t) \cdot \nabla_{z_t^d} \log \tilde{p}(z_t) \right) \\
    &= 0
\end{align}

\section{Proof of \cref{thm:wsp}} \label{apx:proof-wsp}

\subsection{Proof that WSP satisfies \cref{thm:milian}} \label{apx:wsp-satisfies-milian}

\begin{proof}

To prove \cref{thm:wsp}, we will show that $h(t, z_t)$ and $g(t, z_t)$, defined in \cref{eq:wsp}, satisfy (i)-(iii) and (a)-(b) in \cref{thm:milian}. 

\paragraph{Proof that WSP satisfies (i) from \cref{thm:milian}.}

Here, we prove that for each $T > 0$, there exists $C_T > 0$ such that for all $z_t \in K$ and $t \in [0, T]$, $\lVert h(t, z_t) \rVert^2 + \lVert g(t, z_t) \rVert^2 \leq C_T \cdot ( 1 + \lVert z_t \rVert^2)$.

First, we show that $w(z_t)$, defined in \cref{eq:w}, lies in $[0, 1]$. 
Since $\alpha > 0$, and for any $z_t \in K$, $d_s(z_t) \geq 0$ (since distances are non-negative), we have,
\begin{align}
    0 \leq\tanh\left( \alpha \cdot d_s(z_t) \right) \leq 1.
\end{align}
Next, since $\beta > 0$, we know that:
\begin{align}
    0 \leq \beta \cdot \prod_s \underbrace{\frac{e^{-d_s(z_t)}}{\sum_{s'} e^{-d_{s'}(z_t)}}}_{\in [0, 1]} \cdot \underbrace{\tanh\left( \alpha \cdot d_s(z_t) \right)}_{\in [0, 1]}
\end{align}
This then gives us,
\begin{align}
     0 \leq \tanh \underbrace{\left( \beta \cdot \prod_s \frac{e^{-d_s(z_t)}}{\sum_{s'} e^{-d_{s'}(z_t)}} \cdot \tanh\left( \alpha \cdot d_s(z_t) \right) \right)}_{\geq 0}  \leq 1,
\end{align}
thereby showing that $w(z_t) \in [0, 1]$.
Using this, we go on to show that $h$ and $g$ satisfy condition (i) from \cref{thm:milian}.
\begin{align}
    \lVert h(t, z_t) \rVert
    &= \lVert \mathrm{WSP}(\tilde{h}, c_h, t, z_t) \rVert\\
    &= \lVert w(z_t) \cdot \tilde{h}(t, z_t) + (1 - w(z_t)) \cdot c_h(z_t) \rVert\\
    &\leq \lVert w(z_t) \cdot \tilde{h}(t, z_t)\rVert + \lVert (1 - w(z_t)) \cdot c_h(z_t) \rVert\\
    &\leq \lVert 1 \cdot \tilde{h}(t, z_t)\rVert + \lVert (1 - 0) \cdot c_h(z_t) \rVert\\
    &= \lVert \tilde{h}(t, z_t)\rVert + \lVert c_h(z_t) \rVert\\
    &= \lVert \tilde{h}(t, z_t)\rVert + \left\lVert \gamma \cdot \frac{z^* - z_t}{\lVert z^* - z_t \rVert + \epsilon} \right\rVert \\
    &= \lVert \tilde{h}(t, z_t)\rVert + \gamma \cdot \underbrace{\left\lVert \frac{z^* - z_t}{\lVert z^* - z_t \rVert + \epsilon} \right\rVert}_{< 1} \\
    &\leq \lVert \tilde{h}(t, z_t)\rVert + \gamma \\
    &\leq \sqrt{C_T' \cdot (1 + \lVert z_t \rVert^2)} + \gamma
\end{align}
Thus,
\begin{align}
    \lVert h(t, z_t) \rVert^2 
    &\leq \left( \sqrt{C_T' \cdot (1 + \lVert z_t \rVert^2)} + \gamma \right)^2 \\
    &= C_T' \cdot (1 + \lVert z_t \rVert^2) + \gamma^2 + 2 \cdot \gamma \cdot \sqrt{C_T' \cdot (1 + \lVert z_t \rVert^2)} \\
    &\leq  C_T \cdot (1 + \lVert z_t \rVert^2)
\end{align}
for some $C_T > 0$.

Similarly for $\lVert g(t, z_t) \rVert^2$, we have
\begin{align}
    \lVert g(t, z_t) \rVert
    &= \lVert \mathrm{WSP}(\tilde{g}, c_g, t, z_t) \rVert\\ 
    &=\lVert w(z_t) \cdot \tilde{g}(t, z_t) + (1 - w(z_t)) \cdot c_g(z_t) \rVert\\
    &\leq \lVert w(z_t) \cdot \tilde{g}(t, z_t)\rVert + \lVert (1 - w(z_t)) \cdot c_g(z_t) \rVert\\
    &\leq \lVert 1 \cdot \tilde{g}(t, z_t)\rVert + \lVert (1 - 0) \cdot c_g(z_t) \rVert\\
    &= \lVert \tilde{g}(t, z_t)\rVert + \lVert c_g(z_t) \rVert\\
    &= \lVert \tilde{g}(t, z_t)\rVert + \lVert 0 \rVert\\
    &= \lVert \tilde{g}(t, z_t)\rVert\\
    &\leq \sqrt{C_T \cdot (1 + \lVert z_t \rVert^2)}
\end{align}
Thus, $\lVert g(t, z_t) \rVert^2 \leq C_T \cdot (1 + \lVert z_t \rVert^2)$.

\paragraph{Proof that WSP satisfies (ii) from \cref{thm:milian}.}
We now prove that for all $T > 0$, $z_t, z_t' \in K$, and $t \in [0, T]$, $\lVert h(t, z_t) - h(t, z_t') \rVert + \lVert g(t, z_t) - g(t, z_t') \rVert \leq C_T \cdot \lVert z_t - z_t' \rVert$.

We do this as follows:
\begin{align}
    \lVert h(t, z_t) - h(t, z_t') \rVert
    &= \lVert \mathrm{WSP}(\tilde{h}, c_h, t, z_t) - \mathrm{WSP}(\tilde{h}, c_h, t, z_t') \rVert \\
    &= \left\lVert \left( w(z_t) \cdot \tilde{h}(t, z_t) + (1 - w(z_t)) \cdot c_h(z_t) \right) - \left(w(z_t') \cdot \tilde{h}(t, z_t') + (1 - w(z_t')) \cdot c_h(z_t') \right) \right\rVert \\
    &= \left\lVert \left(w(z_t) \cdot \tilde{h}(t, z_t)- w(z_t') \cdot \tilde{h}(t, z_t') \right) + \left((1 - w(z_t)) \cdot c_h(z_t)-(1 - w(z_t')) \cdot c_h(z_t') \right) \right\rVert \\
    &\leq \lVert w(z_t) \cdot \tilde{h}(t, z_t)- w(z_t') \cdot \tilde{h}(t, z_t') \rVert + \lVert(1 - w(z_t)) \cdot c_h(z_t)-(1 - w(z_t')) \cdot c_h(z_t') \rVert
\end{align}
Using the trick by~\citet{mo2171798}, we have:
\begin{align}
    \lVert h(t, z_t) - h(t, z_t') \rVert 
    &\leq \lVert w(z_t) - w(z_t')\rVert \cdot \lVert\tilde{h}(t, z_t) \rVert + \lVert w(z_t') \rVert \cdot \lVert \tilde{h}(t, z_t) - \tilde{h}(t, z_t')\rVert \\ \nonumber
    &\quad + \lVert (1 - w(z_t)) - (1 - w(z_t')) \rVert \cdot \lVert c(z_t) \rVert + \lVert 1 - w(z_t') \rVert \cdot \lVert c(z_t) - c(z_t') \rVert \\
    &\leq \lVert w(z_t) - w(z_t')\rVert \cdot \lVert\tilde{h}(t, z_t) \rVert + 1 \cdot \lVert \tilde{h}(t, z_t) - \tilde{h}(t, z_t')\rVert \\ \nonumber
    &\quad + \lVert w(z_t) - w(z_t') \rVert \cdot \lVert c(z_t) \rVert + 1 \cdot \lVert c(z_t) - c(z_t') \rVert \\
    &= \lVert w(z_t) - w(z_t')\rVert \cdot \underbrace{( \lVert \tilde{h}(t, z_t) \rVert + \lVert c(z_t) \rVert )}_{\text{bounded by const.}} + \underbrace{\lVert \tilde{h}(t, z_t) - \tilde{h}(t, z_t')\rVert}_{< C_T' \cdot \lVert z_t - z_t' \rVert} + \lVert c(z_t) - c(z_t') \rVert
\end{align}
Since $K$ is compact and $\tilde{h}(z_t)$ is continuous and linearly bounded (i.e. $\lVert \tilde{h}(z_t) \rVert^2 \leq C_T ( 1 + \lVert z_t \rVert^2)$), we know that $\lVert \tilde{h}(t, z_t) \rVert$ is bounded above by a constant.
Similarly, $c(z_t)$ is continuous and bounded:
\begin{align}
    \lVert c(z_t) \rVert^2 
    &= \gamma^2 \cdot \left\lVert \frac{z^* - z_t}{\lVert z^* - z_t \rVert + \epsilon} \right\rVert^2
    \leq \gamma^2 \cdot \left\lVert \frac{z^* - z_t}{\epsilon} \right\rVert^2 
    = \frac{\gamma^2}{\epsilon^2} \cdot \lVert z^* - z_t \rVert^2,
\end{align}
so $\lVert c(z_t) \rVert$ is bounded above by a constant.
This leaves us to show that $w(z_t)$ and $c(z_t)$ are Lipschitz.
This is true since both functions are comprised of either composition of Lipschitz functions, or of multiplications of bounded Lipschitz functions, and both of these operations are closed under Lipschitz continuity.

Similarly for $\lVert g(t, z_t) - g(t, z_t') \rVert$, we have,
\begin{align}
    \lVert g(t, z_t) - g(t, z_t') \rVert 
    &= \lVert \mathrm{WSP}(\tilde{g}, c_g, t, z_t) - \mathrm{WSP}(\tilde{g}, c_g, t, z_t') \rVert\\
    &= \lVert \left(w(z_t) \cdot \tilde{g}(t, z_t) + (1 - w(z_t)) \cdot c_g(z_t) \right) - \left(w(z_t') \cdot \tilde{g}(t, z_t') + (1 - w(z_t')) \cdot c_g(z_t') \right) \rVert \\
    &= \lVert w(z_t) \cdot \tilde{g}(t, z_t) - w(z_t') \cdot \tilde{g}(t, z_t') \rVert
\end{align}
Since $w(z_t) \cdot \tilde{g}(t, z_t)$ is the product of a bounded Lipschitz function and a Lipschitz function, we know that $g(t, z_t)$ is also Lipschitz.

\paragraph{Proof that WSP satisfies (iii) from \cref{thm:milian}.}

Here, we prove that for each $z_t \in K$, $h(t, z_t)$ and $g(t, z_t)$ are continuous. 

Since all functions involved are continuous and continuity is closed under addition, subtraction, multiplication and composition, $h(t, z_t)$ and $g(t, z_t)$, defined for $t \geq 0$, are continuous for each $z_t \in K$.

\paragraph{Proof that WSP satisfies (a) from \cref{thm:milian}.}

Here we prove that for all $s \in \{1, \dots, S\}$ and $z_t \in K$ such that $\langle z_t - u_s, v_s \rangle = 0$, we have $\langle h(t, z_t), v_s \rangle \geq 0$.

First, when $\langle z_t - u_s, v_s \rangle = 0$, 
\begin{align}
    d_s(z_t) = \frac{\langle z_t - u_s, v_s \rangle}{\lVert v_s \rVert} = \frac{0}{\lVert v_s \rVert} = 0.
\end{align}
This means that,
\begin{align}
    w(z_t) &= \tanh\left( \beta \cdot \prod_s \frac{e^{-d_s(z_t)}}{\sum_{s'} e^{-d_{s'}(z_t)}} \cdot \tanh\left( \alpha \cdot d_s(z_t) \right) \right)= 0.
\end{align}
Plugging this into $\langle h(t, z_t), v_s \rangle$, we get:
\begin{align}
    \langle h(t, z_t), v_s \rangle 
    &= \langle \mathrm{WSP}(\tilde{h},c_h, t, z_t), v_s \rangle \\
    &=\langle w(z_t) \cdot \tilde{h}(t, z_t) + (1 - w(z_t)) \cdot c_h(z_t) , v_s \rangle \\
    &= \langle 0 \cdot \tilde{h}(t, z_t) + (1 - 0) \cdot c_h(z_t), v_s \rangle \\
    &= \langle  c_h(z_t) , v_s \rangle \\
    &= \left\langle \gamma \cdot \frac{z^* - z_t}{\lVert z^* - z_t \rVert + \epsilon}, v_s \right\rangle \\
    &= \underbrace{\frac{\gamma}{\lVert z^* - z_t \rVert + \epsilon}}_{> 0} \cdot \left\langle z^* - z_t, v_s \right\rangle
\end{align}
Because $z^*$ is the Chebyshev center, it lies strictly in $\operatorname{int}(K)$, so $\langle z^* - u_s, v_s\rangle > 0$.
Thus, when $z_t \in \partial K$, $\langle z^* - z_t, v_s\rangle = \langle z^* - u_s, v_s\rangle - \langle z_t - u_s, v_s\rangle = \langle z^* - u_s, v_s\rangle > 0$, completing the proof.

\paragraph{Proof that WSP satisfies (b) from \cref{thm:milian}.}
Here we prove that for all $s \in \{1, \dots, S\}$ and $z_t \in K$ such that $\langle z_t - u_s, v_s \rangle = 0$, we have $\langle g(t, z_t) \odot e_d, v_s \rangle = 0$ for $t \geq 0$ and $d \in \{1, \dots, D_z\}$.
We do this as follows:
\begin{align}
    \langle g(t, z_t) \odot e_d, v_s \rangle &= \langle \mathrm{WSP}(\tilde{g}, c_g, t, z_t) \odot e_d, v_s \rangle\\ 
    &= \langle \left( w(z_t) \cdot \tilde{g}(t, z_t) + (1 - w(z_t)) \cdot c_g(z_t) \right) \odot e_d, v_s \rangle\\
    &= \langle \left( 0 \cdot \tilde{g}(t, z_t) + (1 - 0) \cdot c_g(z_t) \right) \odot e_d, v_s \rangle\\
    &= \langle c_g(z_t) \odot e_d, v_s \rangle\\
    &= \langle 0 \odot e_d, v_s \rangle\\
    &= \langle 0, v_s \rangle\\
    &= 0
\end{align}

\end{proof}

\subsection{Proof that WSP satisfies \cref{thm:stationary}}

\begin{proof}
Since $c_g = 0$, we have $g(t, z_t) = w(z_t) \cdot \tilde{g}(t, z_t)$ throughout.
Note that condition (iii) concerns the score function $\tilde{p}$ and is
independent of $g$; it must be verified separately by the practitioner.

\paragraph{Proof that $g$ satisfies (i) from \cref{thm:stationary}.}
This follows immediately from the proof of (i) from \cref{thm:milian} in \cref{apx:wsp-satisfies-milian},
which already establishes $\lVert g(t, z_t) \rVert^2 \leq C_T \cdot (1 + \lVert z_t \rVert^2)$.

\paragraph{Proof that $g$ satisfies (ii) from \cref{thm:stationary}.}
The Lipschitz condition $\lVert g(t, z_t) - g(t, z_t') \rVert \leq C_T \cdot \lVert z_t - z_t' \rVert$
follows from the proof of (ii) from \cref{thm:milian} in \cref{apx:wsp-satisfies-milian}.
Now, we prove the Lipschitz condition on $\mathrm{diag}(\nabla_{z_t} g(t, z_t))$.
Since $g(t, z_t) = w(z_t) \cdot \tilde{g}(t, z_t)$, the product rule gives:
\begin{align}
    \mathrm{diag}(\nabla_{z_t} g(t, z_t))
    = \nabla_{z_t} w(z_t) \odot \tilde{g}(t, z_t)
    + w(z_t) \cdot \mathrm{diag}(\nabla_{z_t} \tilde{g}(t, z_t)).
\end{align}
Therefore,
\begin{align}
    &\lVert \mathrm{diag}(\nabla_{z_t} g(t, z_t)) - \mathrm{diag}(\nabla_{z_t'} g(t, z_t')) \rVert \\
    &\leq \underbrace{\lVert \nabla_{z_t} w(z_t) \odot \tilde{g}(t, z_t)
        - \nabla_{z_t'} w(z_t') \odot \tilde{g}(t, z_t') \rVert}_{\Circled{A}}
    + \underbrace{\lVert w(z_t) \cdot \mathrm{diag}(\nabla_{z_t} \tilde{g}(t, z_t))
        - w(z_t') \cdot \mathrm{diag}(\nabla_{z_t'} \tilde{g}(t, z_t')) \rVert}_{\Circled{B}}.
\end{align}
Using the trick by~\citet{mo2171798}, we bound $\Circled{A}$ as follows:
\begin{align}
    \Circled{A}
    &\leq \underbrace{\lVert \nabla_{z_t} w(z_t) - \nabla_{z_t'} w(z_t') \rVert}
        _{\leq C_T \cdot \lVert z_t - z_t' \rVert}
      \cdot \underbrace{\lVert \tilde{g}(t, z_t) \rVert}_{\leq M}
    + \underbrace{\lVert \nabla_{z_t'} w(z_t') \rVert}_{\text{bounded by const.}}
      \cdot \underbrace{\lVert \tilde{g}(t, z_t) - \tilde{g}(t, z_t') \rVert}
        _{\leq C_T \cdot \lVert z_t - z_t' \rVert \text{ via (ii)}},
\end{align}
where $\nabla_{z_t} w(z_t)$ is Lipschitz since $w$ is a smooth composition of smooth
functions (tanh, softmin, product), and $\lVert \nabla_{z_t} w(z_t) \rVert$ is bounded
on the compact set $K$.
Similarly, we bound $\Circled{B}$ as follows:
\begin{align}
    \Circled{B}
    &\leq \underbrace{\lVert w(z_t) - w(z_t') \rVert}_{\leq C_T \cdot \lVert z_t - z_t' \rVert}
      \cdot \underbrace{\lVert \mathrm{diag}(\nabla_{z_t} \tilde{g}(t, z_t)) \rVert}
        _{\text{bounded by const.~via (ii)}}
    + \underbrace{\lVert w(z_t') \rVert}_{\leq 1}
      \cdot \underbrace{\lVert \mathrm{diag}(\nabla_{z_t} \tilde{g}(t, z_t))
        - \mathrm{diag}(\nabla_{z_t'} \tilde{g}(t, z_t')) \rVert}
        _{\leq C_T \cdot \lVert z_t - z_t' \rVert \text{ via (ii)}}.
\end{align}
Altogether, $\lVert \mathrm{diag}(\nabla_{z_t} g(t, z_t))
- \mathrm{diag}(\nabla_{z_t'} g(t, z_t')) \rVert \leq C_T \cdot \lVert z_t - z_t' \rVert$
for some $C_T > 0$.

\paragraph{Proof that $g$ satisfies (iv) from \cref{thm:stationary}.}
We must show $g^d(t, z_t) > 0$ for all $d \in \{1, \dots, D_z\}$ and
$z_t \in \mathrm{int}(K)$.
Since $g(t, z_t) = w(z_t) \cdot \tilde{g}(t, z_t)$, it suffices to show
$w(z_t) > 0$ and $\tilde{g}^d(t, z_t) > 0$ separately.
The latter holds by condition (iv) of \cref{thm:stationary} for $\tilde{g}$.
For the former, since $z_t \in \mathrm{int}(K)$, we have $d_s(z_t) > 0$ for all
$s \in \{1, \dots, S\}$, and therefore:
\begin{align}
    \frac{e^{-d_s(z_t)}}{\sum_{s'} e^{-d_{s'}(z_t)}} > 0 \qquad \text{and} \qquad
    \tanh\left( \alpha \cdot d_s(z_t) \right) > 0,
\end{align}
so the product $\prod_s \frac{e^{-d_s(z_t)}}{\sum_{s'} e^{-d_{s'}(z_t)}} \cdot
\tanh(\alpha \cdot d_s(z_t)) > 0$, and hence
$w(z_t) = \tanh(\beta \cdot \mathrm{positive}) > 0$.

\paragraph{Proof that $g$ satisfies (v) from \cref{thm:stationary}.}
We must show that for each $s \in \{1, \dots, S\}$, there exists $L_s < \infty$
such that $\lVert n_s \odot g(z_t) \rVert_2 \leq L_s \cdot d_s(z_t)$ for all
$z_t \in K$.
Let $M_{\tilde{g}} = \sup_{z_t \in K} \lVert \tilde{g}(t, z_t) \rVert_2$, which is finite
since $K$ is compact and $\tilde{g}$ is continuous.
We first bound $w(z_t)$ linearly in $d_s(z_t)$:
\begin{align}
    w(z_t) 
    &= \tanh\left( \beta \cdot \prod_{s} \frac{e^{-d_{s}(z_t)}}{\sum_{s'} e^{-d_{s'}(z_t)}} \cdot \tanh\left( \alpha \cdot d_{s}(z_t) \right) \right) \\
    &\leq \beta \cdot \prod_{s} \frac{e^{-d_{s}(z_t)}}{\sum_{s'} e^{-d_{s'}(z_t)}} \cdot \tanh\left( \alpha \cdot d_{s}(z_t) \right) 
    \qquad \text{since } \tanh(x) \leq x \text{ for } x \geq 0 \\
\end{align}
Since all factors in the product above lie in $[0, 1]$, the product can be bounded by any single factor:
\begin{align}    
    w(z_t) 
    &\leq \beta \cdot \underbrace{\frac{e^{-d_s(z_t)}}{\sum_{s'} e^{-d_{s'}(z_t)}}}_{\leq 1} \cdot \tanh\left( \alpha \cdot d_s(z_t) \right) \\
    &\leq \beta \cdot \tanh\left( \alpha \cdot d_s(z_t) \right) \\
    &\leq \alpha\beta \cdot d_s(z_t)
    \qquad \text{since } \tanh(x) \leq x \text{ for } x \geq 0.
\end{align}
Therefore,
\begin{align}
    \lVert n_s \odot g(z_t) \rVert_2
    \leq \lVert g(z_t) \rVert_2
    = w(z_t) \cdot \lVert \tilde{g}(t, z_t) \rVert_2
    \leq \alpha\beta \cdot d_s(z_t) \cdot M_{\tilde{g}}
    = L_s \cdot d_s(z_t),
\end{align}
where $L_s = \alpha\beta M_{\tilde{g}} < \infty$.

\paragraph{Proof that $g$ satisfies (b) from \cref{thm:stationary}.}
This is identical to the proof of (b) from \cref{thm:milian} in \cref{apx:wsp-satisfies-milian}, since
the condition is the same.

\end{proof}

\section{Discussion of Assumptions} \label{apx:assumptions}

The assumptions in \cref{thm:milian,thm:stationary} are easily satisfied when $h$, $g$, and $\log \tilde{p}(z_t)$ are parameterized by NNs.

\paragraph{Lipschitz Continuity with Respect to Inputs.}
Lipschitz continuous functions are closed under composition, making a large class of NNs Lipschitz continuous by construction.
Additionally, there exist many easy and empirically effective methods for explicitly obtaining Lipschitz continuity, for example via weight normalization (e.g.~\citet{miyato2018spectral}), regularization (e.g.~\citet{liu2022learning}), and architecture design (e.g.~\citet{anil2019sorting}).
Altogether, this allows us to conveniently satisfy the Lipschitz continuity assumptions for $h$, $g$, and $\log \tilde{p}(z_t)$---(ii) in \cref{thm:milian}, and (ii) and (iii) in \cref{thm:stationary}.

Next, \citet{hurault2022gradient} (Proposition 2) proved that if every function in a collection is differentiable with uniformly bounded, Lipschitz gradients with respect to the inputs, the gradients of their composition are themselves Lipschitz with respect to the inputs.
This makes a large class of NNs satisfy the second half of (ii) from \cref{thm:stationary}.
This property is known as ``Lipschitz smoothness,'' and can also be encouraged explicitly, for example via mixup regularization~\citep{gyawali2020enhancing}.

\paragraph{Linearly Bounded NNs.} 
We parameterize all NNs here with a composition of continuous functions, thereby making them continuous.
And since continuous functions on compact spaces are bounded, we easily satisfy (i) from \cref{thm:milian,thm:stationary}.

\paragraph{Differentiability and Continuity of Partials of NNs.} 
All NNs here use continuously differentiable activation functions, so they are continuously differentiable with continuous partials.

\section{Latent SDEs with Pathwise Expansions for Modeling Suicide Risk} \label{apx:latent-sde}

\paragraph{Notation.} 
We observe each patient $n \in [1, \dots, N]$ at times $t_1, \dots, t_M$. 
Note that the observation times and the number of observations differ by patient, but for notational simplicity, we will denote them as if each patient has observations at the same times.
Let $x_t^{n,d} \in \{0, \dots, 10\}$ denote patient $n$'s response to Likert-scale question $d$ at time $t$.
Similarly, denote $z_t^{n,d}$ as the $d$th component of patient $n$'s latent psychological state at time $t$.
Let $x_t^n \in \{ 0, \dots, 10 \}^{D_x}$ denote patient $n$'s response to all $D_x$ survey questions, and let $z_t^n \in [0, 1]^{D_z}$ denote the patient's $D_z$-dimensional latent space, confined to the unit cube.

\paragraph{Generative Process.}
We assume patient $n$'s data is generated via:
\begin{align}
z_{t_0}^n &\sim \mathcal{N}_{[0, 1]}(\mu_0, \sigma^2_0 \cdot I_{D_z}), \qquad (\text{initial state, drawn from a $[0, 1]$-truncated normal}) \\
z_{t_m}^n | z_{t_{m-1}}^n &\sim p(\cdot | z_{t_{m-1}}^n) = z_{t_{m-1}}^n + \int_{t_{m-1}}^{t_m} \underbrace{h(z_t; \theta)}_{\text{prior drift}} \cdot dt + \underbrace{g(z_t; \theta)}_{\text{ prior diffusion}} \circ dB_t, \label{eq:prior-z} \\
x_{t_m}^{n,d} | z_{t_m}^{n, d} &\sim \mathrm{Cat}(\lambda(z_{t_m}^{n,d}; \sigma_\epsilon)), \qquad (\text{ordinal likelihood of survey data given latent state})
\end{align}
wherein the dynamics, $h: K \to \mathbb{R}^{D_z}$ and $g: K \to \mathbb{R}_{\geq 0}^{D_z}$, are time-independent, and where,
\begin{align*}
\lambda(z_t^d; \sigma_\epsilon) &= \Phi\left( \frac{ b^{1:12} - z_t^d}{\sigma_\epsilon} \right) - \Phi\left( \frac{b^{0:11} - z_t^d}{\sigma_\epsilon} \right), \quad
b = \begin{bmatrix}
-\infty & 0.05 & 0.15 & 0.25 & \dots & 0.95 & \infty
\end{bmatrix},
\end{align*}
in which $\Phi(\cdot)$ is the CDF of a standard normal and $b$ represents the boundaries for 0-10 range ordinal likelihood.

\paragraph{Latent Dimensionality.}
Note that in these experiments, we set the latent dimensionality equal to the data dimensionality, $D_z = D_x$. 
In many applications, one typically chooses $D_z < D_x$, treating the latent space as a compressed representation of the observations. 
However, when modeling psychological state, there is often no single low-dimensional ``underlying condition'' that generates the observed symptoms; instead, mental health challenges are characterized by the way in which the symptoms themselves interact, regulate, or reinforce one another---this is known as the ``networks approach'' in clinical psychology~\citep{borsboom2022systems,robinaugh2020network}.
For this reason, each latent dimension directly corresponds to an observed dimension.

\paragraph{Eliminating SDE Solvers with a Pathwise Series Expansion.}
Fitting the Latent SDE above requires backpropagation through a slow, numerically unstable SDE solver.
In contrast, ODE solvers are known to be more numerically stable, accurate, and well-behaved when used with adaptive step-sizes (e.g.~\citet{lou2023reflected}).
Thus, we replace Brownian motion with the Karhunen-Lo\`eve expansion~\citep{sarkka2019applied}:
\begin{align}
d\widehat{B}_t | \xi &= \sum\limits_{r=1}^R \sqrt{\frac{2}{T}} \cos\left( \frac{(2 \cdot r - 1) \cdot \pi \cdot t}{2T} \right) \cdot \xi^r \cdot dt, \qquad \xi \sim \mathcal{N}(0, I_R),  \label{eq:kl-expansion}
\end{align} 
where $T$ is the end-time of the process (in this work, it's the end-time of the EMA study) and $I_R$ is an $R$-dimensional identity matrix.
This expansion consists of a sum of $R$ randomly weighted ODEs that, as $R \to \infty$, converge to $dB_t$, leading the overall equation to converge to the Stratonovich SDE~\citep{wong1965convergence}.
Using a finite $R$, we obtain an approximation of the above model that we can fit using ODE solvers.
Following \citet{ghosh2022differentiable}, we replace \cref{eq:prior-z} above with:
\begin{align}
    z_{t_m}^n | z_{t_{m-1}}^n, \xi^n &= z_{t_{m-1}}^n + \int_{t_{m-1}}^{t_m} h(z_t; \theta) \cdot dt + g(z_t; \theta) \cdot d\widehat{B}_t, \qquad \xi^n \sim \mathcal{N}(0, I_R),
\end{align}
wherein $z_t^n$ is now a \emph{deterministic} function of a new latent variable, $\xi^n$.

\paragraph{Fitting Latent SDEs to Data.} \label{apx:model-fitting}
Our goal is to find parameters, $\Theta = \{ \mu_0, \sigma_0, \theta, \sigma_\epsilon \}$, that maximize the log marginal likelihood (LML) of the observed data:
\begin{align*}
    \Theta^* &= \mathrm{argmax}_\Theta \frac{1}{N} \sum\limits_{n=1}^N \log p(X^n; \Theta) 
    = \mathrm{argmax}_\Theta \frac{1}{N} \sum\limits_{n=1}^N \log \int_{z_{t_0}} \int_\xi p(X^n, \xi, z_{t_0}; \Theta) \cdot d\xi \cdot dz_{t_0},
\end{align*}
where $X^n$ represents all of patient $n$'s training data.
Since the above integrals are intractable, we compute a variational lower bound to the LML instead:
\begin{align}
    \log p(X^n; \Theta)    
    \geq \mathbb{E}_{q(\xi, z_{t_0} | X^n; \Phi)} \left[ \log \frac{p(X^n, \xi, z_{t_0}; \Theta)}{q(\xi, z_{t_0} | X^n; \Phi)} \right] = \text{ELBO}(X^n; \Theta, \Phi),
\end{align}
where 
\begin{align}
    q(\xi, z_{t_0} | X^n; \Phi) = \mathcal{N}(\mu_\xi^n, \sigma_\xi^n \cdot \sigma_\xi^n \cdot I_R) \cdot \mathcal{N}_{[0, 1]}( \mu_{z_0}^n, \sigma_{z_0}^n \cdot \sigma_{z_0}^n \cdot I_{D_z}),
\end{align}
is our variational family, and $\Phi = \{ \mu_\xi^n, \sigma_\xi^n, \mu_{z_0}^n, \sigma_{z_0}^n \}_{n=1}^N$ is the set of all variational parameters.
We maximize the ELBO with respect to the model and inference parameters, $\Theta$ and $\Phi$, using stochastic gradient descent.

\paragraph{Approximate Predictive Log-Likelihood.}
Let $\Theta^*$ and $\Phi^*$ denote the model and inference parameters that maximize the ELBO as obtained by gradient descent. 
Then the predictive log-likelihood is,
\begin{align}
    \log p(x_{t^*}^n | X^n; \Theta^*) = \log \mathbb{E}_{p(\xi, z_{t_0} | X^n; \Theta^*)} \left[ p(x_{t^*}^n | \xi, z_{t_0}; \Theta^*) \right] 
    \approx \log \mathbb{E}_{  q(\xi, z_{t_0} | X^n; \Phi^*)} \left[ p(x_{t^*}^n | \xi, z_{t_0}; \Theta^*) \right],
\end{align}
obtained via Monte Carlo approximation.

\paragraph{Forecasting.}
We sample from the approximate posterior and solve the corresponding differential equation to obtain forecasts.
That is, for each draw from $q(\xi, z_{t_0} | X^n; \Phi^*)$, we evaluate:
\begin{align}
    z_{t^*}^n | z_{t_0}, \xi &= z_{t_0} + \int_{t_0}^{t^*} h(z_t; \theta) \cdot dt + g(z_t; \theta) \cdot d\widehat{B}_t.
\end{align}

\section{Experimental Setups} \label{apx:setup}

\paragraph{Software.}
All experiments were conducted in Jax~\citep{bradbury2018jax} with NumPyro~\citep{phan2019composable}, Diffrax~\citep{kidger2021on} and Chex.

\subsection{Experimental Setup for \cref{fig:intuition}} 
\label{apx:expert-dynamics-setup}

Suppose we were given dynamics from a domain expert and we wanted to transform them via WSP to ensure they remained viable in any polyhedron.

\paragraph{Expert-Given Dynamics.}
Consider the dynamics,
\begin{align}
\tilde{h}(t, z_t) = \begin{bmatrix}
    2 \cdot z_t^0 + 5.25 \cdot z_t^1 - 3.625 \\
    5.25 \cdot z_t^0 - 2 \cdot z_t^1 - 1.625
\end{bmatrix}, \qquad
\tilde{g}(t, z_t) = \begin{bmatrix}
    0.3 \\
    0.3
\end{bmatrix}, \label{eq:spiral-dynamics}
\end{align}
which spiral counterclockwise outward from the point $(0.5, 0.5)$.

\paragraph{Polyhedra.}
We instantiate \cref{def:polyhedron} as follows to obtain three polyhedra:
\begin{itemize}
    \item Right-Angle Triangle:
    \begin{align*}
        u_1 &= \begin{bmatrix} 0.0 & 0.0 \end{bmatrix}, \qquad 
        u_2 = \begin{bmatrix} 0.0 & 0.0 \end{bmatrix}, \qquad 
        u_3 = \begin{bmatrix} 0.5 & 0.5 \end{bmatrix}, \\
        v_1 &= \begin{bmatrix} 1.0 & 0.0 \end{bmatrix}, \qquad 
        v_2 = \begin{bmatrix} 0.0 & 1.0 \end{bmatrix}, \qquad 
        v_3 = \begin{bmatrix} -\sqrt{0.5} & -\sqrt{0.5} \end{bmatrix}.
    \end{align*}

    \item Unit Square:
    \begin{align*}
        u_1 &= \begin{bmatrix} 0.0 & 0.0 \end{bmatrix}, \quad 
        u_2 = \begin{bmatrix} 0.0 & 0.0 \end{bmatrix}, \quad 
        u_3 = \begin{bmatrix} 1.0 & 1.0 \end{bmatrix}, \quad
        u_4 = \begin{bmatrix} 1.0 & 1.0 \end{bmatrix}, \\
        v_1 &= \begin{bmatrix} 1.0 & 0.0 \end{bmatrix}, \quad 
        v_2 = \begin{bmatrix} 0.0 & 1.0 \end{bmatrix}, \quad 
        v_3 = \begin{bmatrix} -1.0 & 0.0 \end{bmatrix}, \quad
        v_4 = \begin{bmatrix} 0.0 & -1.0 \end{bmatrix}.
    \end{align*}
    
    \item Lopsided Pentagon:
    \begin{align*}
        u_1 &= \begin{bmatrix} 0.1 & 0.1 \end{bmatrix}, 
        u_2 = \begin{bmatrix} 0.1 & 0.1 \end{bmatrix}, 
        u_3 = \begin{bmatrix} 1.1 & 1.1 \end{bmatrix}, 
        u_4 = \begin{bmatrix} 1.1 & 1.1  \end{bmatrix}, 
        u_5 = \begin{bmatrix} 0.8 & 0.8  \end{bmatrix}, \\
        v_1 &= \begin{bmatrix} 1.0 & 0.1 \end{bmatrix}, 
        v_2 = \begin{bmatrix} 0.1 & 1.0 \end{bmatrix}, 
        v_3 = \begin{bmatrix} -1.0 & 0.2 \end{bmatrix}, 
        v_4 = \begin{bmatrix} 0.2 & -1.0 \end{bmatrix},
        v_5 = \begin{bmatrix} -\sqrt{0.5} & -\sqrt{0.5} \end{bmatrix}.
    \end{align*}
    
\end{itemize}

\paragraph{Dynamics.}
We transform the dynamics from \cref{eq:spiral-dynamics} via WSP in \cref{eq:wsp} to remain within each of the above polyhedra.
We use the following hyper-parameters:
\begin{itemize}
    \item Right-Angle Triangle: $\alpha = 5.0$, $\beta = 100.0$, $\gamma = 2.0$, $\epsilon = 0.1$.
    \item Unit Square: $\alpha = 5.0$, $\beta = 1000.0$, $\gamma = 2.0$, $\epsilon = 0.1$.
    \item Lopsided Pentagon: $\alpha = 10.0$, $\beta = 8000.0$, $\gamma = 2.0$, $\epsilon = 0.1$.
\end{itemize}

\paragraph{Differential Equation Solver.}
All SDE trajectories began at $z_0 = \begin{bmatrix} 0.05 & 0.85 \end{bmatrix}^\intercal$ and were solved for time interval $t \in [0.0, 5.0]$ via the It\^o-Milstein solver~\citep{itoMilsteinSolver} with a step-size of $0.01$.

\subsection{Experimental Setup for \cref{fig:ito-inductive-bias,fig:stationary-inductive-bias,fig:strat-inductive-bias}} \label{apx:synthetic-setup}

\paragraph{Dynamics.}
Due to Limitation 2 (\cref{sec:intro}), initialization plays a crucial role in the success of expressive SDE-based models.
This is because many datasets (e.g.~images) lie in compact Euclidean subspaces.
In early-stage training, SDE trajectories often leave the region, requiring a large number of gradient steps just to return to it (while not necessarily fitting the data well), causing optimization to get stuck in poor local optima.
In late-stage training, small perturbations to the dynamics may, again, yield trajectories that lie outside the region.
As such, to empirically compare the inductive bias of WSP (\cref{eq:wsp}) against baselines (\cref{eq:unconstrained,eq:ito,eq:ito-absorbed}), we solve SDEs given by NNs $h$ and $g$ with randomly sampled weights.
We define the viable region, $K = [0, 1]$, to be an interval, and specifically choose to set $z_0 = 0.99$ near the boundary to stress-test the chain-rule based SDEs in \cref{eq:ito,eq:ito-absorbed} to show that once close to the boundary, they will struggle to return to the interior of $K$. 
While simple, \emph{these preliminary experiments already show WSP boasts a stark improvement in inductive bias in comparison to baselines.}

\paragraph{Architecture.}
In all experiments presented here, we used 3-layer NNs with 64 hidden units and CELU activation~\citep{barron2017continuously}.
We repeated these experiments with 2-layer and 4-layer NNs and observed the \emph{exact same behavior}, so we have omitted them for brevity.
We also repeated these experiments with other continuous activation functions---GeLU~\citep{hendrycks2016gaussian}, ELU~\citep{clevert2015fast}, SELU~\citep{klambauer2017self}, and SiLU~\citep{elfwing2018sigmoid}---and we observed the \emph{exact same type of behavior}, so we have omitted them for brevity.

\paragraph{Random Restarts.}
For each SDE in \cref{eq:unconstrained,eq:ito,eq:ito-absorbed,eq:wsp}, we randomly drew the weights using the Glorot normal initialization~\citep{glorot2010understanding}.
In each plot, we repeated this initialization 5 times, drawing 3 samples for each initialization.

\paragraph{Differential Equation Solver.}
In \cref{fig:ito-inductive-bias}, we used the It\^o-Milstein SDE solver~\citep{itoMilsteinSolver}.
In \cref{fig:strat-inductive-bias}, we used the Dormand-Prince 8/7 ODE solver~\citep{dormand1980family}.
In all experiments, we simulated the dynamics for $t \in [0, 5]$ with a step size of $0.001$.
We purposefully chose a small step size to ensure the faithfulness of the SDE solutions to the dynamics.

\paragraph{Pathwise Expansion.}
We used a truncation of $R = 40$ terms in the pathwise expansion (\cref{eq:kl-expansion}) for the experiment in \cref{fig:strat-inductive-bias}.
We repeated the experiments with $R = 20$, $100$, and $200$ and observed the \emph{exact same behavior}, so we have omitted them for brevity.

\subsection{Experimental Setup for \cref{table:performance,fig:wsp-qualitative-inline,fig:wsp-qualitative-185,fig:wsp-qualitative-114,fig:wsp-qualitative-143,fig:wsp-qualitative-144,fig:wsp-qualitative-149,fig:wsp-qualitative-90,fig:wsp-qualitative-56,fig:wsp-qualitative-5,fig:wsp-qualitative-15,fig:wsp-qualitative-17,fig:wsp-qualitative-88,fig:wsp-qualitative-23,fig:wsp-prior-samples}} \label{apx:real-setup}

\paragraph{Dynamics and Architecture.}
In all experiments, we used 3-layer NNs with 64 hidden units and GELU activations for the unconstrained drift and diffusion, $\tilde{h}$ and $\tilde{g}$. 
For the diffusion, we additionally applied a softplus activation~\citep{dugas2000incorporating} at the output to enforce positivity.

When using WSP to transform $\tilde{h}$ and $\tilde{g}$ so that the SDE solution remains viable, we additionally clipped the state $z_t$ to ensure it remains viable before passing it into the dynamics. 
This was necessary because, although our dynamics are viable, standard ODE/SDE solvers are not. 
Developing solvers that are themselves viable is an important direction for future work.

\paragraph{Differential Equation Solver.} 
Thanks to the pathwise expansion (described in \cref{apx:latent-sde}), we used the Dormand-Prince 8/7 ODE solver~\citep{dormand1980family} using an adaptive step size initialized at $0.01$, with a minimum of $0.001$, and with a tolerance of $0.001$.

\paragraph{Warm Starts.} 
Since optimizing the ELBO for neural latent SDEs is challenging, we propose to start optimization with the following warm starts:
\begin{enumerate}
    \item \textbf{Drift.} We encourage the drift to be $0$ to encourage the model to explore both positive and negative values of drift. We do this by minimizing the following loss function for $5000$ gradient steps:
    \begin{align}
        \mathrm{argmin}_\theta \lVert \tilde{h}(z_t; \theta) - 0.0 \rVert_2^2,
    \end{align}
    wherein $\tilde{h}$ is the drift before applying WSP.
    
    \item \textbf{Diffusion.} We initialized the diffusion term at $0.5$ to encourage the model to learn a nonzero diffusion. 
    Without this initialization, the model frequently collapsed to solutions with diffusion equal to 0. 
    In that case, the noise variable $\xi$ is unused in the generative process, so its posterior matches its prior, the KL term in the ELBO becomes zero, and optimization gets trapped in a local optimum where only the reconstruction term contributes.
    We do this by minimizing the following loss function for 5000 gradient steps:
    \begin{align}
        \mathrm{argmin}_\theta \lVert \tilde{g}(z_t; \theta) - 0.5 \rVert_2^2,
    \end{align}
    wherein $\tilde{g}$ is the diffusion before applying WSP.

\end{enumerate}

\paragraph{Optimization.} We used the Adam optimizer~\citep{kingma2014adam} with a linear learning rate scheduler going from 1e-4 to 5e-5 and a total of 80,000 gradient steps. To be able to evaluate the inductive bias of the model, we also fixed the variational standard deviations for all $N$ posteriors to a small value of $\sigma^n_\xi = 0.05$ throughout optimization, such that the model does not simply increase uncertainty to match the stochasticity in the data (see \cref{sec:results} for additional discussion).

\paragraph{Random Restarts.}
We repeated each experiment 5 times, each time randomly drawing the weights of the drift and diffusion NNs using a Glorot normal initializer~\citep{glorot2010understanding}, which adapts scaling to the arithmetic average of the numbers of inputs and outputs.

\paragraph{Train/Interpolation/Forecast Data Split.}
We split the study time horizon into interpolation and forecasting regimes. To create a challenging forecasting task, we defined the test set as all observations after the median time step (computed across all patients and time points) and evaluated models on predicting this held-out second half, thereby stress-testing their inductive bias. For interpolation, we randomly held out 20\% of the time steps in the first half of the time horizon for each patient (10\% of all time steps). This interpolation set lies within the training window and evaluates how well the learned dynamics fit observed data.

\paragraph{Model Selection.} 
Across the random restarts, we selected the model with the highest posterior predictive on the interpolation set, treating it like a validation set.

\paragraph{Model Evaluation.}
We assessed the quality of learned models along several axes:
\begin{itemize}
    \item \textbf{Inductive Bias.} 
    We evaluated each model's posterior predictive on the test set.
    
    \item \textbf{Learning Dynamics.} 
    We evaluated each model's posterior predictive on the interpolation set.
    Optimization for models that have worse interpolation set performance got stuck in poor local optima. 

    \item \textbf{Constraint Satisfaction.}
    We evaluated how well each model satisfied the constraints proposed in \cref{thm:milian}.
    Since our real data experiments all require the SDE solution to lie in the unit cube, the constraints are: (a) $h^d(z) \geq 0$ if $z^d = 0$, and $h^d(z) \leq 0$ if $z^d = 1$, and (b) $g^d(z) = 0$ if $z^d \in \{ 0, 1 \}$.
    We do this via the metrics proposed below, which we approximate via 100 Monte Carlo samples.  
    In these metrics, $z \cdot (1 - e_d) + e_d$ is $z$ with dimension $d$ set to 1, and $z \cdot (1 - e_d)$ is $z$ with dimension $d$ set to 0.
    \begin{itemize}
        \item Drift Validity Proportion (DRVP):
        \begin{align}
            \text{DRVP} = \frac{1}{2 D_z} \cdot \sum\limits_{d=1}^{D_z} \mathbb{E}_{z \sim \mathcal{U}[0, 1]} \left[ \mathbb{I}\left( h^d(z \cdot (1 - e_d) + e_d ) \leq 0 \right) + \mathbb{I}\left( h^d(z \cdot (1 - e_d)) \geq 0 \right) \right].
        \end{align}        
        This computes the fraction of boundary for which the drift satisfies the constraints (higher is better).
        
        \item Diffusion Validity Proportion (DIVP):
        \begin{align}
            \text{DIVP} = \frac{1}{2 D_z} \cdot \sum\limits_{d=1}^{D_z} \mathbb{E}_{z \sim \mathcal{U}[0, 1]} \left[ \mathbb{I}\left( g^d(z \cdot (1 - e_d) + e_d ) < \epsilon \right) + \mathbb{I}\left( g^d(z \cdot (1 - e_d)) < \epsilon \right) \right],
        \end{align}
        for some small $\epsilon = 0.001$ to allow for very small positive values, since $g(z)$ uses a softplus activation.
        This computes the fraction of boundary for which the diffusion satisfies the constraints (higher is better).
        
        \item Diffusion Distance to Valid (DIDV):
        \begin{align}
            \text{DIDV} = \frac{1}{2 D_z} \cdot \sum\limits_{d=1}^{D_z} \mathbb{E}_{z \sim \mathcal{U}[0, 1]} \left[ |g^d(z \cdot (1 - e_d) + e_d)| + |g^d(z \cdot (1 - e_d))| \right].
        \end{align}      
        This metric computes the average distance from the diffusion's value on the boundary to 0 (lower is better).
        
    \end{itemize}
    
\end{itemize}

\section{Descriptions of Real Data} \label{apx:real-data}

\paragraph{The PhysioNet GLOBEM Data.} 
GLOBEM is a public, multi-year collection of datasets for longitudinal human behavior collected from 2018 to 2022~\citep{goldberger2000physiobank,globem}. 
A total of 497 unique undergraduates in the United States were recruited via email to participate in the study. 
The collection is split into 4 datasets, one for each year, and each containing 155, 218, 137, and 195 participants, respectively. 

In their study, participants downloaded a mobile app and wore a fitness tracker for 10 weeks. 
In this work, we specifically focused on the EMA data collected via mobile surveys pertaining directly to mental health, which consisted of 6 different psychometric instruments:
\begin{itemize}
    \item Perceived Stress Scale (PSS-4)~\citep{pss4}: measures stress levels during the last month, ranging from 0 to 16, with higher values indicating more perceived stress. This was collected every Wednesday.
    \item Positive Affect and Negative Affect Schedule (PANAS)~\citep{panas}: captures positive and negative affects, respectively, each on a range from 0 to 20. Higher values indicate higher levels of positive/negative affect, respectively. This was collected every Wednesday and Sunday.
    \item Patient Health Questionnaire Mental Health (PHQ-4-MH)~\citep{phq4} ranges from 0 to 12, with higher values indicating a higher risk of mental health struggles. This was collected every Sunday.
    \item Patient Health Questionnaire Anxiety (PHQ-4-A)~\citep{phq4} ranges from 0 to 6, with higher values indicating a higher risk of anxiety. This was collected every Sunday. 
    \item Patient Health Questionnaire Depression (PHQ-4-D)~\citep{phq4} ranges from 0 to 6, with higher values indicating a higher risk of depression. This was collected every Sunday.
\end{itemize}
Of the four years of data available, we chose to use ``DS2'' from 2019, ``DS3'' from 2020, and ``DS4'' from 2021 because these three datasets share the same EMA questions (``DS1'' from 2018 did not use the same EMA questions as the other 3 years).
To ensure that we had sufficient data from each patient used in training, we only included participants who had at least 10 observations recorded, omitting 8 participants from DS2, 6 participants from DS3, and 4 participants from DS4. 

For instruments with response scales extending beyond 0–10, we observed a small proportion of values exceeding 10; specifically,  across all patients and all time steps in the 3 datasets, we found 112 PHQ-4 mental health scores ($0.1\%$), 400 PSS-4 scores ($0.3\%$), 500 positive affect scores ($0.3\%$), and 2,445 negative affect scores ($1.5\%$) greater than 10.
As such, we truncated all values greater than 10 to 10, such that a score of 10 represents the category ``$\geq 10$.''

Because participants entered the study on different calendar dates, we converted all dates to patient-specific, zero-anchored time indices, with time step 0 corresponding to each participant’s own study start date.
In our model, we measured time in days and then rescaled it by a factor of 1/5, compressing the overall time horizon to reduce the computational cost of solving the SDEs.

\paragraph{The ``SMART'' EMA Data.}
A total of 623 individuals reporting suicidal thoughts and/or recent suicidal behavior were enrolled from two Boston-area hospitals, who participated in the Sensing and Mobile Assessment in Real Time (SMART) study at Harvard University~\citep{nock2026using}. This sample included 315 adults (18 years and older) recruited from a psychiatric emergency service and 308 adolescents (ages 12–19) recruited from a psychiatric inpatient unit.

Individuals were excluded if they did not own an iOS or Android smartphone, were unable to provide informed consent or assent, lacked fluency in spoken or written English, or exhibited substantial cognitive or behavioral impairment (e.g.~florid psychosis, intellectual disability, dementia, acute intoxication, or marked agitation or violence).

After enrollment, participants provided written consent/assent, completed a baseline questionnaire, and installed the LifeData app on their smartphones. The app delivered brief self-report surveys. Participants received \$10 for the baseline assessment and \$1 for each ecological momentary assessment (EMA) survey they completed. All procedures were reviewed and approved by the institutional review boards of the participating institutions.

EMA surveys captured momentary suicidal thinking---specifically the urge to engage in suicidal behavior, suicidal intent, and perceived ability to resist suicidal urges---as well as 17 affective states (negative, hopeless, trapped, isolated, burdensome, angry, self-hate, agitated, worried, numb, fatigued, humiliated, desire to escape, desire to avoid, energetic, and positive) rated on a 0–10 Likert scale. 
In this work, we modeled four EMA survey items, all on a 0-10 Likert scale, with higher values indicating higher intensity: stress, negative affect, desire to escape, and urge to die.
Surveys were administered six times per day over a three-month period. The first and last surveys of each day occurred at fixed times selected collaboratively with each participant, whereas the remaining surveys were delivered at random times between these anchors.

Participants also had the option to initiate additional surveys at any time, for instance to report a suicide attempt, non-suicidal self-injury, or other significant events. A dedicated risk-monitoring team reviewed incoming data in real time and intervened when participants reported elevated suicidal intent (further details available upon request).

Similar to the PhysioNet GLOBEM dataset, we also converted all dates to zero-anchored time indices. Since surveys were delivered at random times, we discretized time into 6 intervals per day such that responses within the same interval were considered to be administered in the same time bucket. 
If multiple responses from the same patient fell into the same bucket ($84\%$ of observations), we took the maximum to capture peak distress. 
This reduces the amount of noisy data while preserving critical psychological signals. 
The final time indices were rescaled by a factor of 1/20 to compress the time horizon and reduce the computational cost of solving the SDEs.
Finally, after this preprocessing, to ensure that we had sufficient data from each patient used in training, we only included participants who had at least 20 observations recorded, resulting in a final cohort of 201 patients.

\FloatBarrier

\section{Results} \label{apx:results}

\subsection{Inductive Bias with Smooth, Pathwise Expansion of Brownian Motion} \label{apx:results-kl-expansion}

\begin{figure*}[t]
    \centering

    \includegraphics[width=0.48\textwidth]{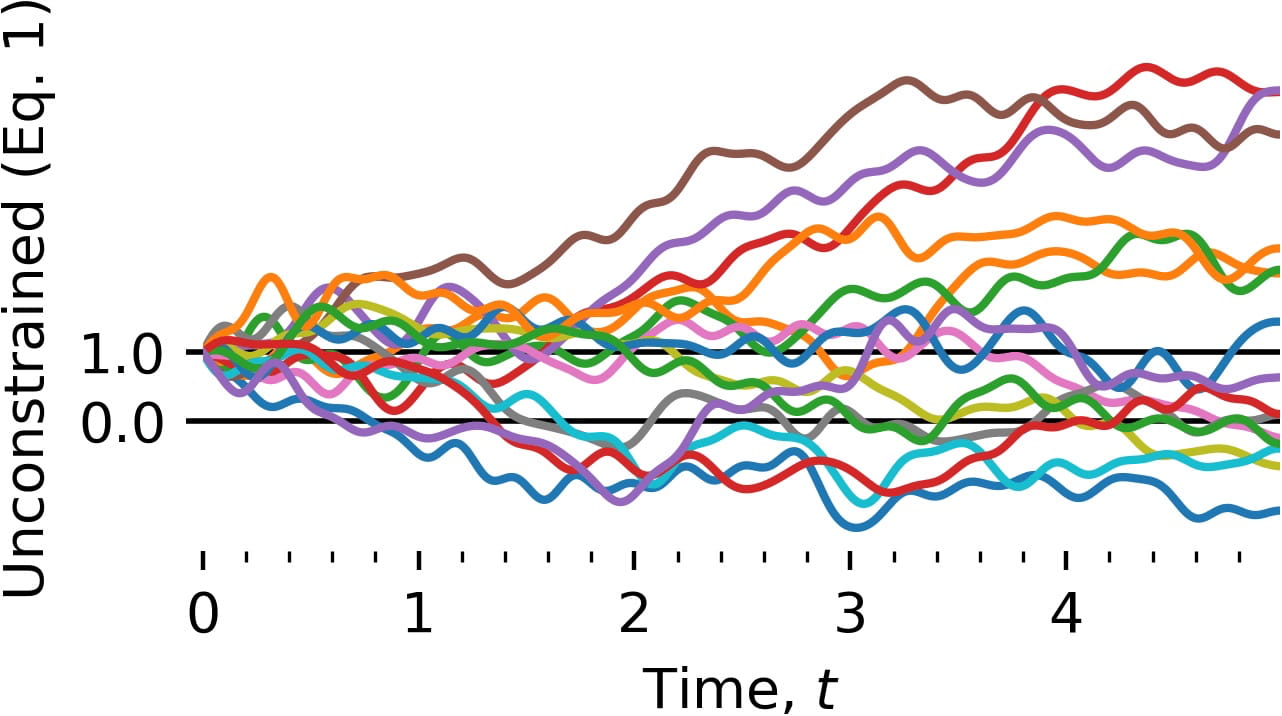}
    ~
    \includegraphics[width=0.48\textwidth]{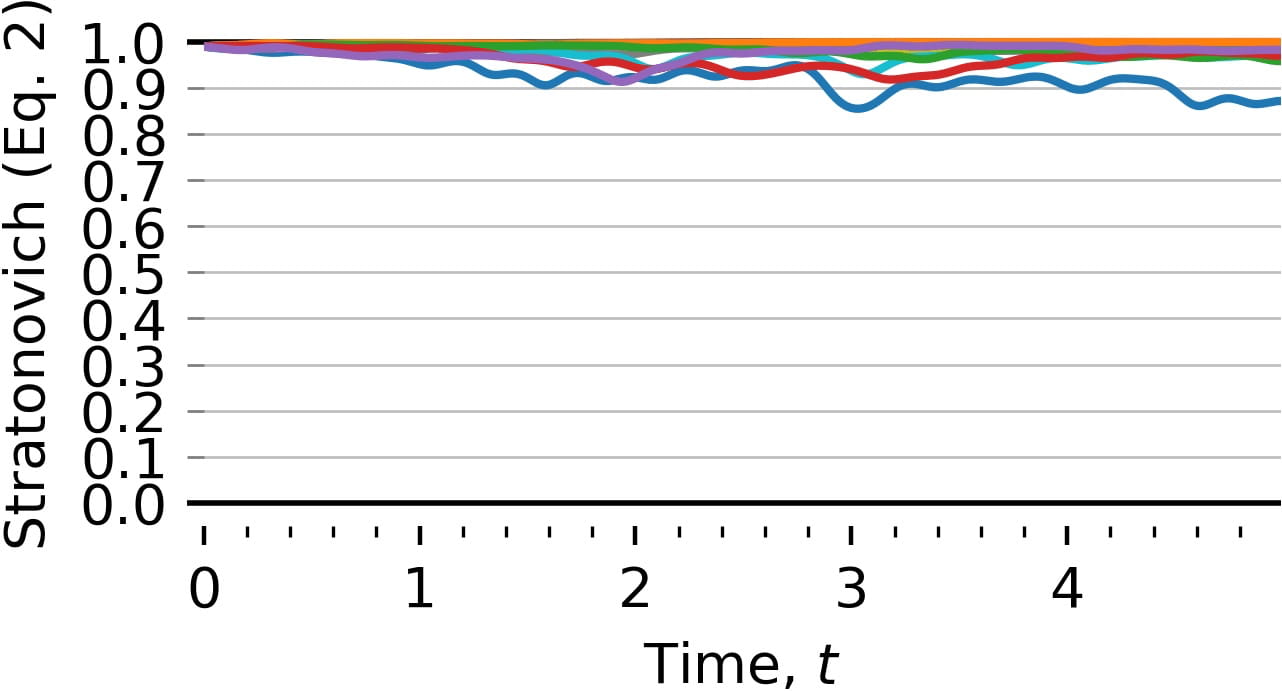}

    \vspace{1mm}

    \includegraphics[width=0.48\textwidth]{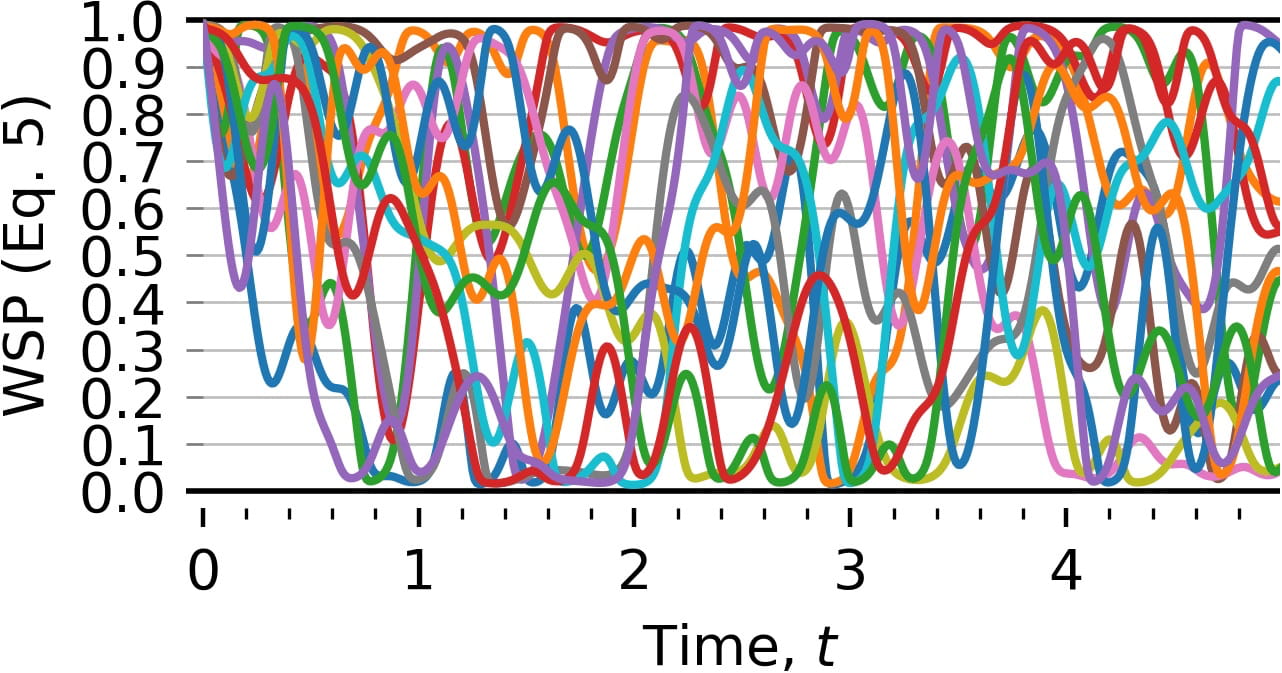}      
    ~
    \includegraphics[width=0.48\textwidth]{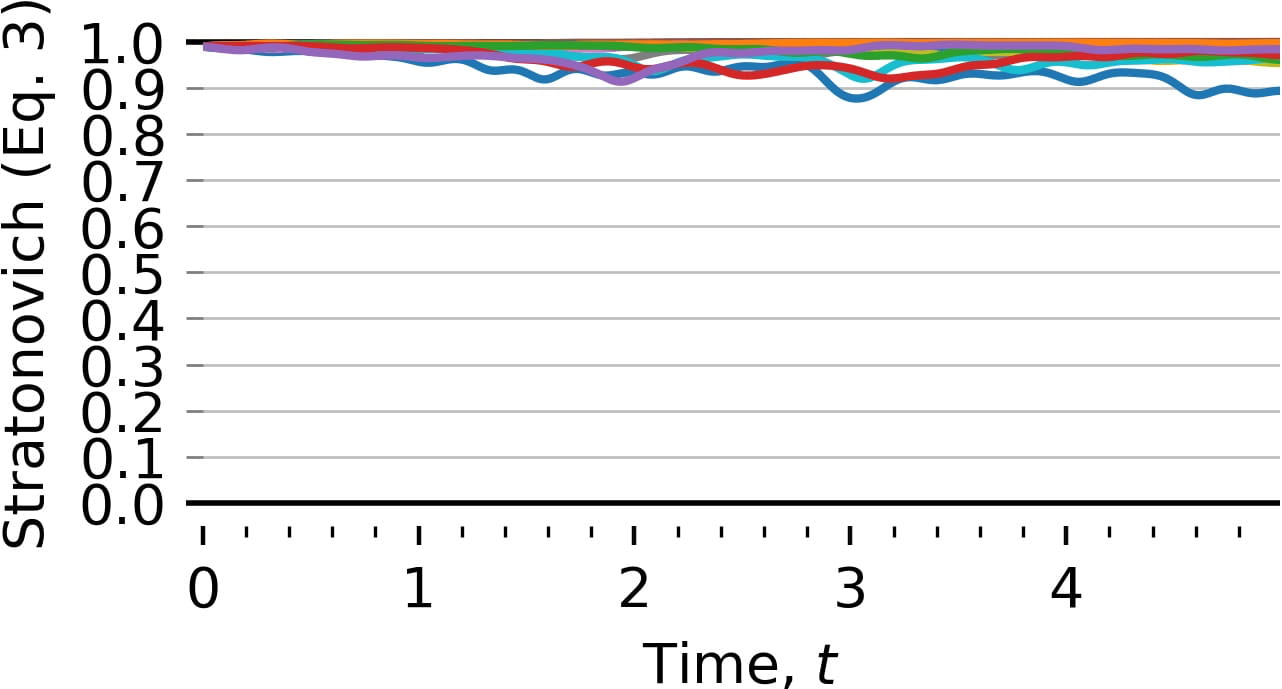}
    
    \caption{\textbf{WSP exhibits better inductive bias than baselines given smooth, pathwise expansion of Brownian motion.} Top left: Stratonovich-SDE with NN quickly leaves $K = [0, 1]$. Top \& bottom right: Stratonovich-SDE transformed via sigmoid sticks to the boundary. Bottom left: Stratonovich-SDE with WSP successfully remains in $K$. Note: for \cref{eq:ito,eq:ito-absorbed}, we used the Stratonovich chain-rule instead of It\^o's lemma.}
    \label{fig:strat-inductive-bias}
\end{figure*}

Since SDE solvers are slow and unstable, prior work focused on finding mechanisms to use ODE solvers instead.
ODE solvers are known to be more numerically stable, accurate, and well-behaved when used with adaptive step-sizes.
Prior work has used several strategies to accomplish this.
For example, prior work approximates the first two moments of the time-marginal using the Fokker-Planck-Kolmogorov (FPK) equation, which can then be solved using an ODE solver---this is known as the ``Gaussian Assumed Approximation'' \citep{sarkka2019applied}.
In diffusion models, prior work derived an FPK-based, fast, numerically stable process that samples from the same distribution as the SDE using an ODE solver---this is called the ``probability flow ODE''~\citep{song2021scorebased}.
Finally, prior work (e.g.~\citet{ghosh2022differentiable}) replaces Brownian motion with the Karhunen-Lo\`eve Expansion~\citep{sarkka2019applied}---see \cref{eq:kl-expansion}.
In \cref{fig:strat-inductive-bias}, we empirically demonstrate that WSP exhibits the same inductive bias in comparison to baselines, even under this pathwise expansion.

\FloatBarrier

\subsection{Stationary SDEs on r-Polyhedra} \label{apx:results-stationary}

In \cref{fig:stationary-inductive-bias}, we show that, for any neural diffusion with WSP (\cref{eq:wsp}) with randomly generated weights, we can always construct a drift, given by \cref{thm:stationary}, with dynamics that are viable in $K = [0, 1]$ and induce the target time-marginal.
Moreover, like their general counterparts, these stationary dynamics overcome the shortcomings of the baseline dynamics in \cref{eq:unconstrained,eq:ito,eq:ito-absorbed}.

\begin{figure*}[t]
    \centering

    \includegraphics[width=1.0\textwidth]{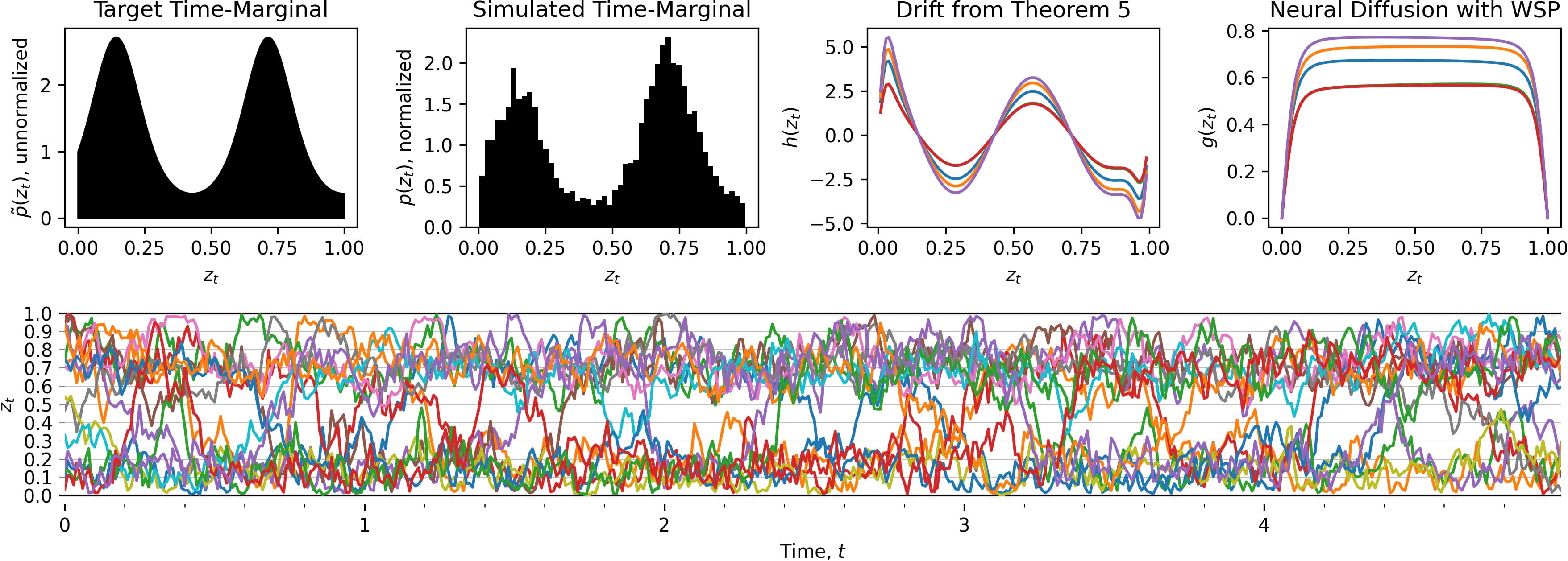}

    \vspace{3mm}
    
    \includegraphics[width=1.0\textwidth]{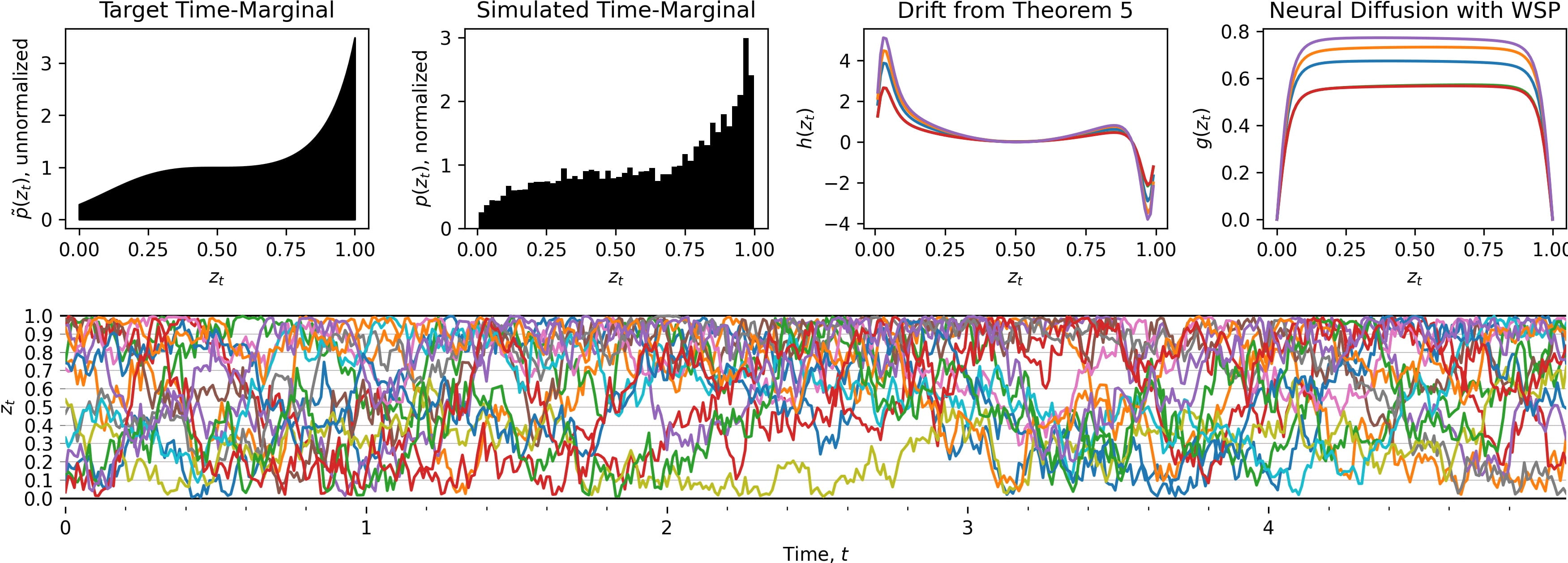}

    \vspace{3mm}
    
    \includegraphics[width=1.0\textwidth]{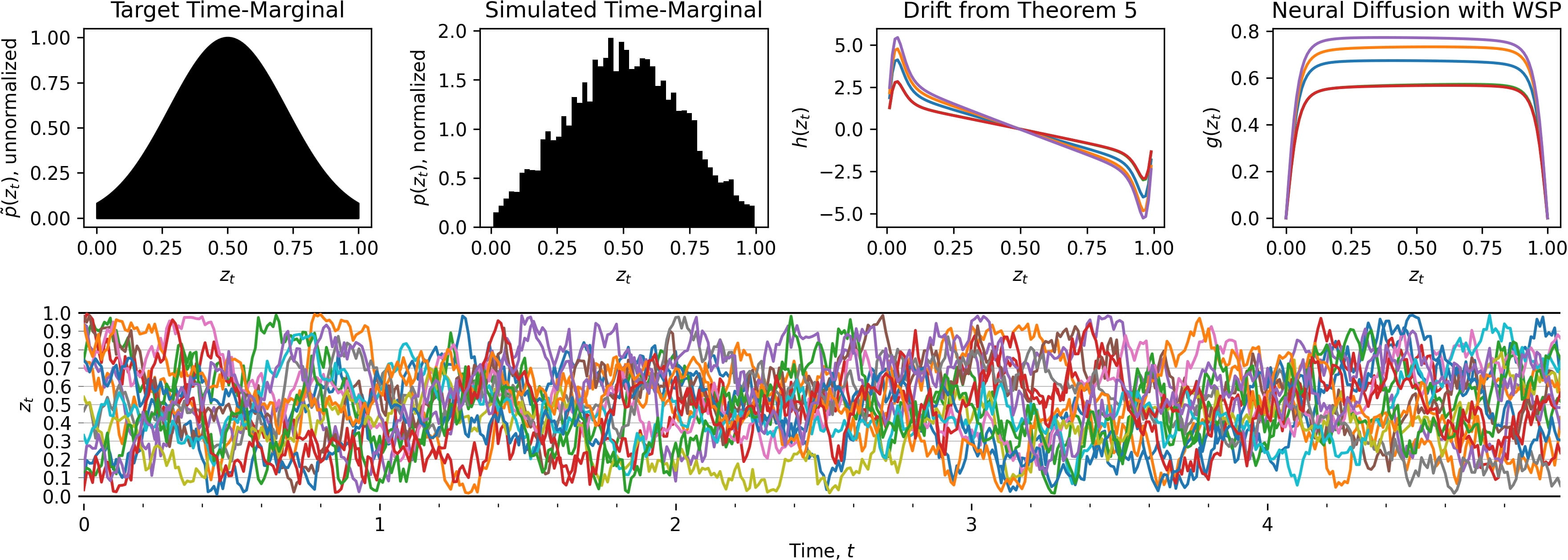}

    \caption{\textbf{Stationary SDE exhibits better inductive bias than baselines.} Given a target time-marginal, given any diffusion with WSP, we can always derive a corresponding drift via \cref{thm:stationary} that is viable in $K$ and has the target stationary distribution. Like the general dynamics, these dynamics overcome the shortcomings of the baseline dynamics in \cref{eq:unconstrained,eq:ito,eq:ito-absorbed}. Here, our diffusion is a NN with randomly initialized weights, with each color corresponding to a different seed. Note: the target time-marginal is \emph{not} normalized.}
    \label{fig:stationary-inductive-bias}
\end{figure*}

\FloatBarrier
\clearpage

\subsection{Visualizations of Latent SDE Posterior Samples from the ``SMART'' Data} \label{apx:viz-qualitative}

\cref{fig:wsp-qualitative-185,fig:wsp-qualitative-114,fig:wsp-qualitative-143,fig:wsp-qualitative-144,fig:wsp-qualitative-149,fig:wsp-qualitative-90,fig:wsp-qualitative-56,fig:wsp-qualitative-5,fig:wsp-qualitative-15,fig:wsp-qualitative-17,fig:wsp-qualitative-88,fig:wsp-qualitative-23} visualize the posteriors of latent neural SDEs of various patients in the ``SMART'' dataset across all baselines---\textcolor{colorVanillaStrong}{\textbf{Vanilla}}, \textcolor{colorVanillaClipStrong}{\textbf{Vanilla+Clip}}, \textcolor{colorWSPNoClipStrong}{WSP}, and \textcolor{colorWSPStrong}{\textbf{WSP+Clip}}.
These results consistently show that:
\begin{enumerate}
    \item WSP-based dynamics respect the specified clinical constraints and avoid assigning probability mass to impossible outcomes: the constraint-satisfaction metrics are perfect for WSP, by construction, and are far from perfect for Vanilla \emph{and} Vanilla+Clip. 

    \item These constraints guide optimization toward better minima: the WSP-based model exhibits better interpolation performance, indicating optimization converged to better optima.
    In contrast, the Vanilla and Vanilla+Clip model often struggles to fit the interpolation set, even though it lies in the same time region as the training data.

    \item Altogether, the WSP-based model exhibits a better inductive bias for forecasting.
\end{enumerate}
We, again, note that although WSP substantially improves forecasts, it still cannot reliably infer a patient's true state in the second half of the study just from the first half---EMA data is too stochastic for \emph{any} model to be this accurate.
\emph{Accordingly, we present examples where WSP forecasts are strikingly good (closely tracking the forecasting set or its overall trend) alongside examples where all dynamics miss the target entirely.}
Taken together, these results show that WSP's inductive bias can markedly improve forecasting, and suggest that incorporating additional clinical knowledge could yield similarly large gains.

\begin{figure*}[p]
    \centering

    \includegraphics[width=0.38\textwidth]{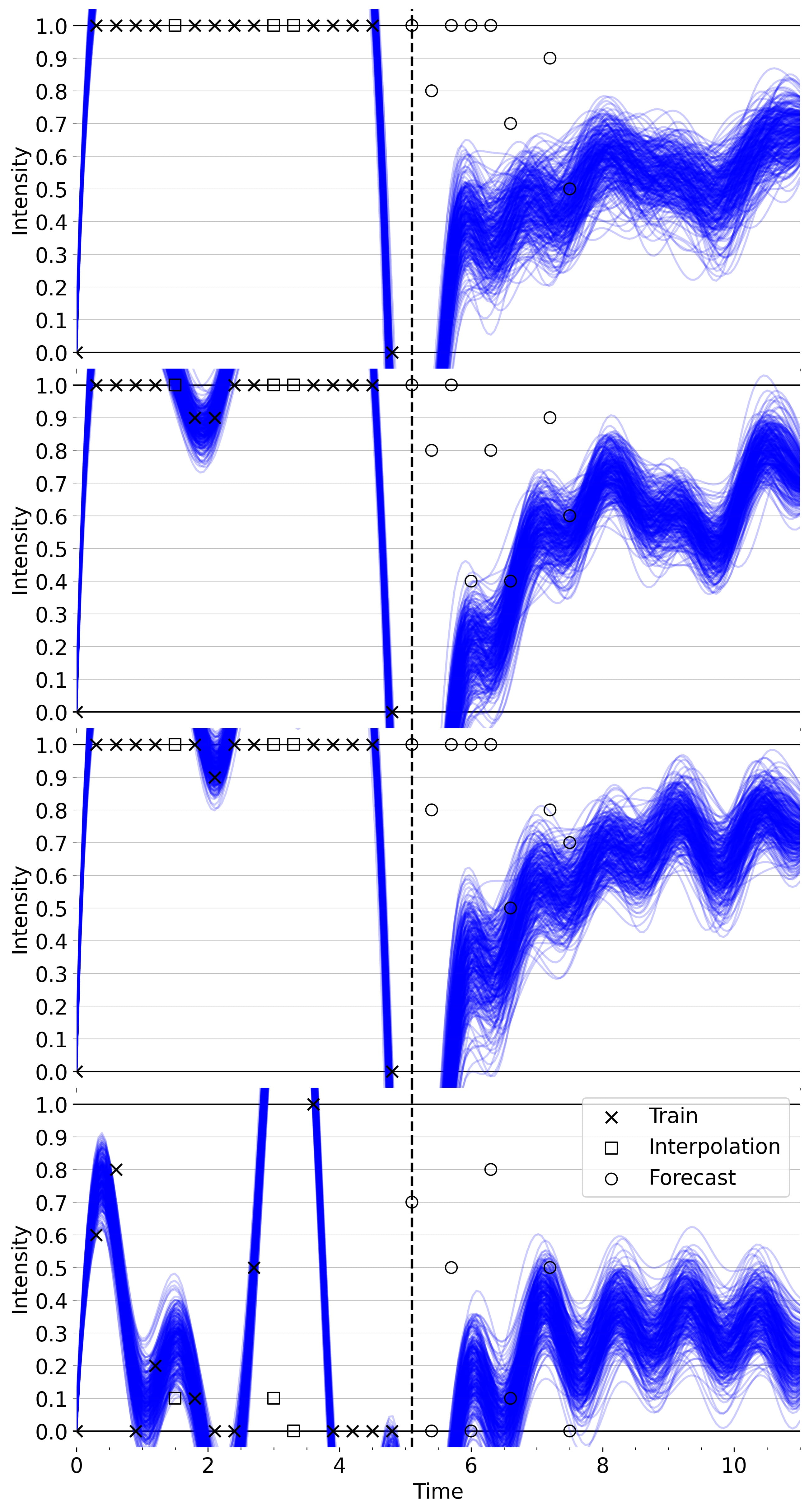}
    ~
    \includegraphics[width=0.38\textwidth]{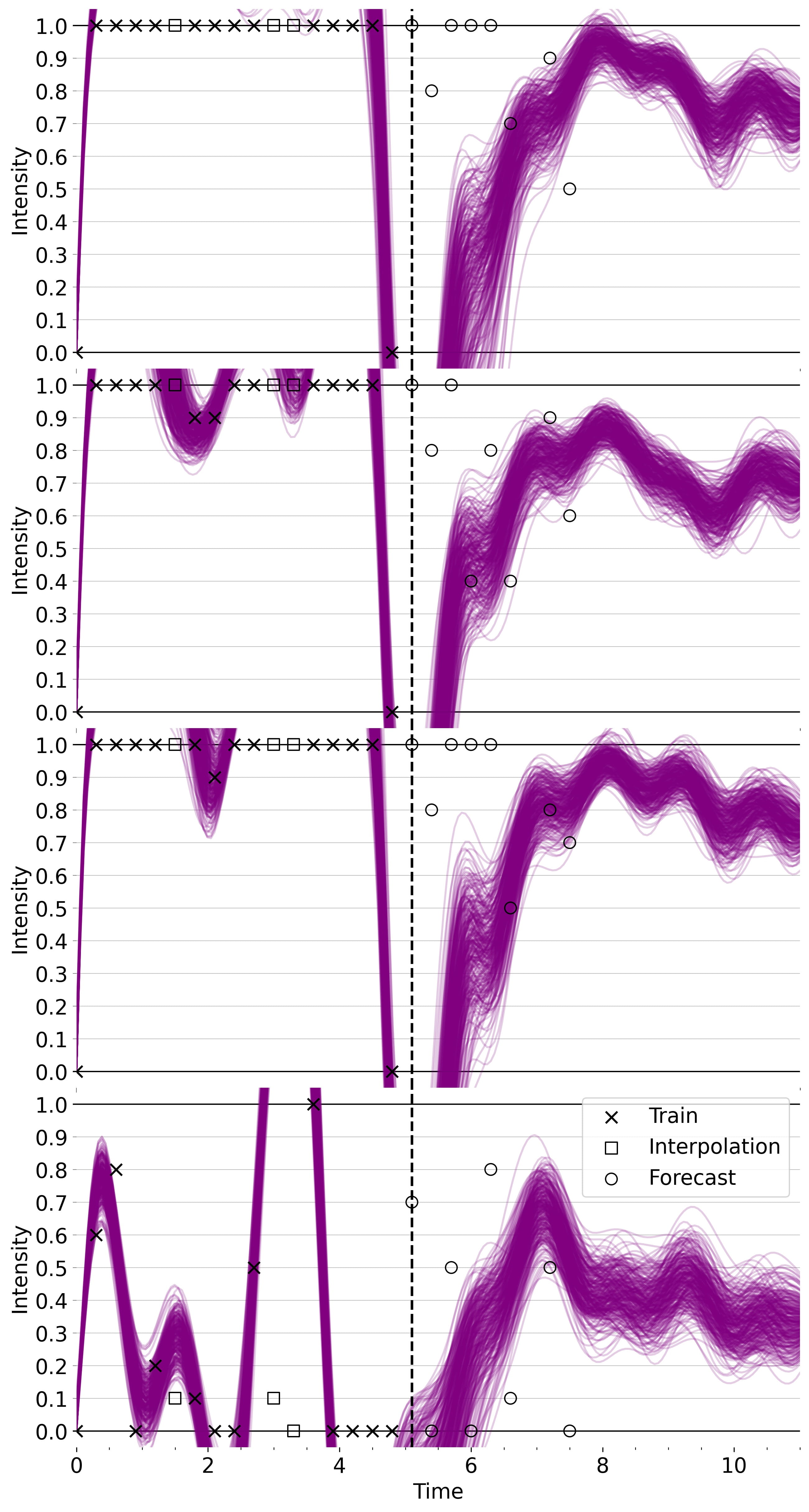}

    \includegraphics[width=0.38\textwidth]{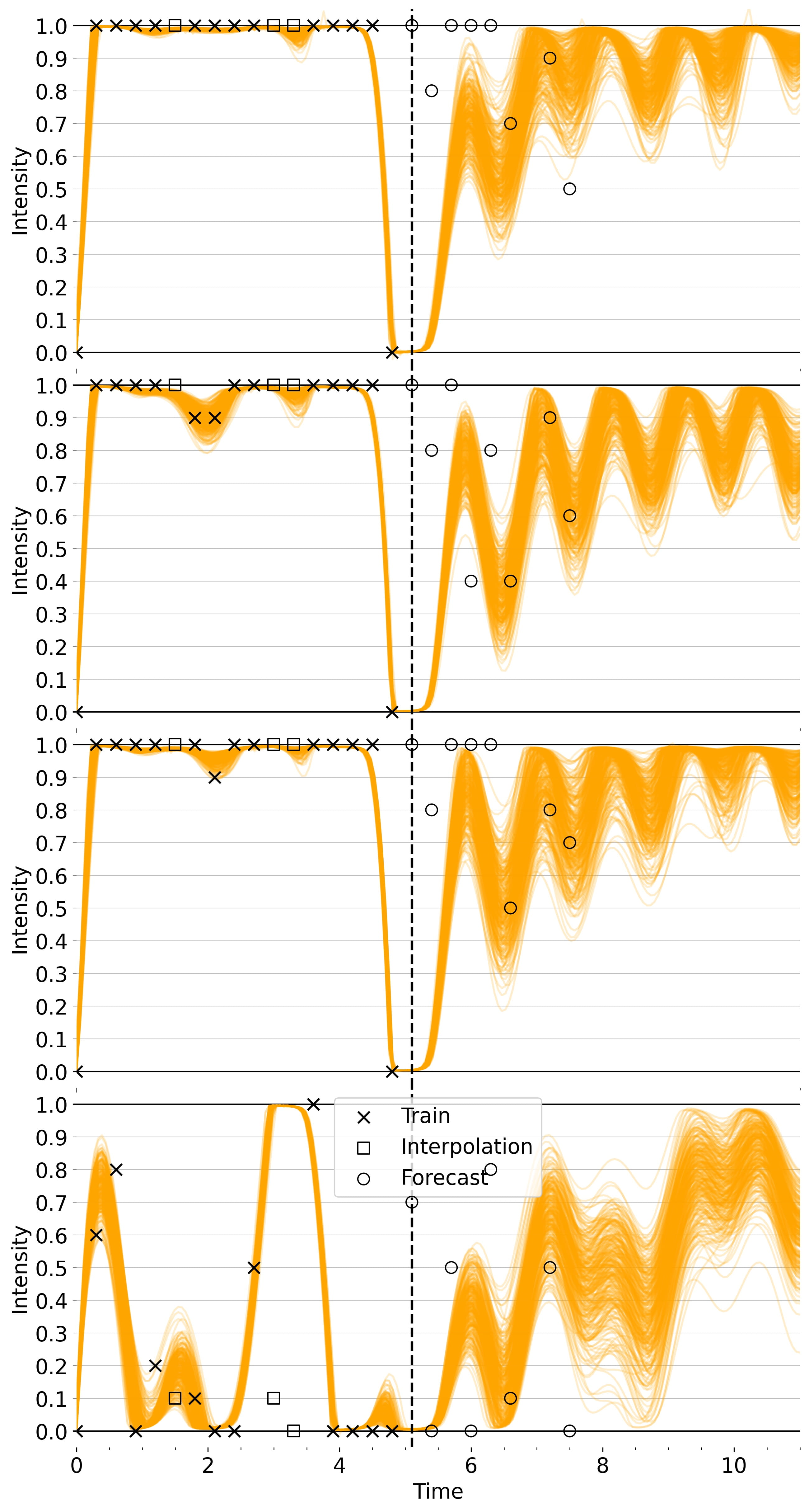}
    ~
    \includegraphics[width=0.38\textwidth]{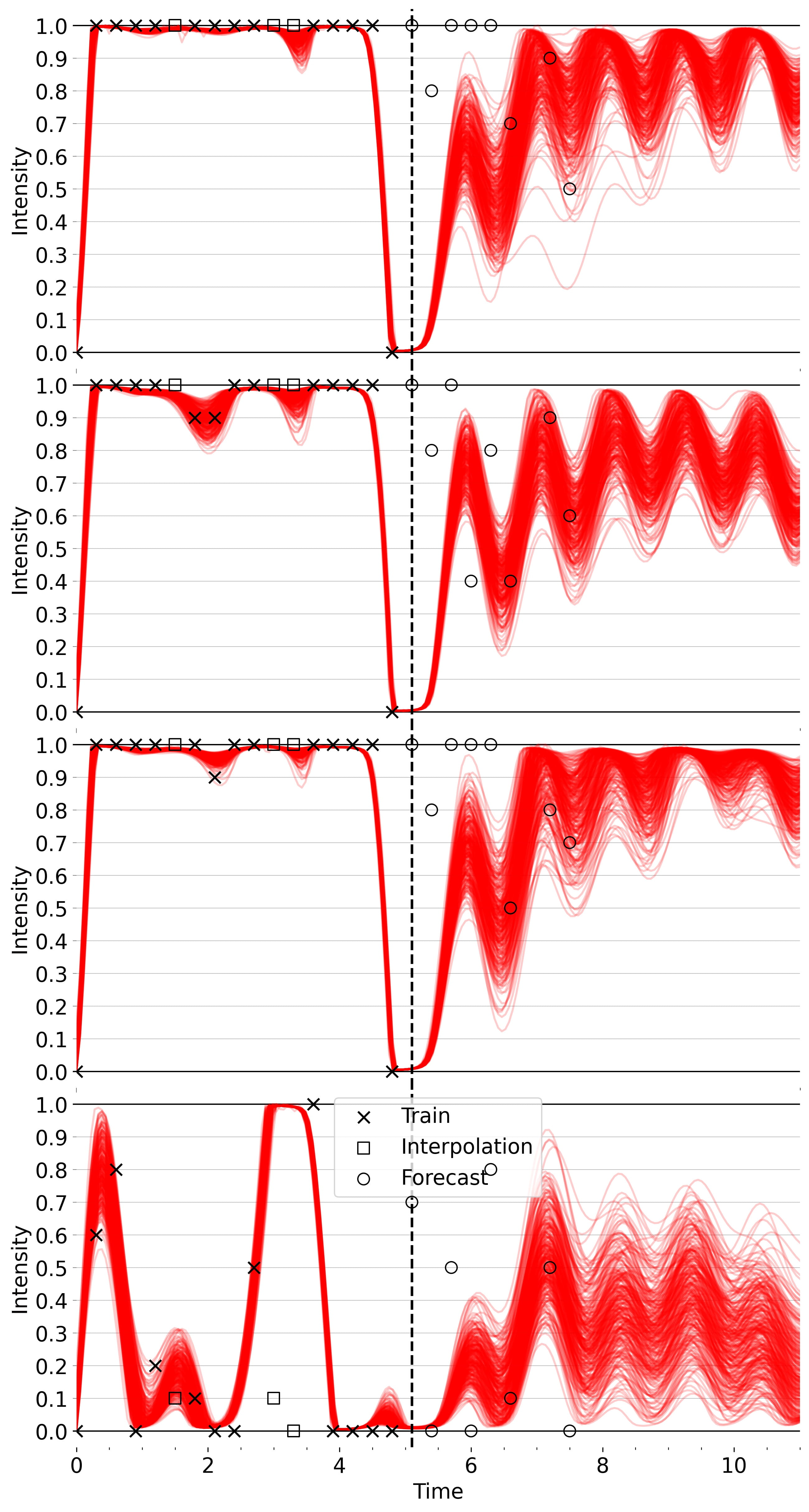}
    
    \caption{\textbf{Qualitative Comparison of Posterior for Specific Patient from \textcolor{colorVanillaStrong}{Vanilla} (top left), \textcolor{colorVanillaClipStrong}{Vanilla+Clip} (top right), \textcolor{colorWSPNoClipStrong}{WSP} (bottom left), and \textcolor{colorWSPStrong}{WSP+Clip} (bottom right).} Each row represents a different EMA survey item, listed in \cref{apx:real-data}. See discussion in \cref{apx:viz-qualitative}.}
    \label{fig:wsp-qualitative-185}
\end{figure*}

\begin{figure*}[p]
    \centering

    \includegraphics[width=0.38\textwidth]{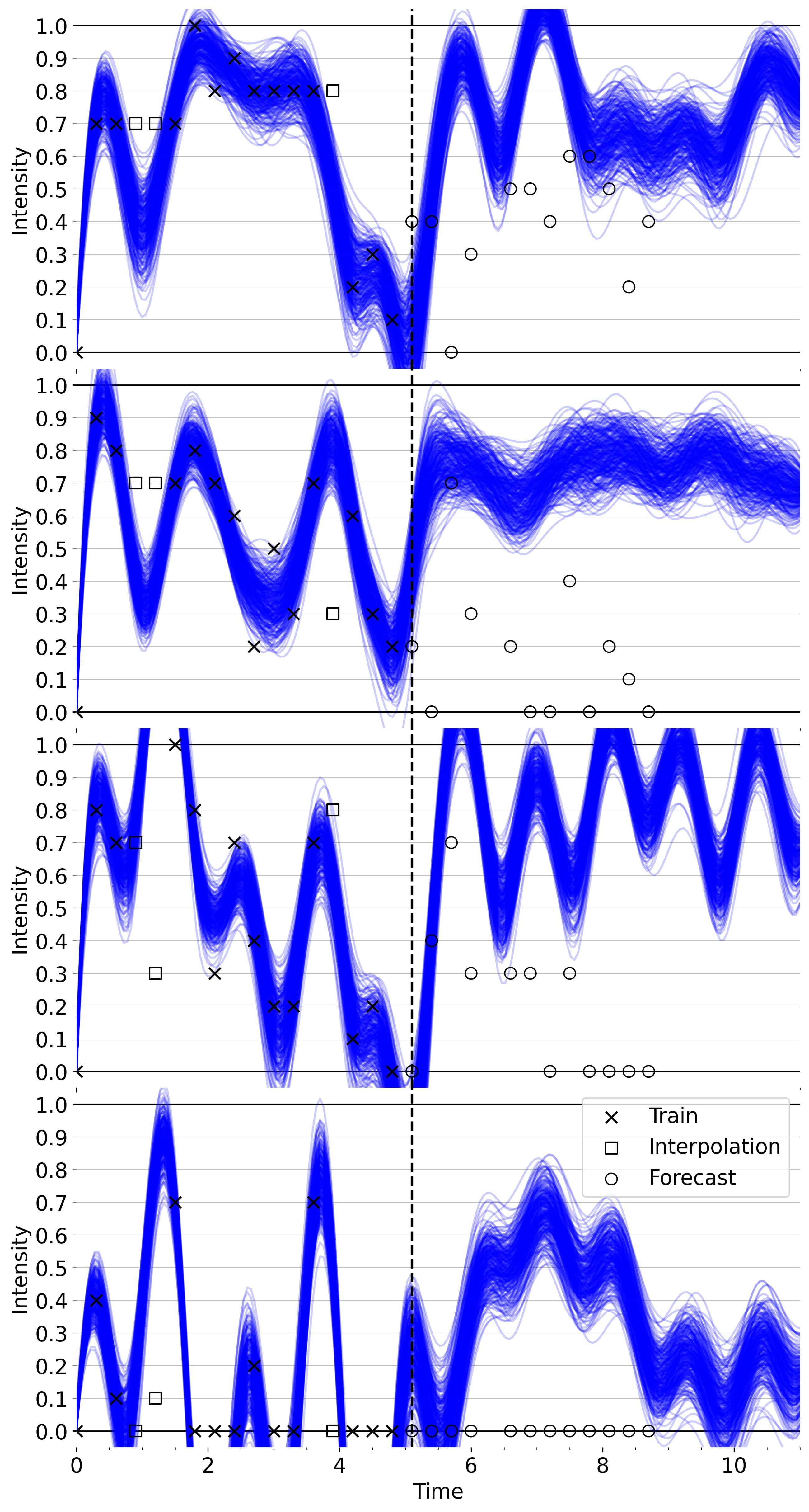}
    ~
    \includegraphics[width=0.38\textwidth]{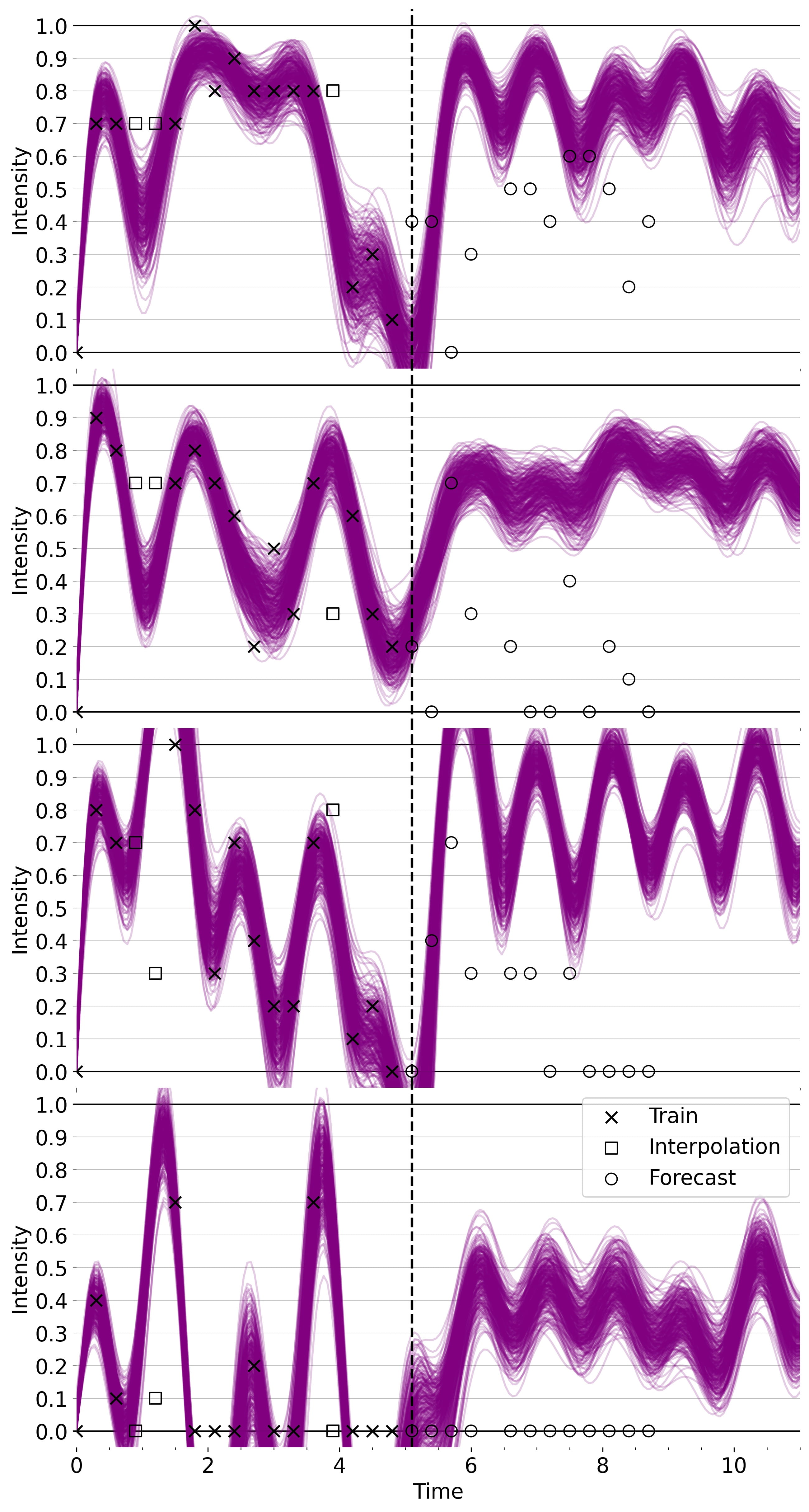}

    \includegraphics[width=0.38\textwidth]{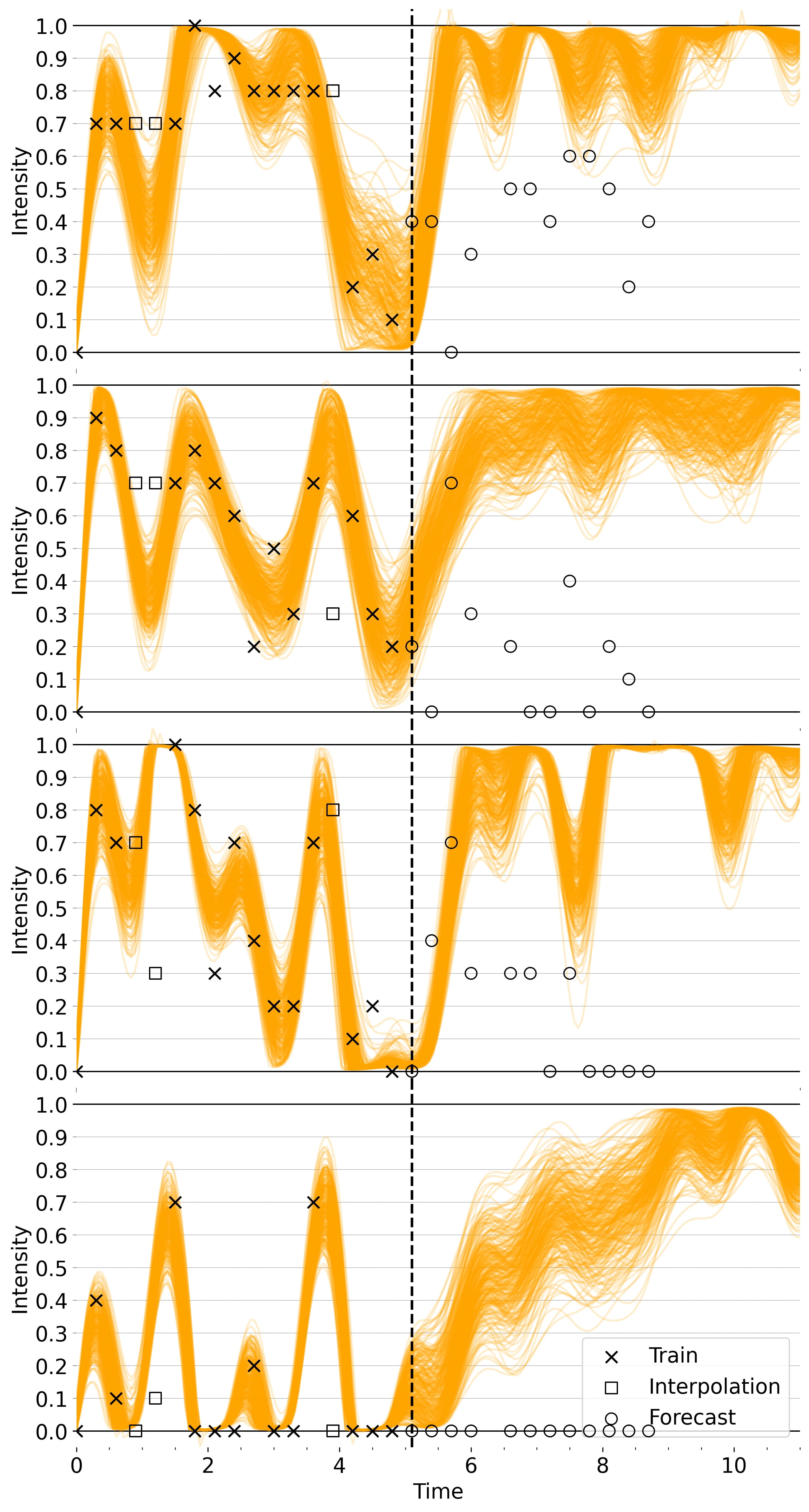}
    ~
    \includegraphics[width=0.38\textwidth]{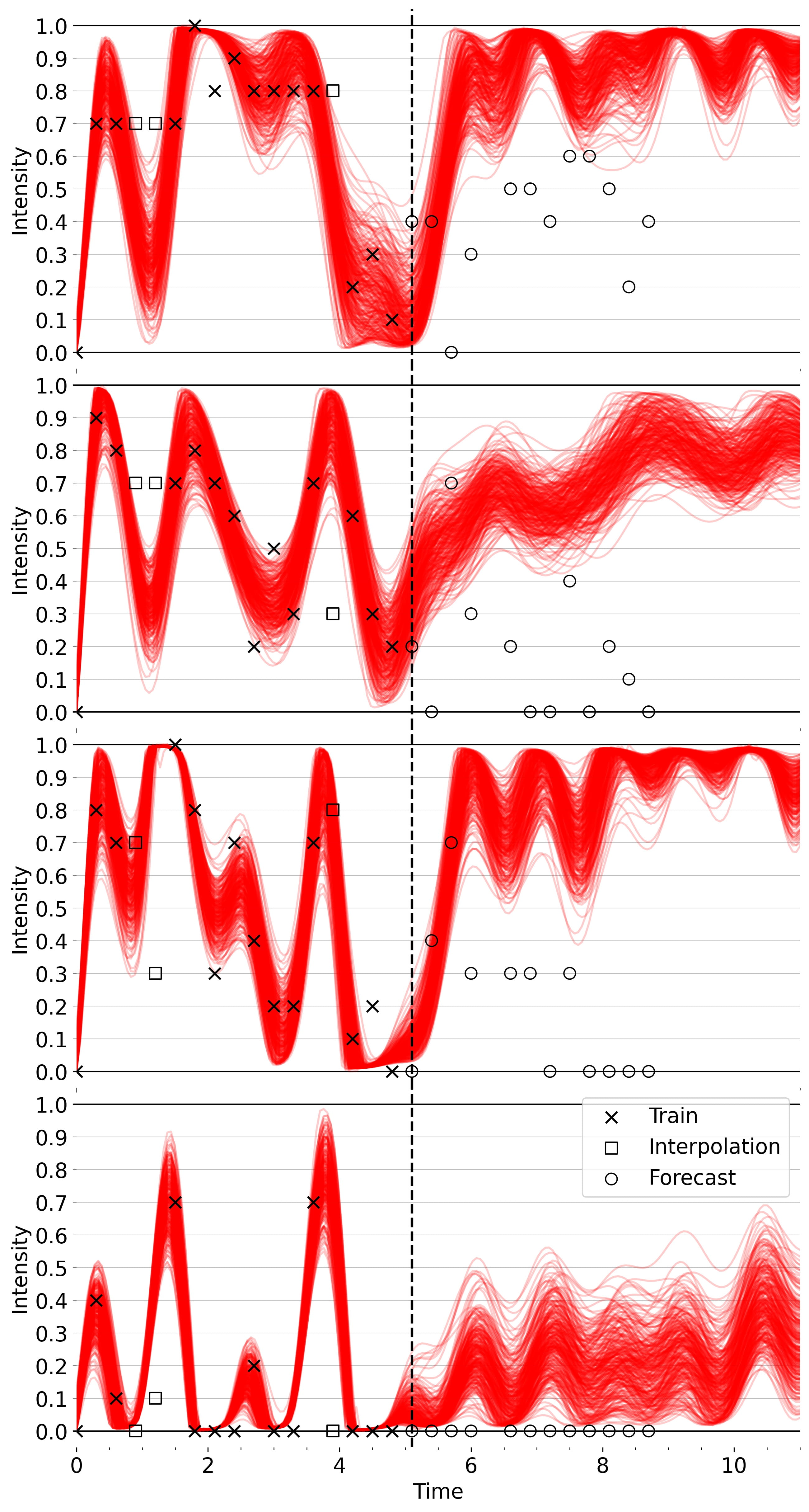}
    
    \caption{\textbf{Qualitative Comparison of Posterior for Specific Patient from \textcolor{colorVanillaStrong}{Vanilla} (top left), \textcolor{colorVanillaClipStrong}{Vanilla+Clip} (top right), \textcolor{colorWSPNoClipStrong}{WSP} (bottom left), and \textcolor{colorWSPStrong}{WSP+Clip} (bottom right).} Each row represents a different EMA survey item, listed in \cref{apx:real-data}. See discussion in \cref{apx:viz-qualitative}.}
    \label{fig:wsp-qualitative-114}
\end{figure*}

\begin{figure*}[p]
    \centering

    \includegraphics[width=0.38\textwidth]{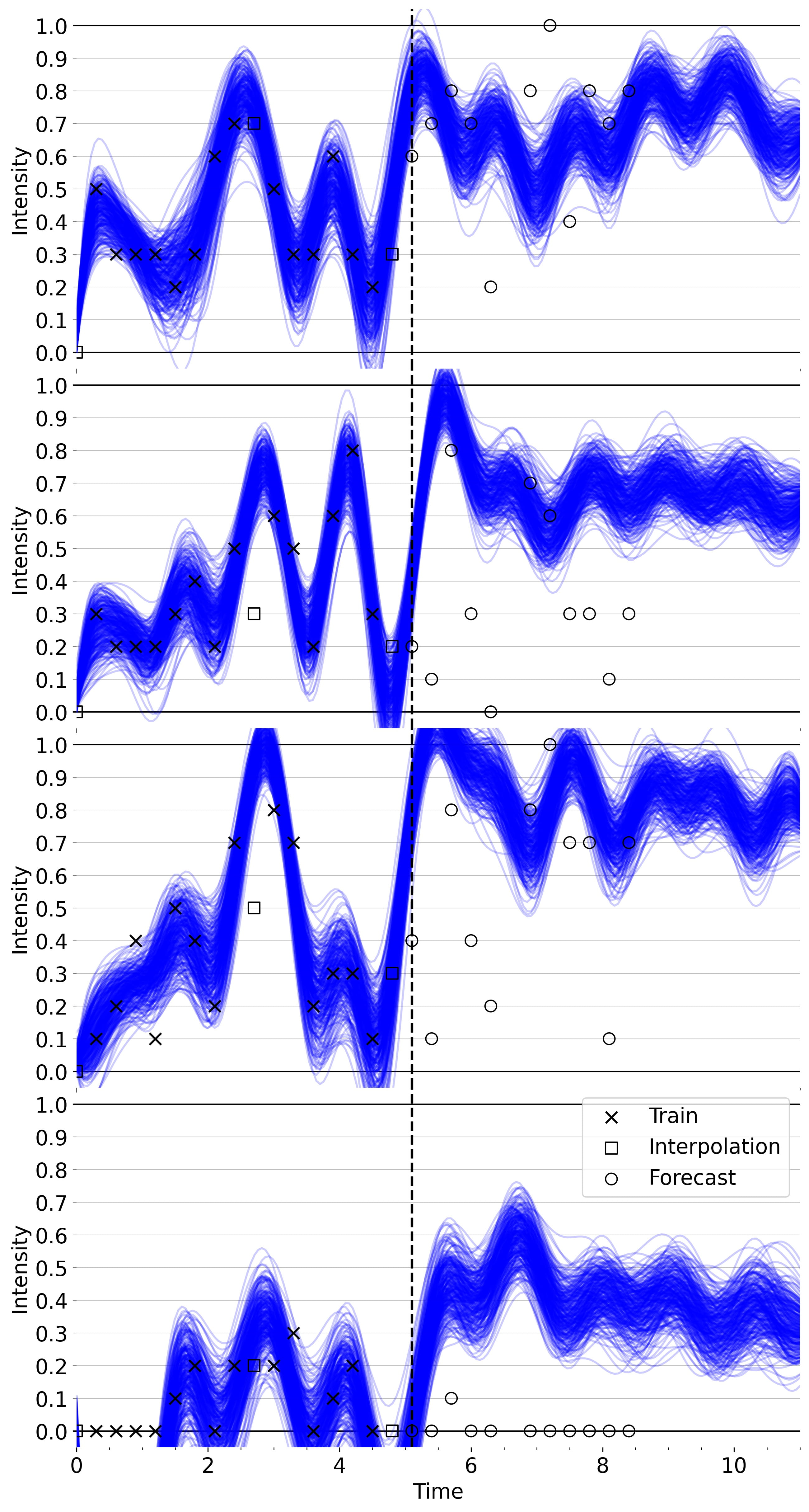}
    ~
    \includegraphics[width=0.38\textwidth]{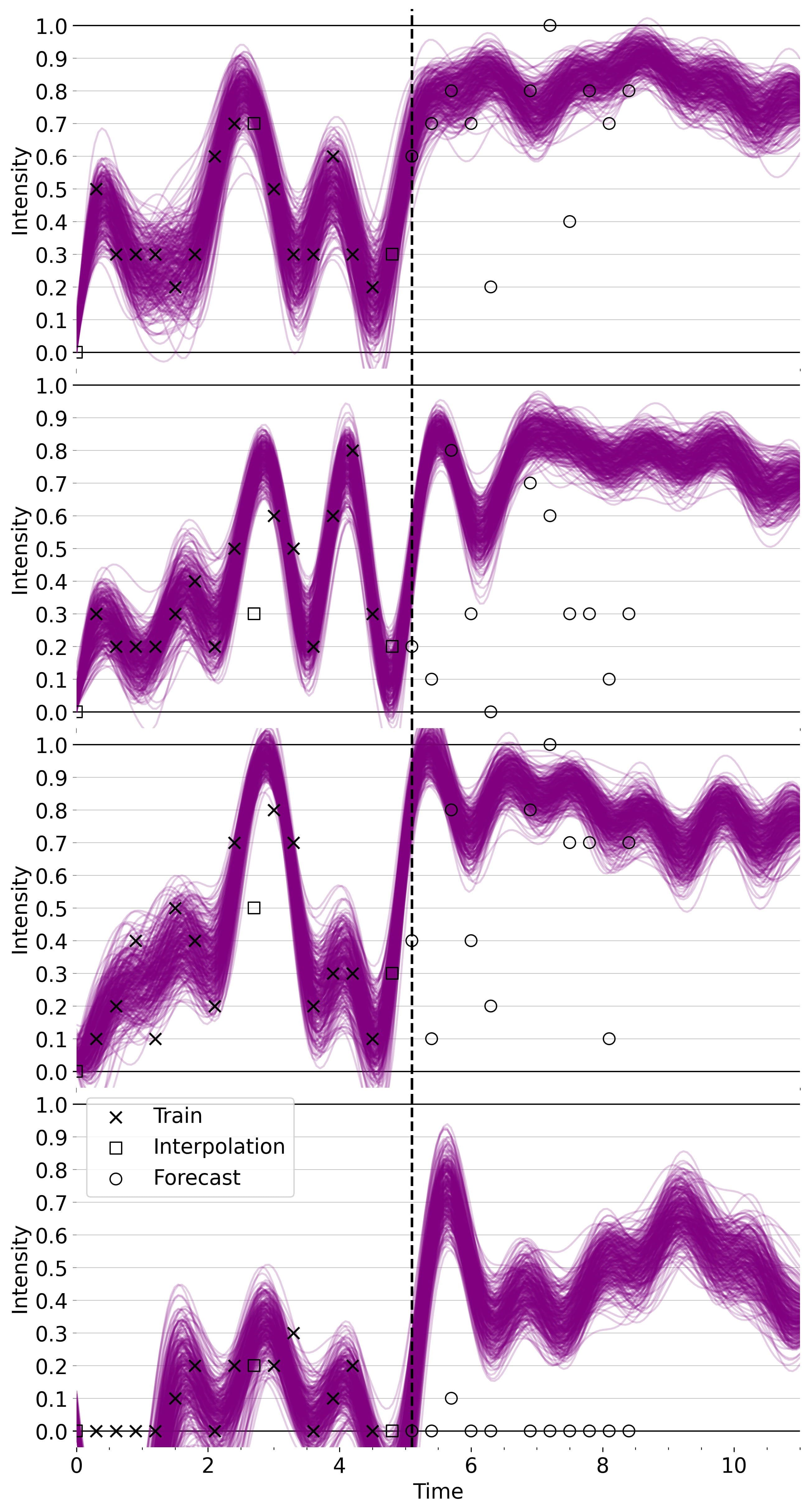}

    \includegraphics[width=0.38\textwidth]{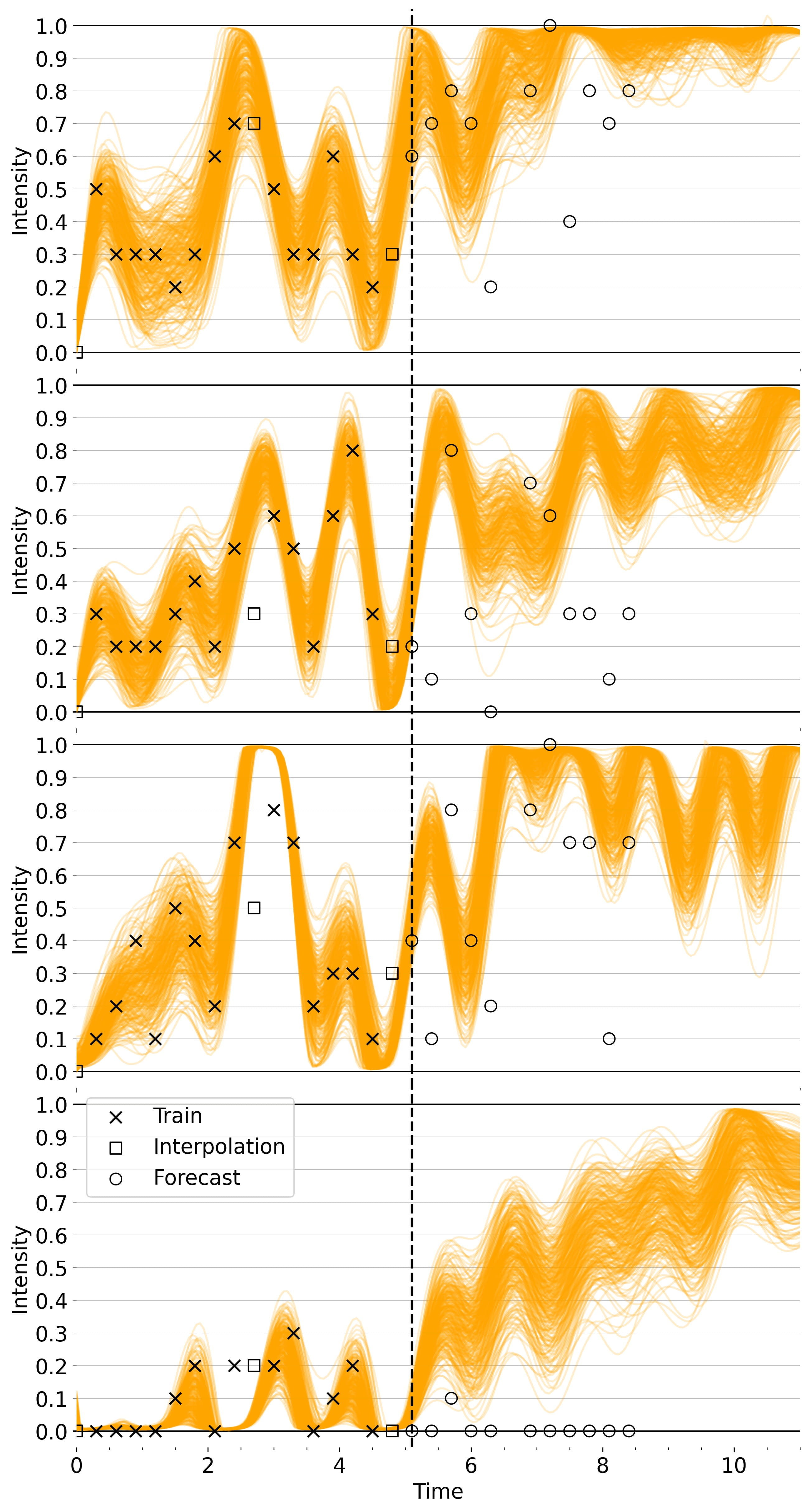}
    ~
    \includegraphics[width=0.38\textwidth]{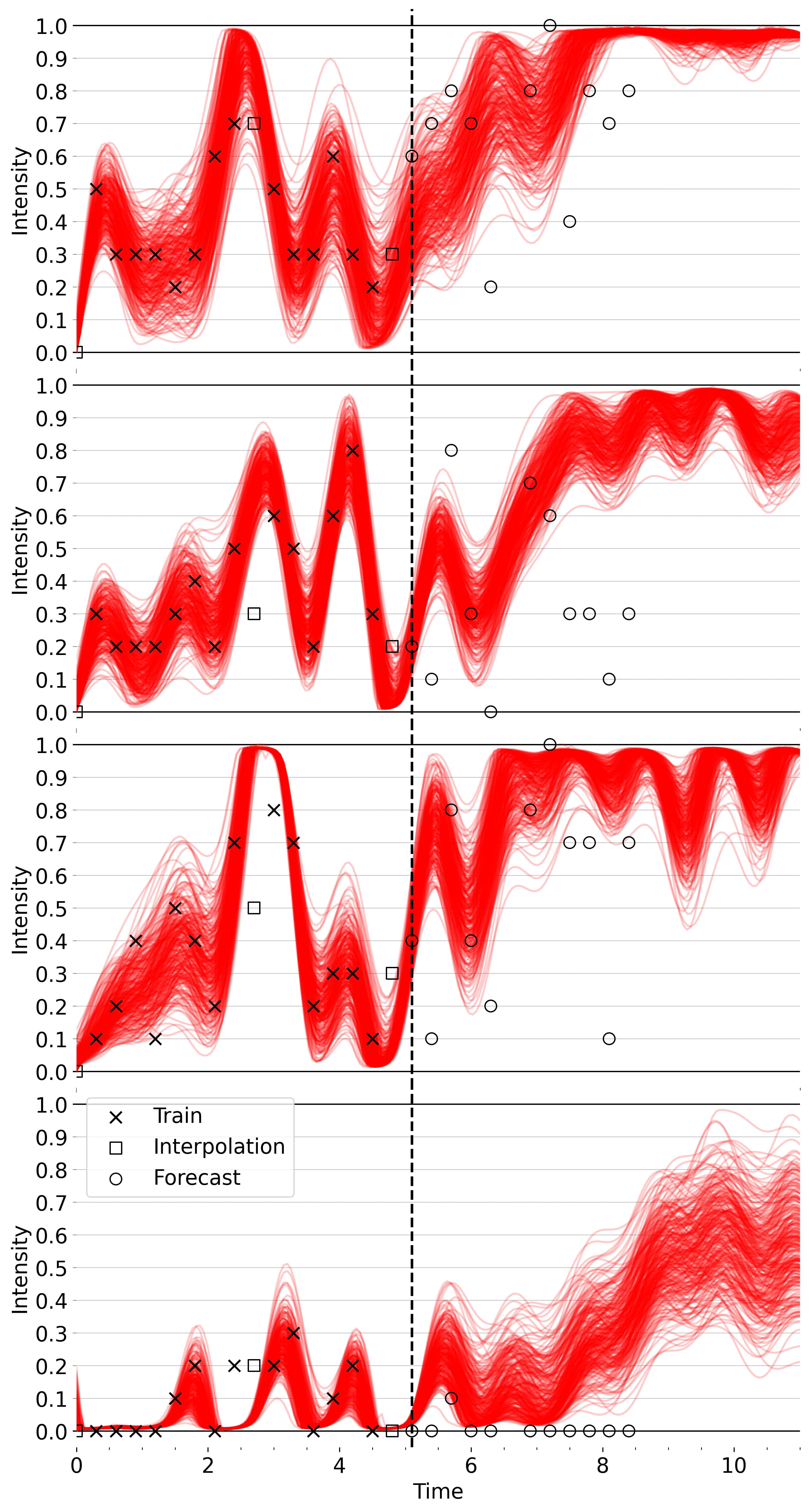}
    
    \caption{\textbf{Qualitative Comparison of Posterior for Specific Patient from \textcolor{colorVanillaStrong}{Vanilla} (top left), \textcolor{colorVanillaClipStrong}{Vanilla+Clip} (top right), \textcolor{colorWSPNoClipStrong}{WSP} (bottom left), and \textcolor{colorWSPStrong}{WSP+Clip} (bottom right).} Each row represents a different EMA survey item, listed in \cref{apx:real-data}. See discussion in \cref{apx:viz-qualitative}.}
    \label{fig:wsp-qualitative-90}
\end{figure*}

\begin{figure*}[p]
    \centering

    \includegraphics[width=0.38\textwidth]{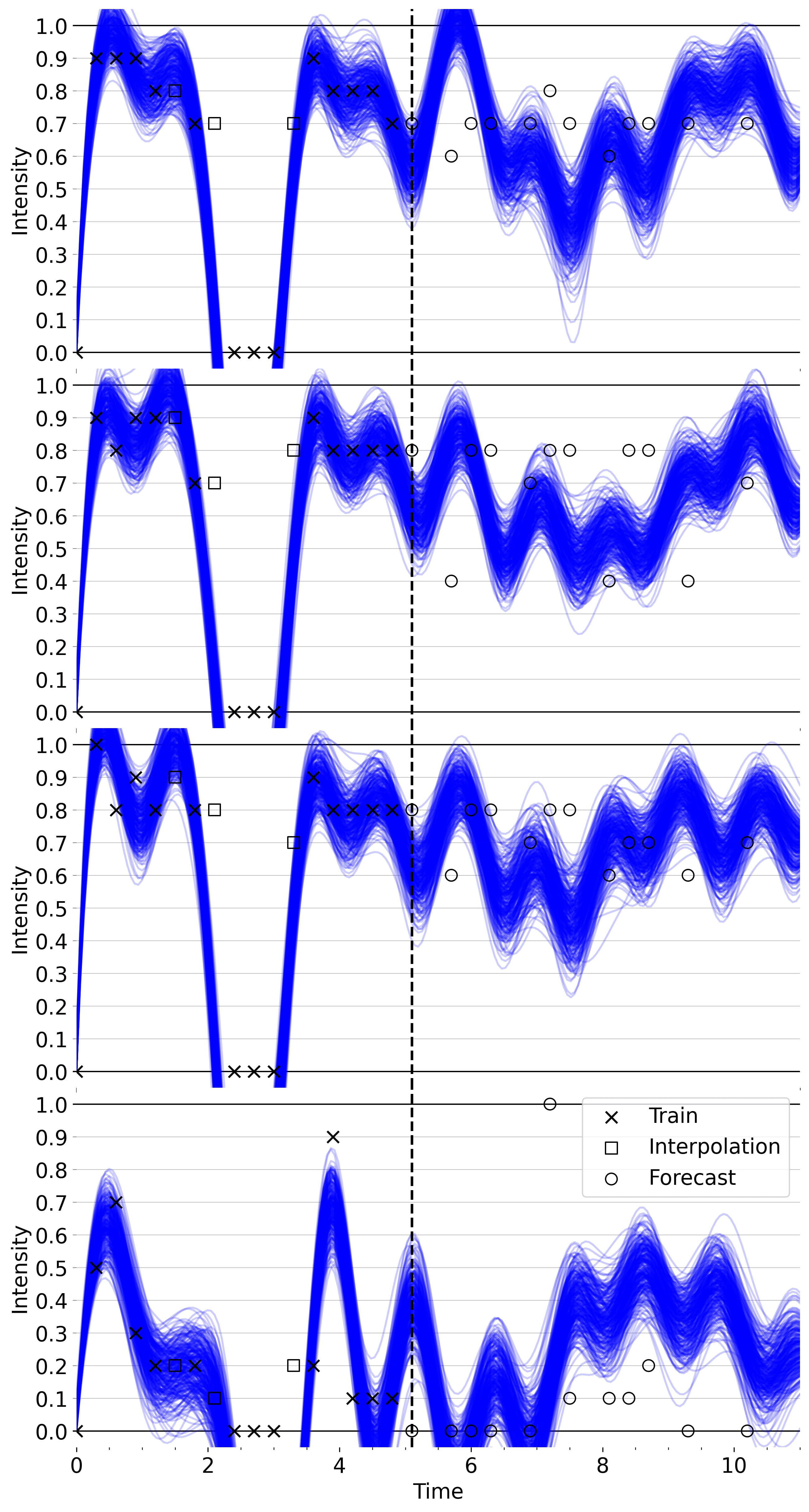}
    ~
    \includegraphics[width=0.38\textwidth]{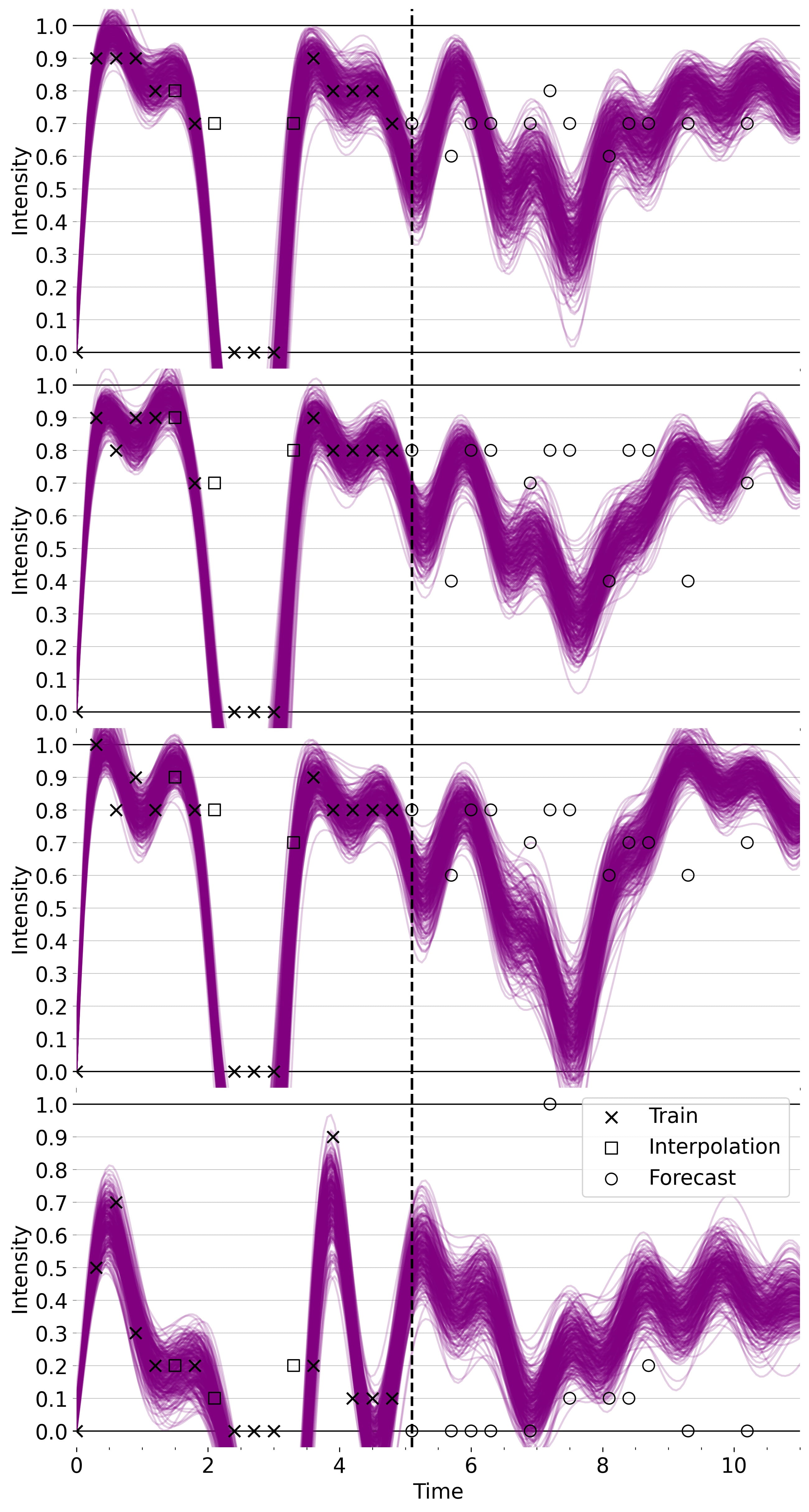}

    \includegraphics[width=0.38\textwidth]{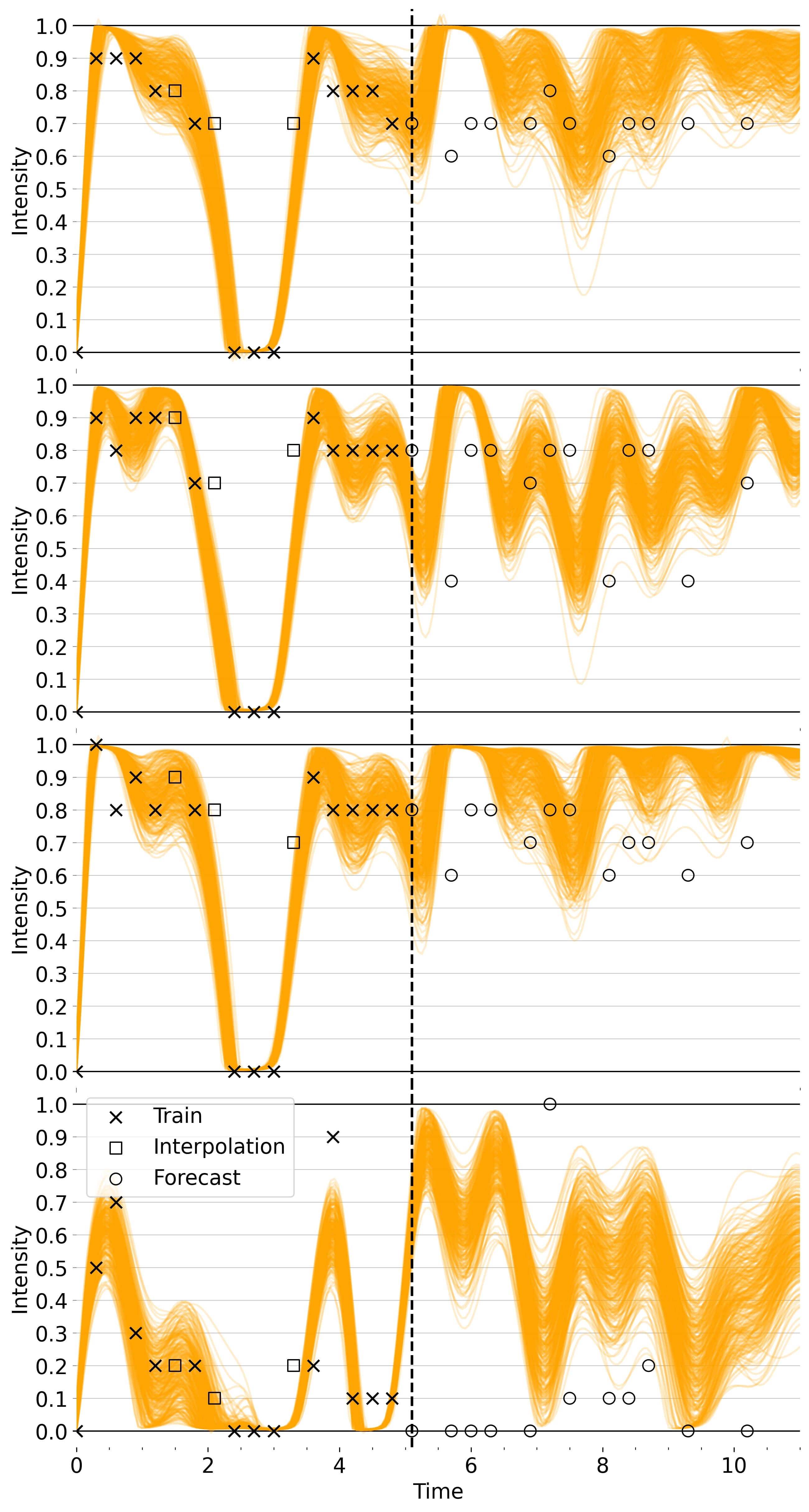}
    ~
    \includegraphics[width=0.38\textwidth]{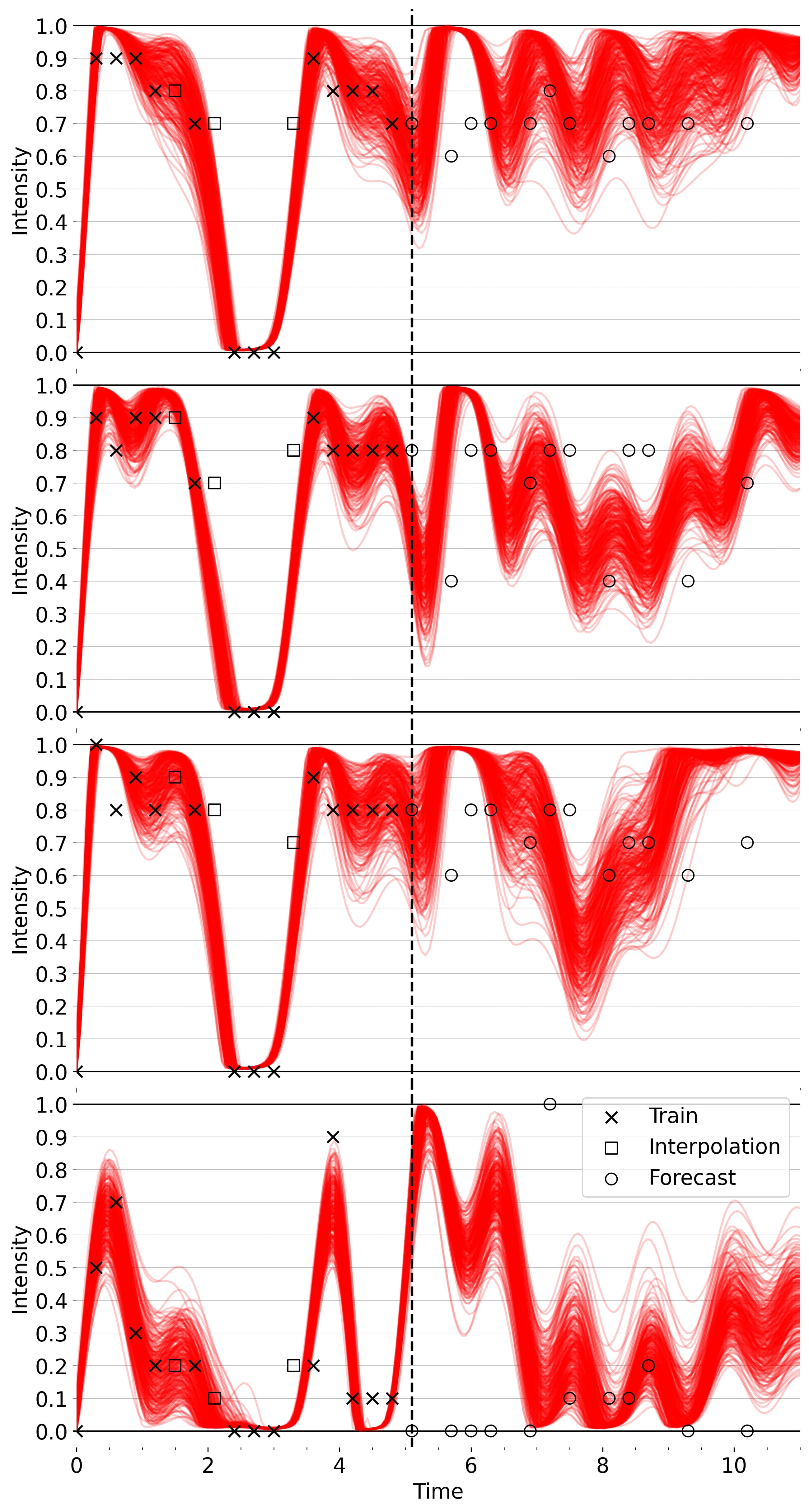}
    
    \caption{\textbf{Qualitative Comparison of Posterior for Specific Patient from \textcolor{colorVanillaStrong}{Vanilla} (top left), \textcolor{colorVanillaClipStrong}{Vanilla+Clip} (top right), \textcolor{colorWSPNoClipStrong}{WSP} (bottom left), and \textcolor{colorWSPStrong}{WSP+Clip} (bottom right).} Each row represents a different EMA survey item, listed in \cref{apx:real-data}. See discussion in \cref{apx:viz-qualitative}.}
    \label{fig:wsp-qualitative-143}
\end{figure*}

\begin{figure*}[p]
    \centering

    \includegraphics[width=0.38\textwidth]{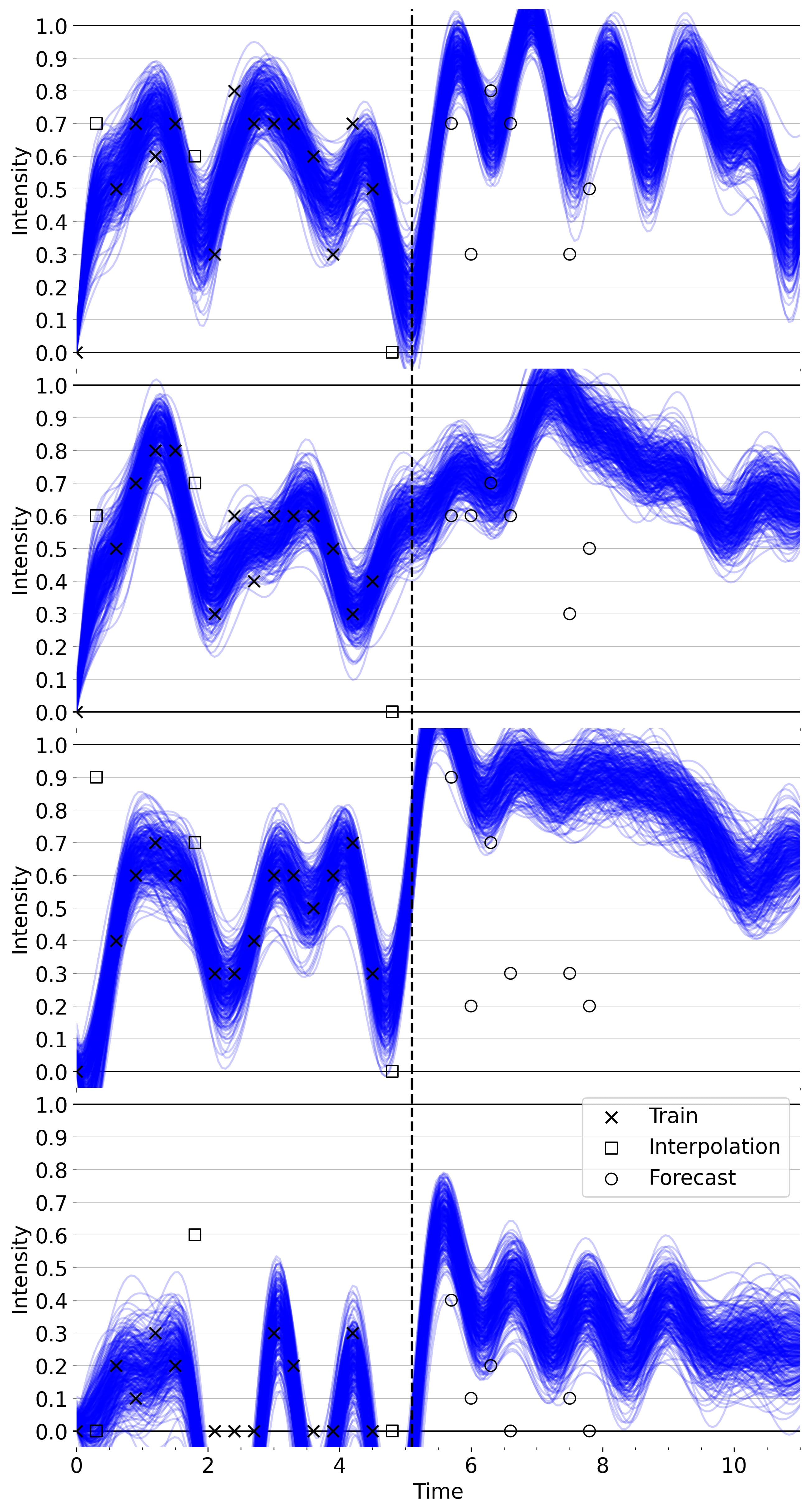}
    ~
    \includegraphics[width=0.38\textwidth]{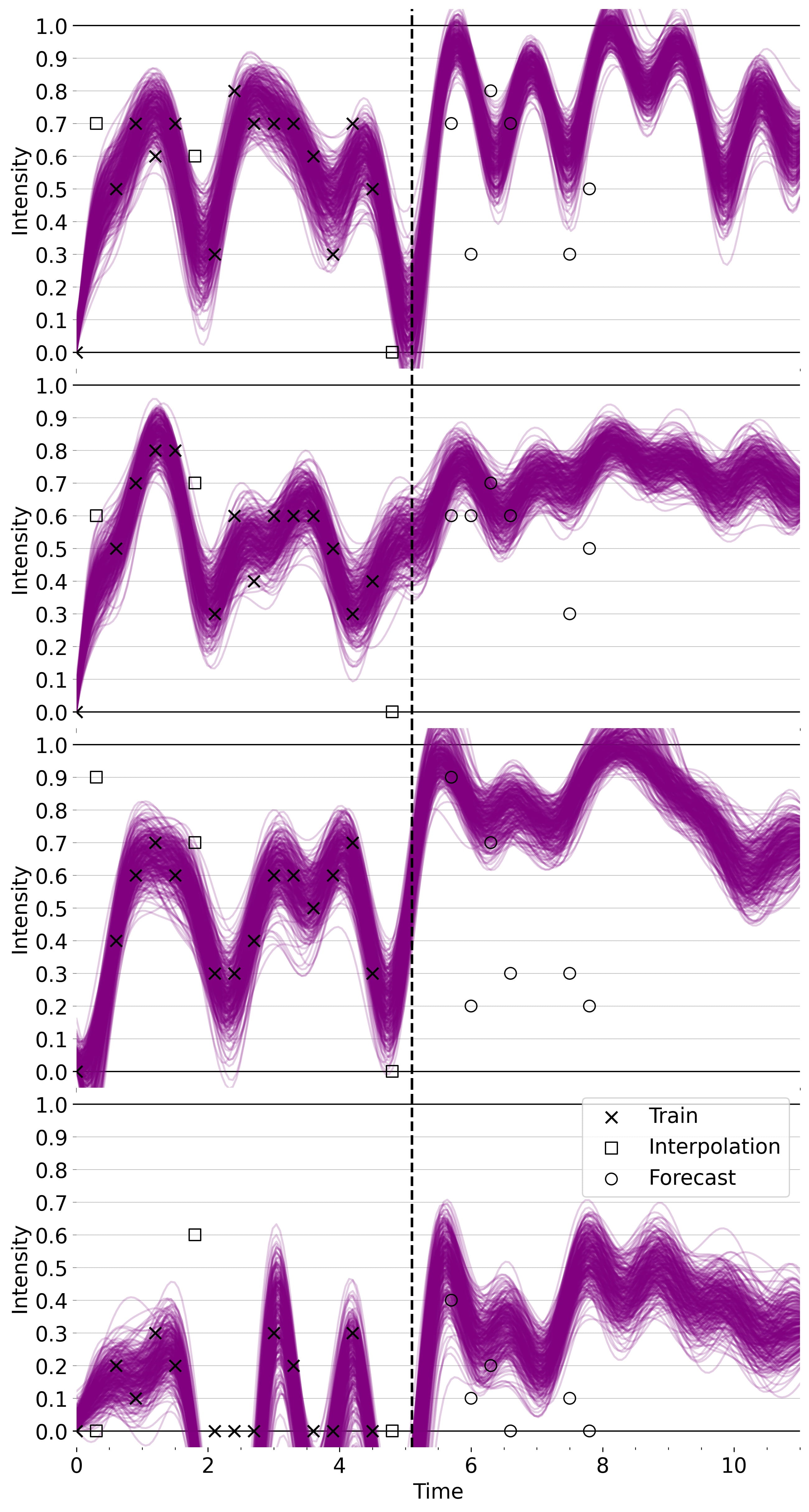}

    \includegraphics[width=0.38\textwidth]{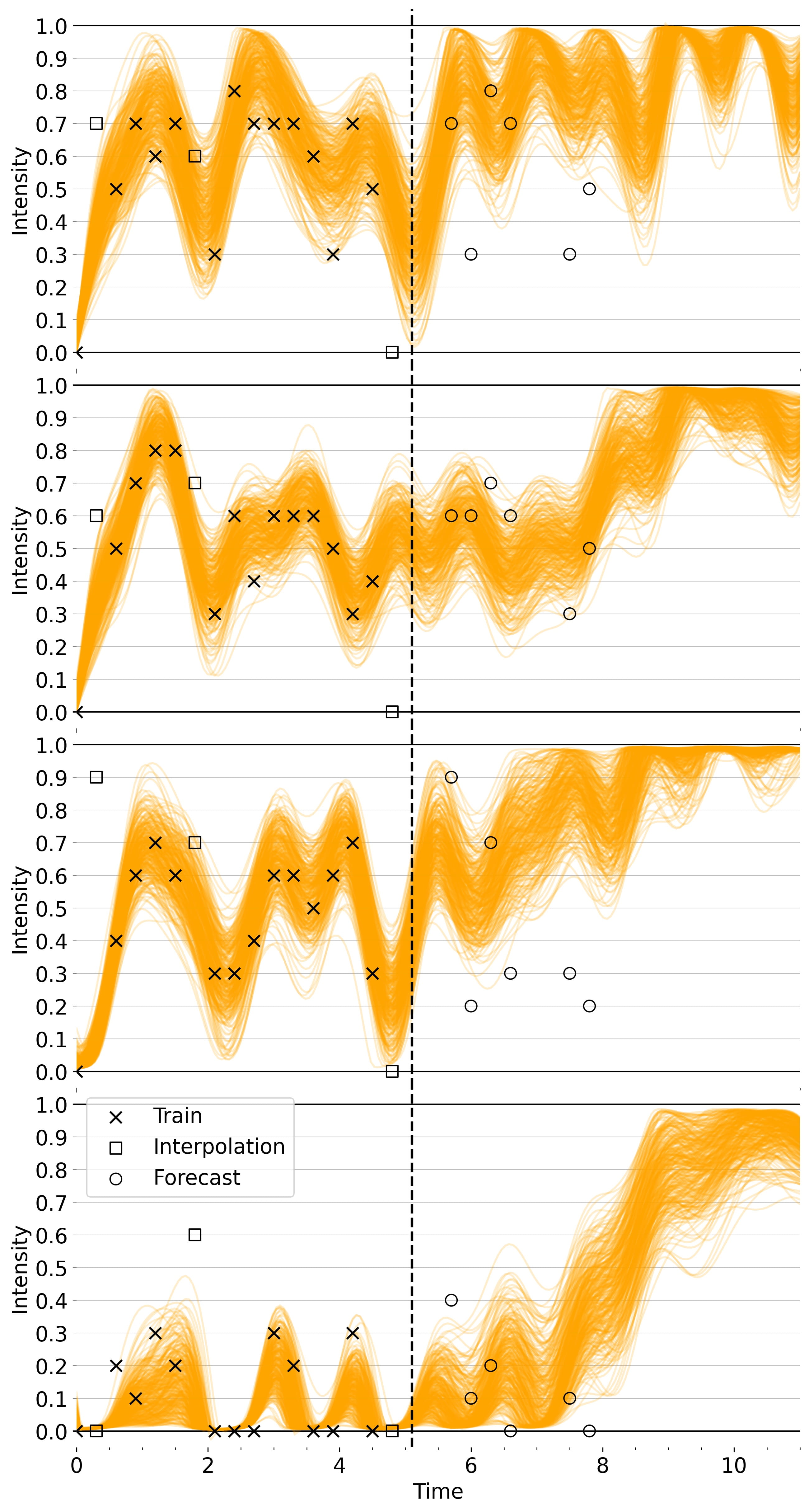}
    ~
    \includegraphics[width=0.38\textwidth]{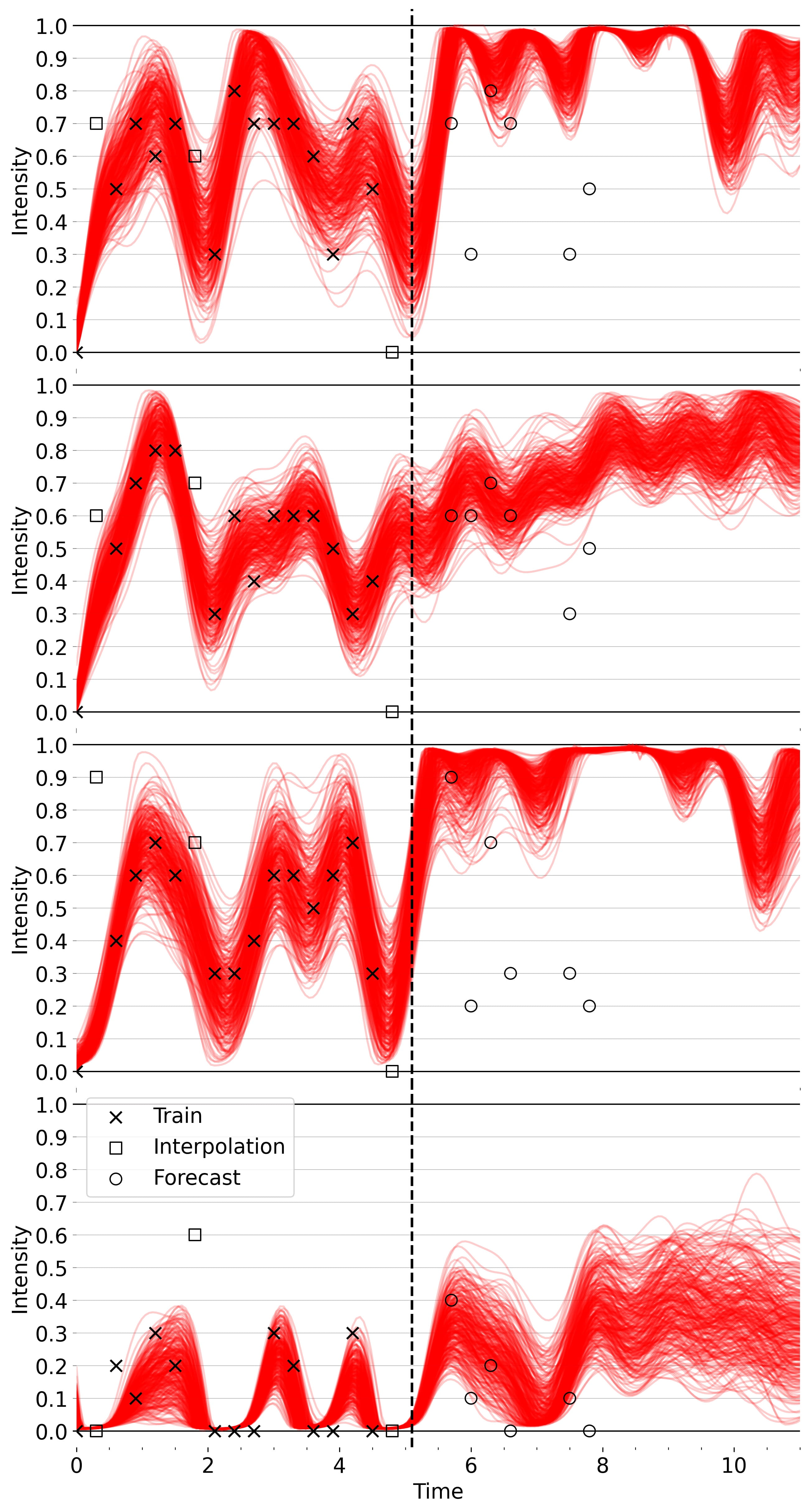}
    
    \caption{\textbf{Qualitative Comparison of Posterior for Specific Patient from \textcolor{colorVanillaStrong}{Vanilla} (top left), \textcolor{colorVanillaClipStrong}{Vanilla+Clip} (top right), \textcolor{colorWSPNoClipStrong}{WSP} (bottom left), and \textcolor{colorWSPStrong}{WSP+Clip} (bottom right).} Each row represents a different EMA survey item, listed in \cref{apx:real-data}. See discussion in \cref{apx:viz-qualitative}.}
    \label{fig:wsp-qualitative-144}
\end{figure*}

\begin{figure*}[p]
    \centering

    \includegraphics[width=0.38\textwidth]{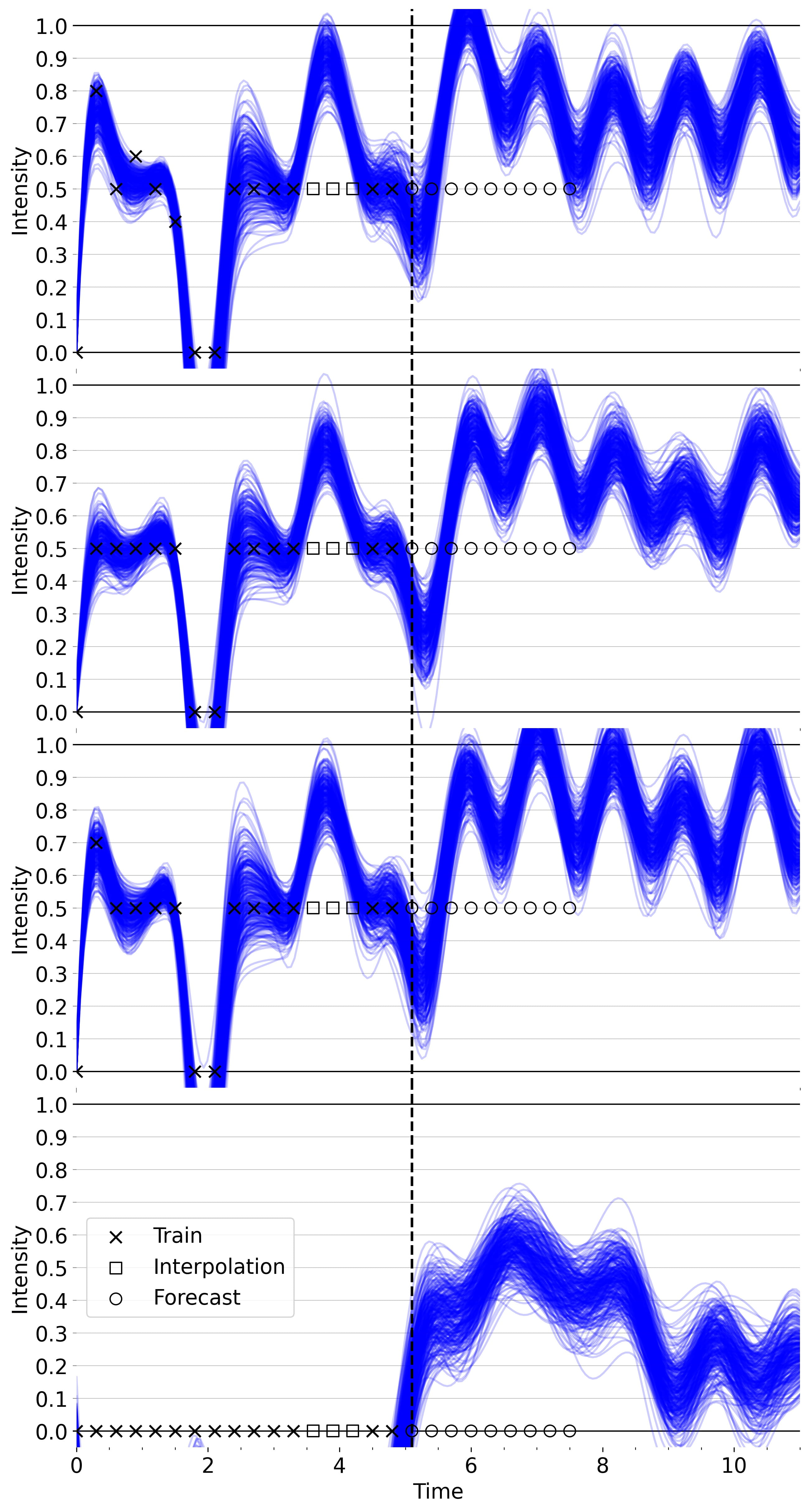}
    ~
    \includegraphics[width=0.38\textwidth]{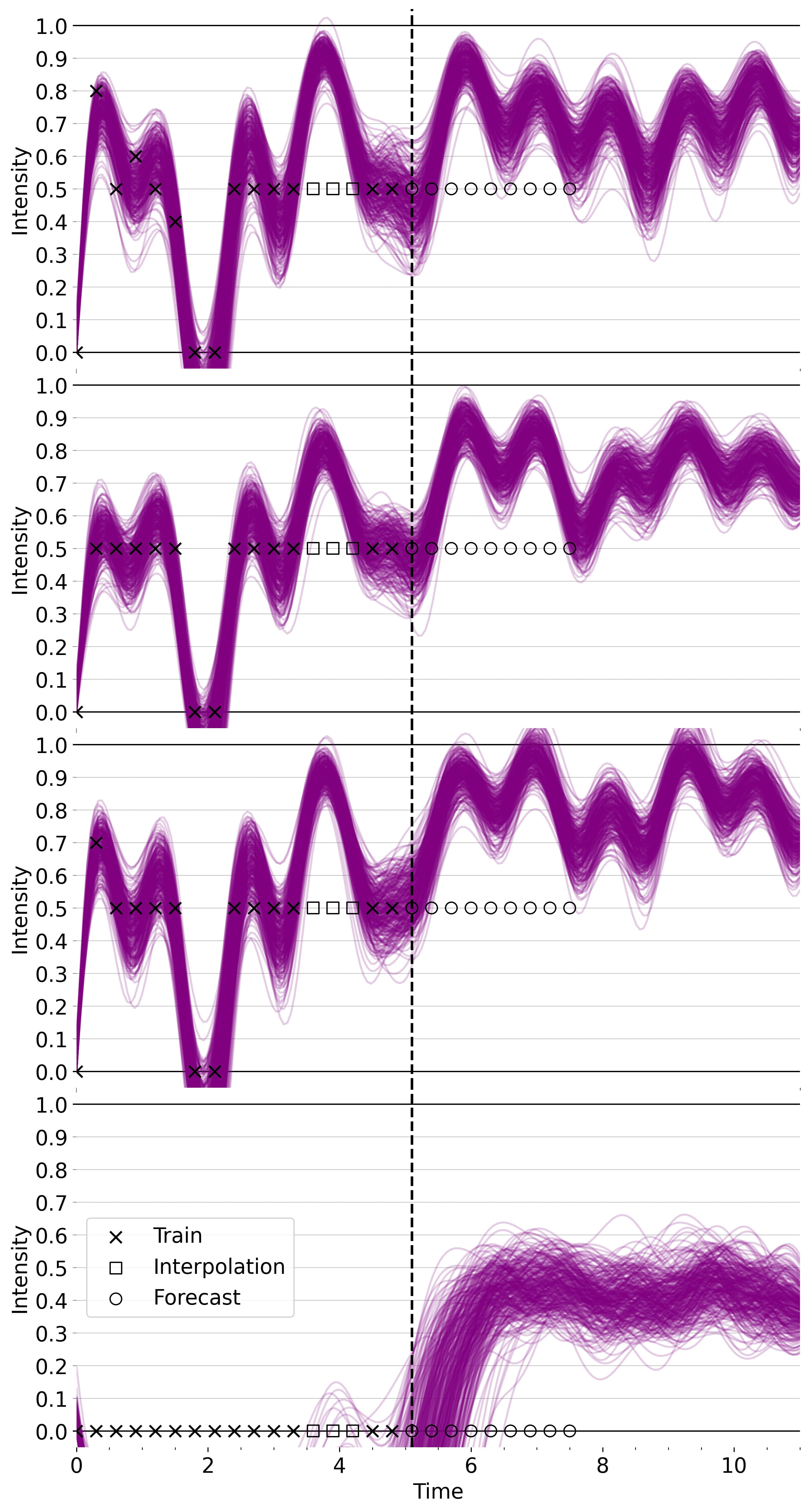}

    \includegraphics[width=0.38\textwidth]{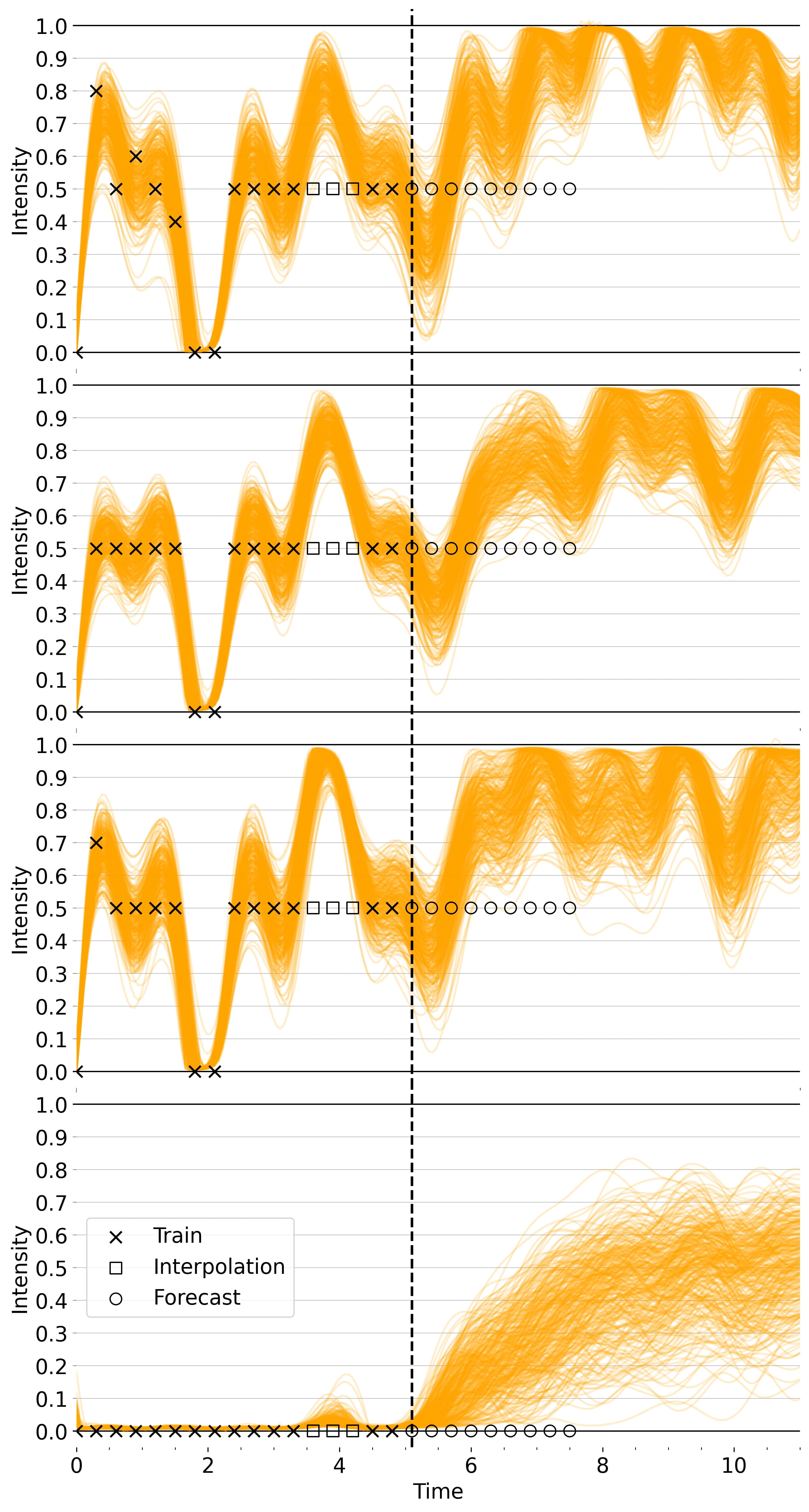}
    ~
    \includegraphics[width=0.38\textwidth]{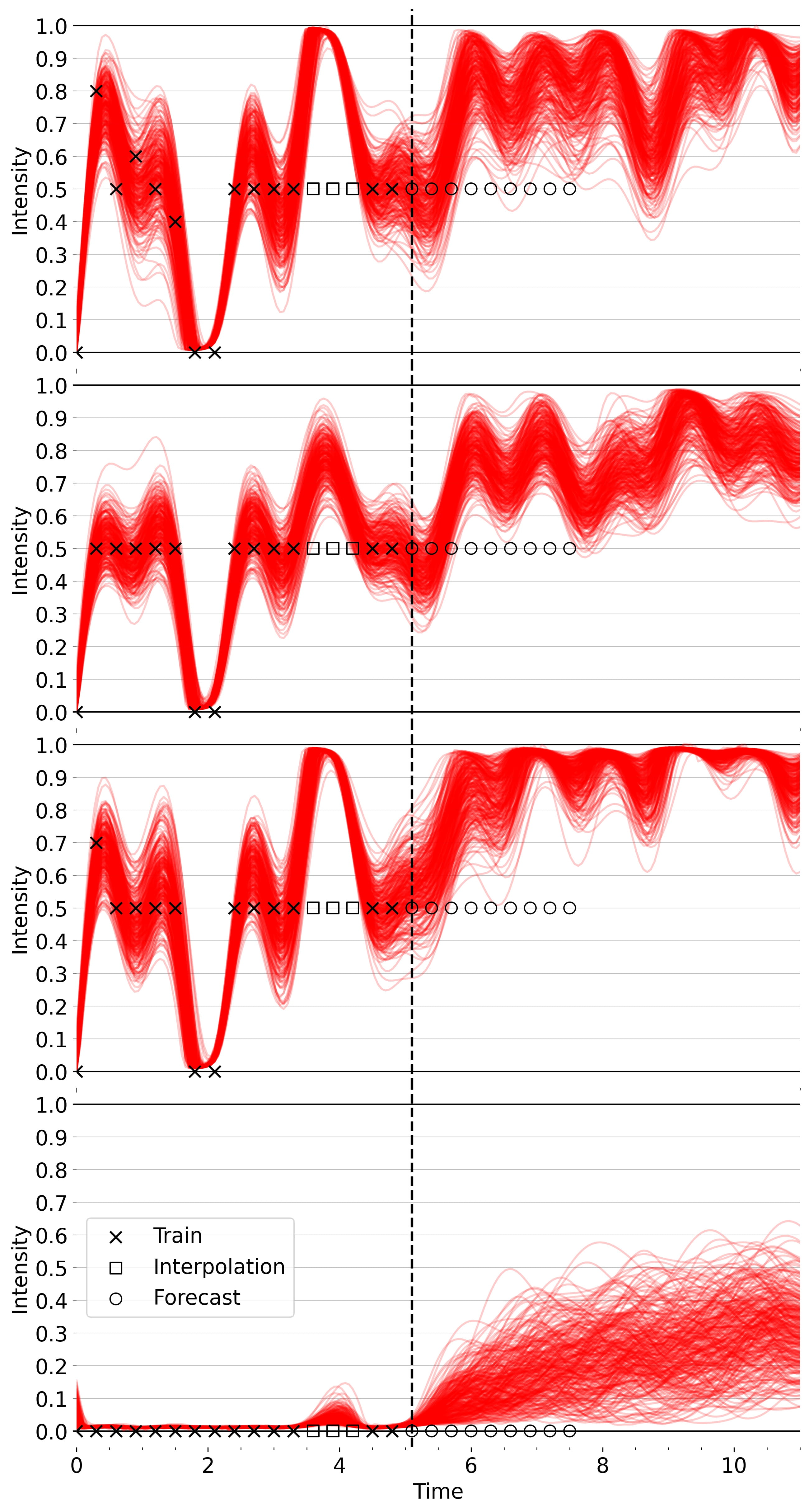}
    
    \caption{\textbf{Qualitative Comparison of Posterior for Specific Patient from \textcolor{colorVanillaStrong}{Vanilla} (top left), \textcolor{colorVanillaClipStrong}{Vanilla+Clip} (top right), \textcolor{colorWSPNoClipStrong}{WSP} (bottom left), and \textcolor{colorWSPStrong}{WSP+Clip} (bottom right).} Each row represents a different EMA survey item, listed in \cref{apx:real-data}. For this particular patient, all models make poor forecasts for the top three dimensions, but WSP does substantially improve forecasting for the fourth dimension (bottom row). See discussion in \cref{apx:viz-qualitative}.}
    \label{fig:wsp-qualitative-56}
\end{figure*}

\begin{figure*}[p]
    \centering

    \includegraphics[width=0.38\textwidth]{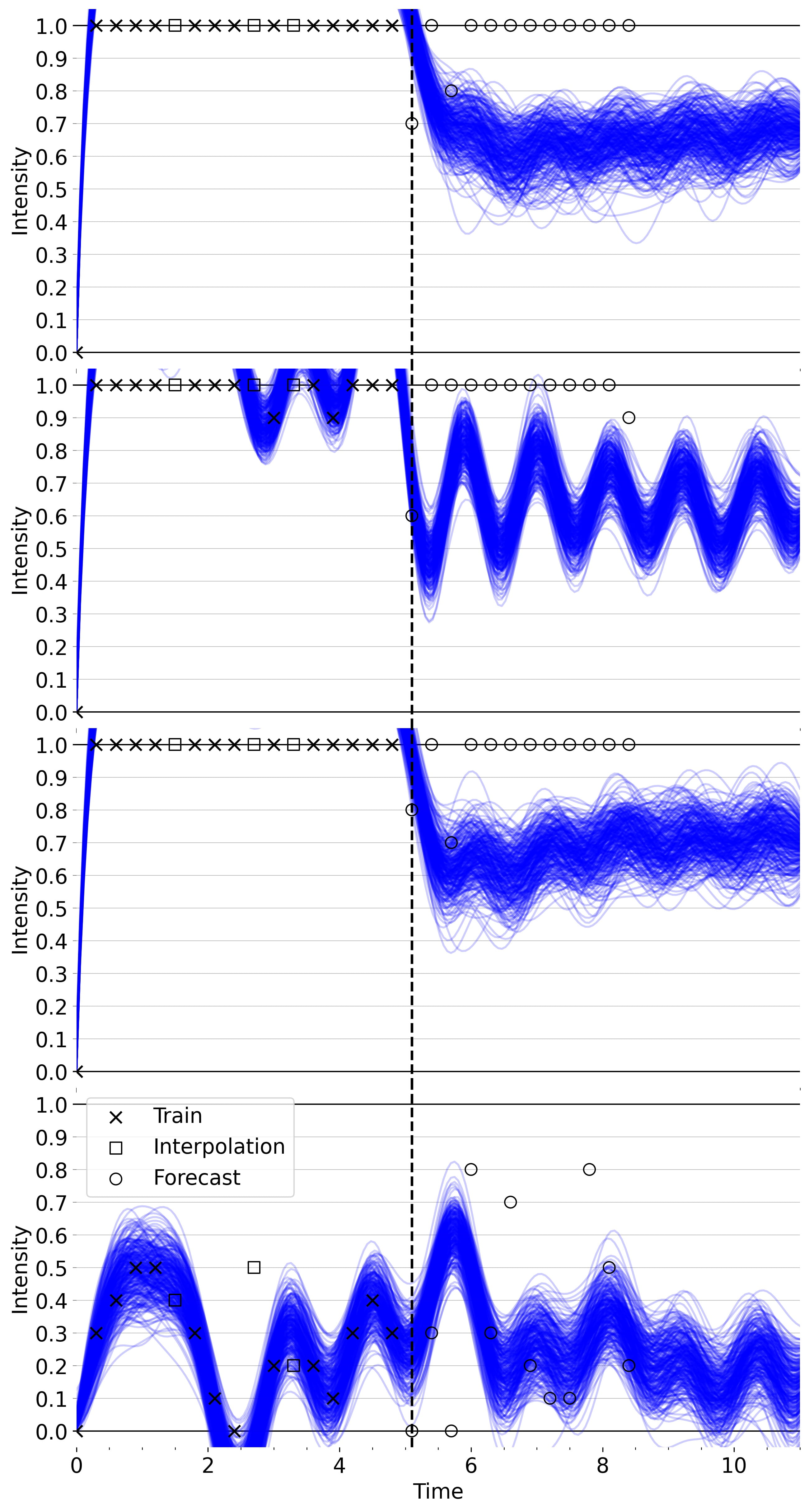}
    ~
    \includegraphics[width=0.38\textwidth]{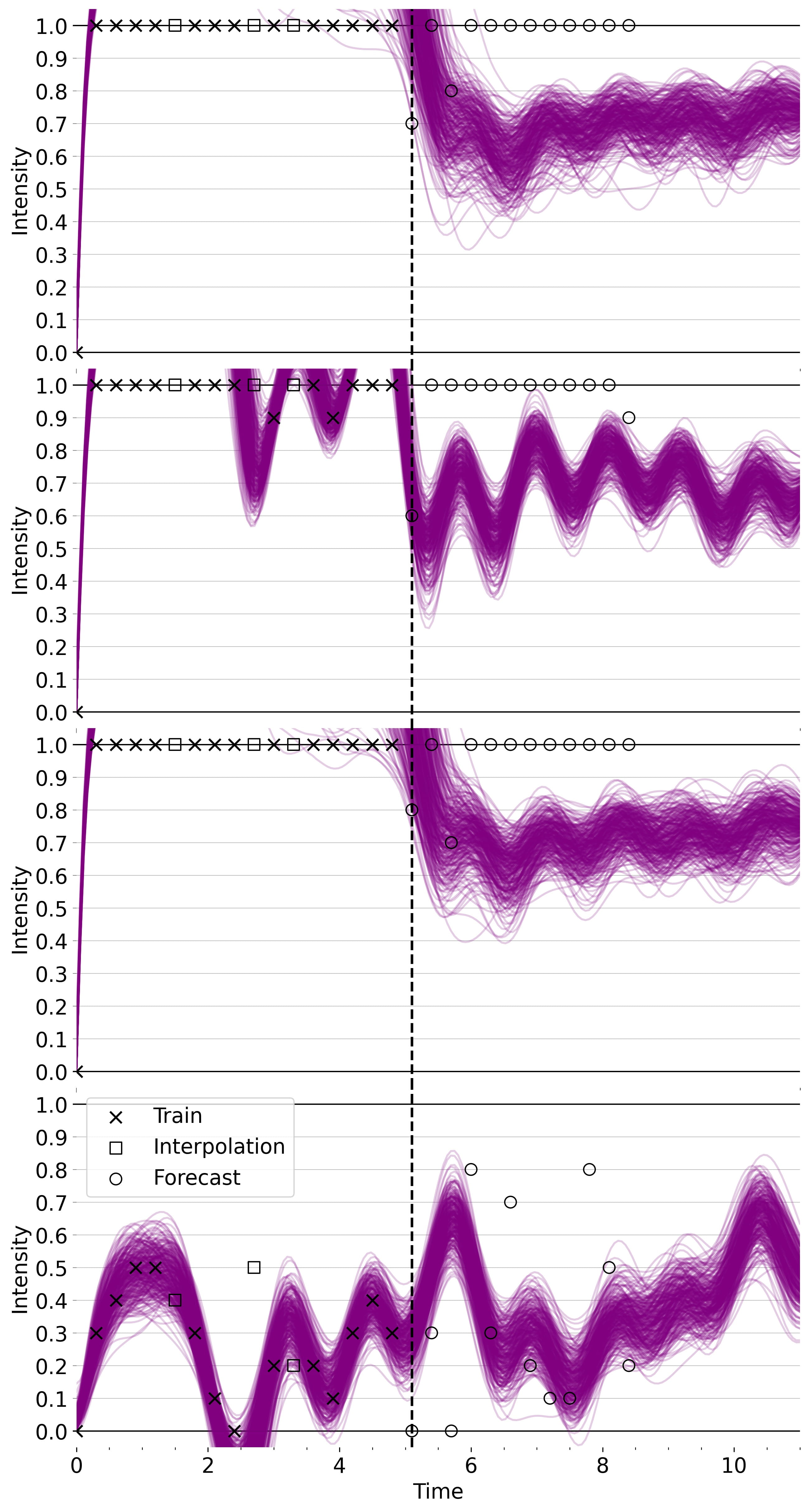}

    \includegraphics[width=0.38\textwidth]{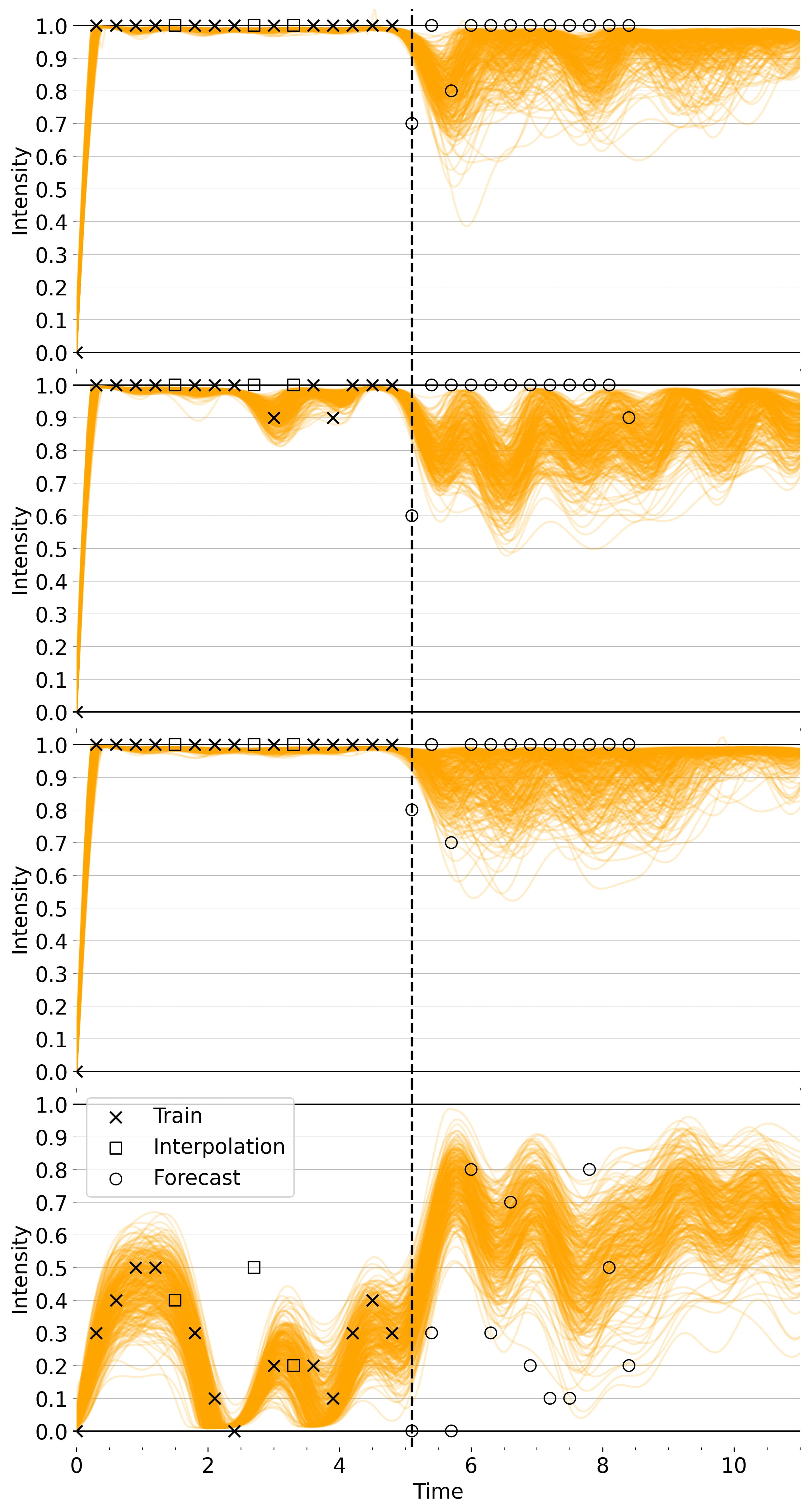}
    ~
    \includegraphics[width=0.38\textwidth]{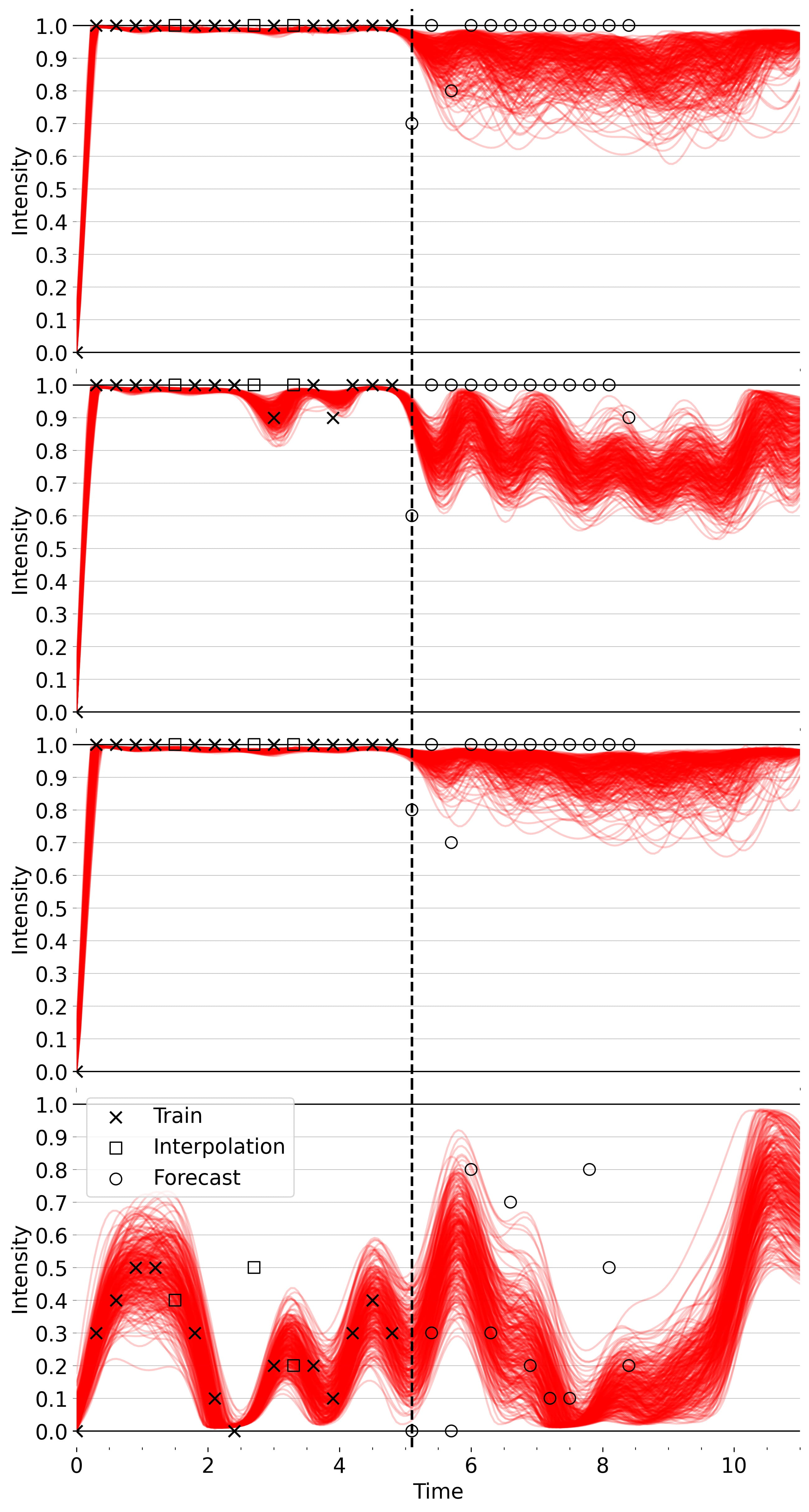}
    
    \caption{\textbf{Qualitative Comparison of Posterior for Specific Patient from \textcolor{colorVanillaStrong}{Vanilla} (top left), \textcolor{colorVanillaClipStrong}{Vanilla+Clip} (top right), \textcolor{colorWSPNoClipStrong}{WSP} (bottom left), and \textcolor{colorWSPStrong}{WSP+Clip} (bottom right).} Each row represents a different EMA survey item, listed in \cref{apx:real-data}. See discussion in \cref{apx:viz-qualitative}.}
    \label{fig:wsp-qualitative-149}
\end{figure*}

\begin{figure*}[p]
    \centering

    \includegraphics[width=0.38\textwidth]{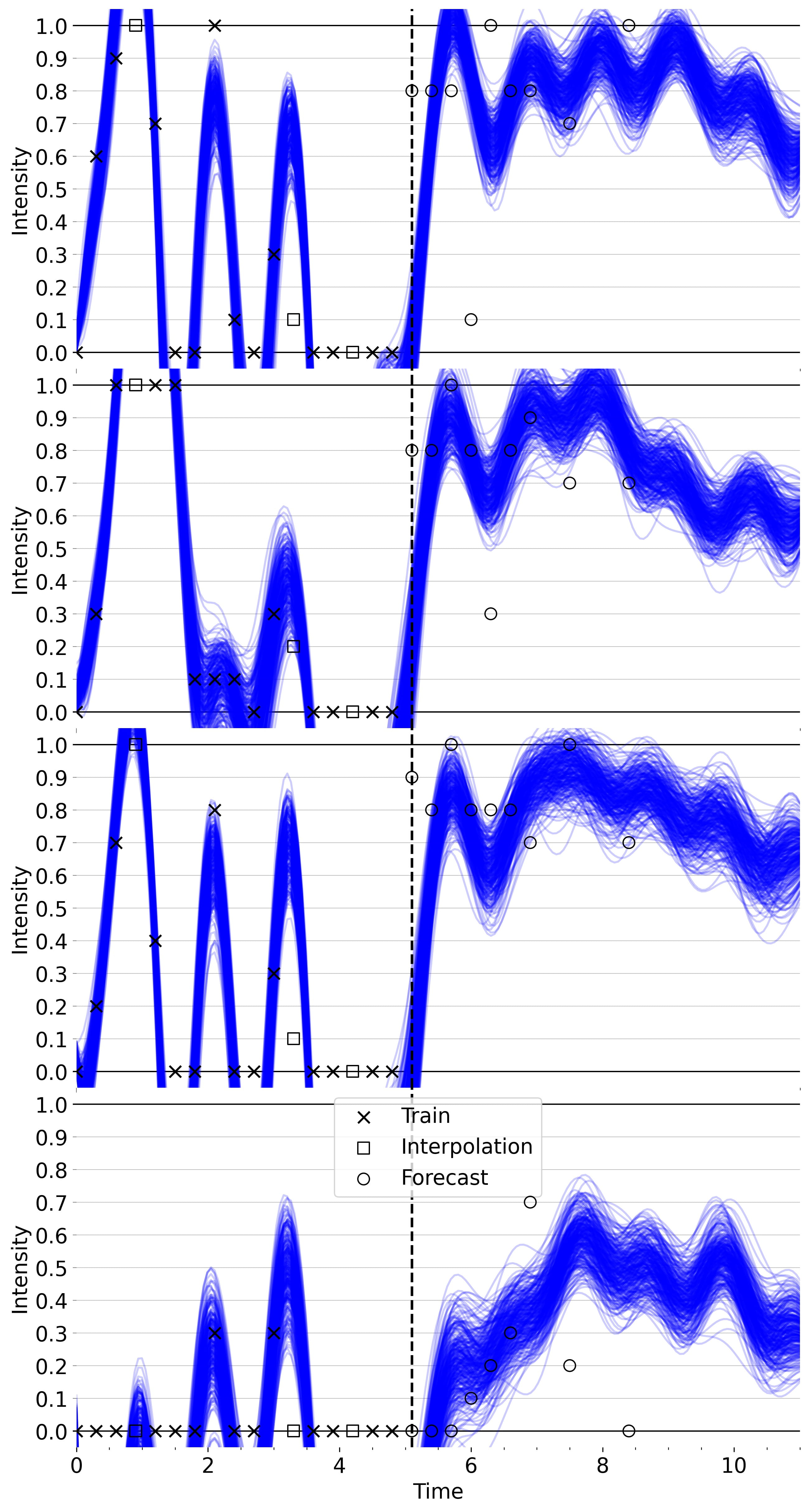}
    ~
    \includegraphics[width=0.38\textwidth]{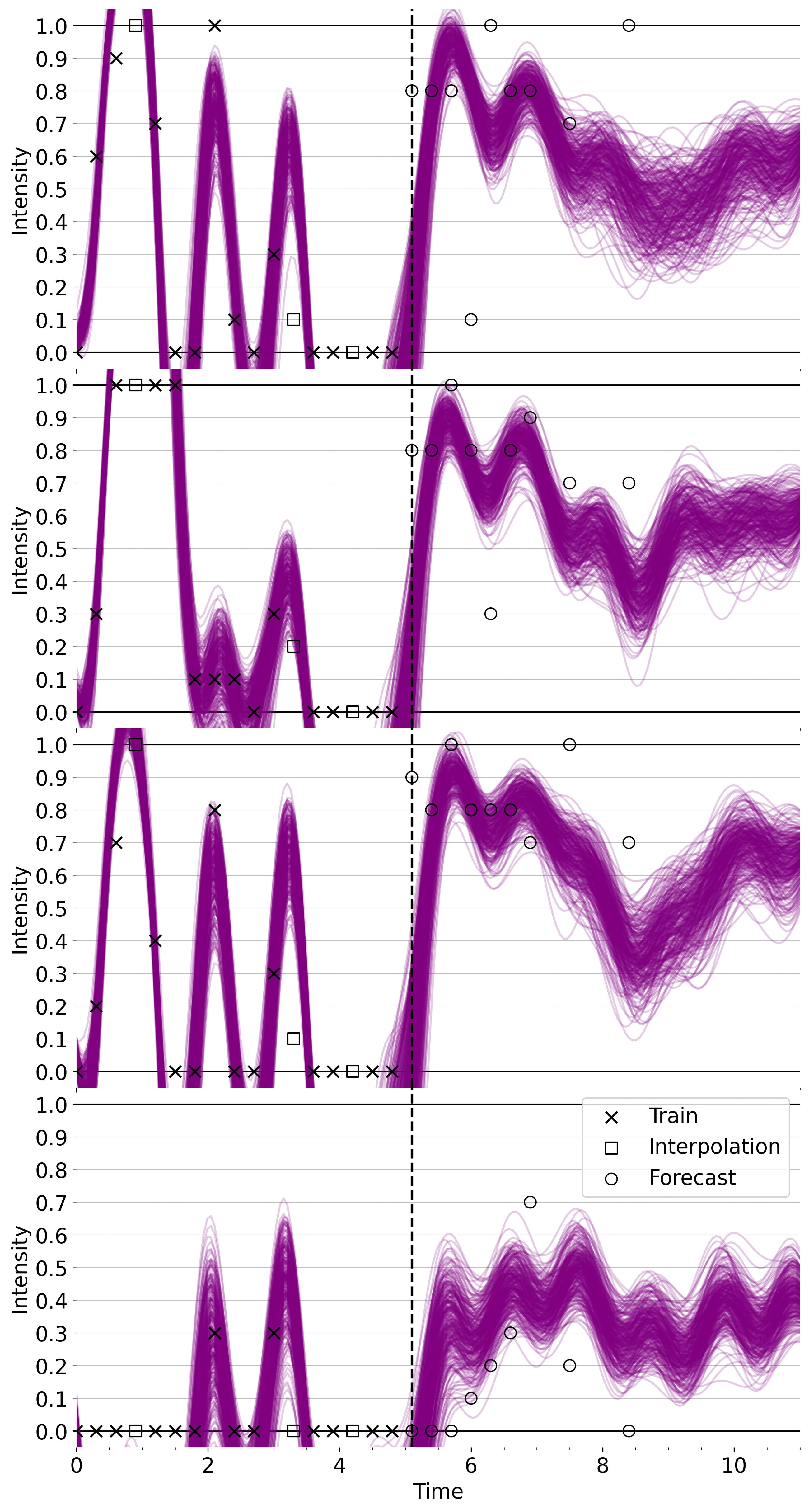}

    \includegraphics[width=0.38\textwidth]{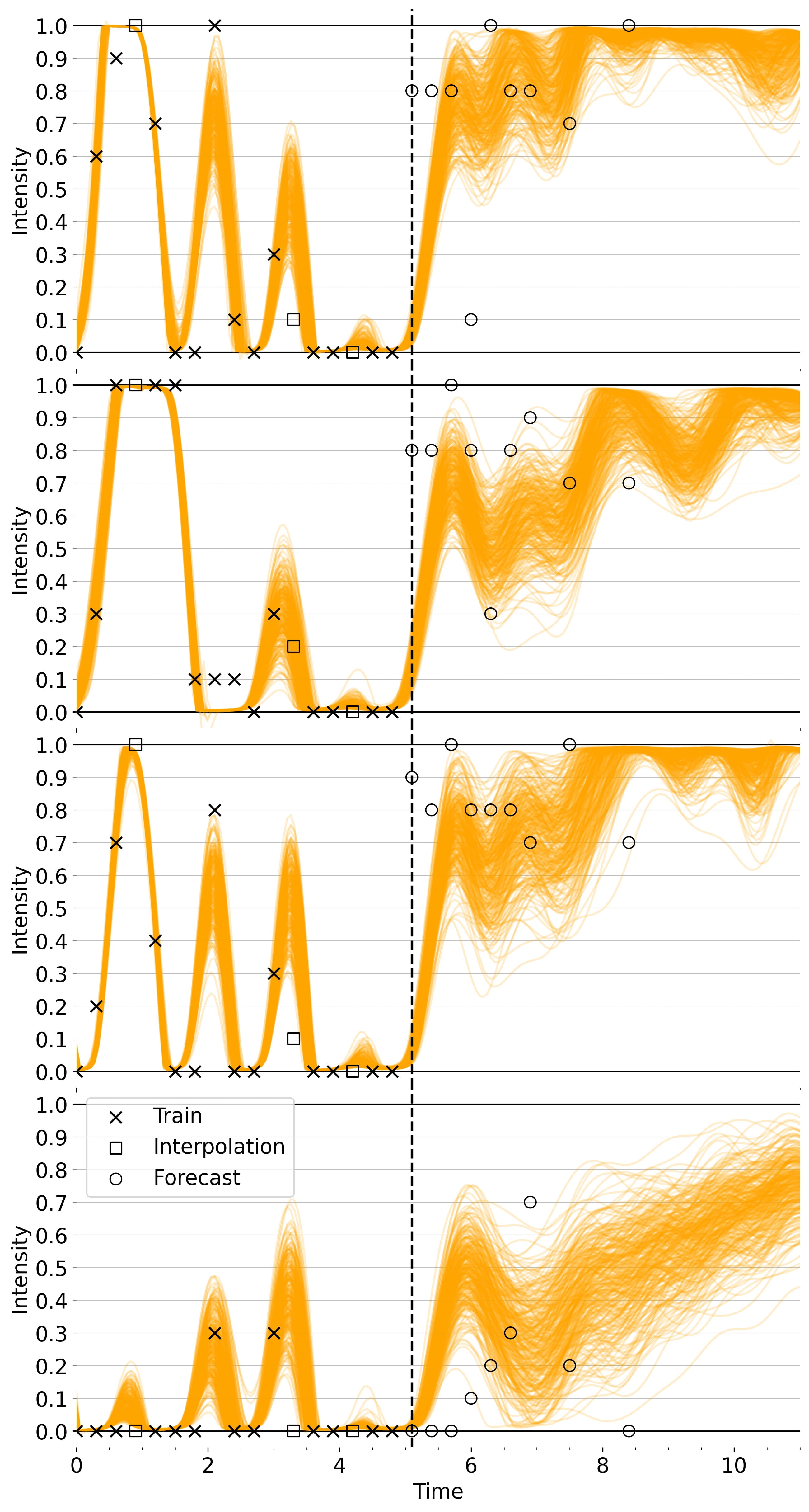}
    ~
    \includegraphics[width=0.38\textwidth]{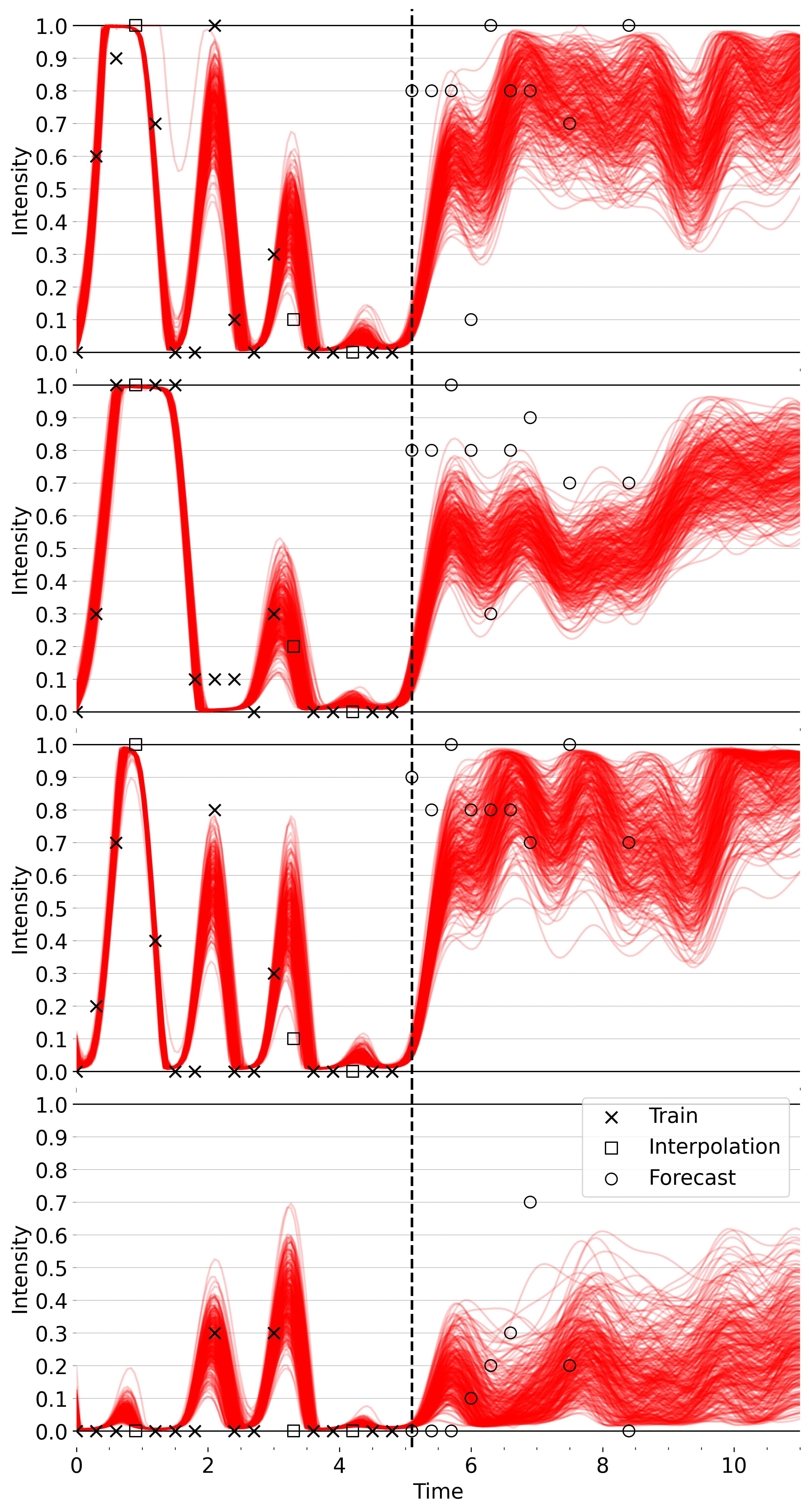}
    
    \caption{\textbf{Qualitative Comparison of Posterior for Specific Patient from \textcolor{colorVanillaStrong}{Vanilla} (top left), \textcolor{colorVanillaClipStrong}{Vanilla+Clip} (top right), \textcolor{colorWSPNoClipStrong}{WSP} (bottom left), and \textcolor{colorWSPStrong}{WSP+Clip} (bottom right).} Each row represents a different EMA survey item, listed in \cref{apx:real-data}. See discussion in \cref{apx:viz-qualitative}.}
    \label{fig:wsp-qualitative-5}
\end{figure*}

\begin{figure*}[p]
    \centering

    \includegraphics[width=0.38\textwidth]{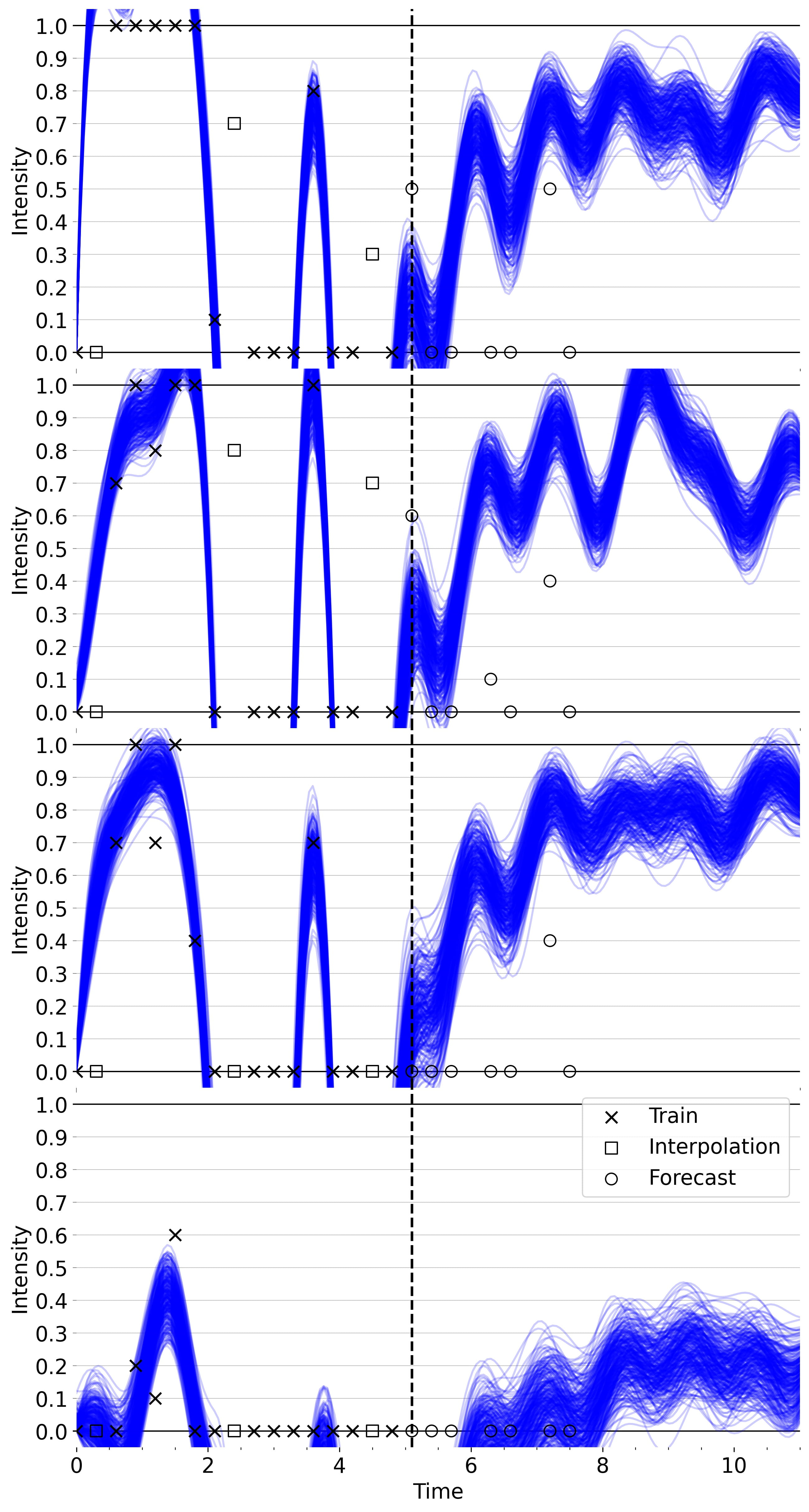}
    ~
    \includegraphics[width=0.38\textwidth]{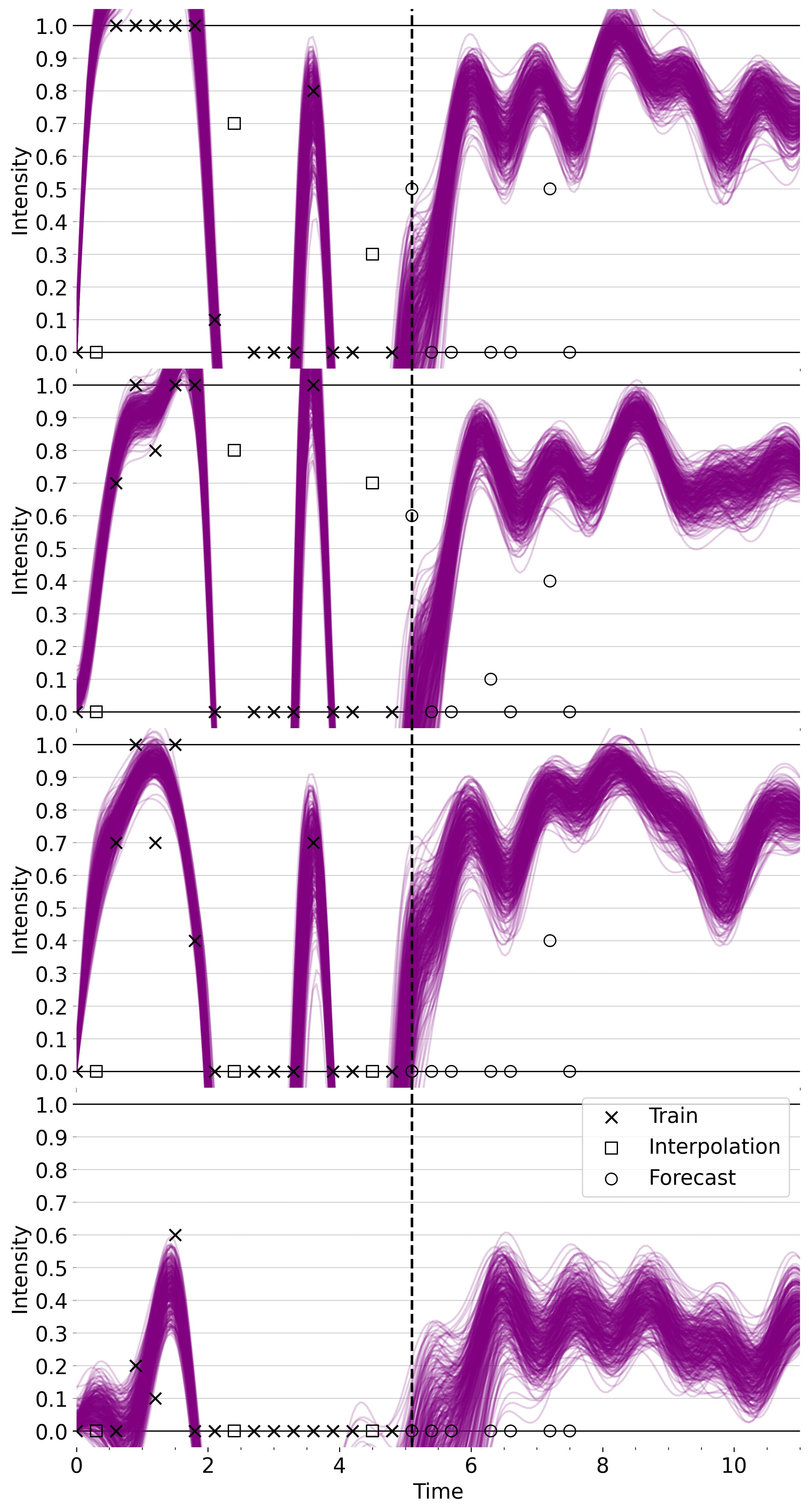}

    \includegraphics[width=0.38\textwidth]{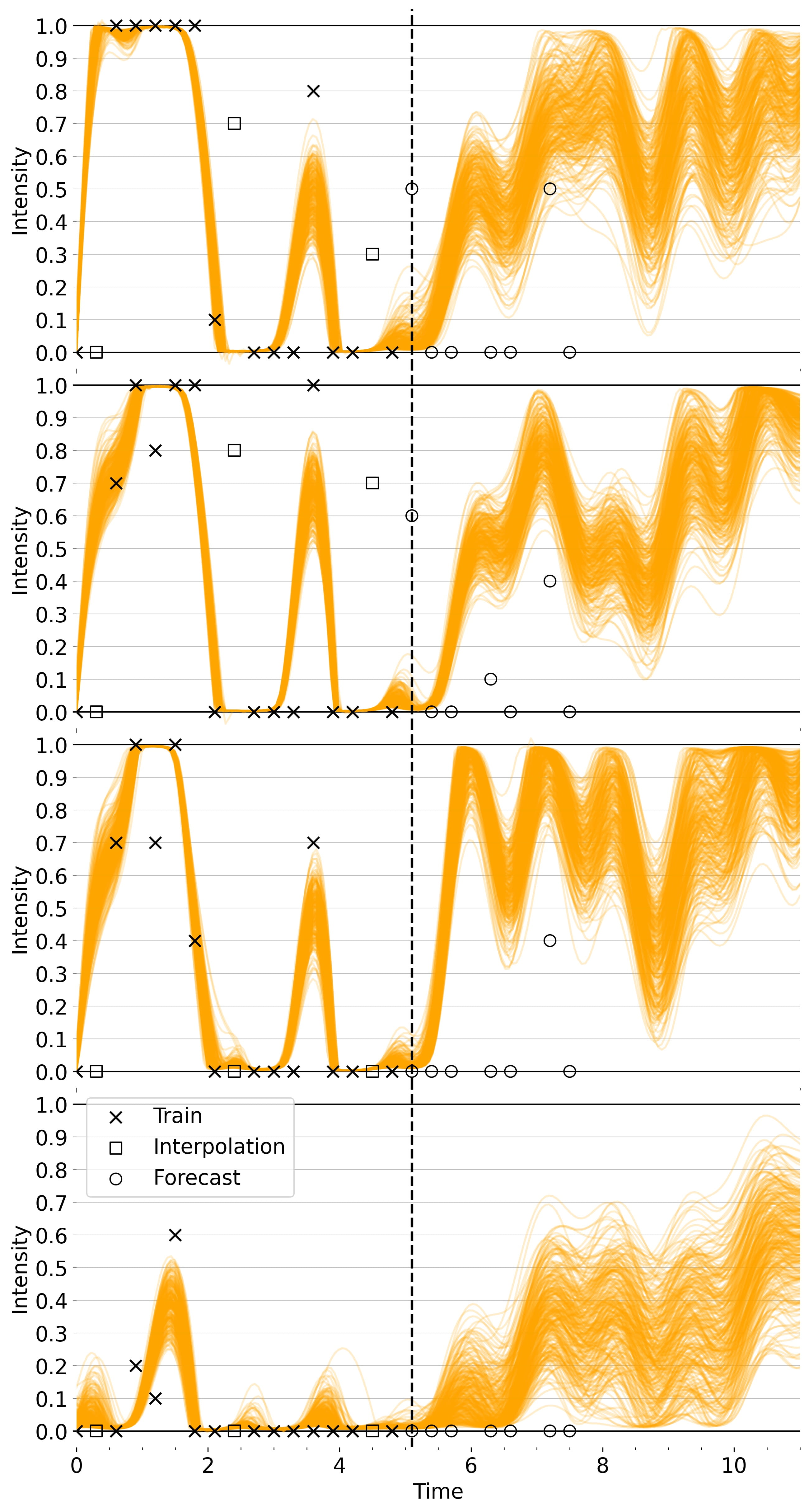}
    ~
    \includegraphics[width=0.38\textwidth]{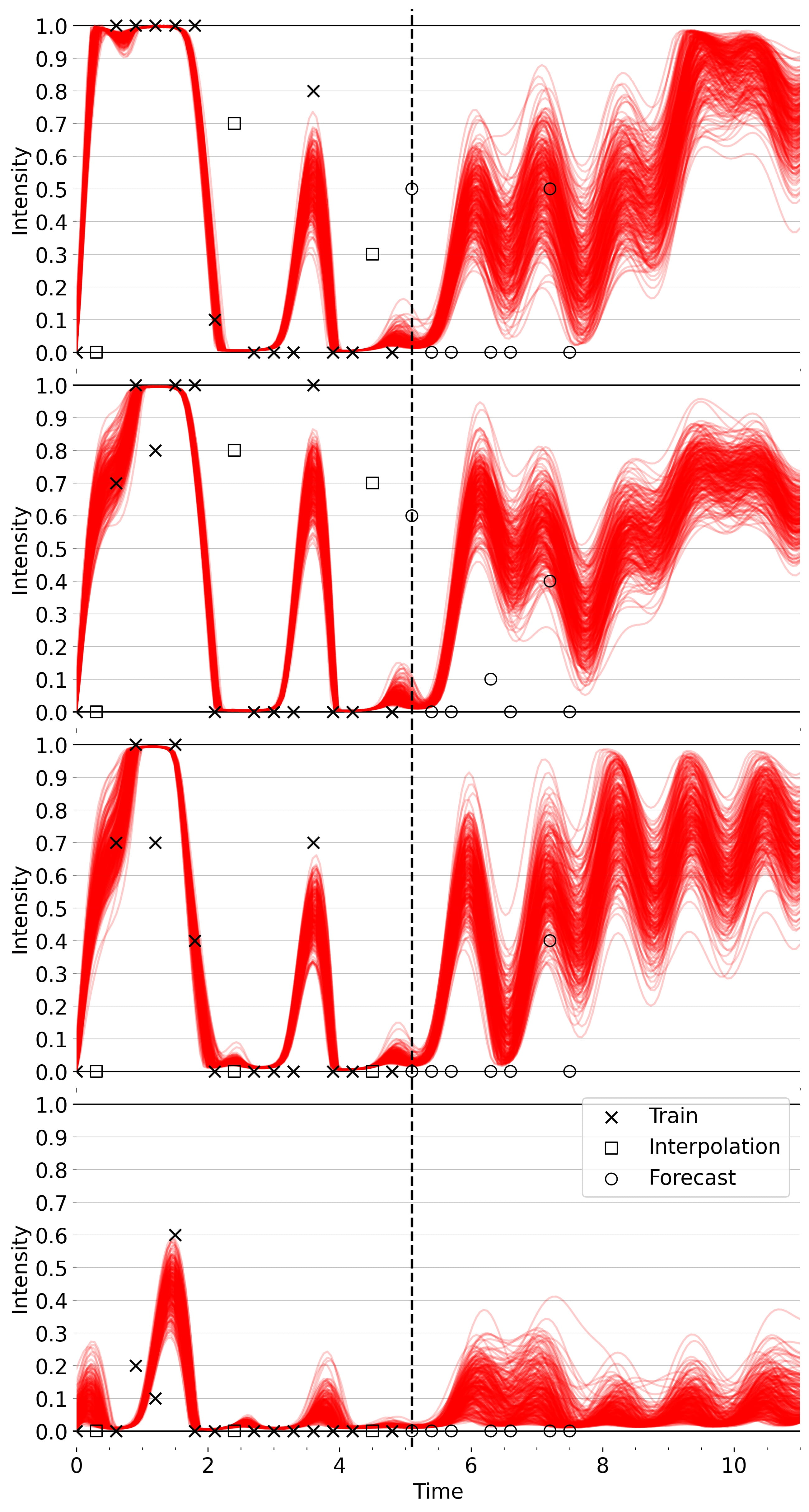}
    
    \caption{\textbf{Qualitative Comparison of Posterior for Specific Patient from \textcolor{colorVanillaStrong}{Vanilla} (top left), \textcolor{colorVanillaClipStrong}{Vanilla+Clip} (top right), \textcolor{colorWSPNoClipStrong}{WSP} (bottom left), and \textcolor{colorWSPStrong}{WSP+Clip} (bottom right).} Each row represents a different EMA survey item, listed in \cref{apx:real-data}. See discussion in \cref{apx:viz-qualitative}.}
    \label{fig:wsp-qualitative-15}
\end{figure*}

\begin{figure*}[p]
    \centering

    \includegraphics[width=0.38\textwidth]{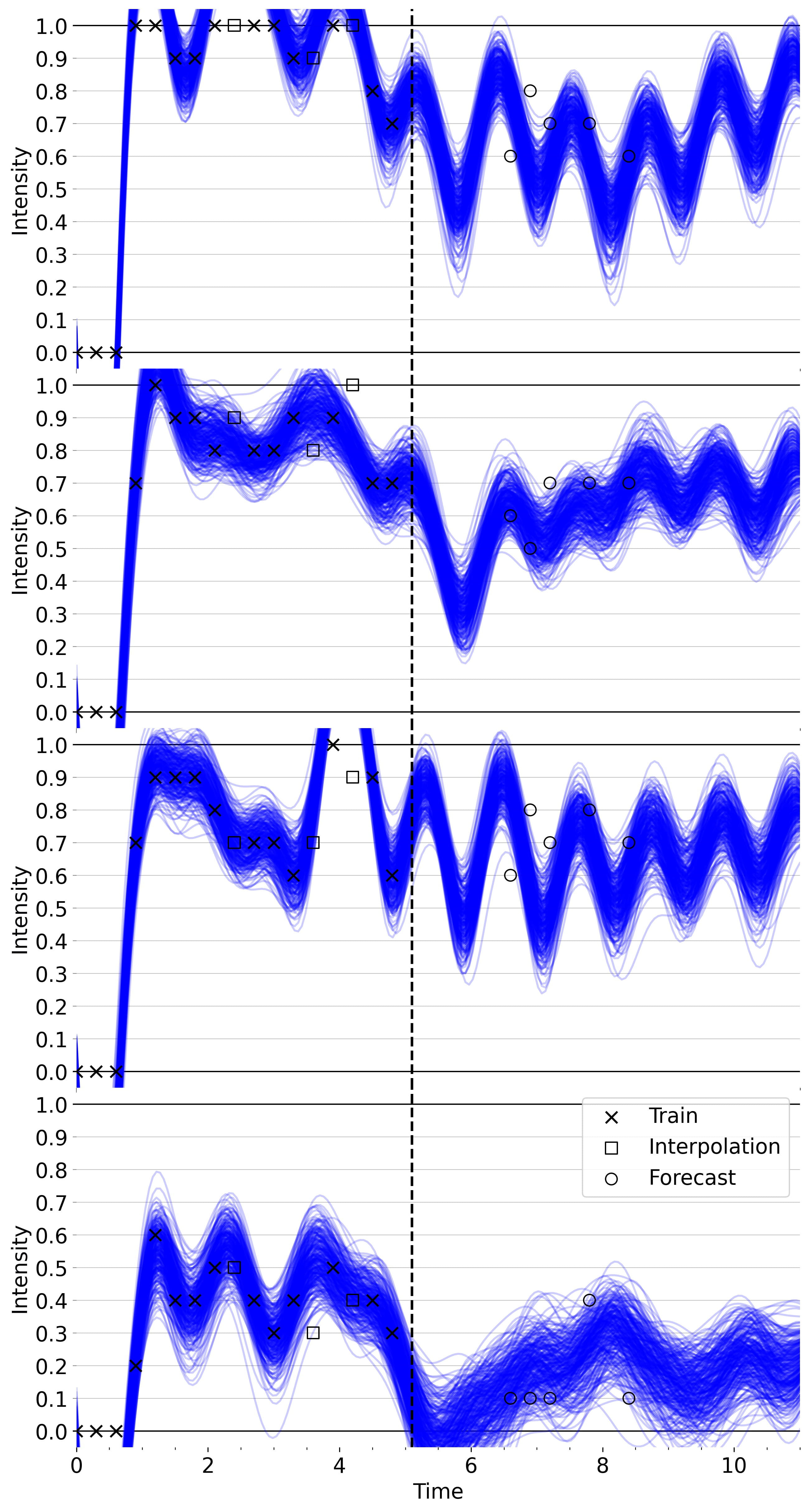}
    ~
    \includegraphics[width=0.38\textwidth]{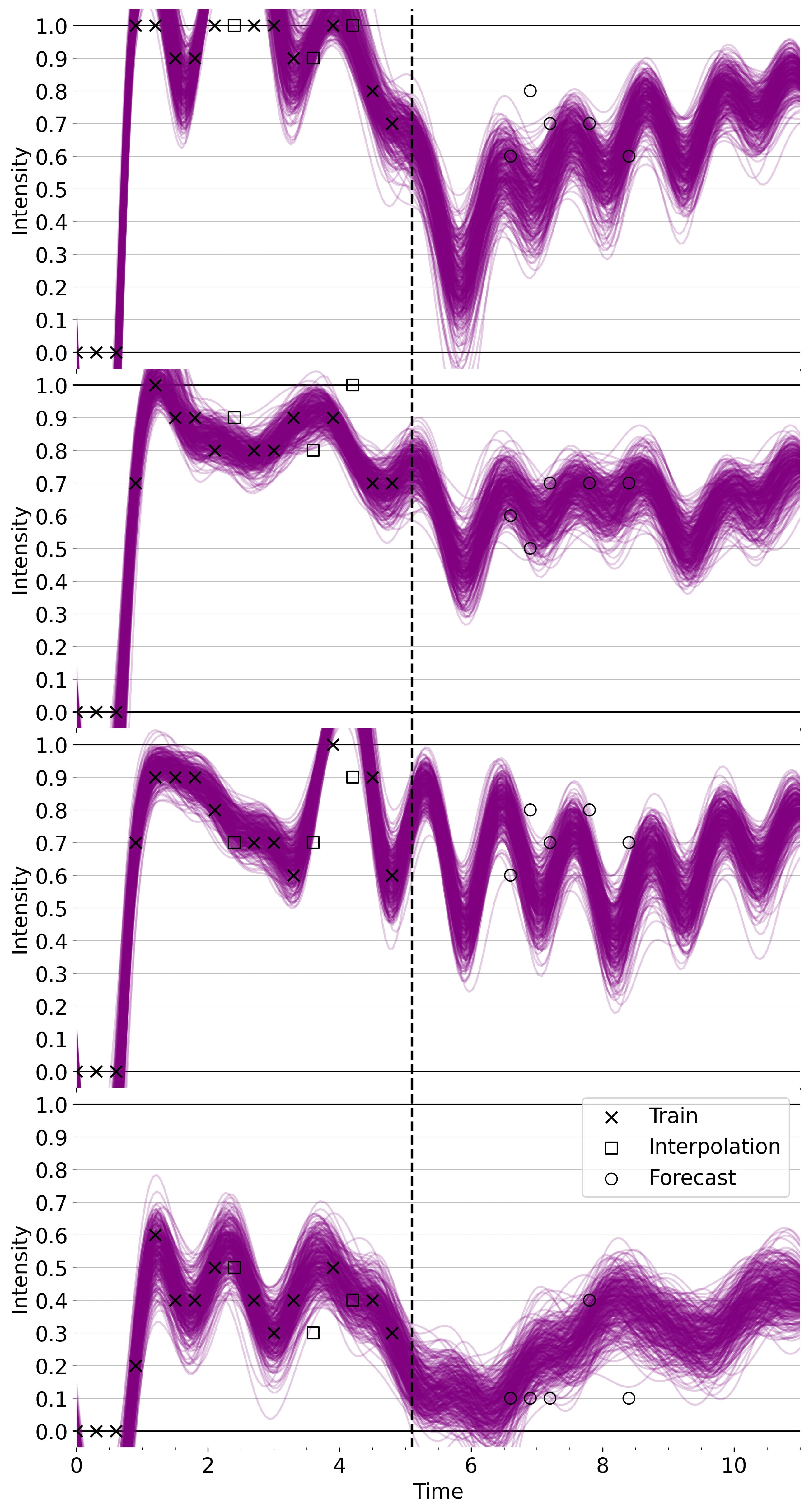}

    \includegraphics[width=0.38\textwidth]{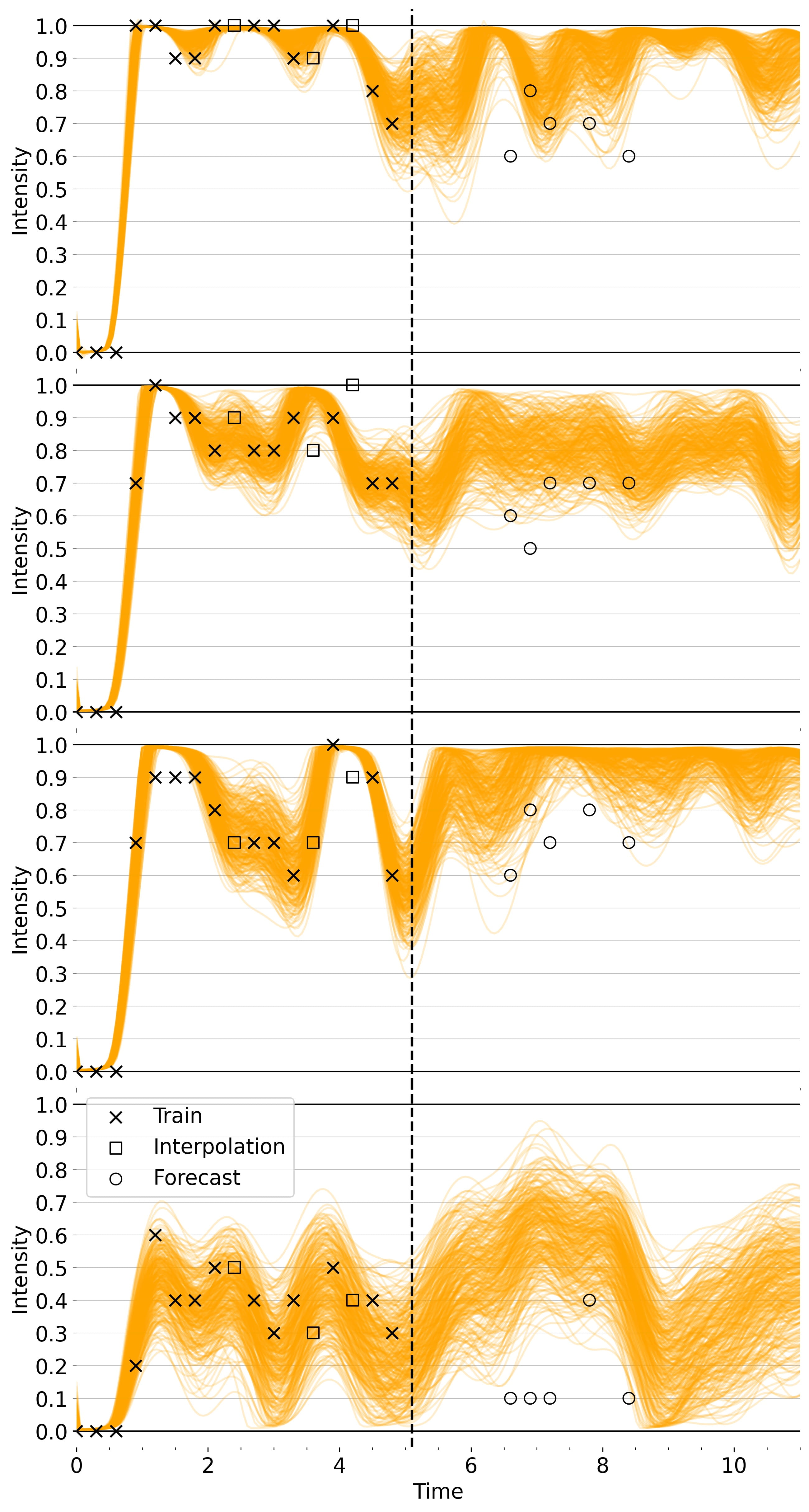}
    ~
    \includegraphics[width=0.38\textwidth]{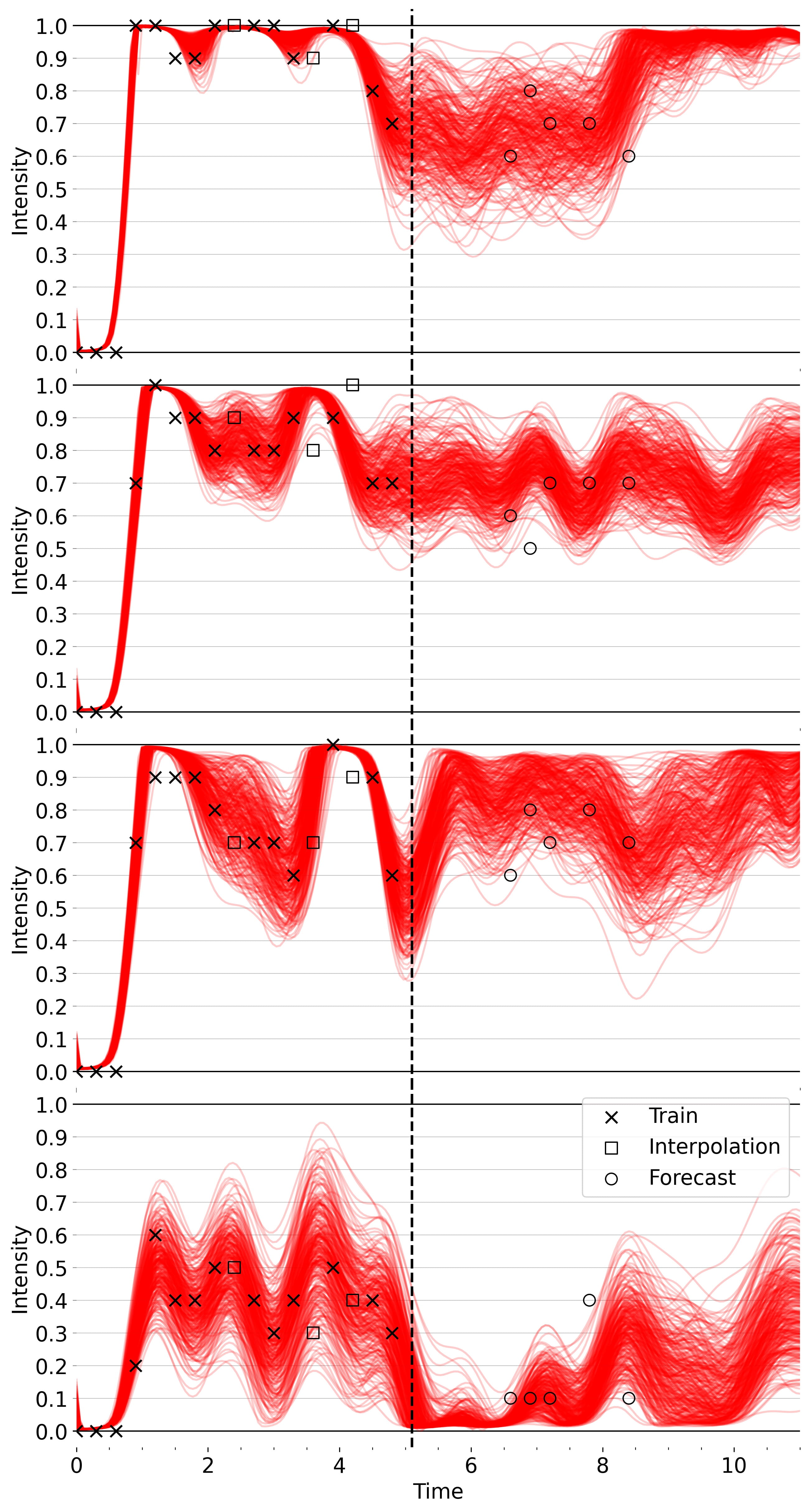}
    
    \caption{\textbf{Qualitative Comparison of Posterior for Specific Patient from \textcolor{colorVanillaStrong}{Vanilla} (top left), \textcolor{colorVanillaClipStrong}{Vanilla+Clip} (top right), \textcolor{colorWSPNoClipStrong}{WSP} (bottom left), and \textcolor{colorWSPStrong}{WSP+Clip} (bottom right).} Each row represents a different EMA survey item, listed in \cref{apx:real-data}. See discussion in \cref{apx:viz-qualitative}.}
    \label{fig:wsp-qualitative-17}
\end{figure*}

\begin{figure*}[p]
    \centering

    \includegraphics[width=0.38\textwidth]{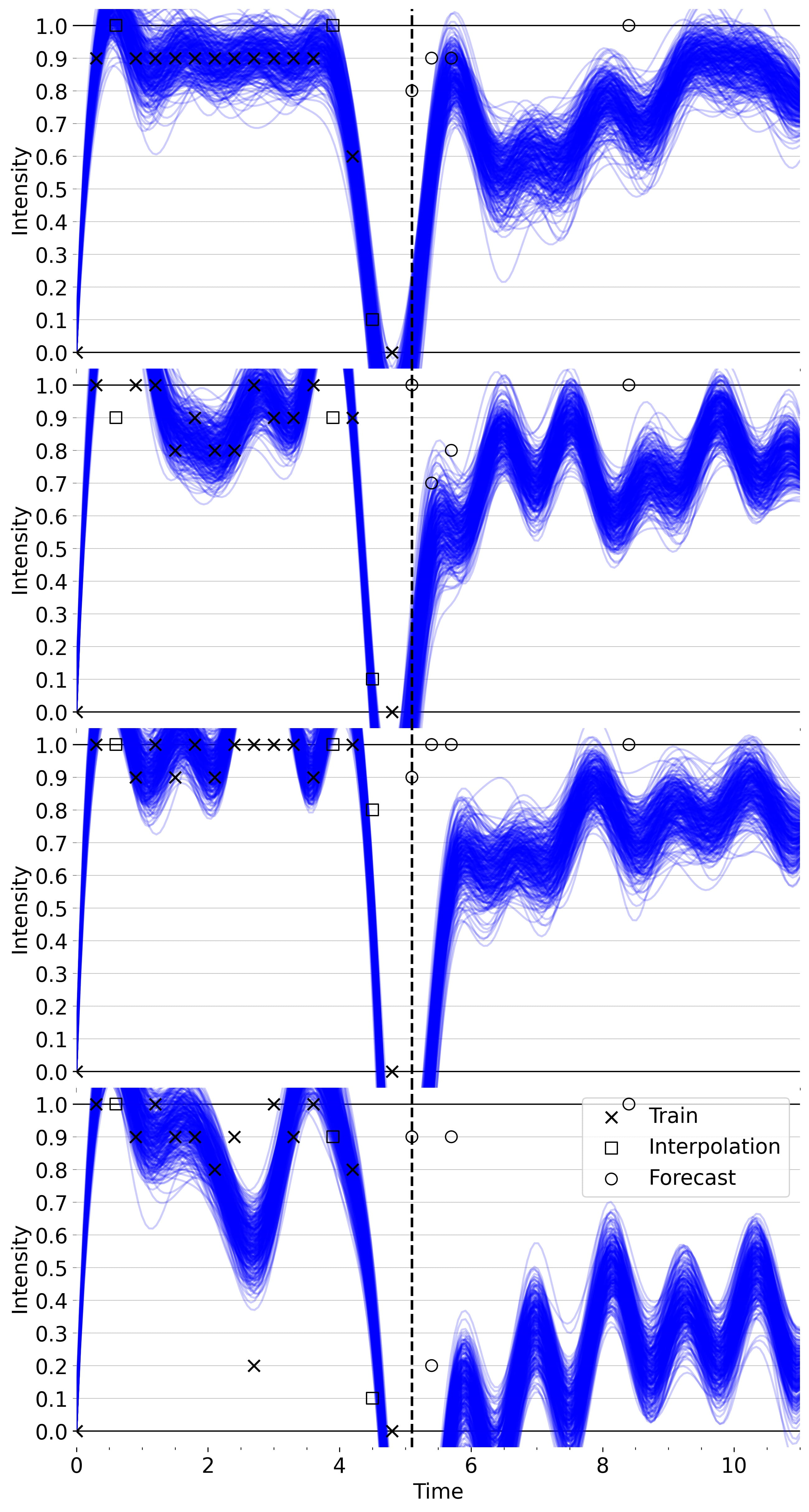}
    ~
    \includegraphics[width=0.38\textwidth]{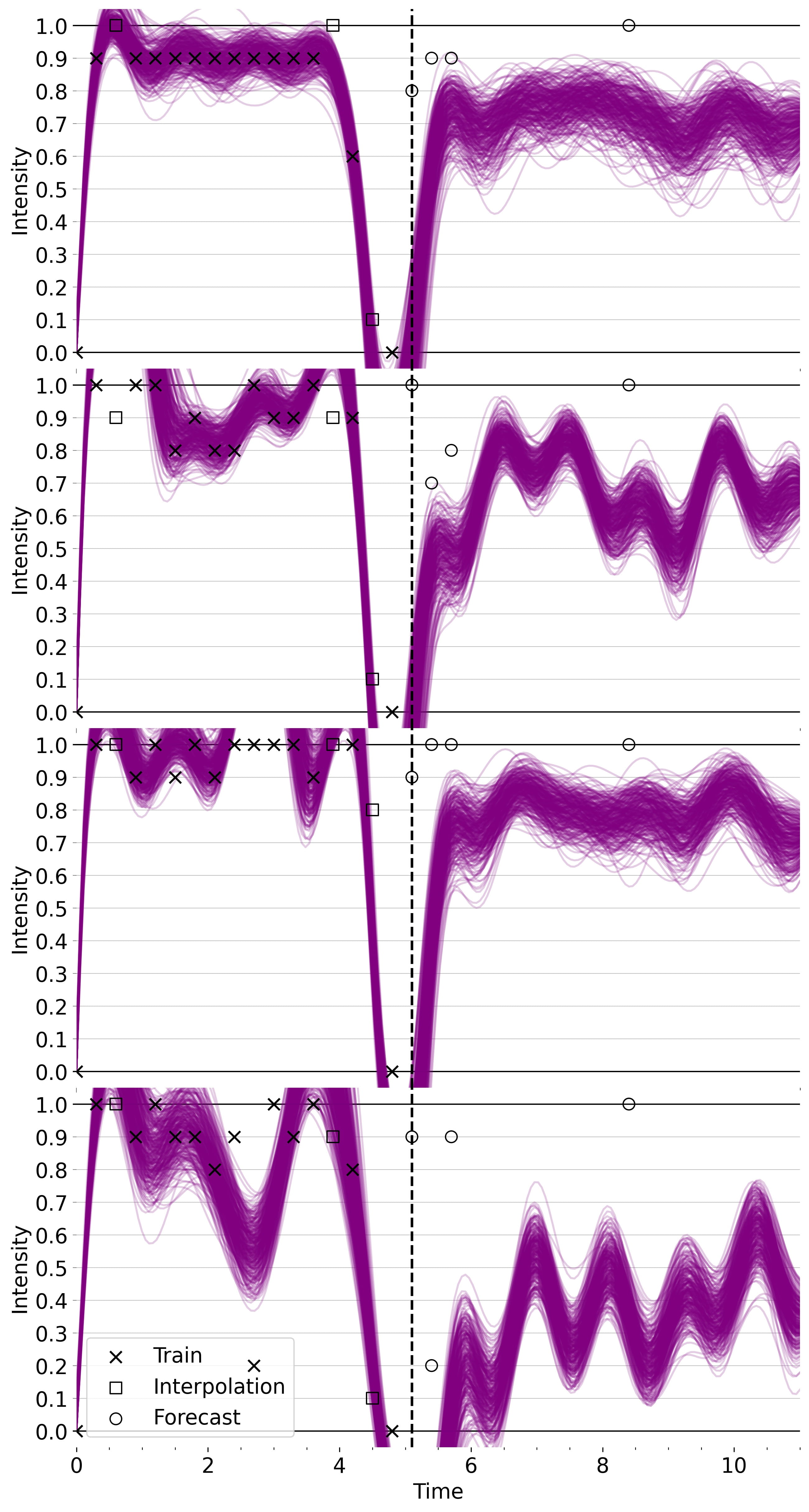}

    \includegraphics[width=0.38\textwidth]{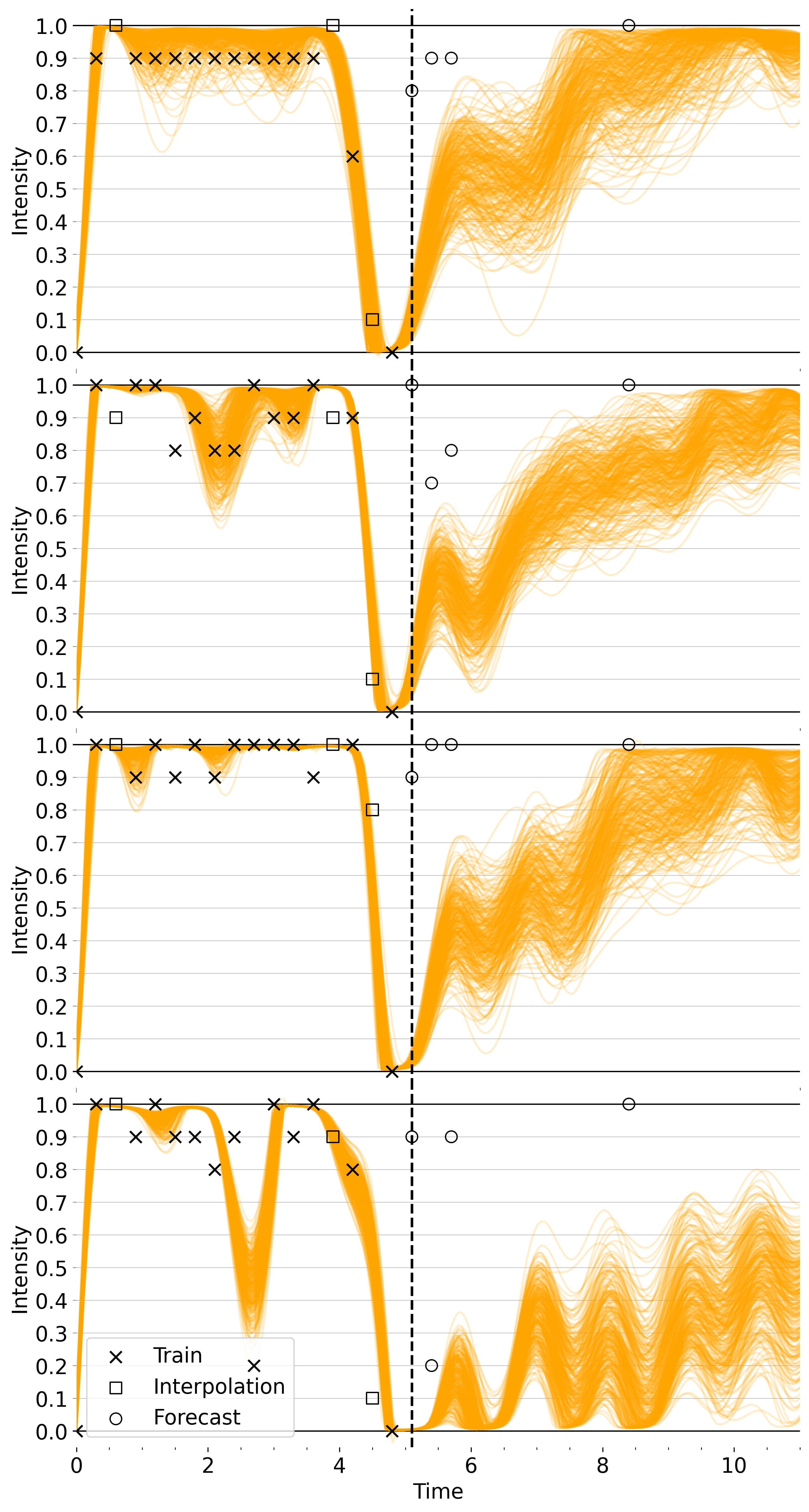}
    ~
    \includegraphics[width=0.38\textwidth]{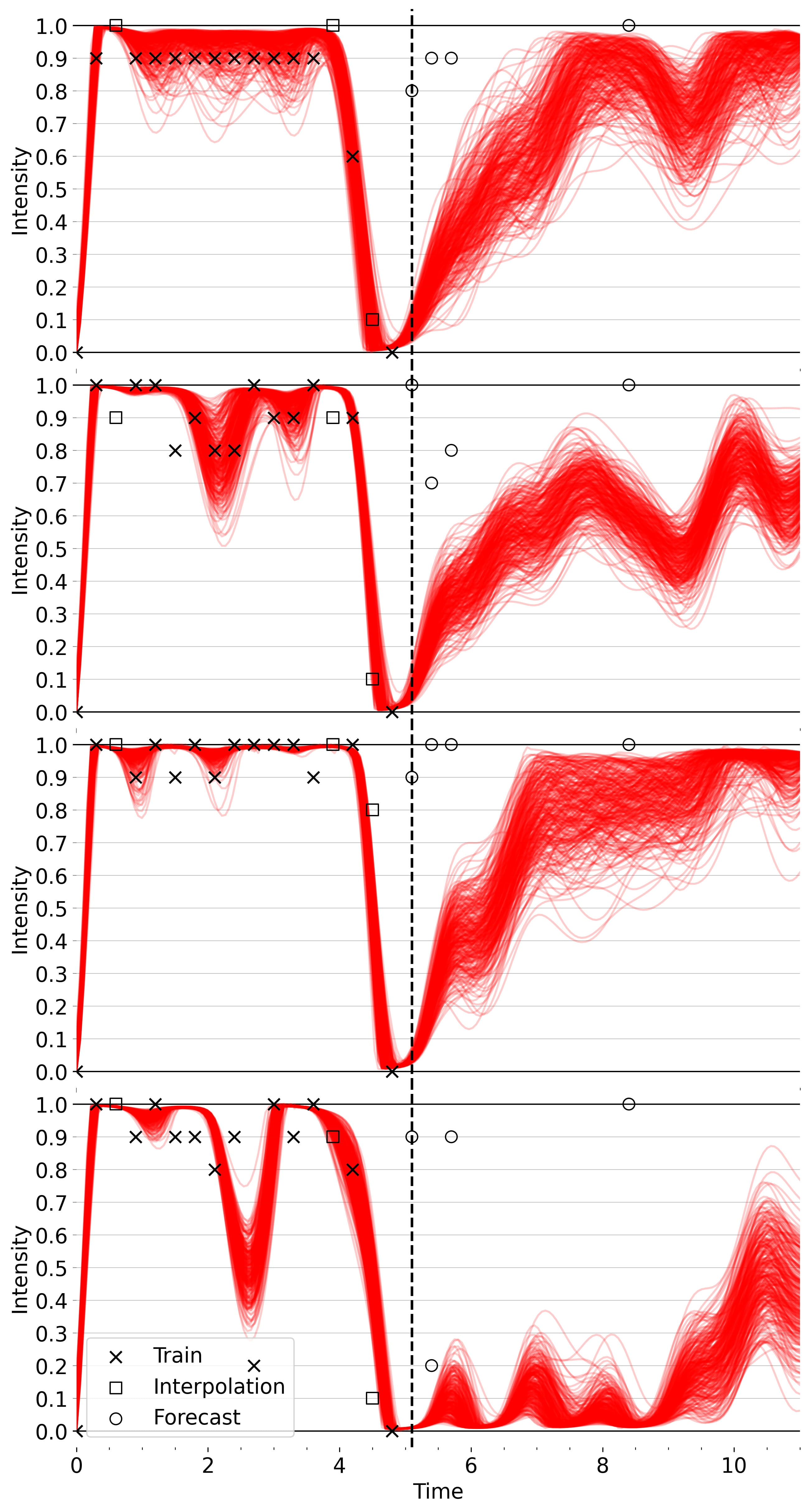}
    
    \caption{\textbf{Qualitative Comparison of Posterior for Specific Patient from \textcolor{colorVanillaStrong}{Vanilla} (top left), \textcolor{colorVanillaClipStrong}{Vanilla+Clip} (top right), \textcolor{colorWSPNoClipStrong}{WSP} (bottom left), and \textcolor{colorWSPStrong}{WSP+Clip} (bottom right).} Each row represents a different EMA survey item, listed in \cref{apx:real-data}. See discussion in \cref{apx:viz-qualitative}.}
    \label{fig:wsp-qualitative-88}
\end{figure*}

\begin{figure*}[p]
    \centering

    \includegraphics[width=0.38\textwidth]{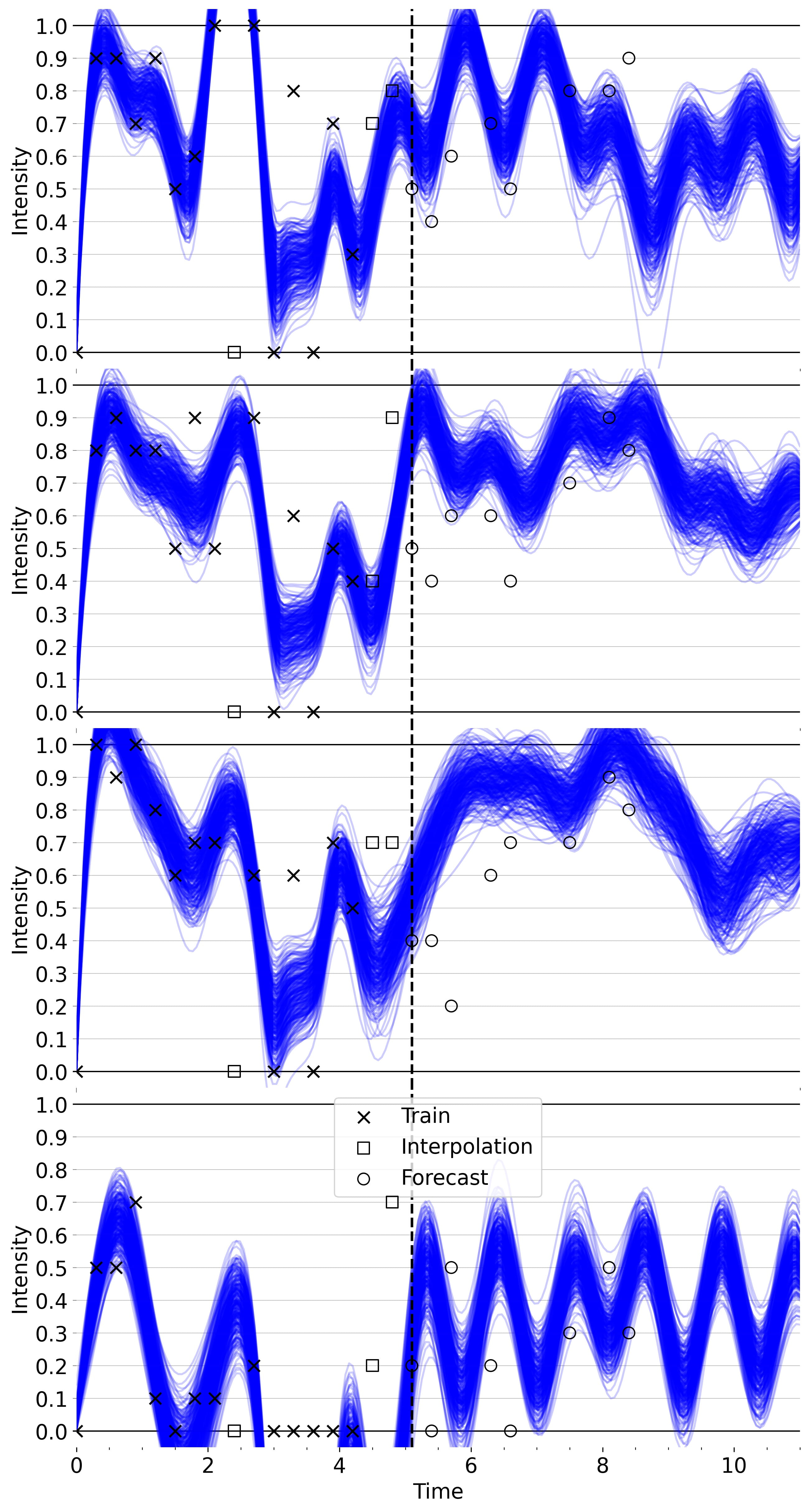}
    ~
    \includegraphics[width=0.38\textwidth]{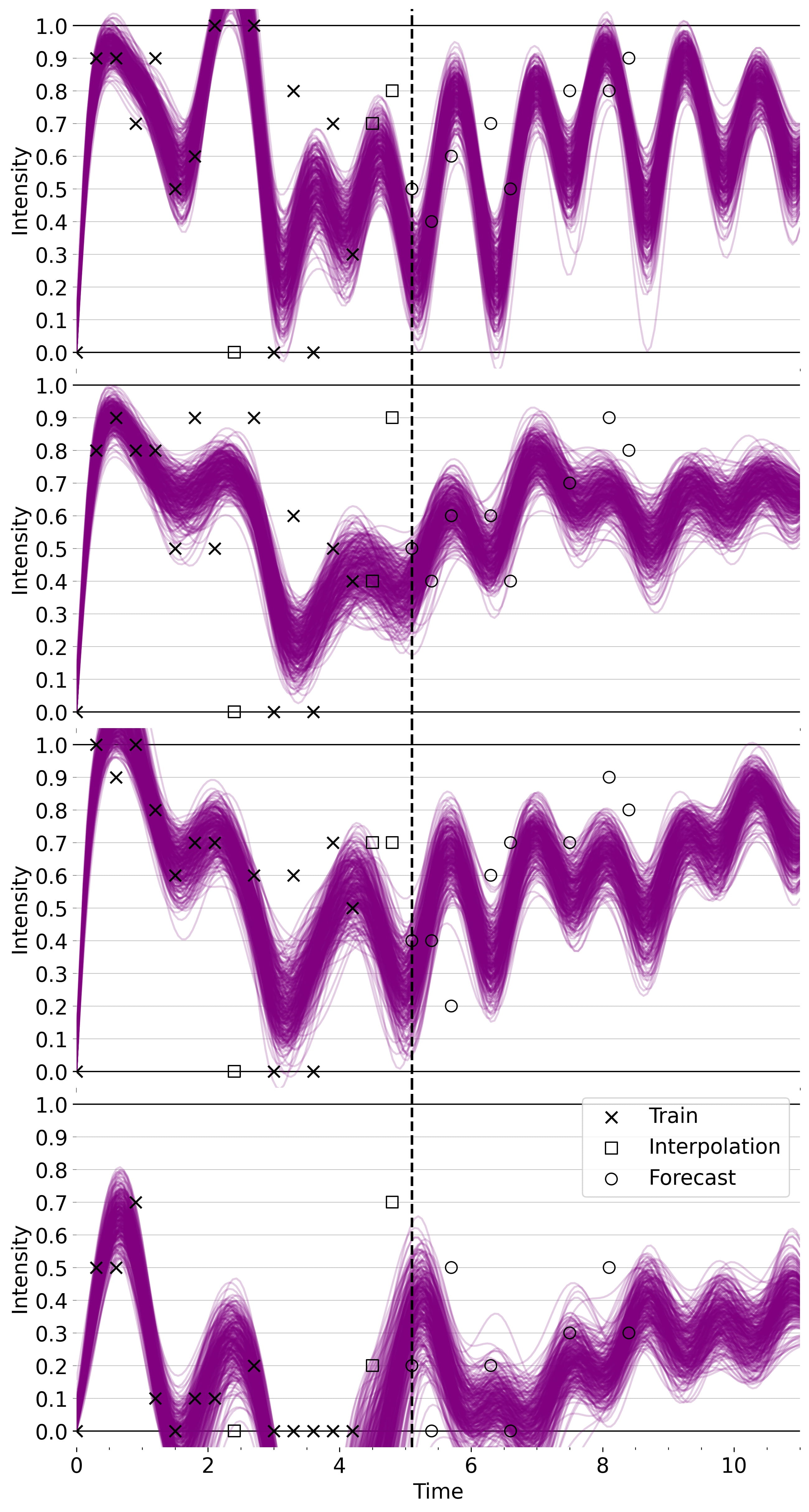}

    \includegraphics[width=0.38\textwidth]{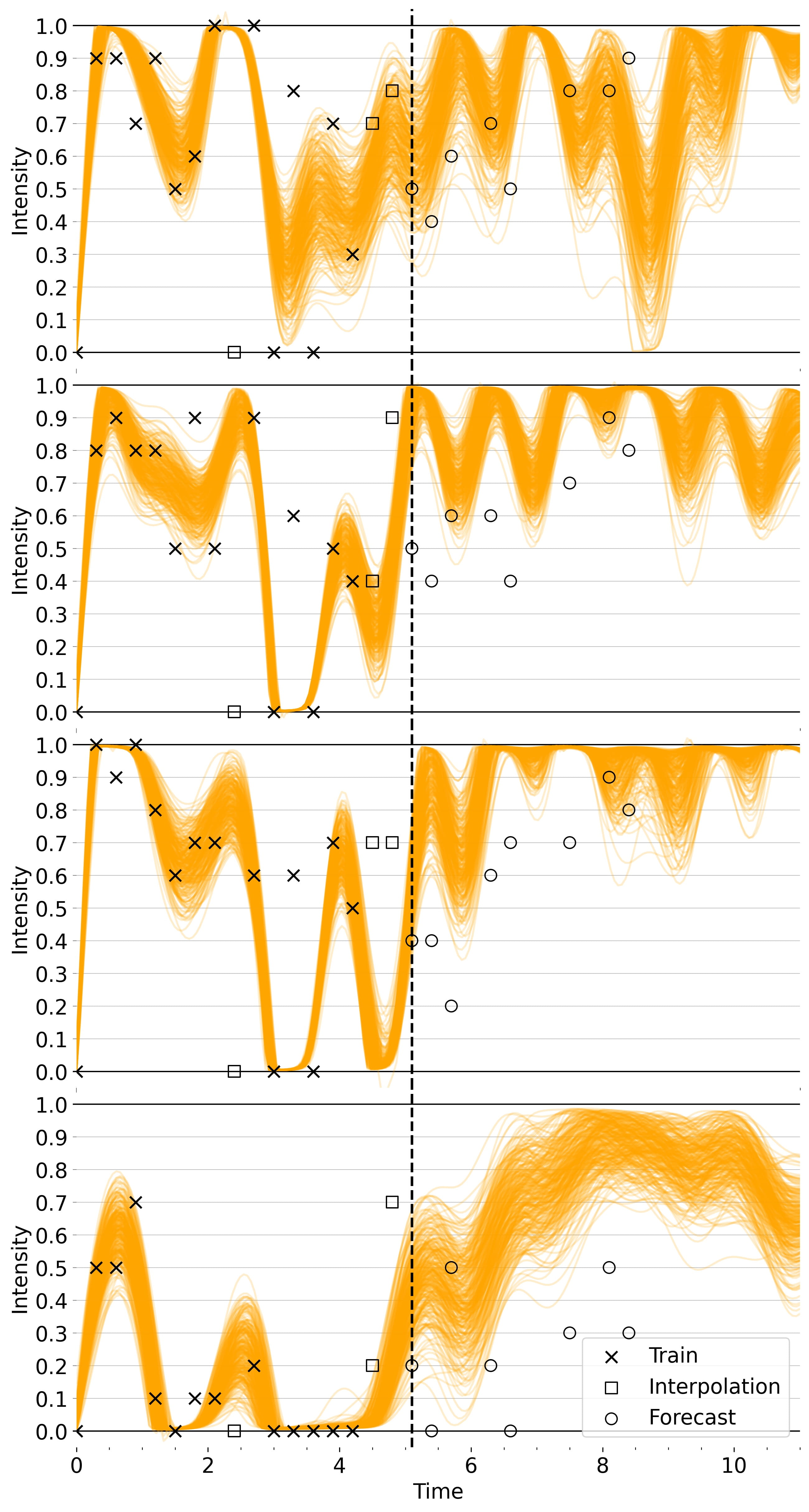}
    ~
    \includegraphics[width=0.38\textwidth]{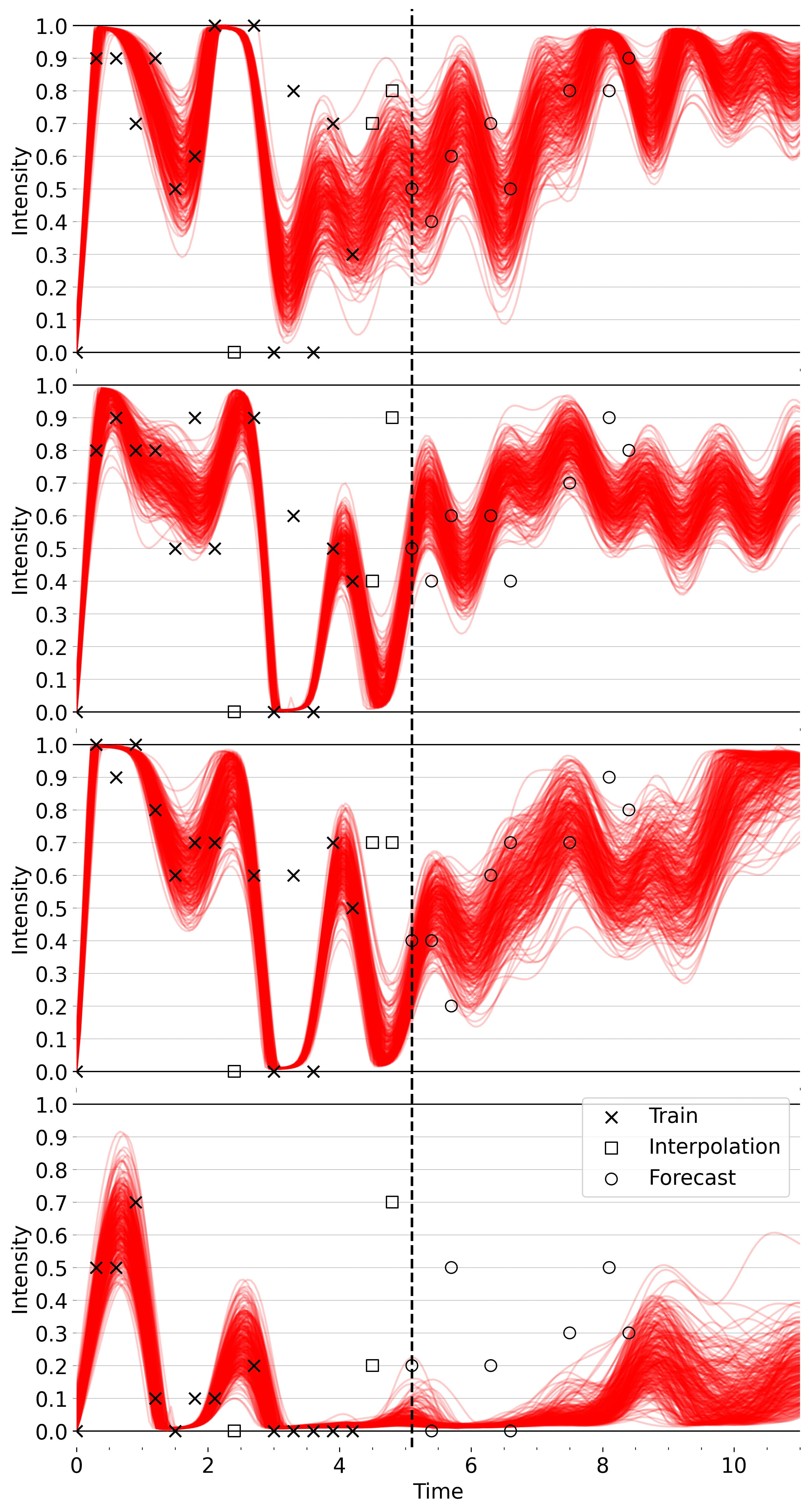}
    
    \caption{\textbf{Qualitative Comparison of Posterior for Specific Patient from \textcolor{colorVanillaStrong}{Vanilla} (top left), \textcolor{colorVanillaClipStrong}{Vanilla+Clip} (top right), \textcolor{colorWSPNoClipStrong}{WSP} (bottom left), and \textcolor{colorWSPStrong}{WSP+Clip} (bottom right).} Each row represents a different EMA survey item, listed in \cref{apx:real-data}. See discussion in \cref{apx:viz-qualitative}.}
    \label{fig:wsp-qualitative-23}
\end{figure*}

\FloatBarrier

\subsection{Visualizations of Latent SDE Prior Samples from the ``SMART'' Data}

\begin{figure*}[h]
    \centering

    \includegraphics[width=0.33\textwidth]{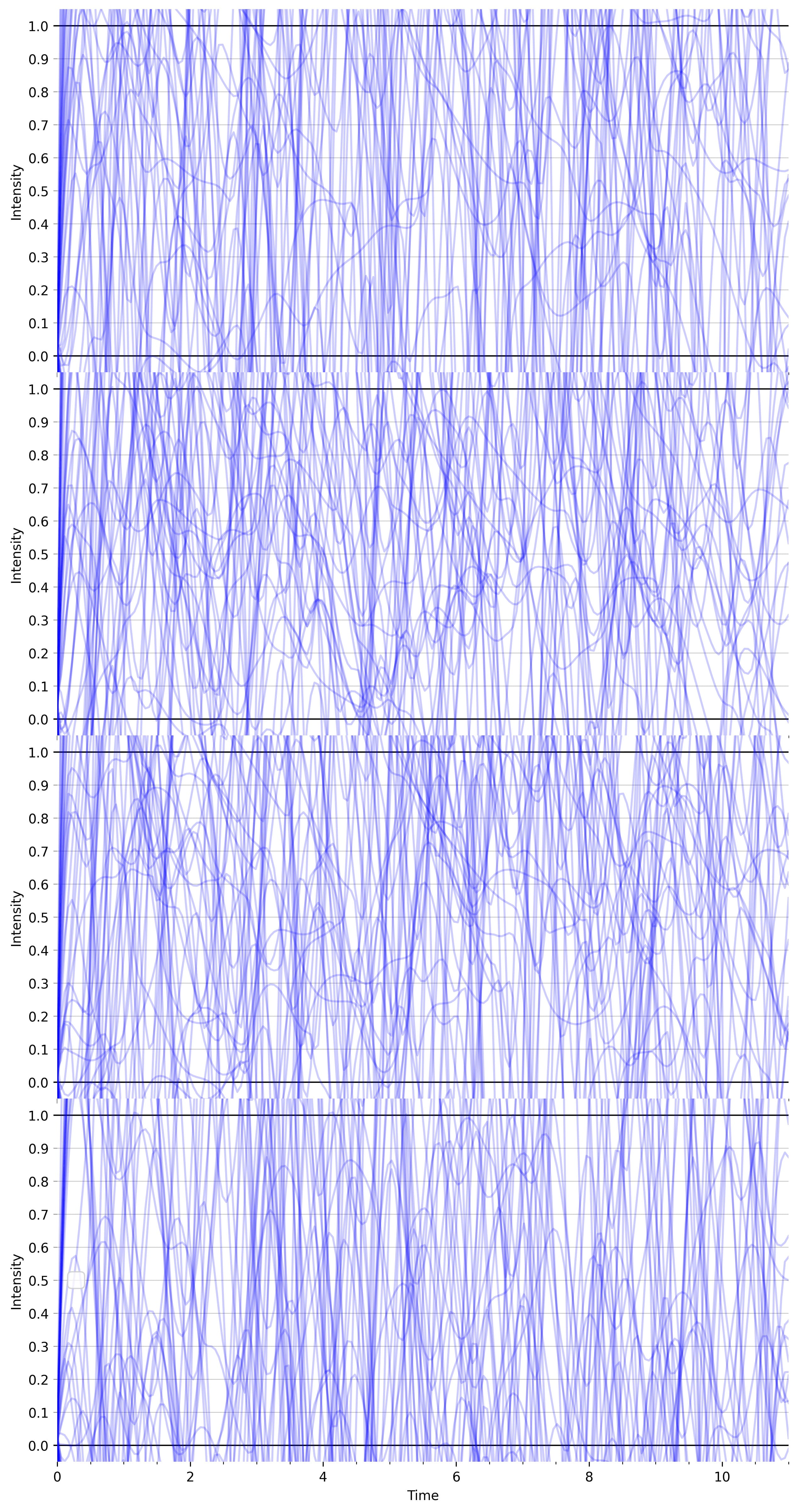}
    ~
    \includegraphics[width=0.33\textwidth]{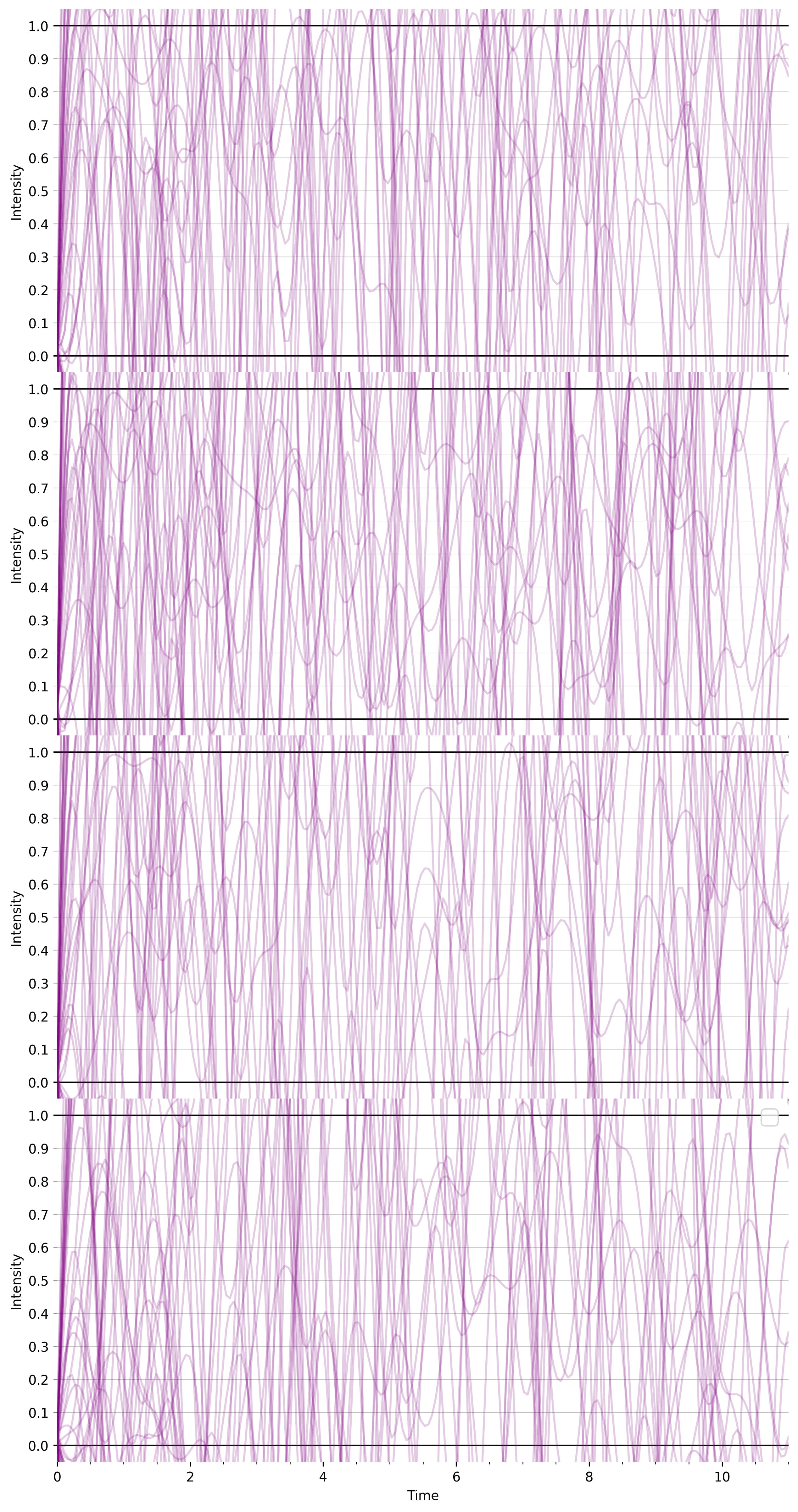}

    \includegraphics[width=0.33\textwidth]{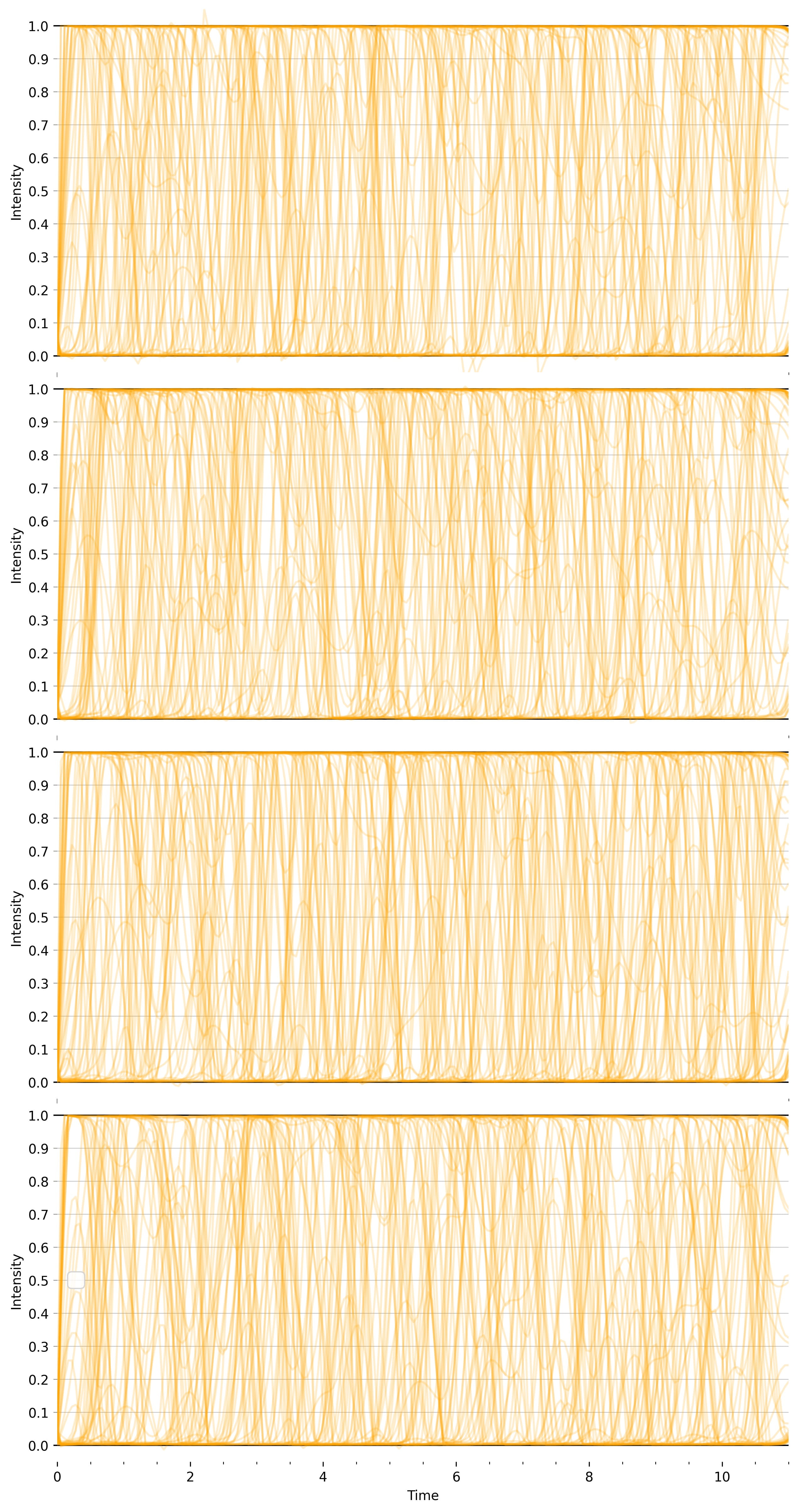}
    ~
    \includegraphics[width=0.33\textwidth]{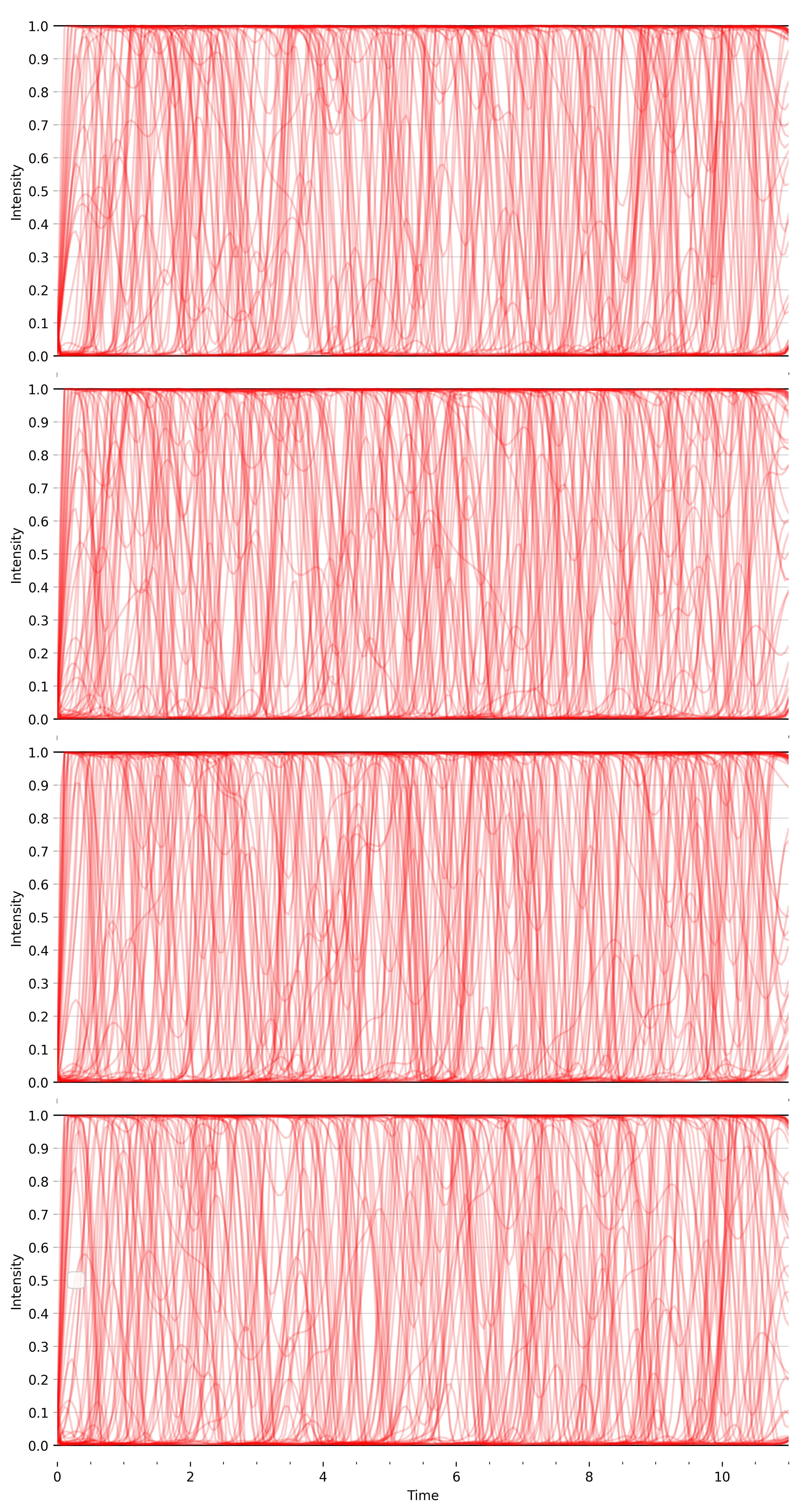}
    
    \caption{\textbf{Qualitative Comparison of Prior Samples from \textcolor{colorVanillaStrong}{Vanilla} (top left), \textcolor{colorVanillaClipStrong}{Vanilla+Clip} (top right), \textcolor{colorWSPNoClipStrong}{WSP} (bottom left), and \textcolor{colorWSPStrong}{WSP+Clip} (bottom right).} Each row represents a different EMA survey item, listed in \cref{apx:real-data}. WSP ensures prior samples remain viable whereas baselines do not.}
    \label{fig:wsp-prior-samples}
\end{figure*}

\end{document}

%% file: math_commands.tex
\usepackage{bm,amssymb,amsmath}

\makeatletter
\def\th@plain{%
  \thm@notefont{}
  \itshape 
}
\def\th@definition{%
  \thm@notefont{}
  \normalfont 
}
\makeatother

\def\1{\bm{1}}

\DeclareMathAlphabet{\mathsfit}{\encodingdefault}{\sfdefault}{m}{sl}
\SetMathAlphabet{\mathsfit}{bold}{\encodingdefault}{\sfdefault}{bx}{n}

\newcommand{\R}{\mathbb{R}}



%% file: references.bib
@article{wong1965convergence,
  title={On the convergence of ordinary integrals to stochastic integrals},
  author={Wong, Eugene and Zakai, Moshe},
  journal={The Annals of Mathematical Statistics},
  volume={36},
  number={5},
  pages={1560--1564},
  year={1965},
  publisher={JSTOR}
}

@inproceedings{ghosh2022differentiable,
  title={Differentiable Bayesian inference of SDE parameters using a pathwise series expansion of Brownian motion},
  author={Ghosh, Sanmitra and Birrell, Paul J and De Angelis, Daniela},
  booktitle={International Conference on Artificial Intelligence and Statistics},
  pages={10982--10998},
  year={2022},
  organization={PMLR}
}

@article{cai1996generation,
  title={Generation of non-Gaussian stationary stochastic processes},
  author={Cai, GQ and Lin, YK},
  journal={Physical Review E},
  volume={54},
  number={1},
  pages={299},
  year={1996},
  publisher={APS}
}

@book{d2013bounded,
  title={Bounded noises in physics, biology, and engineering},
  author={d'Onofrio, Alberto},
  year={2013},
  publisher={Springer}
}

@article{cresson2012validating,
  title={Validating stochastic models: invariance criteria for systems of stochastic differential equations and the selection of a stochastic Hodgkin-Huxley type model},
  author={Cresson, Jacky and Puig, B{\'e}n{\'e}dicte and Sonner, Stefanie},
  journal={arXiv preprint arXiv:1209.4520},
  year={2012}
}

@article{cresson2016stochastic,
  title={Stochastic Models in Biology and the Invariance Problem.},
  author={Cresson, Jacky and Puig, B{\'e}n{\'e}dicte and Sonner, Stefanie},
  journal={Discrete \& Continuous Dynamical Systems-Series B},
  volume={21},
  number={7},
  year={2016}
}

@article{rohanizadegan2020discrete,
  title={Discrete attachment to a cellulolytic biofilm modeled by an It{\^o} stochastic differential equation},
  author={Rohanizadegan, Yousef and Sonner, Stefanie and Eberl, Hermann J},
  year={2020}
}

@article{cresson2018note,
  title={A note on a derivation method for SDE models: Applications in biology and viability criteria},
  author={Cresson, Jacky and Sonner, Stefanie},
  journal={Stochastic Analysis and Applications},
  volume={36},
  number={2},
  pages={224--239},
  year={2018},
  publisher={Taylor \& Francis}
}

@inproceedings{milian1995stochastic,
  title={Stochastic viability and a comparison theorem},
  author={Milian, Anna},
  booktitle={Colloquium Mathematicum},
  volume={68},
  number={2},
  pages={297--316},
  year={1995},
  organization={Polska Akademia Nauk. Instytut Matematyczny PAN}
}

@BOOK{aubin1991,
  title     = "Viability Theory",
  author    = "Aubin, Jean-Pierre",
  publisher = "Birkhauser Boston",
  series    = "Modern Birkh{\"a}user Classics",
  month     =  jan,
  year      =  1991,
  address   = "Secaucus, NJ",
  language  = "en"
}

@inproceedings{lou2023reflected,
  title={Reflected diffusion models},
  author={Lou, Aaron and Ermon, Stefano},
  booktitle={International Conference on Machine Learning},
  pages={22675--22701},
  year={2023},
  organization={PMLR}
}

@article{christopher2024constrained,
  title={Constrained synthesis with projected diffusion models},
  author={Christopher, Jacob K and Baek, Stephen and Fioretto, Nando},
  journal={Advances in Neural Information Processing Systems},
  volume={37},
  pages={89307--89333},
  year={2024}
}

@article{fishman2023diffusion,
  title={Diffusion Models for Constrained Domains},
  author={Nic Fishman and Leo Klarner and Valentin De Bortoli and Emile Mathieu and Michael John Hutchinson},
  journal={Transactions on Machine Learning Research},
  issn={2835-8856},
  year={2023}
}

@article{archambeau2007variational,
  title={Variational inference for diffusion processes},
  author={Archambeau, C{\'e}dric and Opper, Manfred and Shen, Yuan and Cornford, Dan and Shawe-Taylor, John},
  journal={Advances in neural information processing systems},
  volume={20},
  year={2007}
}

@inproceedings{li2020scalable,
  title={Scalable gradients for stochastic differential equations},
  author={Li, Xuechen and Wong, Ting-Kam Leonard and Chen, Ricky TQ and Duvenaud, David},
  booktitle={International Conference on Artificial Intelligence and Statistics},
  pages={3870--3882},
  year={2020},
  organization={PMLR}
}

@inproceedings{xu2022infinitely,
  title={Infinitely deep bayesian neural networks with stochastic differential equations},
  author={Xu, Winnie and Chen, Ricky TQ and Li, Xuechen and Duvenaud, David},
  booktitle={International Conference on Artificial Intelligence and Statistics},
  pages={721--738},
  year={2022},
  organization={PMLR}
}

@inproceedings{song2021scorebased,
  title={Score-Based Generative Modeling through Stochastic Differential Equations},
  author={Yang Song and Jascha Sohl-Dickstein and Diederik P Kingma and Abhishek Kumar and Stefano Ermon and Ben Poole},
  booktitle={International Conference on Learning Representations},
  year={2021}
}

@InProceedings{ansari2023neural,
  title = 	 {Neural Continuous-Discrete State Space Models for Irregularly-Sampled Time Series},
  author =       {Ansari, Abdul Fatir and Heng, Alvin and Lim, Andre and Soh, Harold},
  booktitle = 	 {Proceedings of the 40th International Conference on Machine Learning},
  pages = 	 {926--951},
  year = 	 {2023},
  editor = 	 {Krause, Andreas and Brunskill, Emma and Cho, Kyunghyun and Engelhardt, Barbara and Sabato, Sivan and Scarlett, Jonathan},
  volume = 	 {202},
  series = 	 {Proceedings of Machine Learning Research},
  month = 	 {23--29 Jul},
  publisher =    {PMLR},
}

@inproceedings{oh2024stable,
  title={Stable Neural Stochastic Differential Equations in Analyzing Irregular Time Series Data},
  author={YongKyung Oh and Dongyoung Lim and Sungil Kim},
  booktitle={The Twelfth International Conference on Learning Representations},
  year={2024}
}

@article{saharia2022photorealistic,
  title={Photorealistic text-to-image diffusion models with deep language understanding},
  author={Saharia, Chitwan and Chan, William and Saxena, Saurabh and Li, Lala and Whang, Jay and Denton, Emily L and Ghasemipour, Kamyar and Gontijo Lopes, Raphael and Karagol Ayan, Burcu and Salimans, Tim and others},
  journal={Advances in neural information processing systems},
  volume={35},
  pages={36479--36494},
  year={2022}
}

@misc{ding2008numerical,
  title={Numerical solutions for reflected stochastic differential equations},
  author={Ding, Deng and Zhang, Ying Ying},
  year={2008},
  publisher={International Press of Boston}
}

@article{fishman2023metropolis,
  title={Metropolis sampling for constrained diffusion models},
  author={Fishman, Nic and Klarner, Leo and Mathieu, Emile and Hutchinson, Michael and De Bortoli, Valentin},
  journal={Advances in Neural Information Processing Systems},
  volume={36},
  pages={62296--62331},
  year={2023}
}

@book{pilipenko2014introduction,
  title={An introduction to stochastic differential equations with reflection},
  author={Pilipenko, Andrey},
  volume={1},
  year={2014},
  publisher={Universit{\"a}tsverlag Potsdam}
}

@article{ma2015complete,
  title={A complete recipe for stochastic gradient MCMC},
  author={Ma, Yi-An and Chen, Tianqi and Fox, Emily},
  journal={Advances in neural information processing systems},
  volume={28},
  year={2015}
}

@misc{ormandy2019linking, 
  title={Linking Sampling and Stochastic Differential Equations}, 
  url={https://chrisorm.github.io/SDE-S.html}, 
  author={Ormandy, Chris},
  year={2019}
}

@article{santner1986note,
  title={A note on A. Albert and JA Anderson's conditions for the existence of maximum likelihood estimates in logistic regression models},
  author={Santner, Thomas J and Duffy, Diane E},
  journal={Biometrika},
  volume={73},
  number={3},
  pages={755--758},
  year={1986},
  publisher={Oxford University Press}
}

@book{oksendal2013stochastic,
  title={Stochastic differential equations: an introduction with applications},
  author={Oksendal, Bernt},
  year={2013},
  publisher={Springer Science \& Business Media}
}

@MISC {mo2171798,
    TITLE = {Prove that the product of two Lipschitz functions is locally Lipschitz.},
    AUTHOR = {Behnam Esmayli},
    HOWPUBLISHED = {Mathematics Stack Exchange},
    NOTE = {URL:https://math.stackexchange.com/q/2171798 (version: 2021-10-12)},
    EPRINT = {https://math.stackexchange.com/q/2171798},
    URL = {https://math.stackexchange.com/q/2171798},
    year={2017}
}

@book{sarkka2019applied,
  title={Applied stochastic differential equations},
  author={S{\"a}rkk{\"a}, Simo and Solin, Arno},
  volume={10},
  year={2019},
  publisher={Cambridge University Press}
}

@article{bradbury2018jax,
  title={JAX: composable transformations of Python+ NumPy programs},
  author={Bradbury, James and Frostig, Roy and Hawkins, Peter and Johnson, Matthew James and Leary, Chris and Maclaurin, Dougal and Necula, George and Paszke, Adam and VanderPlas, Jake and Wanderman-Milne, Skye and others},
  year={2018}
}

@article{phan2019composable,
  title={Composable effects for flexible and accelerated probabilistic programming in NumPyro},
  author={Phan, Du and Pradhan, Neeraj and Jankowiak, Martin},
  journal={arXiv preprint arXiv:1912.11554},
  year={2019}
}

@phdthesis{kidger2021on,
    title={{O}n {N}eural {D}ifferential {E}quations},
    author={Patrick Kidger},
    year={2021},
    school={University of Oxford},
}

@article{barron2017continuously,
  title={Continuously differentiable exponential linear units},
  author={Barron, Jonathan T},
  journal={arXiv preprint arXiv:1704.07483},
  year={2017}
}

@article{hendrycks2016gaussian,
  title={Gaussian error linear units (gelus)},
  author={Hendrycks, Dan and Gimpel, Kevin},
  journal={arXiv preprint arXiv:1606.08415},
  year={2016}
}

@article{clevert2015fast,
  title={Fast and accurate deep network learning by exponential linear units (elus)},
  author={Clevert, Djork-Arn{\'e} and Unterthiner, Thomas and Hochreiter, Sepp},
  journal={arXiv preprint arXiv:1511.07289},
  year={2015}
}

@article{elfwing2018sigmoid,
  title={Sigmoid-weighted linear units for neural network function approximation in reinforcement learning},
  author={Elfwing, Stefan and Uchibe, Eiji and Doya, Kenji},
  journal={Neural networks},
  volume={107},
  pages={3--11},
  year={2018},
  publisher={Elsevier}
}

@article{klambauer2017self,
  title={Self-normalizing neural networks},
  author={Klambauer, G{\"u}nter and Unterthiner, Thomas and Mayr, Andreas and Hochreiter, Sepp},
  journal={Advances in neural information processing systems},
  volume={30},
  year={2017}
}

@inproceedings{glorot2010understanding,
  title={Understanding the difficulty of training deep feedforward neural networks},
  author={Glorot, Xavier and Bengio, Yoshua},
  booktitle={Proceedings of the thirteenth international conference on artificial intelligence and statistics},
  pages={249--256},
  year={2010},
  organization={JMLR Workshop and Conference Proceedings}
}

@article{miyato2018spectral,
  title={Spectral normalization for generative adversarial networks},
  author={Miyato, Takeru and Kataoka, Toshiki and Koyama, Masanori and Yoshida, Yuichi},
  journal={arXiv preprint arXiv:1802.05957},
  year={2018}
}

@inproceedings{liu2022learning,
  title={Learning smooth neural functions via lipschitz regularization},
  author={Liu, Hsueh-Ti Derek and Williams, Francis and Jacobson, Alec and Fidler, Sanja and Litany, Or},
  booktitle={ACM SIGGRAPH 2022 Conference Proceedings},
  pages={1--13},
  year={2022}
}

@inproceedings{anil2019sorting,
  title={Sorting out Lipschitz function approximation},
  author={Anil, Cem and Lucas, James and Grosse, Roger},
  booktitle={International conference on machine learning},
  pages={291--301},
  year={2019},
  organization={PMLR}
}

@inproceedings{gyawali2020enhancing,
  title={Enhancing mixup-based semi-supervised learningwith explicit lipschitz regularization},
  author={Gyawali, Prashnna and Ghimire, Sandesh and Wang, Linwei},
  booktitle={2020 IEEE International Conference on Data Mining (ICDM)},
  pages={1046--1051},
  year={2020},
  organization={IEEE}
}

@inproceedings{hurault2022gradient,
  title={Gradient Step Denoiser for convergent Plug-and-Play},
  author={Hurault, Samuel and Leclaire, Arthur and Papadakis, Nicolas},
  booktitle={International Conference on Learning Representations (ICLR'22)},
  year={2022}
}

@inproceedings{kidger2021neural,
  title={Neural sdes as infinite-dimensional gans},
  author={Kidger, Patrick and Foster, James and Li, Xuechen and Lyons, Terry J},
  booktitle={International conference on machine learning},
  pages={5453--5463},
  year={2021},
  organization={PMLR}
}

@article{issa2023non,
  title={Non-adversarial training of Neural SDEs with signature kernel scores},
  author={Issa, Zacharia and Horvath, Blanka and Lemercier, Maud and Salvi, Cristopher},
  journal={Advances in Neural Information Processing Systems},
  volume={36},
  pages={11102--11126},
  year={2023}
}

@inproceedings{zhang2025efficient,
title={Efficient Training of Neural Stochastic Differential Equations by Matching Finite Dimensional Distributions},
author={Jianxin Zhang and Josh Viktorov and Doosan Jung and Emily Pitler},
booktitle={The Thirteenth International Conference on Learning Representations},
year={2025}
}

@article{boyd2004convex,
  title={Convex optimization},
  author={Boyd, Stephen},
  journal={Cambridge UP},
  year={2004}
}

@article{shiffman2008ecological,
  title={Ecological momentary assessment},
  author={Shiffman, Saul and Stone, Arthur A and Hufford, Michael R},
  journal={Annu. Rev. Clin. Psychol.},
  volume={4},
  pages={1--32},
  year={2008},
  publisher={Annual Reviews}
}

@article{wang2023mathematical,
	Title = {Mathematical and Computational Modeling of Suicidal Thinking as a Complex Dynamical System},
	Author = {Wang, Shirley and Robinaugh, Donald and Millner, Alexander and Fortgang, Rebecca and Hamm, Sharina and Barrow, Coby and Nock, Matthew},
	DOI = {10.31234/osf.io/b29cs},
	Journal = {PsyArXiv},
	Year = {2023},
	URL = {https://doi.org/10.31234/osf.io/b29cs},
}

@article{zhang2024trajectory,
  title={Trajectory flow matching with applications to clinical time series modelling},
  author={Zhang, Xi Nicole and Pu, Yuan and Kawamura, Yuki and Loza, Andrew and Bengio, Yoshua and Shung, Dennis and Tong, Alexander},
  journal={Advances in Neural Information Processing Systems},
  volume={37},
  pages={107198--107224},
  year={2024}
}

@article{tonekaboni2021unsupervised,
  title={Unsupervised representation learning for time series with temporal neighborhood coding},
  author={Tonekaboni, Sana and Eytan, Danny and Goldenberg, Anna},
  journal={arXiv preprint arXiv:2106.00750},
  year={2021}
}

@InProceedings{weatherhead22learning,
  title = 	 {Learning Unsupervised Representations for ICU Timeseries},
  author =       {Weatherhead, Addison and Greer, Robert and Moga, Michael-Alice and Mazwi, Mjaye and Eytan, Danny and Goldenberg, Anna and Tonekaboni, Sana},
  booktitle = 	 {Proceedings of the Conference on Health, Inference, and Learning},
  pages = 	 {152--168},
  year = 	 {2022},
  editor = 	 {Flores, Gerardo and Chen, George H and Pollard, Tom and Ho, Joyce C and Naumann, Tristan},
  volume = 	 {174},
  series = 	 {Proceedings of Machine Learning Research},
  month = 	 {07--08 Apr},
  publisher =    {PMLR},
  url = 	 {https://proceedings.mlr.press/v174/weatherhead22a.html},
}

@inproceedings{tank2015bayesian,
author = {Tank, Alex and Foti, Nicholas J. and Fox, Emily B.},
title = {Bayesian structure learning for stationary time series},
year = {2015},
isbn = {9780996643108},
publisher = {AUAI Press},
address = {Arlington, Virginia, USA},
booktitle = {Proceedings of the Thirty-First Conference on Uncertainty in Artificial Intelligence},
pages = {872–881},
numpages = {10},
location = {Amsterdam, Netherlands},
series = {UAI'15}
}

@article{ammerman2022using,
  title={Using intensive time sampling methods to capture daily suicidal ideation: a systematic review},
  author={Ammerman, Brooke A and Law, Keyne C},
  journal={Journal of Affective Disorders},
  volume={299},
  pages={108--117},
  year={2022},
  publisher={Elsevier}
}

@article{robinaugh2024advancing,
  title={Advancing the network theory of mental disorders: A computational model of panic disorder.},
  author={Robinaugh, Donald J and Haslbeck, Jonas and Waldorp, Lourens J and Kossakowski, Jolanda J and Fried, Eiko I and Millner, Alexander J and McNally, Richard J and Ryan, Ois{\'\i}n and de Ron, Jill and van der Maas, Han LJ and others},
  journal={Psychological review},
  volume={131},
  number={6},
  pages={1482},
  year={2024},
  publisher={American Psychological Association}
}

@article{coppersmith2023mapping,
  title={Mapping the timescale of suicidal thinking},
  author={Coppersmith, Daniel DL and Ryan, Ois{\'\i}n and Fortgang, Rebecca G and Millner, Alexander J and Kleiman, Evan M and Nock, Matthew K},
  journal={Proceedings of the National Academy of Sciences},
  volume={120},
  number={17},
  pages={e2215434120},
  year={2023},
  publisher={National Academy of Sciences}
}

@article{millner2020advancing,
  title={Advancing the understanding of suicide: The need for formal theory and rigorous descriptive research},
  author={Millner, Alexander J and Robinaugh, Donald J and Nock, Matthew K},
  journal={Trends in cognitive sciences},
  volume={24},
  number={9},
  pages={704--716},
  year={2020},
  publisher={Elsevier}
}

@article{phq4,
    title = {An ultra-brief screening scale for anxiety and depression: the PHQ-4},
    author={Kroenke, Kurt and Spitzer, Robert L and Williams, Janet B W and Löwe, Bernd},
    journal = {Psychosomatics},
    volume={50},
    number={6},
    pages={613--612},
    year = {2009},
    publisher={Elsevier}
}

@article{pss4,
    title = {A global measure of perceived stress},
    author={Cohen, Sheldon and Kamarck, Tom and Mermelstein, Robin},
    journal = {Journal of Health and Social Behavior},
    volume={24},
    number={4},
    pages={385--396},
    year = {1983},
    publisher={JSTOR}
}

@article{panas,
    title = {Development and validation of brief measures of positive and negative affect: the PANAS scales},
    author={Watson, David and Clark, Lee and Tellegen, Auke},
    journal = {Journal of Personality and Social Psychology},
    volume={54},
    number={6},
    pages={1063--1070},
    year = {1988},
    publisher={APA}
}

@article{globem,
  title={GLOBEM dataset: multi-year datasets for longitudinal human behavior modeling generalization},
  author={Xu, Xuhai and Zhang, Han and Sefidgar, Yasaman and Ren, Yiyi and Liu, Xin and Seo, Woosuk and Brown, Jennifer and Kuehn, Kevin and Merrill, Mike and Nurius, Paula and others},
  journal={Advances in neural information processing systems},
  volume={35},
  pages={24655--24692},
  year={2022}
}

@article{goldberger2000physiobank,
  title={PhysioBank, PhysioToolkit, and PhysioNet: components of a new research resource for complex physiologic signals},
  author={Goldberger, Ary L and Amaral, Luis AN and Glass, Leon and Hausdorff, Jeffrey M and Ivanov, Plamen Ch and Mark, Roger G and Mietus, Joseph E and Moody, George B and Peng, Chung-Kang and Stanley, H Eugene},
  journal={circulation},
  volume={101},
  number={23},
  pages={e215--e220},
  year={2000},
  publisher={Lippincott Williams \& Wilkins}
}

@article{dormand1980family,
  title={A family of embedded Runge-Kutta formulae},
  author={Dormand, John R and Prince, Peter J},
  journal={Journal of computational and applied mathematics},
  volume={6},
  number={1},
  pages={19--26},
  year={1980},
  publisher={Elsevier}
}

@article{itoMilsteinSolver,
author = {Mil’shtejn, G. N.},
title = {Approximate Integration of Stochastic Differential Equations},
journal = {Theory of Probability \& Its Applications},
volume = {19},
number = {3},
pages = {557-562},
year = {1975},
doi = {10.1137/1119062},
URL = {https://doi.org/10.1137/1119062},
}

@article{yacoby2022mitigating,
  title={Mitigating the effects of non-identifiability on inference for Bayesian neural networks with latent variables},
  author={Yacoby, Yaniv and Pan, Weiwei and Doshi-Velez, Finale},
  journal={Journal of Machine Learning Research},
  volume={23},
  number={244},
  pages={1--54},
  year={2022}
}

@article{kingma2014adam,
  title={Adam: A method for stochastic optimization},
  author={Kingma, Diederik P and Ba, Jimmy},
  journal={arXiv preprint arXiv:1412.6980},
  year={2014}
}

@article{dugas2000incorporating,
  title={Incorporating second-order functional knowledge for better option pricing},
  author={Dugas, Charles and Bengio, Yoshua and B{\'e}lisle, Fran{\c{c}}ois and Nadeau, Claude and Garcia, Ren{\'e}},
  journal={Advances in neural information processing systems},
  volume={13},
  year={2000}
}

@article{borsboom2022systems,
  title={Systems-based approaches to mental disorders are the only game in town},
  author={Borsboom, Denny and Haslbeck, Jonas MB and Robinaugh, Donald J},
  journal={World Psychiatry},
  volume={21},
  number={3},
  pages={420},
  year={2022}
}

@article{millner2017describing,
  title={Describing and measuring the pathway to suicide attempts: A preliminary study},
  author={Millner, Alexander J and Lee, Michael D and Nock, Matthew K},
  journal={Suicide and Life-Threatening Behavior},
  volume={47},
  number={3},
  pages={353--369},
  year={2017},
  publisher={Wiley Online Library}
}

@misc{who2025suicide,
  author       = {{World Health Organization}},
  title        = {Suicide},
  year         = {2025},
  howpublished = {\url{https://www.who.int/news-room/fact-sheets/detail/suicide}},
  note         = {Accessed March 18, 2026}
}

@article{schweidtmann2024review,
  title={A review and perspective on hybrid modeling methodologies},
  author={Schweidtmann, Artur M and Zhang, Dongda and Von Stosch, Moritz},
  journal={Digital Chemical Engineering},
  volume={10},
  pages={100136},
  year={2024},
  publisher={Elsevier}
}

@inproceedings{bartosh2025sde,
title={{SDE} Matching: Scalable and Simulation-Free Training of Latent Stochastic Differential Equations},
author={Grigory Bartosh and Dmitry Vetrov and Christian A. Naesseth},
booktitle={Forty-second International Conference on Machine Learning},
year={2025},
url={https://openreview.net/forum?id=0Hd1lh52Fi}
}

@InProceedings{coker22a,
  title = 	 { Wide Mean-Field Bayesian Neural Networks Ignore the Data },
  author =       {Coker, Beau and Bruinsma, Wessel P. and Burt, David R. and Pan, Weiwei and Doshi-Velez, Finale},
  booktitle = 	 {Proceedings of The 25th International Conference on Artificial Intelligence and Statistics},
  pages = 	 {5276--5333},
  year = 	 {2022},
  editor = 	 {Camps-Valls, Gustau and Ruiz, Francisco J. R. and Valera, Isabel},
  volume = 	 {151},
  series = 	 {Proceedings of Machine Learning Research},
  month = 	 {28--30 Mar},
  publisher =    {PMLR},
  url = 	 {https://proceedings.mlr.press/v151/coker22a.html},
}

@article{foong2020expressiveness,
  title={On the expressiveness of approximate inference in bayesian neural networks},
  author={Foong, Andrew and Burt, David and Li, Yingzhen and Turner, Richard},
  journal={Advances in Neural Information Processing Systems},
  volume={33},
  pages={15897--15908},
  year={2020}
}

@ARTICLE{yacoby2020failure,
       author = {{Yacoby}, Yaniv and {Pan}, Weiwei and {Doshi-Velez}, Finale},
        title = "{Failure Modes of Variational Autoencoders and Their Effects on Downstream Tasks}",
      journal = {arXiv e-prints},
         year = 2020,
        month = jul,
          eid = {arXiv:2007.07124},
        pages = {arXiv:2007.07124},
          doi = {10.48550/arXiv.2007.07124},
archivePrefix = {arXiv},
       eprint = {2007.07124},
 primaryClass = {stat.ML},
       adsurl = {https://ui.adsabs.harvard.edu/abs/2020arXiv200707124Y}
}

@article{robinaugh2020network,
  title={The network approach to psychopathology: A review of the literature 2008--2018 and an agenda for future research},
  author={Robinaugh, Donald J and Hoekstra, Ria HA and Toner, Emma R and Borsboom, Denny},
  journal={Psychological medicine},
  volume={50},
  number={3},
  pages={353--366},
  year={2020},
  publisher={Cambridge University Press}
}

@book{revuz2013continuous,
  title={Continuous martingales and Brownian motion},
  author={Revuz, Daniel and Yor, Marc},
  year={2013},
  publisher={Springer Science \& Business Media}
}

@article{mayerhofer2011strong,
  title={On strong solutions for positive definite jump diffusions},
  author={Mayerhofer, Eberhard and Pfaffel, Oliver and Stelzer, Robert},
  journal={Stochastic processes and their applications},
  volume={121},
  number={9},
  pages={2072--2086},
  year={2011},
  publisher={Elsevier}
}

@book{mckean2024stochastic,
  title={Stochastic integrals},
  author={McKean, Henry P},
  volume={353},
  year={2024},
  publisher={American Mathematical Society}
}

@article{bru1991wishart,
  title={Wishart processes},
  author={Bru, Marie-France},
  journal={Journal of Theoretical Probability},
  volume={4},
  number={4},
  pages={725--751},
  year={1991},
  publisher={Springer}
}

@article{nock2026understanding,
  title={Understanding, Predicting, and Preventing Suicide: Recent Advances Using Digital and Computational Methods},
  author={Nock, Matthew K and Wang, Shirley B},
  journal={Current Directions in Psychological Science},
  pages={09637214251414021},
  year={2026},
  publisher={SAGE Publications Sage CA: Los Angeles, CA}
}

@article{nock2026using,
  author = {Nock, Matthew K. and Kleiman, Evan M. and Bentley, Kate H. and Fortgang, Rebecca G. and Millner, Alexander J. and Zuromski, Kelly L. and Bear, Adam and Christie, Alexis and Daniel, Merryn and DeMarco, Dylan and Follet, Lia and Kelly, Flynn and Neveux, Hope and Obi-Obasi, Onyinye and Ricard, Jordyn R. and Ramlal, Narise and Tambedou, Tida and Yacoby, Yaniv and Bird, Suzanne A. and Buonapane, Ralph and Donovan, Abigail and Mair, Patrick and Onnela, J. P. and Picard, Rosalind and Smoller, Jordan W.},
  year = {2026},
  title = {Using smartphone surveys to predict next-week suicide attempts},
  journal = {Journal of Psychopathology and Clinical Science},
}
